\RequirePackage{fix-cm}
\documentclass{article}
\usepackage[utf8]{inputenc}
\usepackage[T1]{fontenc}
\usepackage[left=1in,right=1in,top=1in,bottom=1in]{geometry}
\usepackage{microtype}
\usepackage{natbib}
\setcitestyle{authoryear,round,citesep={;},aysep={,},yysep={;}}

\usepackage{amsmath,amsfonts,bm}

\def\eqref#1{equation~\ref{#1}}

\def\1{\bm{1}}

\DeclareMathAlphabet{\mathsfit}{\encodingdefault}{\sfdefault}{m}{sl}
\SetMathAlphabet{\mathsfit}{bold}{\encodingdefault}{\sfdefault}{bx}{n}

\usepackage{url}
\usepackage{graphicx}
\usepackage{xcolor}
\usepackage{colortbl}
\usepackage{amsthm,aliascnt}
\newtheorem{proposition}{Proposition}[section]

\newtheorem{corollary}[proposition]{Corollary}
\theoremstyle{definition}

\usepackage{needspace,etoolbox}
\definecolor{LoRAColor}{HTML}{425B76}
\definecolor{OFTColor}{HTML}{B76E4A}
\definecolor{RestoredColor}{HTML}{268978}
\definecolor{NLLReductionColor}{HTML}{B75164}
\definecolor{DoRAColor}{HTML}{3F7C73}
\definecolor{HRAColor}{HTML}{78627F}
\definecolor{PiSSAColor}{HTML}{A85C68}
\definecolor{MiLoRAColor}{HTML}{9B8448}
\definecolor{FullFTColor}{HTML}{748092}
\definecolor{BaseColor}{HTML}{8E99A8}
\definecolor{InkColor}{HTML}{000000}
\definecolor{MutedColor}{HTML}{667182}
\definecolor{GuideColor}{HTML}{DDE2E8}
\definecolor{MinorGuideColor}{HTML}{EFF2F5}
\definecolor{SpineColor}{HTML}{414A57}
\definecolor{RuleColor}{HTML}{D4DAE2}
\definecolor{PanelColor}{HTML}{F7F9FC}
\definecolor{LinkColor}{HTML}{FF0000}
\definecolor{CiteColor}{HTML}{0000FF}
\definecolor{AccentColor}{HTML}{2458A6}
\definecolor{BaselineBarColor}{HTML}{ABB2BC}
\definecolor{SharedTuningColor}{HTML}{8796A9}
\definecolor{MethodTuningColor}{HTML}{536A84}
\definecolor{EarlyDiagnosticColor}{HTML}{4B8C97}
\definecolor{FinalDiagnosticColor}{HTML}{316678}
\definecolor{ContentsBlue}{HTML}{001473}

\usepackage[colorlinks=true,linkcolor=LinkColor,citecolor=CiteColor,urlcolor=CiteColor]{hyperref}
\usepackage[capitalise,nameinlink]{cleveref}
\newcommand{\loraname}{\textcolor{LoRAColor}{LoRA}}
\newcommand{\oftname}{\textcolor{OFTColor}{OFT}}
\newcommand{\bestcell}[2]{\cellcolor{#1!16}#2}
\newcommand{\bestmean}[1]{\ensuremath{\boldsymbol{#1}}}
\usepackage{booktabs}
\usepackage{placeins}
\title{Are Parameter-Efficient Fine-tuning Methods Really Different?}
\author{%
  Yikuan Li\thanks{Department of Statistics, The University of Warwick, UK. Email: \href{mailto:Yikuan.Li@warwick.ac.uk}{\texttt{Yikuan.Li@warwick.ac.uk}}}
  ~~~
  Pinyan Lu\thanks{Key Laboratory of Interdisciplinary Research of Computation and Economics, Shanghai University of Finance and Economics, China. Email: \href{mailto:lu.pinyan@mail.shufe.edu.cn}{\texttt{lu.pinyan@mail.shufe.edu.cn}}}
  ~~~
  Fanghui Liu\thanks{School of Mathematical Sciences, Institute of Natural Sciences and MOE-LSC, Shanghai Jiao Tong University, China. Email: \href{mailto:fanghui.liu@sjtu.edu.cn}{\texttt{fanghui.liu@sjtu.edu.cn}} (Corresponding author)} 
}
\date{}
\usepackage{eso-pic}
\newcommand{\headerlogo}[2]{%
  \raisebox{-0.5\height}{\includegraphics[height=#1]{#2}}%
}
\AddToShipoutPictureFG*{%
  \AtPageUpperLeft{%
    \raisebox{-1.45cm}{%
      \hspace*{\dimexpr 1in+\oddsidemargin\relax}%
      \makebox[\textwidth][l]{%
        \headerlogo{1.05cm}{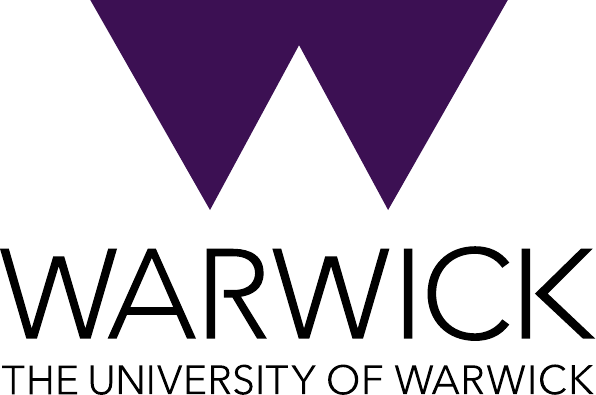}%
        \hspace{0.55cm}%
        \headerlogo{1.075cm}{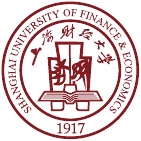}%
        \hspace{0.55cm}%
        \headerlogo{1.115cm}{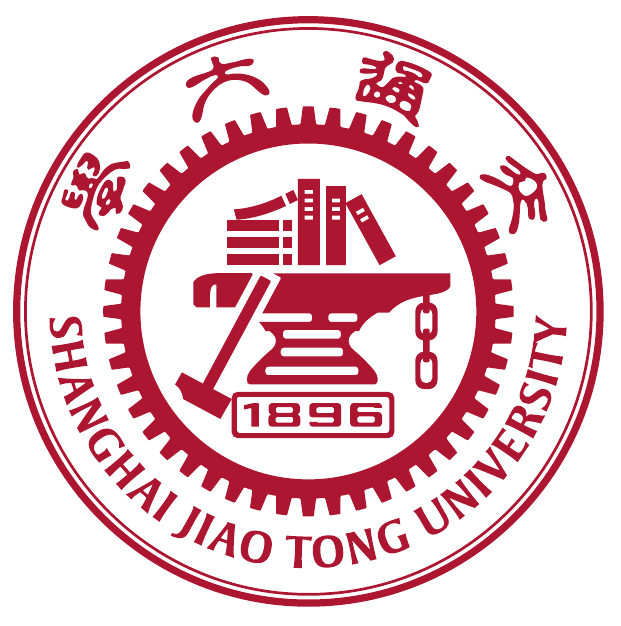}%
      }%
    }%
  }%
}
\AddToShipoutPictureFG*{%
  \AtPageUpperLeft{%
    \raisebox{-2.2cm}{%
      \hspace*{\dimexpr 1in+\oddsidemargin\relax}%
      \makebox[\textwidth][l]{\rule{\textwidth}{0.5pt}}%
    }%
  }%
}

\providecommand{\doi}[1]{\href{https://doi.org/#1}{doi:~\nolinkurl{#1}}}
\hypersetup{pdftitle={Are Parameter-Efficient Fine-tuning Methods Really Different?},
  pdfauthor={Yikuan Li, Pinyan Lu, Fanghui Liu}}

\usepackage{xurl}
\renewcommand{\bibfont}{\small}

\usepackage[font=small,labelfont=bf,justification=raggedright,
  singlelinecheck=false]{caption}
\AtBeginEnvironment{table}{%
  \renewenvironment{center}{\par\centering}{\par}%
  \renewcommand{\arraystretch}{1.12}%
}
\pretocmd{\section}{\Needspace{7\baselineskip}}{}{}
\pretocmd{\subsection}{\Needspace{5\baselineskip}}{}{}
\BeforeBeginEnvironment{proposition}{\Needspace{9\baselineskip}}
\BeforeBeginEnvironment{corollary}{\Needspace{10\baselineskip}}


\newcommand{\result}[1]{\ifcsname result@#1\endcsname\csname result@#1\endcsname\else\PackageError{paper}{Missing numerical value #1}{Define the numerical value in data/numbers.tex.}\fi}
\expandafter\def\csname result@fx22_lora_restore_dino_mean\endcsname{68.12}
\expandafter\def\csname result@fx22_lora_restore_dino_sd\endcsname{0.45}
\expandafter\def\csname result@fx22_lora_restore_dino_vs_reconstruct_mean\endcsname{0.28}
\expandafter\def\csname result@fx22_lora_restore_dino_vs_reconstruct_sd\endcsname{0.04}
\expandafter\def\csname result@fx22_lora_trained_dino_mean\endcsname{68.06}
\expandafter\def\csname result@fx22_lora_trained_dino_sd\endcsname{0.78}
\expandafter\def\csname result@fx22_lora_trained_ratio_mean\endcsname{1.12}
\expandafter\def\csname result@fx22_lora_trained_ratio_sd\endcsname{0.09}
\expandafter\def\csname result@fx22_lora_trained_dW_mean\endcsname{0.783}
\expandafter\def\csname result@fx22_lora_trained_dW_sd\endcsname{0.029}
\expandafter\def\csname result@fx22_lora_trained_dS_mean\endcsname{0.0088}
\expandafter\def\csname result@fx22_lora_trained_dS_sd\endcsname{0.0005}
\expandafter\def\csname result@fx22_lora_reconstruct_dino_mean\endcsname{67.84}
\expandafter\def\csname result@fx22_lora_reconstruct_dino_sd\endcsname{0.47}
\expandafter\def\csname result@fx22_lora_spectrum_only_dino_mean\endcsname{36.24}
\expandafter\def\csname result@fx22_lora_spectrum_only_dino_sd\endcsname{1.13}
\expandafter\def\csname result@fx22_oft_restore_dino_mean\endcsname{59.18}
\expandafter\def\csname result@fx22_oft_restore_dino_sd\endcsname{5.42}
\expandafter\def\csname result@fx22_oft_restore_dino_vs_reconstruct_mean\endcsname{-0.15}
\expandafter\def\csname result@fx22_oft_restore_dino_vs_reconstruct_sd\endcsname{0.57}
\expandafter\def\csname result@fx22_oft_trained_dino_mean\endcsname{59.04}
\expandafter\def\csname result@fx22_oft_trained_dino_sd\endcsname{5.51}
\expandafter\def\csname result@fx22_oft_trained_ratio_mean\endcsname{0.07}
\expandafter\def\csname result@fx22_oft_trained_ratio_sd\endcsname{0.01}
\expandafter\def\csname result@fx22_oft_trained_dW_mean\endcsname{10.587}
\expandafter\def\csname result@fx22_oft_trained_dW_sd\endcsname{0.085}
\expandafter\def\csname result@fx22_oft_trained_dS_mean\endcsname{0.0071}
\expandafter\def\csname result@fx22_oft_trained_dS_sd\endcsname{0.0013}
\expandafter\def\csname result@fx22_oft_reconstruct_dino_mean\endcsname{59.33}
\expandafter\def\csname result@fx22_oft_reconstruct_dino_sd\endcsname{5.91}
\expandafter\def\csname result@fx22_oft_spectrum_only_dino_mean\endcsname{36.17}
\expandafter\def\csname result@fx22_oft_spectrum_only_dino_sd\endcsname{1.09}
\expandafter\def\csname result@fx22_dora_restore_dino_mean\endcsname{70.59}
\expandafter\def\csname result@fx22_dora_restore_dino_sd\endcsname{3.75}
\expandafter\def\csname result@fx22_dora_restore_dino_vs_reconstruct_mean\endcsname{-0.36}
\expandafter\def\csname result@fx22_dora_restore_dino_vs_reconstruct_sd\endcsname{0.36}
\expandafter\def\csname result@fx22_dora_trained_dino_mean\endcsname{70.47}
\expandafter\def\csname result@fx22_dora_trained_dino_sd\endcsname{3.96}
\expandafter\def\csname result@fx22_dora_trained_ratio_mean\endcsname{1.28}
\expandafter\def\csname result@fx22_dora_trained_ratio_sd\endcsname{0.02}
\expandafter\def\csname result@fx22_dora_trained_dW_mean\endcsname{1.586}
\expandafter\def\csname result@fx22_dora_trained_dW_sd\endcsname{0.062}
\expandafter\def\csname result@fx22_dora_trained_dS_mean\endcsname{0.0204}
\expandafter\def\csname result@fx22_dora_trained_dS_sd\endcsname{0.0007}
\expandafter\def\csname result@fx22_dora_reconstruct_dino_mean\endcsname{70.95}
\expandafter\def\csname result@fx22_dora_reconstruct_dino_sd\endcsname{4.05}
\expandafter\def\csname result@fx22_dora_spectrum_only_dino_mean\endcsname{36.24}
\expandafter\def\csname result@fx22_dora_spectrum_only_dino_sd\endcsname{1.11}
\expandafter\def\csname result@fx22_pissa_restore_dino_mean\endcsname{69.03}
\expandafter\def\csname result@fx22_pissa_restore_dino_sd\endcsname{2.96}
\expandafter\def\csname result@fx22_pissa_restore_dino_vs_reconstruct_mean\endcsname{-0.05}
\expandafter\def\csname result@fx22_pissa_restore_dino_vs_reconstruct_sd\endcsname{0.83}
\expandafter\def\csname result@fx22_pissa_trained_dino_mean\endcsname{68.88}
\expandafter\def\csname result@fx22_pissa_trained_dino_sd\endcsname{3.45}
\expandafter\def\csname result@fx22_pissa_trained_ratio_mean\endcsname{5.79}
\expandafter\def\csname result@fx22_pissa_trained_ratio_sd\endcsname{0.04}
\expandafter\def\csname result@fx22_pissa_trained_dW_mean\endcsname{2.477}
\expandafter\def\csname result@fx22_pissa_trained_dW_sd\endcsname{0.027}
\expandafter\def\csname result@fx22_pissa_trained_dS_mean\endcsname{0.1434}
\expandafter\def\csname result@fx22_pissa_trained_dS_sd\endcsname{0.0021}
\expandafter\def\csname result@fx22_pissa_reconstruct_dino_mean\endcsname{69.07}
\expandafter\def\csname result@fx22_pissa_reconstruct_dino_sd\endcsname{2.91}
\expandafter\def\csname result@fx22_pissa_spectrum_only_dino_mean\endcsname{36.21}
\expandafter\def\csname result@fx22_pissa_spectrum_only_dino_sd\endcsname{1.24}
\expandafter\def\csname result@fx22_milora_restore_dino_mean\endcsname{66.06}
\expandafter\def\csname result@fx22_milora_restore_dino_sd\endcsname{2.15}
\expandafter\def\csname result@fx22_milora_restore_dino_vs_reconstruct_mean\endcsname{0.03}
\expandafter\def\csname result@fx22_milora_restore_dino_vs_reconstruct_sd\endcsname{0.60}
\expandafter\def\csname result@fx22_milora_trained_dino_mean\endcsname{66.13}
\expandafter\def\csname result@fx22_milora_trained_dino_sd\endcsname{2.10}
\expandafter\def\csname result@fx22_milora_trained_ratio_mean\endcsname{8.73}
\expandafter\def\csname result@fx22_milora_trained_ratio_sd\endcsname{0.06}
\expandafter\def\csname result@fx22_milora_trained_dW_mean\endcsname{2.019}
\expandafter\def\csname result@fx22_milora_trained_dW_sd\endcsname{0.046}
\expandafter\def\csname result@fx22_milora_trained_dS_mean\endcsname{0.1763}
\expandafter\def\csname result@fx22_milora_trained_dS_sd\endcsname{0.0029}
\expandafter\def\csname result@fx22_milora_reconstruct_dino_mean\endcsname{66.03}
\expandafter\def\csname result@fx22_milora_reconstruct_dino_sd\endcsname{2.17}
\expandafter\def\csname result@fx22_milora_spectrum_only_dino_mean\endcsname{36.34}
\expandafter\def\csname result@fx22_milora_spectrum_only_dino_sd\endcsname{1.35}
\expandafter\def\csname result@fx22_hra_restore_dino_mean\endcsname{63.54}
\expandafter\def\csname result@fx22_hra_restore_dino_sd\endcsname{5.99}
\expandafter\def\csname result@fx22_hra_restore_dino_vs_reconstruct_mean\endcsname{-0.15}
\expandafter\def\csname result@fx22_hra_restore_dino_vs_reconstruct_sd\endcsname{0.27}
\expandafter\def\csname result@fx22_hra_trained_dino_mean\endcsname{63.75}
\expandafter\def\csname result@fx22_hra_trained_dino_sd\endcsname{5.76}
\expandafter\def\csname result@fx22_hra_trained_ratio_mean\endcsname{0.01}
\expandafter\def\csname result@fx22_hra_trained_ratio_sd\endcsname{0.00}
\expandafter\def\csname result@fx22_hra_trained_dW_mean\endcsname{11.744}
\expandafter\def\csname result@fx22_hra_trained_dW_sd\endcsname{0.260}
\expandafter\def\csname result@fx22_hra_trained_dS_mean\endcsname{0.0017}
\expandafter\def\csname result@fx22_hra_trained_dS_sd\endcsname{0.0000}
\expandafter\def\csname result@fx22_hra_reconstruct_dino_mean\endcsname{63.69}
\expandafter\def\csname result@fx22_hra_reconstruct_dino_sd\endcsname{5.78}
\expandafter\def\csname result@fx22_hra_spectrum_only_dino_mean\endcsname{36.19}
\expandafter\def\csname result@fx22_hra_spectrum_only_dino_sd\endcsname{1.06}
\expandafter\def\csname result@fx22_lora_params\endcsname{38.34}
\expandafter\def\csname result@fx22_oft_params\endcsname{39.17}
\expandafter\def\csname result@fx22_dora_params\endcsname{39.23}
\expandafter\def\csname result@fx22_pissa_params\endcsname{38.34}
\expandafter\def\csname result@fx22_milora_params\endcsname{38.34}
\expandafter\def\csname result@fx22_hra_params\endcsname{38.09}
\expandafter\def\csname result@cross_cat_selected_settings\endcsname{4}
\expandafter\def\csname result@cross_cat_lora_lowest\endcsname{3}
\expandafter\def\csname result@cross_cat_lora_lowest_two\endcsname{4}
\expandafter\def\csname result@cross_cat_pissa_above_lora_seeds\endcsname{12}
\expandafter\def\csname result@cross_lora_math_lowest\endcsname{4}
\expandafter\def\csname result@cross_math_matched_blocks\endcsname{29}
\expandafter\def\csname result@cross_math_matched_pissa_highest\endcsname{29}
\expandafter\def\csname result@cross_math_matched_pissa_seed_highest\endcsname{87}
\expandafter\def\csname result@cross_coding_matched_blocks\endcsname{10}
\expandafter\def\csname result@cross_coding_matched_pissa_highest\endcsname{10}
\expandafter\def\csname result@cross_coding_matched_pissa_seed_highest\endcsname{30}
\expandafter\def\csname result@cross_cat_matched_blocks\endcsname{32}
\expandafter\def\csname result@cross_cat_matched_pissa_highest\endcsname{21}
\expandafter\def\csname result@cross_cat_matched_pissa_highest_two\endcsname{32}
\expandafter\def\csname result@cross_cat_hra_matched_lowest\endcsname{25}
\expandafter\def\csname result@ext_code_qwen_base_he\endcsname{58.54}
\expandafter\def\csname result@ext_code_qwen_base_hep\endcsname{51.22}
\expandafter\def\csname result@ext_code_qwen_base_mb\endcsname{76.72}
\expandafter\def\csname result@ext_code_qwen_base_mbp\endcsname{64.02}
\expandafter\def\csname result@ext_code_llama_base_he\endcsname{38.41}
\expandafter\def\csname result@ext_code_llama_base_hep\endcsname{31.71}
\expandafter\def\csname result@ext_code_llama_base_mb\endcsname{61.38}
\expandafter\def\csname result@ext_code_llama_base_mbp\endcsname{50.79}
\expandafter\def\csname result@ext_code_llama_dora_0p0003_he_mean\endcsname{42.07}
\expandafter\def\csname result@ext_code_llama_dora_0p0003_he_sd\endcsname{0.61}
\expandafter\def\csname result@ext_code_llama_dora_0p0003_hep_mean\endcsname{37.80}
\expandafter\def\csname result@ext_code_llama_dora_0p0003_hep_sd\endcsname{1.06}
\expandafter\def\csname result@ext_code_llama_dora_0p0003_mb_mean\endcsname{67.55}
\expandafter\def\csname result@ext_code_llama_dora_0p0003_mb_sd\endcsname{1.10}
\expandafter\def\csname result@ext_code_llama_dora_0p0003_mbp_mean\endcsname{56.08}
\expandafter\def\csname result@ext_code_llama_dora_0p0003_mbp_sd\endcsname{0.00}
\expandafter\def\csname result@ext_code_llama_dora_0p0003_nll_mean\endcsname{13.51}
\expandafter\def\csname result@ext_code_llama_dora_0p0003_nll_sd\endcsname{0.97}
\expandafter\def\csname result@ext_code_llama_hra_0p001_he_mean\endcsname{39.84}
\expandafter\def\csname result@ext_code_llama_hra_0p001_he_sd\endcsname{1.27}
\expandafter\def\csname result@ext_code_llama_hra_0p001_hep_mean\endcsname{35.57}
\expandafter\def\csname result@ext_code_llama_hra_0p001_hep_sd\endcsname{0.70}
\expandafter\def\csname result@ext_code_llama_hra_0p001_mb_mean\endcsname{66.75}
\expandafter\def\csname result@ext_code_llama_hra_0p001_mb_sd\endcsname{0.76}
\expandafter\def\csname result@ext_code_llama_hra_0p001_mbp_mean\endcsname{55.64}
\expandafter\def\csname result@ext_code_llama_hra_0p001_mbp_sd\endcsname{1.19}
\expandafter\def\csname result@ext_code_llama_hra_0p001_nll_mean\endcsname{15.09}
\expandafter\def\csname result@ext_code_llama_hra_0p001_nll_sd\endcsname{1.11}
\expandafter\def\csname result@ext_code_llama_lora_0p0003_he_mean\endcsname{41.46}
\expandafter\def\csname result@ext_code_llama_lora_0p0003_he_sd\endcsname{1.06}
\expandafter\def\csname result@ext_code_llama_lora_0p0003_hep_mean\endcsname{36.99}
\expandafter\def\csname result@ext_code_llama_lora_0p0003_hep_sd\endcsname{0.93}
\expandafter\def\csname result@ext_code_llama_lora_0p0003_mb_mean\endcsname{65.87}
\expandafter\def\csname result@ext_code_llama_lora_0p0003_mb_sd\endcsname{0.46}
\expandafter\def\csname result@ext_code_llama_lora_0p0003_mbp_mean\endcsname{54.76}
\expandafter\def\csname result@ext_code_llama_lora_0p0003_mbp_sd\endcsname{0.46}
\expandafter\def\csname result@ext_code_llama_lora_0p0003_nll_mean\endcsname{10.73}
\expandafter\def\csname result@ext_code_llama_lora_0p0003_nll_sd\endcsname{0.56}
\expandafter\def\csname result@ext_code_llama_milora_0p0003_he_mean\endcsname{42.28}
\expandafter\def\csname result@ext_code_llama_milora_0p0003_he_sd\endcsname{0.93}
\expandafter\def\csname result@ext_code_llama_milora_0p0003_hep_mean\endcsname{37.80}
\expandafter\def\csname result@ext_code_llama_milora_0p0003_hep_sd\endcsname{0.00}
\expandafter\def\csname result@ext_code_llama_milora_0p0003_mb_mean\endcsname{65.61}
\expandafter\def\csname result@ext_code_llama_milora_0p0003_mb_sd\endcsname{0.26}
\expandafter\def\csname result@ext_code_llama_milora_0p0003_mbp_mean\endcsname{54.06}
\expandafter\def\csname result@ext_code_llama_milora_0p0003_mbp_sd\endcsname{0.67}
\expandafter\def\csname result@ext_code_llama_milora_0p0003_nll_mean\endcsname{22.61}
\expandafter\def\csname result@ext_code_llama_milora_0p0003_nll_sd\endcsname{1.53}
\expandafter\def\csname result@ext_code_llama_oft_0p0003_he_mean\endcsname{43.09}
\expandafter\def\csname result@ext_code_llama_oft_0p0003_he_sd\endcsname{1.53}
\expandafter\def\csname result@ext_code_llama_oft_0p0003_hep_mean\endcsname{39.02}
\expandafter\def\csname result@ext_code_llama_oft_0p0003_hep_sd\endcsname{1.06}
\expandafter\def\csname result@ext_code_llama_oft_0p0003_mb_mean\endcsname{65.61}
\expandafter\def\csname result@ext_code_llama_oft_0p0003_mb_sd\endcsname{1.61}
\expandafter\def\csname result@ext_code_llama_oft_0p0003_mbp_mean\endcsname{54.41}
\expandafter\def\csname result@ext_code_llama_oft_0p0003_mbp_sd\endcsname{1.88}
\expandafter\def\csname result@ext_code_llama_oft_0p0003_nll_mean\endcsname{25.01}
\expandafter\def\csname result@ext_code_llama_oft_0p0003_nll_sd\endcsname{2.49}
\expandafter\def\csname result@ext_code_llama_pissa_0p0001_he_mean\endcsname{45.53}
\expandafter\def\csname result@ext_code_llama_pissa_0p0001_he_sd\endcsname{0.35}
\expandafter\def\csname result@ext_code_llama_pissa_0p0001_hep_mean\endcsname{40.45}
\expandafter\def\csname result@ext_code_llama_pissa_0p0001_hep_sd\endcsname{0.93}
\expandafter\def\csname result@ext_code_llama_pissa_0p0001_mb_mean\endcsname{66.14}
\expandafter\def\csname result@ext_code_llama_pissa_0p0001_mb_sd\endcsname{0.26}
\expandafter\def\csname result@ext_code_llama_pissa_0p0001_mbp_mean\endcsname{53.88}
\expandafter\def\csname result@ext_code_llama_pissa_0p0001_mbp_sd\endcsname{0.40}
\expandafter\def\csname result@ext_code_llama_pissa_0p0001_nll_mean\endcsname{43.00}
\expandafter\def\csname result@ext_code_llama_pissa_0p0001_nll_sd\endcsname{0.73}
\expandafter\def\csname result@ext_code_qwen_dora_0p0001_he_mean\endcsname{59.96}
\expandafter\def\csname result@ext_code_qwen_dora_0p0001_he_sd\endcsname{0.93}
\expandafter\def\csname result@ext_code_qwen_dora_0p0001_hep_mean\endcsname{52.64}
\expandafter\def\csname result@ext_code_qwen_dora_0p0001_hep_sd\endcsname{0.70}
\expandafter\def\csname result@ext_code_qwen_dora_0p0001_mb_mean\endcsname{76.63}
\expandafter\def\csname result@ext_code_qwen_dora_0p0001_mb_sd\endcsname{1.00}
\expandafter\def\csname result@ext_code_qwen_dora_0p0001_mbp_mean\endcsname{64.37}
\expandafter\def\csname result@ext_code_qwen_dora_0p0001_mbp_sd\endcsname{0.61}
\expandafter\def\csname result@ext_code_qwen_dora_0p0001_nll_mean\endcsname{1.89}
\expandafter\def\csname result@ext_code_qwen_dora_0p0001_nll_sd\endcsname{0.47}
\expandafter\def\csname result@ext_code_qwen_hra_0p001_he_mean\endcsname{58.94}
\expandafter\def\csname result@ext_code_qwen_hra_0p001_he_sd\endcsname{2.14}
\expandafter\def\csname result@ext_code_qwen_hra_0p001_hep_mean\endcsname{51.83}
\expandafter\def\csname result@ext_code_qwen_hra_0p001_hep_sd\endcsname{1.83}
\expandafter\def\csname result@ext_code_qwen_hra_0p001_mb_mean\endcsname{77.25}
\expandafter\def\csname result@ext_code_qwen_hra_0p001_mb_sd\endcsname{0.70}
\expandafter\def\csname result@ext_code_qwen_hra_0p001_mbp_mean\endcsname{64.46}
\expandafter\def\csname result@ext_code_qwen_hra_0p001_mbp_sd\endcsname{0.85}
\expandafter\def\csname result@ext_code_qwen_hra_0p001_nll_mean\endcsname{2.92}
\expandafter\def\csname result@ext_code_qwen_hra_0p001_nll_sd\endcsname{1.08}
\expandafter\def\csname result@ext_code_qwen_lora_0p0001_he_mean\endcsname{59.96}
\expandafter\def\csname result@ext_code_qwen_lora_0p0001_he_sd\endcsname{0.70}
\expandafter\def\csname result@ext_code_qwen_lora_0p0001_hep_mean\endcsname{52.24}
\expandafter\def\csname result@ext_code_qwen_lora_0p0001_hep_sd\endcsname{0.35}
\expandafter\def\csname result@ext_code_qwen_lora_0p0001_mb_mean\endcsname{76.90}
\expandafter\def\csname result@ext_code_qwen_lora_0p0001_mb_sd\endcsname{0.40}
\expandafter\def\csname result@ext_code_qwen_lora_0p0001_mbp_mean\endcsname{64.29}
\expandafter\def\csname result@ext_code_qwen_lora_0p0001_mbp_sd\endcsname{0.53}
\expandafter\def\csname result@ext_code_qwen_lora_0p0001_nll_mean\endcsname{1.60}
\expandafter\def\csname result@ext_code_qwen_lora_0p0001_nll_sd\endcsname{0.49}
\expandafter\def\csname result@ext_code_qwen_milora_0p0001_he_mean\endcsname{58.94}
\expandafter\def\csname result@ext_code_qwen_milora_0p0001_he_sd\endcsname{0.35}
\expandafter\def\csname result@ext_code_qwen_milora_0p0001_hep_mean\endcsname{51.02}
\expandafter\def\csname result@ext_code_qwen_milora_0p0001_hep_sd\endcsname{0.93}
\expandafter\def\csname result@ext_code_qwen_milora_0p0001_mb_mean\endcsname{76.63}
\expandafter\def\csname result@ext_code_qwen_milora_0p0001_mb_sd\endcsname{0.55}
\expandafter\def\csname result@ext_code_qwen_milora_0p0001_mbp_mean\endcsname{64.02}
\expandafter\def\csname result@ext_code_qwen_milora_0p0001_mbp_sd\endcsname{0.53}
\expandafter\def\csname result@ext_code_qwen_milora_0p0001_nll_mean\endcsname{3.11}
\expandafter\def\csname result@ext_code_qwen_milora_0p0001_nll_sd\endcsname{0.35}
\expandafter\def\csname result@ext_code_qwen_oft_0p0001_he_mean\endcsname{58.94}
\expandafter\def\csname result@ext_code_qwen_oft_0p0001_he_sd\endcsname{0.70}
\expandafter\def\csname result@ext_code_qwen_oft_0p0001_hep_mean\endcsname{51.83}
\expandafter\def\csname result@ext_code_qwen_oft_0p0001_hep_sd\endcsname{0.61}
\expandafter\def\csname result@ext_code_qwen_oft_0p0001_mb_mean\endcsname{77.43}
\expandafter\def\csname result@ext_code_qwen_oft_0p0001_mb_sd\endcsname{0.55}
\expandafter\def\csname result@ext_code_qwen_oft_0p0001_mbp_mean\endcsname{64.55}
\expandafter\def\csname result@ext_code_qwen_oft_0p0001_mbp_sd\endcsname{0.53}
\expandafter\def\csname result@ext_code_qwen_oft_0p0001_nll_mean\endcsname{2.23}
\expandafter\def\csname result@ext_code_qwen_oft_0p0001_nll_sd\endcsname{0.47}
\expandafter\def\csname result@ext_code_qwen_pissa_1em05_he_mean\endcsname{59.15}
\expandafter\def\csname result@ext_code_qwen_pissa_1em05_he_sd\endcsname{1.06}
\expandafter\def\csname result@ext_code_qwen_pissa_1em05_hep_mean\endcsname{52.44}
\expandafter\def\csname result@ext_code_qwen_pissa_1em05_hep_sd\endcsname{1.61}
\expandafter\def\csname result@ext_code_qwen_pissa_1em05_mb_mean\endcsname{76.63}
\expandafter\def\csname result@ext_code_qwen_pissa_1em05_mb_sd\endcsname{0.40}
\expandafter\def\csname result@ext_code_qwen_pissa_1em05_mbp_mean\endcsname{63.93}
\expandafter\def\csname result@ext_code_qwen_pissa_1em05_mbp_sd\endcsname{0.61}
\expandafter\def\csname result@ext_code_qwen_pissa_1em05_nll_mean\endcsname{4.25}
\expandafter\def\csname result@ext_code_qwen_pissa_1em05_nll_sd\endcsname{0.31}
\expandafter\def\csname result@ext_flux_flux_dora_cap4_0p0005_params\endcsname{5.68}
\expandafter\def\csname result@ext_flux_flux_dora_cap4_0p0005_dino_mean\endcsname{69.92}
\expandafter\def\csname result@ext_flux_flux_dora_cap4_0p0005_dino_sd\endcsname{2.05}
\expandafter\def\csname result@ext_flux_flux_dora_cap4_0p0005_drift_mean\endcsname{16.93}
\expandafter\def\csname result@ext_flux_flux_dora_cap4_0p0005_drift_sd\endcsname{3.89}
\expandafter\def\csname result@ext_flux_flux_dora_cap4_0p0005_clip_mean\endcsname{26.81}
\expandafter\def\csname result@ext_flux_flux_dora_cap4_0p0005_clip_sd\endcsname{0.62}
\expandafter\def\csname result@ext_flux_flux_dora_cap4_0p0005_validation_mean\endcsname{71.87}
\expandafter\def\csname result@ext_flux_flux_dora_cap4_0p0005_validation_sd\endcsname{2.31}
\expandafter\def\csname result@ext_flux_flux_dora_0p0003_params\endcsname{10.48}
\expandafter\def\csname result@ext_flux_flux_dora_0p0003_dino_mean\endcsname{68.77}
\expandafter\def\csname result@ext_flux_flux_dora_0p0003_dino_sd\endcsname{2.44}
\expandafter\def\csname result@ext_flux_flux_dora_0p0003_drift_mean\endcsname{14.77}
\expandafter\def\csname result@ext_flux_flux_dora_0p0003_drift_sd\endcsname{0.57}
\expandafter\def\csname result@ext_flux_flux_dora_0p0003_clip_mean\endcsname{27.40}
\expandafter\def\csname result@ext_flux_flux_dora_0p0003_clip_sd\endcsname{2.09}
\expandafter\def\csname result@ext_flux_flux_dora_0p0003_validation_mean\endcsname{70.42}
\expandafter\def\csname result@ext_flux_flux_dora_0p0003_validation_sd\endcsname{2.84}
\expandafter\def\csname result@ext_flux_flux_dora_cap16_5em05_params\endcsname{20.06}
\expandafter\def\csname result@ext_flux_flux_dora_cap16_5em05_dino_mean\endcsname{66.38}
\expandafter\def\csname result@ext_flux_flux_dora_cap16_5em05_dino_sd\endcsname{2.31}
\expandafter\def\csname result@ext_flux_flux_dora_cap16_5em05_drift_mean\endcsname{13.83}
\expandafter\def\csname result@ext_flux_flux_dora_cap16_5em05_drift_sd\endcsname{2.73}
\expandafter\def\csname result@ext_flux_flux_dora_cap16_5em05_clip_mean\endcsname{28.27}
\expandafter\def\csname result@ext_flux_flux_dora_cap16_5em05_clip_sd\endcsname{0.66}
\expandafter\def\csname result@ext_flux_flux_dora_cap16_5em05_validation_mean\endcsname{73.73}
\expandafter\def\csname result@ext_flux_flux_dora_cap16_5em05_validation_sd\endcsname{1.29}
\expandafter\def\csname result@ext_flux_flux_dora_cap32_0p0001_params\endcsname{39.23}
\expandafter\def\csname result@ext_flux_flux_dora_cap32_0p0001_dino_mean\endcsname{70.97}
\expandafter\def\csname result@ext_flux_flux_dora_cap32_0p0001_dino_sd\endcsname{3.86}
\expandafter\def\csname result@ext_flux_flux_dora_cap32_0p0001_drift_mean\endcsname{15.98}
\expandafter\def\csname result@ext_flux_flux_dora_cap32_0p0001_drift_sd\endcsname{2.36}
\expandafter\def\csname result@ext_flux_flux_dora_cap32_0p0001_clip_mean\endcsname{27.34}
\expandafter\def\csname result@ext_flux_flux_dora_cap32_0p0001_clip_sd\endcsname{0.60}
\expandafter\def\csname result@ext_flux_flux_dora_cap32_0p0001_validation_mean\endcsname{72.33}
\expandafter\def\csname result@ext_flux_flux_dora_cap32_0p0001_validation_sd\endcsname{2.88}
\expandafter\def\csname result@ext_flux_flux_hra_cap16_0p0005_params\endcsname{4.92}
\expandafter\def\csname result@ext_flux_flux_hra_cap16_0p0005_dino_mean\endcsname{68.40}
\expandafter\def\csname result@ext_flux_flux_hra_cap16_0p0005_dino_sd\endcsname{1.09}
\expandafter\def\csname result@ext_flux_flux_hra_cap16_0p0005_drift_mean\endcsname{14.83}
\expandafter\def\csname result@ext_flux_flux_hra_cap16_0p0005_drift_sd\endcsname{4.11}
\expandafter\def\csname result@ext_flux_flux_hra_cap16_0p0005_clip_mean\endcsname{26.95}
\expandafter\def\csname result@ext_flux_flux_hra_cap16_0p0005_clip_sd\endcsname{2.05}
\expandafter\def\csname result@ext_flux_flux_hra_cap16_0p0005_validation_mean\endcsname{72.99}
\expandafter\def\csname result@ext_flux_flux_hra_cap16_0p0005_validation_sd\endcsname{1.85}
\expandafter\def\csname result@ext_flux_flux_hra_0p0003_params\endcsname{9.83}
\expandafter\def\csname result@ext_flux_flux_hra_0p0003_dino_mean\endcsname{70.10}
\expandafter\def\csname result@ext_flux_flux_hra_0p0003_dino_sd\endcsname{2.95}
\expandafter\def\csname result@ext_flux_flux_hra_0p0003_drift_mean\endcsname{13.47}
\expandafter\def\csname result@ext_flux_flux_hra_0p0003_drift_sd\endcsname{4.26}
\expandafter\def\csname result@ext_flux_flux_hra_0p0003_clip_mean\endcsname{26.63}
\expandafter\def\csname result@ext_flux_flux_hra_0p0003_clip_sd\endcsname{1.47}
\expandafter\def\csname result@ext_flux_flux_hra_0p0003_validation_mean\endcsname{73.11}
\expandafter\def\csname result@ext_flux_flux_hra_0p0003_validation_sd\endcsname{0.95}
\expandafter\def\csname result@ext_flux_flux_hra_cap62_0p0005_params\endcsname{19.05}
\expandafter\def\csname result@ext_flux_flux_hra_cap62_0p0005_dino_mean\endcsname{50.89}
\expandafter\def\csname result@ext_flux_flux_hra_cap62_0p0005_dino_sd\endcsname{23.55}
\expandafter\def\csname result@ext_flux_flux_hra_cap62_0p0005_drift_mean\endcsname{32.82}
\expandafter\def\csname result@ext_flux_flux_hra_cap62_0p0005_drift_sd\endcsname{12.12}
\expandafter\def\csname result@ext_flux_flux_hra_cap62_0p0005_clip_mean\endcsname{21.10}
\expandafter\def\csname result@ext_flux_flux_hra_cap62_0p0005_clip_sd\endcsname{5.13}
\expandafter\def\csname result@ext_flux_flux_hra_cap62_0p0005_validation_mean\endcsname{70.95}
\expandafter\def\csname result@ext_flux_flux_hra_cap62_0p0005_validation_sd\endcsname{3.50}
\expandafter\def\csname result@ext_flux_flux_hra_cap124_0p0003_params\endcsname{38.09}
\expandafter\def\csname result@ext_flux_flux_hra_cap124_0p0003_dino_mean\endcsname{63.63}
\expandafter\def\csname result@ext_flux_flux_hra_cap124_0p0003_dino_sd\endcsname{5.99}
\expandafter\def\csname result@ext_flux_flux_hra_cap124_0p0003_drift_mean\endcsname{32.82}
\expandafter\def\csname result@ext_flux_flux_hra_cap124_0p0003_drift_sd\endcsname{2.95}
\expandafter\def\csname result@ext_flux_flux_hra_cap124_0p0003_clip_mean\endcsname{23.27}
\expandafter\def\csname result@ext_flux_flux_hra_cap124_0p0003_clip_sd\endcsname{0.99}
\expandafter\def\csname result@ext_flux_flux_hra_cap124_0p0003_validation_mean\endcsname{72.51}
\expandafter\def\csname result@ext_flux_flux_hra_cap124_0p0003_validation_sd\endcsname{4.16}
\expandafter\def\csname result@ext_flux_flux_lora_cap4_0p0003_params\endcsname{4.79}
\expandafter\def\csname result@ext_flux_flux_lora_cap4_0p0003_dino_mean\endcsname{68.90}
\expandafter\def\csname result@ext_flux_flux_lora_cap4_0p0003_dino_sd\endcsname{0.87}
\expandafter\def\csname result@ext_flux_flux_lora_cap4_0p0003_drift_mean\endcsname{14.23}
\expandafter\def\csname result@ext_flux_flux_lora_cap4_0p0003_drift_sd\endcsname{1.01}
\expandafter\def\csname result@ext_flux_flux_lora_cap4_0p0003_clip_mean\endcsname{27.00}
\expandafter\def\csname result@ext_flux_flux_lora_cap4_0p0003_clip_sd\endcsname{1.24}
\expandafter\def\csname result@ext_flux_flux_lora_cap4_0p0003_validation_mean\endcsname{73.48}
\expandafter\def\csname result@ext_flux_flux_lora_cap4_0p0003_validation_sd\endcsname{2.90}
\expandafter\def\csname result@ext_flux_flux_lora_0p0001_params\endcsname{9.58}
\expandafter\def\csname result@ext_flux_flux_lora_0p0001_dino_mean\endcsname{68.48}
\expandafter\def\csname result@ext_flux_flux_lora_0p0001_dino_sd\endcsname{1.29}
\expandafter\def\csname result@ext_flux_flux_lora_0p0001_drift_mean\endcsname{14.38}
\expandafter\def\csname result@ext_flux_flux_lora_0p0001_drift_sd\endcsname{0.68}
\expandafter\def\csname result@ext_flux_flux_lora_0p0001_clip_mean\endcsname{28.09}
\expandafter\def\csname result@ext_flux_flux_lora_0p0001_clip_sd\endcsname{1.48}
\expandafter\def\csname result@ext_flux_flux_lora_0p0001_validation_mean\endcsname{71.65}
\expandafter\def\csname result@ext_flux_flux_lora_0p0001_validation_sd\endcsname{1.88}
\expandafter\def\csname result@ext_flux_flux_lora_cap16_0p0001_params\endcsname{19.17}
\expandafter\def\csname result@ext_flux_flux_lora_cap16_0p0001_dino_mean\endcsname{65.78}
\expandafter\def\csname result@ext_flux_flux_lora_cap16_0p0001_dino_sd\endcsname{1.41}
\expandafter\def\csname result@ext_flux_flux_lora_cap16_0p0001_drift_mean\endcsname{12.87}
\expandafter\def\csname result@ext_flux_flux_lora_cap16_0p0001_drift_sd\endcsname{2.45}
\expandafter\def\csname result@ext_flux_flux_lora_cap16_0p0001_clip_mean\endcsname{27.99}
\expandafter\def\csname result@ext_flux_flux_lora_cap16_0p0001_clip_sd\endcsname{1.60}
\expandafter\def\csname result@ext_flux_flux_lora_cap16_0p0001_validation_mean\endcsname{74.88}
\expandafter\def\csname result@ext_flux_flux_lora_cap16_0p0001_validation_sd\endcsname{1.44}
\expandafter\def\csname result@ext_flux_flux_lora_cap32_5em05_params\endcsname{38.34}
\expandafter\def\csname result@ext_flux_flux_lora_cap32_5em05_dino_mean\endcsname{67.63}
\expandafter\def\csname result@ext_flux_flux_lora_cap32_5em05_dino_sd\endcsname{0.88}
\expandafter\def\csname result@ext_flux_flux_lora_cap32_5em05_drift_mean\endcsname{14.98}
\expandafter\def\csname result@ext_flux_flux_lora_cap32_5em05_drift_sd\endcsname{2.53}
\expandafter\def\csname result@ext_flux_flux_lora_cap32_5em05_clip_mean\endcsname{28.78}
\expandafter\def\csname result@ext_flux_flux_lora_cap32_5em05_clip_sd\endcsname{0.85}
\expandafter\def\csname result@ext_flux_flux_lora_cap32_5em05_validation_mean\endcsname{72.46}
\expandafter\def\csname result@ext_flux_flux_lora_cap32_5em05_validation_sd\endcsname{5.94}
\expandafter\def\csname result@ext_flux_flux_milora_cap4_0p0001_params\endcsname{4.79}
\expandafter\def\csname result@ext_flux_flux_milora_cap4_0p0001_dino_mean\endcsname{66.51}
\expandafter\def\csname result@ext_flux_flux_milora_cap4_0p0001_dino_sd\endcsname{1.07}
\expandafter\def\csname result@ext_flux_flux_milora_cap4_0p0001_drift_mean\endcsname{15.89}
\expandafter\def\csname result@ext_flux_flux_milora_cap4_0p0001_drift_sd\endcsname{3.13}
\expandafter\def\csname result@ext_flux_flux_milora_cap4_0p0001_clip_mean\endcsname{26.94}
\expandafter\def\csname result@ext_flux_flux_milora_cap4_0p0001_clip_sd\endcsname{1.39}
\expandafter\def\csname result@ext_flux_flux_milora_cap4_0p0001_validation_mean\endcsname{71.79}
\expandafter\def\csname result@ext_flux_flux_milora_cap4_0p0001_validation_sd\endcsname{5.90}
\expandafter\def\csname result@ext_flux_flux_milora_5em05_params\endcsname{9.58}
\expandafter\def\csname result@ext_flux_flux_milora_5em05_dino_mean\endcsname{65.87}
\expandafter\def\csname result@ext_flux_flux_milora_5em05_dino_sd\endcsname{2.53}
\expandafter\def\csname result@ext_flux_flux_milora_5em05_drift_mean\endcsname{14.58}
\expandafter\def\csname result@ext_flux_flux_milora_5em05_drift_sd\endcsname{4.24}
\expandafter\def\csname result@ext_flux_flux_milora_5em05_clip_mean\endcsname{28.25}
\expandafter\def\csname result@ext_flux_flux_milora_5em05_clip_sd\endcsname{0.21}
\expandafter\def\csname result@ext_flux_flux_milora_5em05_validation_mean\endcsname{69.65}
\expandafter\def\csname result@ext_flux_flux_milora_5em05_validation_sd\endcsname{2.02}
\expandafter\def\csname result@ext_flux_flux_milora_cap16_0p0001_params\endcsname{19.17}
\expandafter\def\csname result@ext_flux_flux_milora_cap16_0p0001_dino_mean\endcsname{65.14}
\expandafter\def\csname result@ext_flux_flux_milora_cap16_0p0001_dino_sd\endcsname{1.60}
\expandafter\def\csname result@ext_flux_flux_milora_cap16_0p0001_drift_mean\endcsname{15.38}
\expandafter\def\csname result@ext_flux_flux_milora_cap16_0p0001_drift_sd\endcsname{2.56}
\expandafter\def\csname result@ext_flux_flux_milora_cap16_0p0001_clip_mean\endcsname{27.07}
\expandafter\def\csname result@ext_flux_flux_milora_cap16_0p0001_clip_sd\endcsname{0.13}
\expandafter\def\csname result@ext_flux_flux_milora_cap16_0p0001_validation_mean\endcsname{73.96}
\expandafter\def\csname result@ext_flux_flux_milora_cap16_0p0001_validation_sd\endcsname{2.00}
\expandafter\def\csname result@ext_flux_flux_milora_cap32_0p0001_params\endcsname{38.34}
\expandafter\def\csname result@ext_flux_flux_milora_cap32_0p0001_dino_mean\endcsname{65.96}
\expandafter\def\csname result@ext_flux_flux_milora_cap32_0p0001_dino_sd\endcsname{1.88}
\expandafter\def\csname result@ext_flux_flux_milora_cap32_0p0001_drift_mean\endcsname{15.44}
\expandafter\def\csname result@ext_flux_flux_milora_cap32_0p0001_drift_sd\endcsname{2.85}
\expandafter\def\csname result@ext_flux_flux_milora_cap32_0p0001_clip_mean\endcsname{27.69}
\expandafter\def\csname result@ext_flux_flux_milora_cap32_0p0001_clip_sd\endcsname{1.70}
\expandafter\def\csname result@ext_flux_flux_milora_cap32_0p0001_validation_mean\endcsname{73.32}
\expandafter\def\csname result@ext_flux_flux_milora_cap32_0p0001_validation_sd\endcsname{2.90}
\expandafter\def\csname result@ext_flux_flux_oft_cap32_0p0003_params\endcsname{4.76}
\expandafter\def\csname result@ext_flux_flux_oft_cap32_0p0003_dino_mean\endcsname{67.43}
\expandafter\def\csname result@ext_flux_flux_oft_cap32_0p0003_dino_sd\endcsname{4.54}
\expandafter\def\csname result@ext_flux_flux_oft_cap32_0p0003_drift_mean\endcsname{27.66}
\expandafter\def\csname result@ext_flux_flux_oft_cap32_0p0003_drift_sd\endcsname{4.85}
\expandafter\def\csname result@ext_flux_flux_oft_cap32_0p0003_clip_mean\endcsname{24.88}
\expandafter\def\csname result@ext_flux_flux_oft_cap32_0p0003_clip_sd\endcsname{1.90}
\expandafter\def\csname result@ext_flux_flux_oft_cap32_0p0003_validation_mean\endcsname{74.58}
\expandafter\def\csname result@ext_flux_flux_oft_cap32_0p0003_validation_sd\endcsname{2.56}
\expandafter\def\csname result@ext_flux_flux_oft_5em05_params\endcsname{9.68}
\expandafter\def\csname result@ext_flux_flux_oft_5em05_dino_mean\endcsname{68.90}
\expandafter\def\csname result@ext_flux_flux_oft_5em05_dino_sd\endcsname{1.03}
\expandafter\def\csname result@ext_flux_flux_oft_5em05_drift_mean\endcsname{15.03}
\expandafter\def\csname result@ext_flux_flux_oft_5em05_drift_sd\endcsname{3.55}
\expandafter\def\csname result@ext_flux_flux_oft_5em05_clip_mean\endcsname{27.88}
\expandafter\def\csname result@ext_flux_flux_oft_5em05_clip_sd\endcsname{2.43}
\expandafter\def\csname result@ext_flux_flux_oft_5em05_validation_mean\endcsname{71.17}
\expandafter\def\csname result@ext_flux_flux_oft_5em05_validation_sd\endcsname{0.91}
\expandafter\def\csname result@ext_flux_flux_oft_cap128_0p0001_params\endcsname{19.51}
\expandafter\def\csname result@ext_flux_flux_oft_cap128_0p0001_dino_mean\endcsname{63.11}
\expandafter\def\csname result@ext_flux_flux_oft_cap128_0p0001_dino_sd\endcsname{1.24}
\expandafter\def\csname result@ext_flux_flux_oft_cap128_0p0001_drift_mean\endcsname{30.83}
\expandafter\def\csname result@ext_flux_flux_oft_cap128_0p0001_drift_sd\endcsname{4.67}
\expandafter\def\csname result@ext_flux_flux_oft_cap128_0p0001_clip_mean\endcsname{23.23}
\expandafter\def\csname result@ext_flux_flux_oft_cap128_0p0001_clip_sd\endcsname{1.11}
\expandafter\def\csname result@ext_flux_flux_oft_cap128_0p0001_validation_mean\endcsname{70.87}
\expandafter\def\csname result@ext_flux_flux_oft_cap128_0p0001_validation_sd\endcsname{7.10}
\expandafter\def\csname result@ext_flux_flux_oft_cap256_0p0001_params\endcsname{39.17}
\expandafter\def\csname result@ext_flux_flux_oft_cap256_0p0001_dino_mean\endcsname{59.54}
\expandafter\def\csname result@ext_flux_flux_oft_cap256_0p0001_dino_sd\endcsname{5.56}
\expandafter\def\csname result@ext_flux_flux_oft_cap256_0p0001_drift_mean\endcsname{43.69}
\expandafter\def\csname result@ext_flux_flux_oft_cap256_0p0001_drift_sd\endcsname{6.55}
\expandafter\def\csname result@ext_flux_flux_oft_cap256_0p0001_clip_mean\endcsname{16.51}
\expandafter\def\csname result@ext_flux_flux_oft_cap256_0p0001_clip_sd\endcsname{5.30}
\expandafter\def\csname result@ext_flux_flux_oft_cap256_0p0001_validation_mean\endcsname{74.71}
\expandafter\def\csname result@ext_flux_flux_oft_cap256_0p0001_validation_sd\endcsname{4.21}
\expandafter\def\csname result@ext_flux_flux_pissa_cap4_0p0001_params\endcsname{4.79}
\expandafter\def\csname result@ext_flux_flux_pissa_cap4_0p0001_dino_mean\endcsname{68.29}
\expandafter\def\csname result@ext_flux_flux_pissa_cap4_0p0001_dino_sd\endcsname{3.41}
\expandafter\def\csname result@ext_flux_flux_pissa_cap4_0p0001_drift_mean\endcsname{18.63}
\expandafter\def\csname result@ext_flux_flux_pissa_cap4_0p0001_drift_sd\endcsname{4.99}
\expandafter\def\csname result@ext_flux_flux_pissa_cap4_0p0001_clip_mean\endcsname{27.80}
\expandafter\def\csname result@ext_flux_flux_pissa_cap4_0p0001_clip_sd\endcsname{1.04}
\expandafter\def\csname result@ext_flux_flux_pissa_cap4_0p0001_validation_mean\endcsname{73.93}
\expandafter\def\csname result@ext_flux_flux_pissa_cap4_0p0001_validation_sd\endcsname{1.62}
\expandafter\def\csname result@ext_flux_flux_pissa_5em05_params\endcsname{9.58}
\expandafter\def\csname result@ext_flux_flux_pissa_5em05_dino_mean\endcsname{69.41}
\expandafter\def\csname result@ext_flux_flux_pissa_5em05_dino_sd\endcsname{1.43}
\expandafter\def\csname result@ext_flux_flux_pissa_5em05_drift_mean\endcsname{18.32}
\expandafter\def\csname result@ext_flux_flux_pissa_5em05_drift_sd\endcsname{2.22}
\expandafter\def\csname result@ext_flux_flux_pissa_5em05_clip_mean\endcsname{28.04}
\expandafter\def\csname result@ext_flux_flux_pissa_5em05_clip_sd\endcsname{0.55}
\expandafter\def\csname result@ext_flux_flux_pissa_5em05_validation_mean\endcsname{72.23}
\expandafter\def\csname result@ext_flux_flux_pissa_5em05_validation_sd\endcsname{3.77}
\expandafter\def\csname result@ext_flux_flux_pissa_cap16_5em05_params\endcsname{19.17}
\expandafter\def\csname result@ext_flux_flux_pissa_cap16_5em05_dino_mean\endcsname{69.98}
\expandafter\def\csname result@ext_flux_flux_pissa_cap16_5em05_dino_sd\endcsname{5.48}
\expandafter\def\csname result@ext_flux_flux_pissa_cap16_5em05_drift_mean\endcsname{19.69}
\expandafter\def\csname result@ext_flux_flux_pissa_cap16_5em05_drift_sd\endcsname{2.20}
\expandafter\def\csname result@ext_flux_flux_pissa_cap16_5em05_clip_mean\endcsname{26.83}
\expandafter\def\csname result@ext_flux_flux_pissa_cap16_5em05_clip_sd\endcsname{0.78}
\expandafter\def\csname result@ext_flux_flux_pissa_cap16_5em05_validation_mean\endcsname{71.70}
\expandafter\def\csname result@ext_flux_flux_pissa_cap16_5em05_validation_sd\endcsname{3.16}
\expandafter\def\csname result@ext_flux_flux_pissa_cap32_5em05_params\endcsname{38.34}
\expandafter\def\csname result@ext_flux_flux_pissa_cap32_5em05_dino_mean\endcsname{68.92}
\expandafter\def\csname result@ext_flux_flux_pissa_cap32_5em05_dino_sd\endcsname{3.47}
\expandafter\def\csname result@ext_flux_flux_pissa_cap32_5em05_drift_mean\endcsname{26.14}
\expandafter\def\csname result@ext_flux_flux_pissa_cap32_5em05_drift_sd\endcsname{2.62}
\expandafter\def\csname result@ext_flux_flux_pissa_cap32_5em05_clip_mean\endcsname{26.79}
\expandafter\def\csname result@ext_flux_flux_pissa_cap32_5em05_clip_sd\endcsname{1.44}
\expandafter\def\csname result@ext_flux_flux_pissa_cap32_5em05_validation_mean\endcsname{71.87}
\expandafter\def\csname result@ext_flux_flux_pissa_cap32_5em05_validation_sd\endcsname{2.13}
\expandafter\def\csname result@rev_llama_dora_7_reconstruct_accuracy_mean\endcsname{71.47}
\expandafter\def\csname result@rev_llama_dora_7_reconstruct_accuracy_sd\endcsname{1.03}
\expandafter\def\csname result@rev_llama_dora_7_restore_accuracy_mean\endcsname{71.17}
\expandafter\def\csname result@rev_llama_dora_7_restore_accuracy_sd\endcsname{0.57}
\expandafter\def\csname result@rev_llama_dora_7_restore_nll_vs_reconstruct_mean\endcsname{-0.00029}
\expandafter\def\csname result@rev_llama_dora_7_restore_nll_vs_reconstruct_sd\endcsname{0.00029}
\expandafter\def\csname result@rev_llama_dora_7_rotation_band_restore_bottom_accuracy_change_mean\endcsname{-0.27}
\expandafter\def\csname result@rev_llama_dora_7_rotation_band_restore_bottom_accuracy_change_sd\endcsname{0.50}
\expandafter\def\csname result@rev_llama_dora_7_rotation_band_restore_mid_accuracy_change_mean\endcsname{-0.30}
\expandafter\def\csname result@rev_llama_dora_7_rotation_band_restore_mid_accuracy_change_sd\endcsname{0.52}
\expandafter\def\csname result@rev_llama_dora_7_rotation_band_restore_top_accuracy_change_mean\endcsname{-45.40}
\expandafter\def\csname result@rev_llama_dora_7_rotation_band_restore_top_accuracy_change_sd\endcsname{10.83}
\expandafter\def\csname result@rev_llama_dora_7_rotation_band_restore_top_nll_change_mean\endcsname{-0.02029}
\expandafter\def\csname result@rev_llama_dora_7_rotation_band_restore_top_nll_change_sd\endcsname{0.00029}
\expandafter\def\csname result@rev_llama_dora_7_trained_accuracy_mean\endcsname{71.10}
\expandafter\def\csname result@rev_llama_dora_7_trained_accuracy_sd\endcsname{0.61}
\expandafter\def\csname result@rev_llama_dora_7_trained_d_W_mean\endcsname{2.200}
\expandafter\def\csname result@rev_llama_dora_7_trained_d_W_sd\endcsname{0.028}
\expandafter\def\csname result@rev_llama_dora_7_trained_d_S_mean\endcsname{0.04447}
\expandafter\def\csname result@rev_llama_dora_7_trained_d_S_sd\endcsname{0.00197}
\expandafter\def\csname result@rev_llama_dora_7_trained_ratio_mean\endcsname{2.02}
\expandafter\def\csname result@rev_llama_dora_7_trained_ratio_sd\endcsname{0.07}
\expandafter\def\csname result@rev_llama_dora_15_reconstruct_accuracy_mean\endcsname{72.40}
\expandafter\def\csname result@rev_llama_dora_15_reconstruct_accuracy_sd\endcsname{1.01}
\expandafter\def\csname result@rev_llama_dora_15_restore_accuracy_mean\endcsname{72.70}
\expandafter\def\csname result@rev_llama_dora_15_restore_accuracy_sd\endcsname{0.17}
\expandafter\def\csname result@rev_llama_dora_15_restore_nll_vs_reconstruct_mean\endcsname{-0.00065}
\expandafter\def\csname result@rev_llama_dora_15_restore_nll_vs_reconstruct_sd\endcsname{0.00034}
\expandafter\def\csname result@rev_llama_dora_15_rotation_band_restore_bottom_accuracy_change_mean\endcsname{-1.00}
\expandafter\def\csname result@rev_llama_dora_15_rotation_band_restore_bottom_accuracy_change_sd\endcsname{0.44}
\expandafter\def\csname result@rev_llama_dora_15_rotation_band_restore_mid_accuracy_change_mean\endcsname{-1.07}
\expandafter\def\csname result@rev_llama_dora_15_rotation_band_restore_mid_accuracy_change_sd\endcsname{0.42}
\expandafter\def\csname result@rev_llama_dora_15_rotation_band_restore_top_accuracy_change_mean\endcsname{-42.10}
\expandafter\def\csname result@rev_llama_dora_15_rotation_band_restore_top_accuracy_change_sd\endcsname{8.64}
\expandafter\def\csname result@rev_llama_dora_15_rotation_band_restore_top_nll_change_mean\endcsname{-0.02078}
\expandafter\def\csname result@rev_llama_dora_15_rotation_band_restore_top_nll_change_sd\endcsname{0.00252}
\expandafter\def\csname result@rev_llama_dora_15_trained_accuracy_mean\endcsname{72.77}
\expandafter\def\csname result@rev_llama_dora_15_trained_accuracy_sd\endcsname{0.87}
\expandafter\def\csname result@rev_llama_dora_15_trained_d_W_mean\endcsname{2.871}
\expandafter\def\csname result@rev_llama_dora_15_trained_d_W_sd\endcsname{0.023}
\expandafter\def\csname result@rev_llama_dora_15_trained_d_S_mean\endcsname{0.06579}
\expandafter\def\csname result@rev_llama_dora_15_trained_d_S_sd\endcsname{0.00819}
\expandafter\def\csname result@rev_llama_dora_15_trained_ratio_mean\endcsname{2.29}
\expandafter\def\csname result@rev_llama_dora_15_trained_ratio_sd\endcsname{0.27}
\expandafter\def\csname result@rev_llama_dora_30_reconstruct_accuracy_mean\endcsname{73.33}
\expandafter\def\csname result@rev_llama_dora_30_reconstruct_accuracy_sd\endcsname{1.90}
\expandafter\def\csname result@rev_llama_dora_30_restore_accuracy_mean\endcsname{73.07}
\expandafter\def\csname result@rev_llama_dora_30_restore_accuracy_sd\endcsname{1.48}
\expandafter\def\csname result@rev_llama_dora_30_restore_nll_vs_reconstruct_mean\endcsname{-0.00208}
\expandafter\def\csname result@rev_llama_dora_30_restore_nll_vs_reconstruct_sd\endcsname{0.00085}
\expandafter\def\csname result@rev_llama_dora_30_rotation_band_restore_bottom_accuracy_change_mean\endcsname{-1.20}
\expandafter\def\csname result@rev_llama_dora_30_rotation_band_restore_bottom_accuracy_change_sd\endcsname{1.44}
\expandafter\def\csname result@rev_llama_dora_30_rotation_band_restore_mid_accuracy_change_mean\endcsname{-1.77}
\expandafter\def\csname result@rev_llama_dora_30_rotation_band_restore_mid_accuracy_change_sd\endcsname{1.17}
\expandafter\def\csname result@rev_llama_dora_30_rotation_band_restore_top_accuracy_change_mean\endcsname{-27.03}
\expandafter\def\csname result@rev_llama_dora_30_rotation_band_restore_top_accuracy_change_sd\endcsname{1.55}
\expandafter\def\csname result@rev_llama_dora_30_rotation_band_restore_top_nll_change_mean\endcsname{-0.02313}
\expandafter\def\csname result@rev_llama_dora_30_rotation_band_restore_top_nll_change_sd\endcsname{0.00277}
\expandafter\def\csname result@rev_llama_dora_30_trained_accuracy_mean\endcsname{73.10}
\expandafter\def\csname result@rev_llama_dora_30_trained_accuracy_sd\endcsname{2.19}
\expandafter\def\csname result@rev_llama_dora_30_trained_d_W_mean\endcsname{3.712}
\expandafter\def\csname result@rev_llama_dora_30_trained_d_W_sd\endcsname{0.015}
\expandafter\def\csname result@rev_llama_dora_30_trained_d_S_mean\endcsname{0.10183}
\expandafter\def\csname result@rev_llama_dora_30_trained_d_S_sd\endcsname{0.00382}
\expandafter\def\csname result@rev_llama_dora_30_trained_ratio_mean\endcsname{2.74}
\expandafter\def\csname result@rev_llama_dora_30_trained_ratio_sd\endcsname{0.09}
\expandafter\def\csname result@rev_llama_hra_16_reconstruct_accuracy_mean\endcsname{70.83}
\expandafter\def\csname result@rev_llama_hra_16_reconstruct_accuracy_sd\endcsname{0.87}
\expandafter\def\csname result@rev_llama_hra_16_restore_accuracy_mean\endcsname{70.50}
\expandafter\def\csname result@rev_llama_hra_16_restore_accuracy_sd\endcsname{0.56}
\expandafter\def\csname result@rev_llama_hra_16_restore_nll_vs_reconstruct_mean\endcsname{0.00010}
\expandafter\def\csname result@rev_llama_hra_16_restore_nll_vs_reconstruct_sd\endcsname{0.00006}
\expandafter\def\csname result@rev_llama_hra_16_rotation_band_restore_bottom_accuracy_change_mean\endcsname{-0.27}
\expandafter\def\csname result@rev_llama_hra_16_rotation_band_restore_bottom_accuracy_change_sd\endcsname{1.43}
\expandafter\def\csname result@rev_llama_hra_16_rotation_band_restore_mid_accuracy_change_mean\endcsname{-0.77}
\expandafter\def\csname result@rev_llama_hra_16_rotation_band_restore_mid_accuracy_change_sd\endcsname{1.07}
\expandafter\def\csname result@rev_llama_hra_16_rotation_band_restore_top_accuracy_change_mean\endcsname{-19.33}
\expandafter\def\csname result@rev_llama_hra_16_rotation_band_restore_top_accuracy_change_sd\endcsname{3.96}
\expandafter\def\csname result@rev_llama_hra_16_rotation_band_restore_top_nll_change_mean\endcsname{-0.04873}
\expandafter\def\csname result@rev_llama_hra_16_rotation_band_restore_top_nll_change_sd\endcsname{0.00166}
\expandafter\def\csname result@rev_llama_hra_16_trained_accuracy_mean\endcsname{70.57}
\expandafter\def\csname result@rev_llama_hra_16_trained_accuracy_sd\endcsname{0.95}
\expandafter\def\csname result@rev_llama_hra_16_trained_d_W_mean\endcsname{5.547}
\expandafter\def\csname result@rev_llama_hra_16_trained_d_W_sd\endcsname{0.023}
\expandafter\def\csname result@rev_llama_hra_16_trained_d_S_mean\endcsname{0.00188}
\expandafter\def\csname result@rev_llama_hra_16_trained_d_S_sd\endcsname{0.00000}
\expandafter\def\csname result@rev_llama_hra_16_trained_ratio_mean\endcsname{0.03}
\expandafter\def\csname result@rev_llama_hra_16_trained_ratio_sd\endcsname{0.00}
\expandafter\def\csname result@rev_llama_hra_32_reconstruct_accuracy_mean\endcsname{71.03}
\expandafter\def\csname result@rev_llama_hra_32_reconstruct_accuracy_sd\endcsname{1.29}
\expandafter\def\csname result@rev_llama_hra_32_restore_accuracy_mean\endcsname{70.40}
\expandafter\def\csname result@rev_llama_hra_32_restore_accuracy_sd\endcsname{1.30}
\expandafter\def\csname result@rev_llama_hra_32_restore_nll_vs_reconstruct_mean\endcsname{0.00012}
\expandafter\def\csname result@rev_llama_hra_32_restore_nll_vs_reconstruct_sd\endcsname{0.00015}
\expandafter\def\csname result@rev_llama_hra_32_rotation_band_restore_bottom_accuracy_change_mean\endcsname{0.40}
\expandafter\def\csname result@rev_llama_hra_32_rotation_band_restore_bottom_accuracy_change_sd\endcsname{1.25}
\expandafter\def\csname result@rev_llama_hra_32_rotation_band_restore_mid_accuracy_change_mean\endcsname{-0.73}
\expandafter\def\csname result@rev_llama_hra_32_rotation_band_restore_mid_accuracy_change_sd\endcsname{0.50}
\expandafter\def\csname result@rev_llama_hra_32_rotation_band_restore_top_accuracy_change_mean\endcsname{-14.07}
\expandafter\def\csname result@rev_llama_hra_32_rotation_band_restore_top_accuracy_change_sd\endcsname{5.22}
\expandafter\def\csname result@rev_llama_hra_32_rotation_band_restore_top_nll_change_mean\endcsname{-0.09701}
\expandafter\def\csname result@rev_llama_hra_32_rotation_band_restore_top_nll_change_sd\endcsname{0.00496}
\expandafter\def\csname result@rev_llama_hra_32_trained_accuracy_mean\endcsname{71.00}
\expandafter\def\csname result@rev_llama_hra_32_trained_accuracy_sd\endcsname{1.67}
\expandafter\def\csname result@rev_llama_hra_32_trained_d_W_mean\endcsname{9.077}
\expandafter\def\csname result@rev_llama_hra_32_trained_d_W_sd\endcsname{0.052}
\expandafter\def\csname result@rev_llama_hra_32_trained_d_S_mean\endcsname{0.00189}
\expandafter\def\csname result@rev_llama_hra_32_trained_d_S_sd\endcsname{0.00000}
\expandafter\def\csname result@rev_llama_hra_32_trained_ratio_mean\endcsname{0.02}
\expandafter\def\csname result@rev_llama_hra_32_trained_ratio_sd\endcsname{0.00}
\expandafter\def\csname result@rev_llama_hra_64_reconstruct_accuracy_mean\endcsname{71.07}
\expandafter\def\csname result@rev_llama_hra_64_reconstruct_accuracy_sd\endcsname{1.18}
\expandafter\def\csname result@rev_llama_hra_64_restore_accuracy_mean\endcsname{70.93}
\expandafter\def\csname result@rev_llama_hra_64_restore_accuracy_sd\endcsname{0.83}
\expandafter\def\csname result@rev_llama_hra_64_restore_nll_vs_reconstruct_mean\endcsname{-0.00005}
\expandafter\def\csname result@rev_llama_hra_64_restore_nll_vs_reconstruct_sd\endcsname{0.00023}
\expandafter\def\csname result@rev_llama_hra_64_rotation_band_restore_bottom_accuracy_change_mean\endcsname{0.17}
\expandafter\def\csname result@rev_llama_hra_64_rotation_band_restore_bottom_accuracy_change_sd\endcsname{1.04}
\expandafter\def\csname result@rev_llama_hra_64_rotation_band_restore_mid_accuracy_change_mean\endcsname{-1.13}
\expandafter\def\csname result@rev_llama_hra_64_rotation_band_restore_mid_accuracy_change_sd\endcsname{2.12}
\expandafter\def\csname result@rev_llama_hra_64_rotation_band_restore_top_accuracy_change_mean\endcsname{-11.07}
\expandafter\def\csname result@rev_llama_hra_64_rotation_band_restore_top_accuracy_change_sd\endcsname{0.81}
\expandafter\def\csname result@rev_llama_hra_64_rotation_band_restore_top_nll_change_mean\endcsname{-0.08909}
\expandafter\def\csname result@rev_llama_hra_64_rotation_band_restore_top_nll_change_sd\endcsname{0.00459}
\expandafter\def\csname result@rev_llama_hra_64_trained_accuracy_mean\endcsname{71.63}
\expandafter\def\csname result@rev_llama_hra_64_trained_accuracy_sd\endcsname{1.08}
\expandafter\def\csname result@rev_llama_hra_64_trained_d_W_mean\endcsname{6.864}
\expandafter\def\csname result@rev_llama_hra_64_trained_d_W_sd\endcsname{0.027}
\expandafter\def\csname result@rev_llama_hra_64_trained_d_S_mean\endcsname{0.00189}
\expandafter\def\csname result@rev_llama_hra_64_trained_d_S_sd\endcsname{0.00000}
\expandafter\def\csname result@rev_llama_hra_64_trained_ratio_mean\endcsname{0.03}
\expandafter\def\csname result@rev_llama_hra_64_trained_ratio_sd\endcsname{0.00}
\expandafter\def\csname result@rev_llama_lora_7_reconstruct_accuracy_mean\endcsname{71.30}
\expandafter\def\csname result@rev_llama_lora_7_reconstruct_accuracy_sd\endcsname{0.46}
\expandafter\def\csname result@rev_llama_lora_7_restore_accuracy_mean\endcsname{71.07}
\expandafter\def\csname result@rev_llama_lora_7_restore_accuracy_sd\endcsname{0.51}
\expandafter\def\csname result@rev_llama_lora_7_restore_nll_vs_reconstruct_mean\endcsname{0.00006}
\expandafter\def\csname result@rev_llama_lora_7_restore_nll_vs_reconstruct_sd\endcsname{0.00016}
\expandafter\def\csname result@rev_llama_lora_7_rotation_band_restore_bottom_accuracy_change_mean\endcsname{-0.03}
\expandafter\def\csname result@rev_llama_lora_7_rotation_band_restore_bottom_accuracy_change_sd\endcsname{0.71}
\expandafter\def\csname result@rev_llama_lora_7_rotation_band_restore_mid_accuracy_change_mean\endcsname{-0.67}
\expandafter\def\csname result@rev_llama_lora_7_rotation_band_restore_mid_accuracy_change_sd\endcsname{0.59}
\expandafter\def\csname result@rev_llama_lora_7_rotation_band_restore_top_accuracy_change_mean\endcsname{-48.23}
\expandafter\def\csname result@rev_llama_lora_7_rotation_band_restore_top_accuracy_change_sd\endcsname{11.06}
\expandafter\def\csname result@rev_llama_lora_7_rotation_band_restore_top_nll_change_mean\endcsname{-0.02150}
\expandafter\def\csname result@rev_llama_lora_7_rotation_band_restore_top_nll_change_sd\endcsname{0.00224}
\expandafter\def\csname result@rev_llama_lora_7_trained_accuracy_mean\endcsname{71.00}
\expandafter\def\csname result@rev_llama_lora_7_trained_accuracy_sd\endcsname{0.35}
\expandafter\def\csname result@rev_llama_lora_7_trained_d_W_mean\endcsname{2.167}
\expandafter\def\csname result@rev_llama_lora_7_trained_d_W_sd\endcsname{0.020}
\expandafter\def\csname result@rev_llama_lora_7_trained_d_S_mean\endcsname{0.04985}
\expandafter\def\csname result@rev_llama_lora_7_trained_d_S_sd\endcsname{0.00207}
\expandafter\def\csname result@rev_llama_lora_7_trained_ratio_mean\endcsname{2.30}
\expandafter\def\csname result@rev_llama_lora_7_trained_ratio_sd\endcsname{0.08}
\expandafter\def\csname result@rev_llama_lora_15_reconstruct_accuracy_mean\endcsname{72.50}
\expandafter\def\csname result@rev_llama_lora_15_reconstruct_accuracy_sd\endcsname{0.26}
\expandafter\def\csname result@rev_llama_lora_15_restore_accuracy_mean\endcsname{72.23}
\expandafter\def\csname result@rev_llama_lora_15_restore_accuracy_sd\endcsname{0.90}
\expandafter\def\csname result@rev_llama_lora_15_restore_nll_vs_reconstruct_mean\endcsname{-0.00023}
\expandafter\def\csname result@rev_llama_lora_15_restore_nll_vs_reconstruct_sd\endcsname{0.00025}
\expandafter\def\csname result@rev_llama_lora_15_rotation_band_restore_bottom_accuracy_change_mean\endcsname{-0.17}
\expandafter\def\csname result@rev_llama_lora_15_rotation_band_restore_bottom_accuracy_change_sd\endcsname{0.61}
\expandafter\def\csname result@rev_llama_lora_15_rotation_band_restore_mid_accuracy_change_mean\endcsname{-1.57}
\expandafter\def\csname result@rev_llama_lora_15_rotation_band_restore_mid_accuracy_change_sd\endcsname{1.36}
\expandafter\def\csname result@rev_llama_lora_15_rotation_band_restore_top_accuracy_change_mean\endcsname{-34.73}
\expandafter\def\csname result@rev_llama_lora_15_rotation_band_restore_top_accuracy_change_sd\endcsname{6.28}
\expandafter\def\csname result@rev_llama_lora_15_rotation_band_restore_top_nll_change_mean\endcsname{-0.02192}
\expandafter\def\csname result@rev_llama_lora_15_rotation_band_restore_top_nll_change_sd\endcsname{0.00220}
\expandafter\def\csname result@rev_llama_lora_15_trained_accuracy_mean\endcsname{72.33}
\expandafter\def\csname result@rev_llama_lora_15_trained_accuracy_sd\endcsname{0.93}
\expandafter\def\csname result@rev_llama_lora_15_trained_d_W_mean\endcsname{2.845}
\expandafter\def\csname result@rev_llama_lora_15_trained_d_W_sd\endcsname{0.026}
\expandafter\def\csname result@rev_llama_lora_15_trained_d_S_mean\endcsname{0.08243}
\expandafter\def\csname result@rev_llama_lora_15_trained_d_S_sd\endcsname{0.00717}
\expandafter\def\csname result@rev_llama_lora_15_trained_ratio_mean\endcsname{2.90}
\expandafter\def\csname result@rev_llama_lora_15_trained_ratio_sd\endcsname{0.23}
\expandafter\def\csname result@rev_llama_lora_30_reconstruct_accuracy_mean\endcsname{72.40}
\expandafter\def\csname result@rev_llama_lora_30_reconstruct_accuracy_sd\endcsname{0.89}
\expandafter\def\csname result@rev_llama_lora_30_restore_accuracy_mean\endcsname{71.97}
\expandafter\def\csname result@rev_llama_lora_30_restore_accuracy_sd\endcsname{1.43}
\expandafter\def\csname result@rev_llama_lora_30_restore_nll_vs_reconstruct_mean\endcsname{-0.00254}
\expandafter\def\csname result@rev_llama_lora_30_restore_nll_vs_reconstruct_sd\endcsname{0.00068}
\expandafter\def\csname result@rev_llama_lora_30_rotation_band_restore_bottom_accuracy_change_mean\endcsname{-0.60}
\expandafter\def\csname result@rev_llama_lora_30_rotation_band_restore_bottom_accuracy_change_sd\endcsname{1.14}
\expandafter\def\csname result@rev_llama_lora_30_rotation_band_restore_mid_accuracy_change_mean\endcsname{-0.60}
\expandafter\def\csname result@rev_llama_lora_30_rotation_band_restore_mid_accuracy_change_sd\endcsname{1.31}
\expandafter\def\csname result@rev_llama_lora_30_rotation_band_restore_top_accuracy_change_mean\endcsname{-27.43}
\expandafter\def\csname result@rev_llama_lora_30_rotation_band_restore_top_accuracy_change_sd\endcsname{1.07}
\expandafter\def\csname result@rev_llama_lora_30_rotation_band_restore_top_nll_change_mean\endcsname{-0.02506}
\expandafter\def\csname result@rev_llama_lora_30_rotation_band_restore_top_nll_change_sd\endcsname{0.00329}
\expandafter\def\csname result@rev_llama_lora_30_trained_accuracy_mean\endcsname{72.20}
\expandafter\def\csname result@rev_llama_lora_30_trained_accuracy_sd\endcsname{1.28}
\expandafter\def\csname result@rev_llama_lora_30_trained_d_W_mean\endcsname{3.707}
\expandafter\def\csname result@rev_llama_lora_30_trained_d_W_sd\endcsname{0.024}
\expandafter\def\csname result@rev_llama_lora_30_trained_d_S_mean\endcsname{0.12641}
\expandafter\def\csname result@rev_llama_lora_30_trained_d_S_sd\endcsname{0.00536}
\expandafter\def\csname result@rev_llama_lora_30_trained_ratio_mean\endcsname{3.41}
\expandafter\def\csname result@rev_llama_lora_30_trained_ratio_sd\endcsname{0.13}
\expandafter\def\csname result@rev_llama_milora_7_reconstruct_accuracy_mean\endcsname{70.90}
\expandafter\def\csname result@rev_llama_milora_7_reconstruct_accuracy_sd\endcsname{0.44}
\expandafter\def\csname result@rev_llama_milora_7_restore_accuracy_mean\endcsname{71.50}
\expandafter\def\csname result@rev_llama_milora_7_restore_accuracy_sd\endcsname{0.30}
\expandafter\def\csname result@rev_llama_milora_7_restore_nll_vs_reconstruct_mean\endcsname{-0.00046}
\expandafter\def\csname result@rev_llama_milora_7_restore_nll_vs_reconstruct_sd\endcsname{0.00061}
\expandafter\def\csname result@rev_llama_milora_7_rotation_band_restore_bottom_accuracy_change_mean\endcsname{-1.43}
\expandafter\def\csname result@rev_llama_milora_7_rotation_band_restore_bottom_accuracy_change_sd\endcsname{0.49}
\expandafter\def\csname result@rev_llama_milora_7_rotation_band_restore_mid_accuracy_change_mean\endcsname{0.67}
\expandafter\def\csname result@rev_llama_milora_7_rotation_band_restore_mid_accuracy_change_sd\endcsname{0.85}
\expandafter\def\csname result@rev_llama_milora_7_rotation_band_restore_top_accuracy_change_mean\endcsname{-11.93}
\expandafter\def\csname result@rev_llama_milora_7_rotation_band_restore_top_accuracy_change_sd\endcsname{1.76}
\expandafter\def\csname result@rev_llama_milora_7_rotation_band_restore_top_nll_change_mean\endcsname{-0.03790}
\expandafter\def\csname result@rev_llama_milora_7_rotation_band_restore_top_nll_change_sd\endcsname{0.00074}
\expandafter\def\csname result@rev_llama_milora_7_trained_accuracy_mean\endcsname{70.97}
\expandafter\def\csname result@rev_llama_milora_7_trained_accuracy_sd\endcsname{0.80}
\expandafter\def\csname result@rev_llama_milora_7_trained_d_W_mean\endcsname{1.851}
\expandafter\def\csname result@rev_llama_milora_7_trained_d_W_sd\endcsname{0.015}
\expandafter\def\csname result@rev_llama_milora_7_trained_d_S_mean\endcsname{0.16542}
\expandafter\def\csname result@rev_llama_milora_7_trained_d_S_sd\endcsname{0.00088}
\expandafter\def\csname result@rev_llama_milora_7_trained_ratio_mean\endcsname{8.94}
\expandafter\def\csname result@rev_llama_milora_7_trained_ratio_sd\endcsname{0.02}
\expandafter\def\csname result@rev_llama_milora_15_reconstruct_accuracy_mean\endcsname{71.97}
\expandafter\def\csname result@rev_llama_milora_15_reconstruct_accuracy_sd\endcsname{0.86}
\expandafter\def\csname result@rev_llama_milora_15_restore_accuracy_mean\endcsname{71.70}
\expandafter\def\csname result@rev_llama_milora_15_restore_accuracy_sd\endcsname{0.46}
\expandafter\def\csname result@rev_llama_milora_15_restore_nll_vs_reconstruct_mean\endcsname{-0.00042}
\expandafter\def\csname result@rev_llama_milora_15_restore_nll_vs_reconstruct_sd\endcsname{0.00058}
\expandafter\def\csname result@rev_llama_milora_15_rotation_band_restore_bottom_accuracy_change_mean\endcsname{-1.77}
\expandafter\def\csname result@rev_llama_milora_15_rotation_band_restore_bottom_accuracy_change_sd\endcsname{0.85}
\expandafter\def\csname result@rev_llama_milora_15_rotation_band_restore_mid_accuracy_change_mean\endcsname{-1.07}
\expandafter\def\csname result@rev_llama_milora_15_rotation_band_restore_mid_accuracy_change_sd\endcsname{0.15}
\expandafter\def\csname result@rev_llama_milora_15_rotation_band_restore_top_accuracy_change_mean\endcsname{-7.57}
\expandafter\def\csname result@rev_llama_milora_15_rotation_band_restore_top_accuracy_change_sd\endcsname{0.25}
\expandafter\def\csname result@rev_llama_milora_15_rotation_band_restore_top_nll_change_mean\endcsname{-0.04841}
\expandafter\def\csname result@rev_llama_milora_15_rotation_band_restore_top_nll_change_sd\endcsname{0.00096}
\expandafter\def\csname result@rev_llama_milora_15_trained_accuracy_mean\endcsname{71.90}
\expandafter\def\csname result@rev_llama_milora_15_trained_accuracy_sd\endcsname{0.46}
\expandafter\def\csname result@rev_llama_milora_15_trained_d_W_mean\endcsname{2.457}
\expandafter\def\csname result@rev_llama_milora_15_trained_d_W_sd\endcsname{0.013}
\expandafter\def\csname result@rev_llama_milora_15_trained_d_S_mean\endcsname{0.25238}
\expandafter\def\csname result@rev_llama_milora_15_trained_d_S_sd\endcsname{0.00124}
\expandafter\def\csname result@rev_llama_milora_15_trained_ratio_mean\endcsname{10.27}
\expandafter\def\csname result@rev_llama_milora_15_trained_ratio_sd\endcsname{0.04}
\expandafter\def\csname result@rev_llama_milora_30_reconstruct_accuracy_mean\endcsname{72.53}
\expandafter\def\csname result@rev_llama_milora_30_reconstruct_accuracy_sd\endcsname{0.92}
\expandafter\def\csname result@rev_llama_milora_30_restore_accuracy_mean\endcsname{71.97}
\expandafter\def\csname result@rev_llama_milora_30_restore_accuracy_sd\endcsname{0.81}
\expandafter\def\csname result@rev_llama_milora_30_restore_nll_vs_reconstruct_mean\endcsname{0.00208}
\expandafter\def\csname result@rev_llama_milora_30_restore_nll_vs_reconstruct_sd\endcsname{0.00065}
\expandafter\def\csname result@rev_llama_milora_30_rotation_band_restore_bottom_accuracy_change_mean\endcsname{-3.17}
\expandafter\def\csname result@rev_llama_milora_30_rotation_band_restore_bottom_accuracy_change_sd\endcsname{0.91}
\expandafter\def\csname result@rev_llama_milora_30_rotation_band_restore_mid_accuracy_change_mean\endcsname{-1.00}
\expandafter\def\csname result@rev_llama_milora_30_rotation_band_restore_mid_accuracy_change_sd\endcsname{1.08}
\expandafter\def\csname result@rev_llama_milora_30_rotation_band_restore_top_accuracy_change_mean\endcsname{-5.77}
\expandafter\def\csname result@rev_llama_milora_30_rotation_band_restore_top_accuracy_change_sd\endcsname{0.49}
\expandafter\def\csname result@rev_llama_milora_30_rotation_band_restore_top_nll_change_mean\endcsname{-0.04858}
\expandafter\def\csname result@rev_llama_milora_30_rotation_band_restore_top_nll_change_sd\endcsname{0.00341}
\expandafter\def\csname result@rev_llama_milora_30_trained_accuracy_mean\endcsname{72.63}
\expandafter\def\csname result@rev_llama_milora_30_trained_accuracy_sd\endcsname{1.01}
\expandafter\def\csname result@rev_llama_milora_30_trained_d_W_mean\endcsname{3.282}
\expandafter\def\csname result@rev_llama_milora_30_trained_d_W_sd\endcsname{0.019}
\expandafter\def\csname result@rev_llama_milora_30_trained_d_S_mean\endcsname{0.36372}
\expandafter\def\csname result@rev_llama_milora_30_trained_d_S_sd\endcsname{0.00318}
\expandafter\def\csname result@rev_llama_milora_30_trained_ratio_mean\endcsname{11.08}
\expandafter\def\csname result@rev_llama_milora_30_trained_ratio_sd\endcsname{0.05}
\expandafter\def\csname result@rev_llama_oft_32_reconstruct_accuracy_mean\endcsname{71.77}
\expandafter\def\csname result@rev_llama_oft_32_reconstruct_accuracy_sd\endcsname{0.81}
\expandafter\def\csname result@rev_llama_oft_32_restore_accuracy_mean\endcsname{71.73}
\expandafter\def\csname result@rev_llama_oft_32_restore_accuracy_sd\endcsname{1.29}
\expandafter\def\csname result@rev_llama_oft_32_restore_nll_vs_reconstruct_mean\endcsname{-0.00007}
\expandafter\def\csname result@rev_llama_oft_32_restore_nll_vs_reconstruct_sd\endcsname{0.00012}
\expandafter\def\csname result@rev_llama_oft_32_rotation_band_restore_bottom_accuracy_change_mean\endcsname{-0.03}
\expandafter\def\csname result@rev_llama_oft_32_rotation_band_restore_bottom_accuracy_change_sd\endcsname{0.38}
\expandafter\def\csname result@rev_llama_oft_32_rotation_band_restore_mid_accuracy_change_mean\endcsname{-0.40}
\expandafter\def\csname result@rev_llama_oft_32_rotation_band_restore_mid_accuracy_change_sd\endcsname{1.56}
\expandafter\def\csname result@rev_llama_oft_32_rotation_band_restore_top_accuracy_change_mean\endcsname{-23.30}
\expandafter\def\csname result@rev_llama_oft_32_rotation_band_restore_top_accuracy_change_sd\endcsname{3.12}
\expandafter\def\csname result@rev_llama_oft_32_rotation_band_restore_top_nll_change_mean\endcsname{-0.03143}
\expandafter\def\csname result@rev_llama_oft_32_rotation_band_restore_top_nll_change_sd\endcsname{0.00099}
\expandafter\def\csname result@rev_llama_oft_32_trained_accuracy_mean\endcsname{71.80}
\expandafter\def\csname result@rev_llama_oft_32_trained_accuracy_sd\endcsname{1.37}
\expandafter\def\csname result@rev_llama_oft_32_trained_d_W_mean\endcsname{5.023}
\expandafter\def\csname result@rev_llama_oft_32_trained_d_W_sd\endcsname{0.011}
\expandafter\def\csname result@rev_llama_oft_32_trained_d_S_mean\endcsname{0.00190}
\expandafter\def\csname result@rev_llama_oft_32_trained_d_S_sd\endcsname{0.00000}
\expandafter\def\csname result@rev_llama_oft_32_trained_ratio_mean\endcsname{0.04}
\expandafter\def\csname result@rev_llama_oft_32_trained_ratio_sd\endcsname{0.00}
\expandafter\def\csname result@rev_llama_oft_64_reconstruct_accuracy_mean\endcsname{73.03}
\expandafter\def\csname result@rev_llama_oft_64_reconstruct_accuracy_sd\endcsname{0.29}
\expandafter\def\csname result@rev_llama_oft_64_restore_accuracy_mean\endcsname{73.57}
\expandafter\def\csname result@rev_llama_oft_64_restore_accuracy_sd\endcsname{0.49}
\expandafter\def\csname result@rev_llama_oft_64_restore_nll_vs_reconstruct_mean\endcsname{0.00003}
\expandafter\def\csname result@rev_llama_oft_64_restore_nll_vs_reconstruct_sd\endcsname{0.00010}
\expandafter\def\csname result@rev_llama_oft_64_rotation_band_restore_bottom_accuracy_change_mean\endcsname{-0.63}
\expandafter\def\csname result@rev_llama_oft_64_rotation_band_restore_bottom_accuracy_change_sd\endcsname{1.42}
\expandafter\def\csname result@rev_llama_oft_64_rotation_band_restore_mid_accuracy_change_mean\endcsname{-1.43}
\expandafter\def\csname result@rev_llama_oft_64_rotation_band_restore_mid_accuracy_change_sd\endcsname{1.22}
\expandafter\def\csname result@rev_llama_oft_64_rotation_band_restore_top_accuracy_change_mean\endcsname{-17.63}
\expandafter\def\csname result@rev_llama_oft_64_rotation_band_restore_top_accuracy_change_sd\endcsname{3.27}
\expandafter\def\csname result@rev_llama_oft_64_rotation_band_restore_top_nll_change_mean\endcsname{-0.06295}
\expandafter\def\csname result@rev_llama_oft_64_rotation_band_restore_top_nll_change_sd\endcsname{0.00097}
\expandafter\def\csname result@rev_llama_oft_64_trained_accuracy_mean\endcsname{73.43}
\expandafter\def\csname result@rev_llama_oft_64_trained_accuracy_sd\endcsname{0.50}
\expandafter\def\csname result@rev_llama_oft_64_trained_d_W_mean\endcsname{6.626}
\expandafter\def\csname result@rev_llama_oft_64_trained_d_W_sd\endcsname{0.015}
\expandafter\def\csname result@rev_llama_oft_64_trained_d_S_mean\endcsname{0.00193}
\expandafter\def\csname result@rev_llama_oft_64_trained_d_S_sd\endcsname{0.00000}
\expandafter\def\csname result@rev_llama_oft_64_trained_ratio_mean\endcsname{0.03}
\expandafter\def\csname result@rev_llama_oft_64_trained_ratio_sd\endcsname{0.00}
\expandafter\def\csname result@rev_llama_oft_128_reconstruct_accuracy_mean\endcsname{72.40}
\expandafter\def\csname result@rev_llama_oft_128_reconstruct_accuracy_sd\endcsname{0.87}
\expandafter\def\csname result@rev_llama_oft_128_restore_accuracy_mean\endcsname{72.33}
\expandafter\def\csname result@rev_llama_oft_128_restore_accuracy_sd\endcsname{0.85}
\expandafter\def\csname result@rev_llama_oft_128_restore_nll_vs_reconstruct_mean\endcsname{0.00004}
\expandafter\def\csname result@rev_llama_oft_128_restore_nll_vs_reconstruct_sd\endcsname{0.00014}
\expandafter\def\csname result@rev_llama_oft_128_rotation_band_restore_bottom_accuracy_change_mean\endcsname{-0.07}
\expandafter\def\csname result@rev_llama_oft_128_rotation_band_restore_bottom_accuracy_change_sd\endcsname{0.58}
\expandafter\def\csname result@rev_llama_oft_128_rotation_band_restore_mid_accuracy_change_mean\endcsname{-0.83}
\expandafter\def\csname result@rev_llama_oft_128_rotation_band_restore_mid_accuracy_change_sd\endcsname{1.12}
\expandafter\def\csname result@rev_llama_oft_128_rotation_band_restore_top_accuracy_change_mean\endcsname{-18.97}
\expandafter\def\csname result@rev_llama_oft_128_rotation_band_restore_top_accuracy_change_sd\endcsname{0.61}
\expandafter\def\csname result@rev_llama_oft_128_rotation_band_restore_top_nll_change_mean\endcsname{-0.03398}
\expandafter\def\csname result@rev_llama_oft_128_rotation_band_restore_top_nll_change_sd\endcsname{0.00105}
\expandafter\def\csname result@rev_llama_oft_128_trained_accuracy_mean\endcsname{72.47}
\expandafter\def\csname result@rev_llama_oft_128_trained_accuracy_sd\endcsname{1.31}
\expandafter\def\csname result@rev_llama_oft_128_trained_d_W_mean\endcsname{3.488}
\expandafter\def\csname result@rev_llama_oft_128_trained_d_W_sd\endcsname{0.007}
\expandafter\def\csname result@rev_llama_oft_128_trained_d_S_mean\endcsname{0.00188}
\expandafter\def\csname result@rev_llama_oft_128_trained_d_S_sd\endcsname{0.00000}
\expandafter\def\csname result@rev_llama_oft_128_trained_ratio_mean\endcsname{0.05}
\expandafter\def\csname result@rev_llama_oft_128_trained_ratio_sd\endcsname{0.00}
\expandafter\def\csname result@rev_llama_pissa_7_reconstruct_accuracy_mean\endcsname{70.67}
\expandafter\def\csname result@rev_llama_pissa_7_reconstruct_accuracy_sd\endcsname{0.55}
\expandafter\def\csname result@rev_llama_pissa_7_restore_accuracy_mean\endcsname{70.47}
\expandafter\def\csname result@rev_llama_pissa_7_restore_accuracy_sd\endcsname{0.25}
\expandafter\def\csname result@rev_llama_pissa_7_restore_nll_vs_reconstruct_mean\endcsname{-0.00509}
\expandafter\def\csname result@rev_llama_pissa_7_restore_nll_vs_reconstruct_sd\endcsname{0.00037}
\expandafter\def\csname result@rev_llama_pissa_7_rotation_band_restore_bottom_accuracy_change_mean\endcsname{-0.23}
\expandafter\def\csname result@rev_llama_pissa_7_rotation_band_restore_bottom_accuracy_change_sd\endcsname{0.38}
\expandafter\def\csname result@rev_llama_pissa_7_rotation_band_restore_mid_accuracy_change_mean\endcsname{-0.10}
\expandafter\def\csname result@rev_llama_pissa_7_rotation_band_restore_mid_accuracy_change_sd\endcsname{0.78}
\expandafter\def\csname result@rev_llama_pissa_7_rotation_band_restore_top_accuracy_change_mean\endcsname{-70.77}
\expandafter\def\csname result@rev_llama_pissa_7_rotation_band_restore_top_accuracy_change_sd\endcsname{0.40}
\expandafter\def\csname result@rev_llama_pissa_7_rotation_band_restore_top_nll_change_mean\endcsname{-0.10749}
\expandafter\def\csname result@rev_llama_pissa_7_rotation_band_restore_top_nll_change_sd\endcsname{0.00253}
\expandafter\def\csname result@rev_llama_pissa_7_trained_accuracy_mean\endcsname{70.77}
\expandafter\def\csname result@rev_llama_pissa_7_trained_accuracy_sd\endcsname{0.40}
\expandafter\def\csname result@rev_llama_pissa_7_trained_d_W_mean\endcsname{1.796}
\expandafter\def\csname result@rev_llama_pissa_7_trained_d_W_sd\endcsname{0.006}
\expandafter\def\csname result@rev_llama_pissa_7_trained_d_S_mean\endcsname{0.15402}
\expandafter\def\csname result@rev_llama_pissa_7_trained_d_S_sd\endcsname{0.00041}
\expandafter\def\csname result@rev_llama_pissa_7_trained_ratio_mean\endcsname{8.58}
\expandafter\def\csname result@rev_llama_pissa_7_trained_ratio_sd\endcsname{0.01}
\expandafter\def\csname result@rev_llama_pissa_15_reconstruct_accuracy_mean\endcsname{70.83}
\expandafter\def\csname result@rev_llama_pissa_15_reconstruct_accuracy_sd\endcsname{0.45}
\expandafter\def\csname result@rev_llama_pissa_15_restore_accuracy_mean\endcsname{71.30}
\expandafter\def\csname result@rev_llama_pissa_15_restore_accuracy_sd\endcsname{0.36}
\expandafter\def\csname result@rev_llama_pissa_15_restore_nll_vs_reconstruct_mean\endcsname{-0.00390}
\expandafter\def\csname result@rev_llama_pissa_15_restore_nll_vs_reconstruct_sd\endcsname{0.00071}
\expandafter\def\csname result@rev_llama_pissa_15_rotation_band_restore_bottom_accuracy_change_mean\endcsname{-0.07}
\expandafter\def\csname result@rev_llama_pissa_15_rotation_band_restore_bottom_accuracy_change_sd\endcsname{0.21}
\expandafter\def\csname result@rev_llama_pissa_15_rotation_band_restore_mid_accuracy_change_mean\endcsname{-0.37}
\expandafter\def\csname result@rev_llama_pissa_15_rotation_band_restore_mid_accuracy_change_sd\endcsname{0.50}
\expandafter\def\csname result@rev_llama_pissa_15_rotation_band_restore_top_accuracy_change_mean\endcsname{-71.00}
\expandafter\def\csname result@rev_llama_pissa_15_rotation_band_restore_top_accuracy_change_sd\endcsname{0.26}
\expandafter\def\csname result@rev_llama_pissa_15_rotation_band_restore_top_nll_change_mean\endcsname{-0.15029}
\expandafter\def\csname result@rev_llama_pissa_15_rotation_band_restore_top_nll_change_sd\endcsname{0.00406}
\expandafter\def\csname result@rev_llama_pissa_15_trained_accuracy_mean\endcsname{71.00}
\expandafter\def\csname result@rev_llama_pissa_15_trained_accuracy_sd\endcsname{0.26}
\expandafter\def\csname result@rev_llama_pissa_15_trained_d_W_mean\endcsname{2.309}
\expandafter\def\csname result@rev_llama_pissa_15_trained_d_W_sd\endcsname{0.005}
\expandafter\def\csname result@rev_llama_pissa_15_trained_d_S_mean\endcsname{0.19040}
\expandafter\def\csname result@rev_llama_pissa_15_trained_d_S_sd\endcsname{0.00148}
\expandafter\def\csname result@rev_llama_pissa_15_trained_ratio_mean\endcsname{8.25}
\expandafter\def\csname result@rev_llama_pissa_15_trained_ratio_sd\endcsname{0.05}
\expandafter\def\csname result@rev_llama_pissa_30_reconstruct_accuracy_mean\endcsname{71.67}
\expandafter\def\csname result@rev_llama_pissa_30_reconstruct_accuracy_sd\endcsname{0.85}
\expandafter\def\csname result@rev_llama_pissa_30_restore_accuracy_mean\endcsname{71.27}
\expandafter\def\csname result@rev_llama_pissa_30_restore_accuracy_sd\endcsname{1.05}
\expandafter\def\csname result@rev_llama_pissa_30_restore_nll_vs_reconstruct_mean\endcsname{-0.00063}
\expandafter\def\csname result@rev_llama_pissa_30_restore_nll_vs_reconstruct_sd\endcsname{0.00010}
\expandafter\def\csname result@rev_llama_pissa_30_rotation_band_restore_bottom_accuracy_change_mean\endcsname{-0.13}
\expandafter\def\csname result@rev_llama_pissa_30_rotation_band_restore_bottom_accuracy_change_sd\endcsname{0.21}
\expandafter\def\csname result@rev_llama_pissa_30_rotation_band_restore_mid_accuracy_change_mean\endcsname{0.10}
\expandafter\def\csname result@rev_llama_pissa_30_rotation_band_restore_mid_accuracy_change_sd\endcsname{0.53}
\expandafter\def\csname result@rev_llama_pissa_30_rotation_band_restore_top_accuracy_change_mean\endcsname{-71.23}
\expandafter\def\csname result@rev_llama_pissa_30_rotation_band_restore_top_accuracy_change_sd\endcsname{0.06}
\expandafter\def\csname result@rev_llama_pissa_30_rotation_band_restore_top_nll_change_mean\endcsname{-0.06317}
\expandafter\def\csname result@rev_llama_pissa_30_rotation_band_restore_top_nll_change_sd\endcsname{0.00101}
\expandafter\def\csname result@rev_llama_pissa_30_trained_accuracy_mean\endcsname{71.23}
\expandafter\def\csname result@rev_llama_pissa_30_trained_accuracy_sd\endcsname{0.06}
\expandafter\def\csname result@rev_llama_pissa_30_trained_d_W_mean\endcsname{1.060}
\expandafter\def\csname result@rev_llama_pissa_30_trained_d_W_sd\endcsname{0.001}
\expandafter\def\csname result@rev_llama_pissa_30_trained_d_S_mean\endcsname{0.04113}
\expandafter\def\csname result@rev_llama_pissa_30_trained_d_S_sd\endcsname{0.00028}
\expandafter\def\csname result@rev_llama_pissa_30_trained_ratio_mean\endcsname{3.88}
\expandafter\def\csname result@rev_llama_pissa_30_trained_ratio_sd\endcsname{0.02}
\expandafter\def\csname result@rev_qwen_dora_7_reconstruct_accuracy_mean\endcsname{92.23}
\expandafter\def\csname result@rev_qwen_dora_7_reconstruct_accuracy_sd\endcsname{0.38}
\expandafter\def\csname result@rev_qwen_dora_7_restore_accuracy_mean\endcsname{92.30}
\expandafter\def\csname result@rev_qwen_dora_7_restore_accuracy_sd\endcsname{0.10}
\expandafter\def\csname result@rev_qwen_dora_7_restore_nll_vs_reconstruct_mean\endcsname{0.00005}
\expandafter\def\csname result@rev_qwen_dora_7_restore_nll_vs_reconstruct_sd\endcsname{0.00008}
\expandafter\def\csname result@rev_qwen_dora_7_rotation_band_restore_bottom_accuracy_change_mean\endcsname{-0.33}
\expandafter\def\csname result@rev_qwen_dora_7_rotation_band_restore_bottom_accuracy_change_sd\endcsname{0.67}
\expandafter\def\csname result@rev_qwen_dora_7_rotation_band_restore_mid_accuracy_change_mean\endcsname{0.17}
\expandafter\def\csname result@rev_qwen_dora_7_rotation_band_restore_mid_accuracy_change_sd\endcsname{0.60}
\expandafter\def\csname result@rev_qwen_dora_7_rotation_band_restore_top_accuracy_change_mean\endcsname{0.07}
\expandafter\def\csname result@rev_qwen_dora_7_rotation_band_restore_top_accuracy_change_sd\endcsname{0.29}
\expandafter\def\csname result@rev_qwen_dora_7_rotation_band_restore_top_nll_change_mean\endcsname{-0.00218}
\expandafter\def\csname result@rev_qwen_dora_7_rotation_band_restore_top_nll_change_sd\endcsname{0.00102}
\expandafter\def\csname result@rev_qwen_dora_7_trained_accuracy_mean\endcsname{92.27}
\expandafter\def\csname result@rev_qwen_dora_7_trained_accuracy_sd\endcsname{0.31}
\expandafter\def\csname result@rev_qwen_dora_7_trained_d_W_mean\endcsname{1.373}
\expandafter\def\csname result@rev_qwen_dora_7_trained_d_W_sd\endcsname{0.010}
\expandafter\def\csname result@rev_qwen_dora_7_trained_d_S_mean\endcsname{0.02003}
\expandafter\def\csname result@rev_qwen_dora_7_trained_d_S_sd\endcsname{0.00212}
\expandafter\def\csname result@rev_qwen_dora_7_trained_ratio_mean\endcsname{1.46}
\expandafter\def\csname result@rev_qwen_dora_7_trained_ratio_sd\endcsname{0.15}
\expandafter\def\csname result@rev_qwen_dora_14_reconstruct_accuracy_mean\endcsname{92.53}
\expandafter\def\csname result@rev_qwen_dora_14_reconstruct_accuracy_sd\endcsname{0.47}
\expandafter\def\csname result@rev_qwen_dora_14_restore_accuracy_mean\endcsname{92.80}
\expandafter\def\csname result@rev_qwen_dora_14_restore_accuracy_sd\endcsname{0.35}
\expandafter\def\csname result@rev_qwen_dora_14_restore_nll_vs_reconstruct_mean\endcsname{-0.00007}
\expandafter\def\csname result@rev_qwen_dora_14_restore_nll_vs_reconstruct_sd\endcsname{0.00008}
\expandafter\def\csname result@rev_qwen_dora_14_rotation_band_restore_bottom_accuracy_change_mean\endcsname{-0.93}
\expandafter\def\csname result@rev_qwen_dora_14_rotation_band_restore_bottom_accuracy_change_sd\endcsname{0.12}
\expandafter\def\csname result@rev_qwen_dora_14_rotation_band_restore_mid_accuracy_change_mean\endcsname{-0.70}
\expandafter\def\csname result@rev_qwen_dora_14_rotation_band_restore_mid_accuracy_change_sd\endcsname{0.62}
\expandafter\def\csname result@rev_qwen_dora_14_rotation_band_restore_top_accuracy_change_mean\endcsname{-1.17}
\expandafter\def\csname result@rev_qwen_dora_14_rotation_band_restore_top_accuracy_change_sd\endcsname{0.51}
\expandafter\def\csname result@rev_qwen_dora_14_rotation_band_restore_top_nll_change_mean\endcsname{-0.00147}
\expandafter\def\csname result@rev_qwen_dora_14_rotation_band_restore_top_nll_change_sd\endcsname{0.00068}
\expandafter\def\csname result@rev_qwen_dora_14_trained_accuracy_mean\endcsname{92.83}
\expandafter\def\csname result@rev_qwen_dora_14_trained_accuracy_sd\endcsname{0.50}
\expandafter\def\csname result@rev_qwen_dora_14_trained_d_W_mean\endcsname{1.863}
\expandafter\def\csname result@rev_qwen_dora_14_trained_d_W_sd\endcsname{0.016}
\expandafter\def\csname result@rev_qwen_dora_14_trained_d_S_mean\endcsname{0.02374}
\expandafter\def\csname result@rev_qwen_dora_14_trained_d_S_sd\endcsname{0.00088}
\expandafter\def\csname result@rev_qwen_dora_14_trained_ratio_mean\endcsname{1.27}
\expandafter\def\csname result@rev_qwen_dora_14_trained_ratio_sd\endcsname{0.04}
\expandafter\def\csname result@rev_qwen_dora_28_reconstruct_accuracy_mean\endcsname{92.50}
\expandafter\def\csname result@rev_qwen_dora_28_reconstruct_accuracy_sd\endcsname{0.36}
\expandafter\def\csname result@rev_qwen_dora_28_restore_accuracy_mean\endcsname{92.47}
\expandafter\def\csname result@rev_qwen_dora_28_restore_accuracy_sd\endcsname{0.15}
\expandafter\def\csname result@rev_qwen_dora_28_restore_nll_vs_reconstruct_mean\endcsname{0.00004}
\expandafter\def\csname result@rev_qwen_dora_28_restore_nll_vs_reconstruct_sd\endcsname{0.00012}
\expandafter\def\csname result@rev_qwen_dora_28_rotation_band_restore_bottom_accuracy_change_mean\endcsname{0.03}
\expandafter\def\csname result@rev_qwen_dora_28_rotation_band_restore_bottom_accuracy_change_sd\endcsname{0.06}
\expandafter\def\csname result@rev_qwen_dora_28_rotation_band_restore_mid_accuracy_change_mean\endcsname{0.67}
\expandafter\def\csname result@rev_qwen_dora_28_rotation_band_restore_mid_accuracy_change_sd\endcsname{0.47}
\expandafter\def\csname result@rev_qwen_dora_28_rotation_band_restore_top_accuracy_change_mean\endcsname{-0.20}
\expandafter\def\csname result@rev_qwen_dora_28_rotation_band_restore_top_accuracy_change_sd\endcsname{0.46}
\expandafter\def\csname result@rev_qwen_dora_28_rotation_band_restore_top_nll_change_mean\endcsname{-0.00272}
\expandafter\def\csname result@rev_qwen_dora_28_rotation_band_restore_top_nll_change_sd\endcsname{0.00066}
\expandafter\def\csname result@rev_qwen_dora_28_trained_accuracy_mean\endcsname{92.33}
\expandafter\def\csname result@rev_qwen_dora_28_trained_accuracy_sd\endcsname{0.15}
\expandafter\def\csname result@rev_qwen_dora_28_trained_d_W_mean\endcsname{0.813}
\expandafter\def\csname result@rev_qwen_dora_28_trained_d_W_sd\endcsname{0.004}
\expandafter\def\csname result@rev_qwen_dora_28_trained_d_S_mean\endcsname{0.01036}
\expandafter\def\csname result@rev_qwen_dora_28_trained_d_S_sd\endcsname{0.00023}
\expandafter\def\csname result@rev_qwen_dora_28_trained_ratio_mean\endcsname{1.27}
\expandafter\def\csname result@rev_qwen_dora_28_trained_ratio_sd\endcsname{0.03}
\expandafter\def\csname result@rev_qwen_hra_16_reconstruct_accuracy_mean\endcsname{91.83}
\expandafter\def\csname result@rev_qwen_hra_16_reconstruct_accuracy_sd\endcsname{0.21}
\expandafter\def\csname result@rev_qwen_hra_16_restore_accuracy_mean\endcsname{91.97}
\expandafter\def\csname result@rev_qwen_hra_16_restore_accuracy_sd\endcsname{0.29}
\expandafter\def\csname result@rev_qwen_hra_16_restore_nll_vs_reconstruct_mean\endcsname{-0.00004}
\expandafter\def\csname result@rev_qwen_hra_16_restore_nll_vs_reconstruct_sd\endcsname{0.00003}
\expandafter\def\csname result@rev_qwen_hra_16_rotation_band_restore_bottom_accuracy_change_mean\endcsname{-0.10}
\expandafter\def\csname result@rev_qwen_hra_16_rotation_band_restore_bottom_accuracy_change_sd\endcsname{0.26}
\expandafter\def\csname result@rev_qwen_hra_16_rotation_band_restore_mid_accuracy_change_mean\endcsname{-0.03}
\expandafter\def\csname result@rev_qwen_hra_16_rotation_band_restore_mid_accuracy_change_sd\endcsname{0.06}
\expandafter\def\csname result@rev_qwen_hra_16_rotation_band_restore_top_accuracy_change_mean\endcsname{-72.03}
\expandafter\def\csname result@rev_qwen_hra_16_rotation_band_restore_top_accuracy_change_sd\endcsname{1.29}
\expandafter\def\csname result@rev_qwen_hra_16_rotation_band_restore_top_nll_change_mean\endcsname{-0.00436}
\expandafter\def\csname result@rev_qwen_hra_16_rotation_band_restore_top_nll_change_sd\endcsname{0.00052}
\expandafter\def\csname result@rev_qwen_hra_16_trained_accuracy_mean\endcsname{91.93}
\expandafter\def\csname result@rev_qwen_hra_16_trained_accuracy_sd\endcsname{0.06}
\expandafter\def\csname result@rev_qwen_hra_16_trained_d_W_mean\endcsname{0.564}
\expandafter\def\csname result@rev_qwen_hra_16_trained_d_W_sd\endcsname{0.004}
\expandafter\def\csname result@rev_qwen_hra_16_trained_d_S_mean\endcsname{0.00153}
\expandafter\def\csname result@rev_qwen_hra_16_trained_d_S_sd\endcsname{0.00000}
\expandafter\def\csname result@rev_qwen_hra_16_trained_ratio_mean\endcsname{0.27}
\expandafter\def\csname result@rev_qwen_hra_16_trained_ratio_sd\endcsname{0.00}
\expandafter\def\csname result@rev_qwen_hra_32_reconstruct_accuracy_mean\endcsname{92.00}
\expandafter\def\csname result@rev_qwen_hra_32_reconstruct_accuracy_sd\endcsname{0.10}
\expandafter\def\csname result@rev_qwen_hra_32_restore_accuracy_mean\endcsname{91.90}
\expandafter\def\csname result@rev_qwen_hra_32_restore_accuracy_sd\endcsname{0.20}
\expandafter\def\csname result@rev_qwen_hra_32_restore_nll_vs_reconstruct_mean\endcsname{0.00002}
\expandafter\def\csname result@rev_qwen_hra_32_restore_nll_vs_reconstruct_sd\endcsname{0.00001}
\expandafter\def\csname result@rev_qwen_hra_32_rotation_band_restore_bottom_accuracy_change_mean\endcsname{-0.10}
\expandafter\def\csname result@rev_qwen_hra_32_rotation_band_restore_bottom_accuracy_change_sd\endcsname{0.46}
\expandafter\def\csname result@rev_qwen_hra_32_rotation_band_restore_mid_accuracy_change_mean\endcsname{-0.43}
\expandafter\def\csname result@rev_qwen_hra_32_rotation_band_restore_mid_accuracy_change_sd\endcsname{0.12}
\expandafter\def\csname result@rev_qwen_hra_32_rotation_band_restore_top_accuracy_change_mean\endcsname{-50.87}
\expandafter\def\csname result@rev_qwen_hra_32_rotation_band_restore_top_accuracy_change_sd\endcsname{3.04}
\expandafter\def\csname result@rev_qwen_hra_32_rotation_band_restore_top_nll_change_mean\endcsname{-0.00576}
\expandafter\def\csname result@rev_qwen_hra_32_rotation_band_restore_top_nll_change_sd\endcsname{0.00019}
\expandafter\def\csname result@rev_qwen_hra_32_trained_accuracy_mean\endcsname{92.10}
\expandafter\def\csname result@rev_qwen_hra_32_trained_accuracy_sd\endcsname{0.17}
\expandafter\def\csname result@rev_qwen_hra_32_trained_d_W_mean\endcsname{0.636}
\expandafter\def\csname result@rev_qwen_hra_32_trained_d_W_sd\endcsname{0.003}
\expandafter\def\csname result@rev_qwen_hra_32_trained_d_S_mean\endcsname{0.00155}
\expandafter\def\csname result@rev_qwen_hra_32_trained_d_S_sd\endcsname{0.00001}
\expandafter\def\csname result@rev_qwen_hra_32_trained_ratio_mean\endcsname{0.24}
\expandafter\def\csname result@rev_qwen_hra_32_trained_ratio_sd\endcsname{0.00}
\expandafter\def\csname result@rev_qwen_hra_64_reconstruct_accuracy_mean\endcsname{91.80}
\expandafter\def\csname result@rev_qwen_hra_64_reconstruct_accuracy_sd\endcsname{0.17}
\expandafter\def\csname result@rev_qwen_hra_64_restore_accuracy_mean\endcsname{92.03}
\expandafter\def\csname result@rev_qwen_hra_64_restore_accuracy_sd\endcsname{0.15}
\expandafter\def\csname result@rev_qwen_hra_64_restore_nll_vs_reconstruct_mean\endcsname{0.00006}
\expandafter\def\csname result@rev_qwen_hra_64_restore_nll_vs_reconstruct_sd\endcsname{0.00008}
\expandafter\def\csname result@rev_qwen_hra_64_rotation_band_restore_bottom_accuracy_change_mean\endcsname{-0.03}
\expandafter\def\csname result@rev_qwen_hra_64_rotation_band_restore_bottom_accuracy_change_sd\endcsname{0.12}
\expandafter\def\csname result@rev_qwen_hra_64_rotation_band_restore_mid_accuracy_change_mean\endcsname{-2.87}
\expandafter\def\csname result@rev_qwen_hra_64_rotation_band_restore_mid_accuracy_change_sd\endcsname{0.76}
\expandafter\def\csname result@rev_qwen_hra_64_rotation_band_restore_top_accuracy_change_mean\endcsname{-77.60}
\expandafter\def\csname result@rev_qwen_hra_64_rotation_band_restore_top_accuracy_change_sd\endcsname{0.56}
\expandafter\def\csname result@rev_qwen_hra_64_rotation_band_restore_top_nll_change_mean\endcsname{-0.00312}
\expandafter\def\csname result@rev_qwen_hra_64_rotation_band_restore_top_nll_change_sd\endcsname{0.00010}
\expandafter\def\csname result@rev_qwen_hra_64_trained_accuracy_mean\endcsname{91.70}
\expandafter\def\csname result@rev_qwen_hra_64_trained_accuracy_sd\endcsname{0.20}
\expandafter\def\csname result@rev_qwen_hra_64_trained_d_W_mean\endcsname{0.286}
\expandafter\def\csname result@rev_qwen_hra_64_trained_d_W_sd\endcsname{0.003}
\expandafter\def\csname result@rev_qwen_hra_64_trained_d_S_mean\endcsname{0.00131}
\expandafter\def\csname result@rev_qwen_hra_64_trained_d_S_sd\endcsname{0.00001}
\expandafter\def\csname result@rev_qwen_hra_64_trained_ratio_mean\endcsname{0.46}
\expandafter\def\csname result@rev_qwen_hra_64_trained_ratio_sd\endcsname{0.01}
\expandafter\def\csname result@rev_qwen_lora_7_reconstruct_accuracy_mean\endcsname{90.60}
\expandafter\def\csname result@rev_qwen_lora_7_reconstruct_accuracy_sd\endcsname{0.85}
\expandafter\def\csname result@rev_qwen_lora_7_restore_accuracy_mean\endcsname{90.73}
\expandafter\def\csname result@rev_qwen_lora_7_restore_accuracy_sd\endcsname{0.55}
\expandafter\def\csname result@rev_qwen_lora_7_restore_accuracy_vs_reconstruct_mean\endcsname{0.13}
\expandafter\def\csname result@rev_qwen_lora_7_restore_nll_vs_reconstruct_mean\endcsname{0.00756}
\expandafter\def\csname result@rev_qwen_lora_7_restore_nll_vs_reconstruct_sd\endcsname{0.00546}
\expandafter\def\csname result@rev_qwen_lora_7_rotation_band_restore_bottom_accuracy_change_mean\endcsname{0.07}
\expandafter\def\csname result@rev_qwen_lora_7_rotation_band_restore_bottom_accuracy_change_sd\endcsname{0.93}
\expandafter\def\csname result@rev_qwen_lora_7_rotation_band_restore_mid_accuracy_change_mean\endcsname{-0.70}
\expandafter\def\csname result@rev_qwen_lora_7_rotation_band_restore_mid_accuracy_change_sd\endcsname{0.44}
\expandafter\def\csname result@rev_qwen_lora_7_rotation_band_restore_top_accuracy_change_mean\endcsname{-82.00}
\expandafter\def\csname result@rev_qwen_lora_7_rotation_band_restore_top_accuracy_change_sd\endcsname{2.40}
\expandafter\def\csname result@rev_qwen_lora_7_rotation_band_restore_top_nll_change_mean\endcsname{-0.01099}
\expandafter\def\csname result@rev_qwen_lora_7_rotation_band_restore_top_nll_change_sd\endcsname{0.00343}
\expandafter\def\csname result@rev_qwen_lora_7_trained_accuracy_mean\endcsname{90.87}
\expandafter\def\csname result@rev_qwen_lora_7_trained_accuracy_sd\endcsname{0.90}
\expandafter\def\csname result@rev_qwen_lora_7_trained_d_W_mean\endcsname{8.154}
\expandafter\def\csname result@rev_qwen_lora_7_trained_d_W_sd\endcsname{0.010}
\expandafter\def\csname result@rev_qwen_lora_7_trained_d_S_mean\endcsname{1.67437}
\expandafter\def\csname result@rev_qwen_lora_7_trained_d_S_sd\endcsname{0.01900}
\expandafter\def\csname result@rev_qwen_lora_7_trained_ratio_mean\endcsname{20.53}
\expandafter\def\csname result@rev_qwen_lora_7_trained_ratio_sd\endcsname{0.25}
\expandafter\def\csname result@rev_qwen_lora_14_reconstruct_accuracy_mean\endcsname{91.53}
\expandafter\def\csname result@rev_qwen_lora_14_reconstruct_accuracy_sd\endcsname{0.29}
\expandafter\def\csname result@rev_qwen_lora_14_restore_accuracy_mean\endcsname{91.73}
\expandafter\def\csname result@rev_qwen_lora_14_restore_accuracy_sd\endcsname{0.15}
\expandafter\def\csname result@rev_qwen_lora_14_restore_nll_vs_reconstruct_mean\endcsname{-0.00002}
\expandafter\def\csname result@rev_qwen_lora_14_restore_nll_vs_reconstruct_sd\endcsname{0.00007}
\expandafter\def\csname result@rev_qwen_lora_14_rotation_band_restore_bottom_accuracy_change_mean\endcsname{-0.33}
\expandafter\def\csname result@rev_qwen_lora_14_rotation_band_restore_bottom_accuracy_change_sd\endcsname{0.47}
\expandafter\def\csname result@rev_qwen_lora_14_rotation_band_restore_mid_accuracy_change_mean\endcsname{-1.17}
\expandafter\def\csname result@rev_qwen_lora_14_rotation_band_restore_mid_accuracy_change_sd\endcsname{0.12}
\expandafter\def\csname result@rev_qwen_lora_14_rotation_band_restore_top_accuracy_change_mean\endcsname{-49.67}
\expandafter\def\csname result@rev_qwen_lora_14_rotation_band_restore_top_accuracy_change_sd\endcsname{1.91}
\expandafter\def\csname result@rev_qwen_lora_14_rotation_band_restore_top_nll_change_mean\endcsname{-0.00139}
\expandafter\def\csname result@rev_qwen_lora_14_rotation_band_restore_top_nll_change_sd\endcsname{0.00026}
\expandafter\def\csname result@rev_qwen_lora_14_trained_accuracy_mean\endcsname{91.73}
\expandafter\def\csname result@rev_qwen_lora_14_trained_accuracy_sd\endcsname{0.29}
\expandafter\def\csname result@rev_qwen_lora_14_trained_d_W_mean\endcsname{0.151}
\expandafter\def\csname result@rev_qwen_lora_14_trained_d_W_sd\endcsname{0.000}
\expandafter\def\csname result@rev_qwen_lora_14_trained_d_S_mean\endcsname{0.00161}
\expandafter\def\csname result@rev_qwen_lora_14_trained_d_S_sd\endcsname{0.00003}
\expandafter\def\csname result@rev_qwen_lora_14_trained_ratio_mean\endcsname{1.07}
\expandafter\def\csname result@rev_qwen_lora_14_trained_ratio_sd\endcsname{0.02}
\expandafter\def\csname result@rev_qwen_lora_28_reconstruct_accuracy_mean\endcsname{92.23}
\expandafter\def\csname result@rev_qwen_lora_28_reconstruct_accuracy_sd\endcsname{0.47}
\expandafter\def\csname result@rev_qwen_lora_28_restore_accuracy_mean\endcsname{92.07}
\expandafter\def\csname result@rev_qwen_lora_28_restore_accuracy_sd\endcsname{0.49}
\expandafter\def\csname result@rev_qwen_lora_28_restore_nll_vs_reconstruct_mean\endcsname{0.00005}
\expandafter\def\csname result@rev_qwen_lora_28_restore_nll_vs_reconstruct_sd\endcsname{0.00005}
\expandafter\def\csname result@rev_qwen_lora_28_rotation_band_restore_bottom_accuracy_change_mean\endcsname{-0.10}
\expandafter\def\csname result@rev_qwen_lora_28_rotation_band_restore_bottom_accuracy_change_sd\endcsname{0.20}
\expandafter\def\csname result@rev_qwen_lora_28_rotation_band_restore_mid_accuracy_change_mean\endcsname{0.20}
\expandafter\def\csname result@rev_qwen_lora_28_rotation_band_restore_mid_accuracy_change_sd\endcsname{0.17}
\expandafter\def\csname result@rev_qwen_lora_28_rotation_band_restore_top_accuracy_change_mean\endcsname{-1.03}
\expandafter\def\csname result@rev_qwen_lora_28_rotation_band_restore_top_accuracy_change_sd\endcsname{0.06}
\expandafter\def\csname result@rev_qwen_lora_28_rotation_band_restore_top_nll_change_mean\endcsname{-0.00259}
\expandafter\def\csname result@rev_qwen_lora_28_rotation_band_restore_top_nll_change_sd\endcsname{0.00023}
\expandafter\def\csname result@rev_qwen_lora_28_trained_accuracy_mean\endcsname{92.20}
\expandafter\def\csname result@rev_qwen_lora_28_trained_accuracy_sd\endcsname{0.53}
\expandafter\def\csname result@rev_qwen_lora_28_trained_d_W_mean\endcsname{0.307}
\expandafter\def\csname result@rev_qwen_lora_28_trained_d_W_sd\endcsname{0.001}
\expandafter\def\csname result@rev_qwen_lora_28_trained_d_S_mean\endcsname{0.00321}
\expandafter\def\csname result@rev_qwen_lora_28_trained_d_S_sd\endcsname{0.00014}
\expandafter\def\csname result@rev_qwen_lora_28_trained_ratio_mean\endcsname{1.05}
\expandafter\def\csname result@rev_qwen_lora_28_trained_ratio_sd\endcsname{0.04}
\expandafter\def\csname result@rev_qwen_milora_7_reconstruct_accuracy_mean\endcsname{92.57}
\expandafter\def\csname result@rev_qwen_milora_7_reconstruct_accuracy_sd\endcsname{0.47}
\expandafter\def\csname result@rev_qwen_milora_7_rotation_band_restore_bottom_accuracy_change_mean\endcsname{-60.07}
\expandafter\def\csname result@rev_qwen_milora_7_rotation_band_restore_bottom_accuracy_change_sd\endcsname{1.01}
\expandafter\def\csname result@rev_qwen_milora_7_rotation_band_restore_mid_accuracy_change_mean\endcsname{-0.93}
\expandafter\def\csname result@rev_qwen_milora_7_rotation_band_restore_mid_accuracy_change_sd\endcsname{0.47}
\expandafter\def\csname result@rev_qwen_milora_7_rotation_band_restore_top_accuracy_change_mean\endcsname{-1.77}
\expandafter\def\csname result@rev_qwen_milora_7_rotation_band_restore_top_accuracy_change_sd\endcsname{0.21}
\expandafter\def\csname result@rev_qwen_milora_7_rotation_band_restore_top_nll_change_mean\endcsname{-0.00245}
\expandafter\def\csname result@rev_qwen_milora_7_rotation_band_restore_top_nll_change_sd\endcsname{0.00008}
\expandafter\def\csname result@rev_qwen_milora_7_trained_accuracy_mean\endcsname{92.53}
\expandafter\def\csname result@rev_qwen_milora_7_trained_accuracy_sd\endcsname{0.15}
\expandafter\def\csname result@rev_qwen_milora_7_trained_d_W_mean\endcsname{0.139}
\expandafter\def\csname result@rev_qwen_milora_7_trained_d_W_sd\endcsname{0.000}
\expandafter\def\csname result@rev_qwen_milora_7_trained_d_S_mean\endcsname{0.00676}
\expandafter\def\csname result@rev_qwen_milora_7_trained_d_S_sd\endcsname{0.00002}
\expandafter\def\csname result@rev_qwen_milora_7_trained_ratio_mean\endcsname{4.87}
\expandafter\def\csname result@rev_qwen_milora_7_trained_ratio_sd\endcsname{0.01}
\expandafter\def\csname result@rev_qwen_milora_14_reconstruct_accuracy_mean\endcsname{92.70}
\expandafter\def\csname result@rev_qwen_milora_14_reconstruct_accuracy_sd\endcsname{0.35}
\expandafter\def\csname result@rev_qwen_milora_14_rotation_band_restore_bottom_accuracy_change_mean\endcsname{-54.83}
\expandafter\def\csname result@rev_qwen_milora_14_rotation_band_restore_bottom_accuracy_change_sd\endcsname{1.30}
\expandafter\def\csname result@rev_qwen_milora_14_rotation_band_restore_mid_accuracy_change_mean\endcsname{-0.53}
\expandafter\def\csname result@rev_qwen_milora_14_rotation_band_restore_mid_accuracy_change_sd\endcsname{0.35}
\expandafter\def\csname result@rev_qwen_milora_14_rotation_band_restore_top_accuracy_change_mean\endcsname{-0.63}
\expandafter\def\csname result@rev_qwen_milora_14_rotation_band_restore_top_accuracy_change_sd\endcsname{0.40}
\expandafter\def\csname result@rev_qwen_milora_14_rotation_band_restore_top_nll_change_mean\endcsname{-0.00290}
\expandafter\def\csname result@rev_qwen_milora_14_rotation_band_restore_top_nll_change_sd\endcsname{0.00014}
\expandafter\def\csname result@rev_qwen_milora_14_trained_accuracy_mean\endcsname{92.60}
\expandafter\def\csname result@rev_qwen_milora_14_trained_accuracy_sd\endcsname{0.30}
\expandafter\def\csname result@rev_qwen_milora_14_trained_d_W_mean\endcsname{0.144}
\expandafter\def\csname result@rev_qwen_milora_14_trained_d_W_sd\endcsname{0.000}
\expandafter\def\csname result@rev_qwen_milora_14_trained_d_S_mean\endcsname{0.00461}
\expandafter\def\csname result@rev_qwen_milora_14_trained_d_S_sd\endcsname{0.00001}
\expandafter\def\csname result@rev_qwen_milora_14_trained_ratio_mean\endcsname{3.20}
\expandafter\def\csname result@rev_qwen_milora_14_trained_ratio_sd\endcsname{0.01}
\expandafter\def\csname result@rev_qwen_milora_28_reconstruct_accuracy_mean\endcsname{92.50}
\expandafter\def\csname result@rev_qwen_milora_28_reconstruct_accuracy_sd\endcsname{0.26}
\expandafter\def\csname result@rev_qwen_milora_28_restore_accuracy_mean\endcsname{92.30}
\expandafter\def\csname result@rev_qwen_milora_28_restore_accuracy_sd\endcsname{0.10}
\expandafter\def\csname result@rev_qwen_milora_28_restore_nll_vs_reconstruct_mean\endcsname{-0.00010}
\expandafter\def\csname result@rev_qwen_milora_28_restore_nll_vs_reconstruct_sd\endcsname{0.00015}
\expandafter\def\csname result@rev_qwen_milora_28_rotation_band_restore_bottom_accuracy_change_mean\endcsname{-1.17}
\expandafter\def\csname result@rev_qwen_milora_28_rotation_band_restore_bottom_accuracy_change_sd\endcsname{0.40}
\expandafter\def\csname result@rev_qwen_milora_28_rotation_band_restore_mid_accuracy_change_mean\endcsname{0.20}
\expandafter\def\csname result@rev_qwen_milora_28_rotation_band_restore_mid_accuracy_change_sd\endcsname{0.10}
\expandafter\def\csname result@rev_qwen_milora_28_rotation_band_restore_top_accuracy_change_mean\endcsname{-0.03}
\expandafter\def\csname result@rev_qwen_milora_28_rotation_band_restore_top_accuracy_change_sd\endcsname{0.45}
\expandafter\def\csname result@rev_qwen_milora_28_rotation_band_restore_top_nll_change_mean\endcsname{-0.00405}
\expandafter\def\csname result@rev_qwen_milora_28_rotation_band_restore_top_nll_change_sd\endcsname{0.00029}
\expandafter\def\csname result@rev_qwen_milora_28_trained_accuracy_mean\endcsname{92.20}
\expandafter\def\csname result@rev_qwen_milora_28_trained_accuracy_sd\endcsname{0.17}
\expandafter\def\csname result@rev_qwen_milora_28_trained_d_W_mean\endcsname{0.346}
\expandafter\def\csname result@rev_qwen_milora_28_trained_d_W_sd\endcsname{0.001}
\expandafter\def\csname result@rev_qwen_milora_28_trained_d_S_mean\endcsname{0.01094}
\expandafter\def\csname result@rev_qwen_milora_28_trained_d_S_sd\endcsname{0.00006}
\expandafter\def\csname result@rev_qwen_milora_28_trained_ratio_mean\endcsname{3.16}
\expandafter\def\csname result@rev_qwen_milora_28_trained_ratio_sd\endcsname{0.01}
\expandafter\def\csname result@rev_qwen_oft_32_reconstruct_accuracy_mean\endcsname{91.67}
\expandafter\def\csname result@rev_qwen_oft_32_reconstruct_accuracy_sd\endcsname{0.75}
\expandafter\def\csname result@rev_qwen_oft_32_restore_accuracy_mean\endcsname{91.73}
\expandafter\def\csname result@rev_qwen_oft_32_restore_accuracy_sd\endcsname{0.68}
\expandafter\def\csname result@rev_qwen_oft_32_restore_nll_vs_reconstruct_mean\endcsname{0.00001}
\expandafter\def\csname result@rev_qwen_oft_32_restore_nll_vs_reconstruct_sd\endcsname{0.00009}
\expandafter\def\csname result@rev_qwen_oft_32_rotation_band_restore_bottom_accuracy_change_mean\endcsname{-0.17}
\expandafter\def\csname result@rev_qwen_oft_32_rotation_band_restore_bottom_accuracy_change_sd\endcsname{0.06}
\expandafter\def\csname result@rev_qwen_oft_32_rotation_band_restore_mid_accuracy_change_mean\endcsname{0.23}
\expandafter\def\csname result@rev_qwen_oft_32_rotation_band_restore_mid_accuracy_change_sd\endcsname{0.84}
\expandafter\def\csname result@rev_qwen_oft_32_rotation_band_restore_top_accuracy_change_mean\endcsname{0.13}
\expandafter\def\csname result@rev_qwen_oft_32_rotation_band_restore_top_accuracy_change_sd\endcsname{0.75}
\expandafter\def\csname result@rev_qwen_oft_32_rotation_band_restore_top_nll_change_mean\endcsname{-0.00772}
\expandafter\def\csname result@rev_qwen_oft_32_rotation_band_restore_top_nll_change_sd\endcsname{0.00056}
\expandafter\def\csname result@rev_qwen_oft_32_trained_accuracy_mean\endcsname{91.77}
\expandafter\def\csname result@rev_qwen_oft_32_trained_accuracy_sd\endcsname{0.71}
\expandafter\def\csname result@rev_qwen_oft_32_trained_d_W_mean\endcsname{4.213}
\expandafter\def\csname result@rev_qwen_oft_32_trained_d_W_sd\endcsname{0.006}
\expandafter\def\csname result@rev_qwen_oft_32_trained_d_S_mean\endcsname{0.00250}
\expandafter\def\csname result@rev_qwen_oft_32_trained_d_S_sd\endcsname{0.00062}
\expandafter\def\csname result@rev_qwen_oft_32_trained_ratio_mean\endcsname{0.06}
\expandafter\def\csname result@rev_qwen_oft_32_trained_ratio_sd\endcsname{0.01}
\expandafter\def\csname result@rev_qwen_oft_64_reconstruct_accuracy_mean\endcsname{90.17}
\expandafter\def\csname result@rev_qwen_oft_64_reconstruct_accuracy_sd\endcsname{0.21}
\expandafter\def\csname result@rev_qwen_oft_64_restore_accuracy_mean\endcsname{89.63}
\expandafter\def\csname result@rev_qwen_oft_64_restore_accuracy_sd\endcsname{0.42}
\expandafter\def\csname result@rev_qwen_oft_64_restore_nll_vs_reconstruct_mean\endcsname{-0.00014}
\expandafter\def\csname result@rev_qwen_oft_64_restore_nll_vs_reconstruct_sd\endcsname{0.00023}
\expandafter\def\csname result@rev_qwen_oft_64_rotation_band_restore_bottom_accuracy_change_mean\endcsname{0.30}
\expandafter\def\csname result@rev_qwen_oft_64_rotation_band_restore_bottom_accuracy_change_sd\endcsname{0.36}
\expandafter\def\csname result@rev_qwen_oft_64_rotation_band_restore_mid_accuracy_change_mean\endcsname{1.37}
\expandafter\def\csname result@rev_qwen_oft_64_rotation_band_restore_mid_accuracy_change_sd\endcsname{0.55}
\expandafter\def\csname result@rev_qwen_oft_64_rotation_band_restore_top_accuracy_change_mean\endcsname{2.27}
\expandafter\def\csname result@rev_qwen_oft_64_rotation_band_restore_top_accuracy_change_sd\endcsname{0.81}
\expandafter\def\csname result@rev_qwen_oft_64_rotation_band_restore_top_nll_change_mean\endcsname{-0.01298}
\expandafter\def\csname result@rev_qwen_oft_64_rotation_band_restore_top_nll_change_sd\endcsname{0.00104}
\expandafter\def\csname result@rev_qwen_oft_64_trained_accuracy_mean\endcsname{89.87}
\expandafter\def\csname result@rev_qwen_oft_64_trained_accuracy_sd\endcsname{0.42}
\expandafter\def\csname result@rev_qwen_oft_64_trained_d_W_mean\endcsname{6.368}
\expandafter\def\csname result@rev_qwen_oft_64_trained_d_W_sd\endcsname{0.022}
\expandafter\def\csname result@rev_qwen_oft_64_trained_d_S_mean\endcsname{0.00340}
\expandafter\def\csname result@rev_qwen_oft_64_trained_d_S_sd\endcsname{0.00085}
\expandafter\def\csname result@rev_qwen_oft_64_trained_ratio_mean\endcsname{0.05}
\expandafter\def\csname result@rev_qwen_oft_64_trained_ratio_sd\endcsname{0.01}
\expandafter\def\csname result@rev_qwen_oft_128_reconstruct_accuracy_mean\endcsname{91.07}
\expandafter\def\csname result@rev_qwen_oft_128_reconstruct_accuracy_sd\endcsname{0.61}
\expandafter\def\csname result@rev_qwen_oft_128_restore_accuracy_mean\endcsname{90.80}
\expandafter\def\csname result@rev_qwen_oft_128_restore_accuracy_sd\endcsname{0.53}
\expandafter\def\csname result@rev_qwen_oft_128_restore_nll_vs_reconstruct_mean\endcsname{0.00001}
\expandafter\def\csname result@rev_qwen_oft_128_restore_nll_vs_reconstruct_sd\endcsname{0.00030}
\expandafter\def\csname result@rev_qwen_oft_128_rotation_band_restore_bottom_accuracy_change_mean\endcsname{0.60}
\expandafter\def\csname result@rev_qwen_oft_128_rotation_band_restore_bottom_accuracy_change_sd\endcsname{0.70}
\expandafter\def\csname result@rev_qwen_oft_128_rotation_band_restore_mid_accuracy_change_mean\endcsname{1.07}
\expandafter\def\csname result@rev_qwen_oft_128_rotation_band_restore_mid_accuracy_change_sd\endcsname{0.91}
\expandafter\def\csname result@rev_qwen_oft_128_rotation_band_restore_top_accuracy_change_mean\endcsname{1.03}
\expandafter\def\csname result@rev_qwen_oft_128_rotation_band_restore_top_accuracy_change_sd\endcsname{0.45}
\expandafter\def\csname result@rev_qwen_oft_128_rotation_band_restore_top_nll_change_mean\endcsname{-0.00669}
\expandafter\def\csname result@rev_qwen_oft_128_rotation_band_restore_top_nll_change_sd\endcsname{0.00015}
\expandafter\def\csname result@rev_qwen_oft_128_trained_accuracy_mean\endcsname{90.70}
\expandafter\def\csname result@rev_qwen_oft_128_trained_accuracy_sd\endcsname{0.36}
\expandafter\def\csname result@rev_qwen_oft_128_trained_d_W_mean\endcsname{3.090}
\expandafter\def\csname result@rev_qwen_oft_128_trained_d_W_sd\endcsname{0.009}
\expandafter\def\csname result@rev_qwen_oft_128_trained_d_S_mean\endcsname{0.00183}
\expandafter\def\csname result@rev_qwen_oft_128_trained_d_S_sd\endcsname{0.00011}
\expandafter\def\csname result@rev_qwen_oft_128_trained_ratio_mean\endcsname{0.06}
\expandafter\def\csname result@rev_qwen_oft_128_trained_ratio_sd\endcsname{0.00}
\expandafter\def\csname result@rev_qwen_pissa_7_reconstruct_accuracy_mean\endcsname{92.30}
\expandafter\def\csname result@rev_qwen_pissa_7_reconstruct_accuracy_sd\endcsname{0.17}
\expandafter\def\csname result@rev_qwen_pissa_7_restore_accuracy_mean\endcsname{92.33}
\expandafter\def\csname result@rev_qwen_pissa_7_restore_accuracy_sd\endcsname{0.25}
\expandafter\def\csname result@rev_qwen_pissa_7_restore_nll_vs_reconstruct_mean\endcsname{-0.00034}
\expandafter\def\csname result@rev_qwen_pissa_7_restore_nll_vs_reconstruct_sd\endcsname{0.00008}
\expandafter\def\csname result@rev_qwen_pissa_7_rotation_band_restore_bottom_accuracy_change_mean\endcsname{-0.17}
\expandafter\def\csname result@rev_qwen_pissa_7_rotation_band_restore_bottom_accuracy_change_sd\endcsname{0.15}
\expandafter\def\csname result@rev_qwen_pissa_7_rotation_band_restore_mid_accuracy_change_mean\endcsname{-0.07}
\expandafter\def\csname result@rev_qwen_pissa_7_rotation_band_restore_mid_accuracy_change_sd\endcsname{0.31}
\expandafter\def\csname result@rev_qwen_pissa_7_rotation_band_restore_top_accuracy_change_mean\endcsname{-86.50}
\expandafter\def\csname result@rev_qwen_pissa_7_rotation_band_restore_top_accuracy_change_sd\endcsname{0.52}
\expandafter\def\csname result@rev_qwen_pissa_7_rotation_band_restore_top_nll_change_mean\endcsname{-0.01057}
\expandafter\def\csname result@rev_qwen_pissa_7_rotation_band_restore_top_nll_change_sd\endcsname{0.00106}
\expandafter\def\csname result@rev_qwen_pissa_7_trained_accuracy_mean\endcsname{92.40}
\expandafter\def\csname result@rev_qwen_pissa_7_trained_accuracy_sd\endcsname{0.10}
\expandafter\def\csname result@rev_qwen_pissa_7_trained_d_W_mean\endcsname{0.522}
\expandafter\def\csname result@rev_qwen_pissa_7_trained_d_W_sd\endcsname{0.000}
\expandafter\def\csname result@rev_qwen_pissa_7_trained_d_S_mean\endcsname{0.02039}
\expandafter\def\csname result@rev_qwen_pissa_7_trained_d_S_sd\endcsname{0.00132}
\expandafter\def\csname result@rev_qwen_pissa_7_trained_ratio_mean\endcsname{3.91}
\expandafter\def\csname result@rev_qwen_pissa_7_trained_ratio_sd\endcsname{0.25}
\expandafter\def\csname result@rev_qwen_pissa_14_reconstruct_accuracy_mean\endcsname{92.07}
\expandafter\def\csname result@rev_qwen_pissa_14_reconstruct_accuracy_sd\endcsname{0.59}
\expandafter\def\csname result@rev_qwen_pissa_14_restore_accuracy_mean\endcsname{92.07}
\expandafter\def\csname result@rev_qwen_pissa_14_restore_accuracy_sd\endcsname{0.40}
\expandafter\def\csname result@rev_qwen_pissa_14_restore_nll_vs_reconstruct_mean\endcsname{0.00014}
\expandafter\def\csname result@rev_qwen_pissa_14_restore_nll_vs_reconstruct_sd\endcsname{0.00041}
\expandafter\def\csname result@rev_qwen_pissa_14_rotation_band_restore_bottom_accuracy_change_mean\endcsname{-0.00}
\expandafter\def\csname result@rev_qwen_pissa_14_rotation_band_restore_bottom_accuracy_change_sd\endcsname{0.36}
\expandafter\def\csname result@rev_qwen_pissa_14_rotation_band_restore_mid_accuracy_change_mean\endcsname{-0.17}
\expandafter\def\csname result@rev_qwen_pissa_14_rotation_band_restore_mid_accuracy_change_sd\endcsname{0.15}
\expandafter\def\csname result@rev_qwen_pissa_14_rotation_band_restore_top_accuracy_change_mean\endcsname{-86.40}
\expandafter\def\csname result@rev_qwen_pissa_14_rotation_band_restore_top_accuracy_change_sd\endcsname{1.51}
\expandafter\def\csname result@rev_qwen_pissa_14_rotation_band_restore_top_nll_change_mean\endcsname{-0.01119}
\expandafter\def\csname result@rev_qwen_pissa_14_rotation_band_restore_top_nll_change_sd\endcsname{0.00137}
\expandafter\def\csname result@rev_qwen_pissa_14_trained_accuracy_mean\endcsname{92.10}
\expandafter\def\csname result@rev_qwen_pissa_14_trained_accuracy_sd\endcsname{0.52}
\expandafter\def\csname result@rev_qwen_pissa_14_trained_d_W_mean\endcsname{0.643}
\expandafter\def\csname result@rev_qwen_pissa_14_trained_d_W_sd\endcsname{0.001}
\expandafter\def\csname result@rev_qwen_pissa_14_trained_d_S_mean\endcsname{0.02227}
\expandafter\def\csname result@rev_qwen_pissa_14_trained_d_S_sd\endcsname{0.00119}
\expandafter\def\csname result@rev_qwen_pissa_14_trained_ratio_mean\endcsname{3.46}
\expandafter\def\csname result@rev_qwen_pissa_14_trained_ratio_sd\endcsname{0.19}
\expandafter\def\csname result@rev_qwen_pissa_28_reconstruct_accuracy_mean\endcsname{92.20}
\expandafter\def\csname result@rev_qwen_pissa_28_reconstruct_accuracy_sd\endcsname{1.15}
\expandafter\def\csname result@rev_qwen_pissa_28_restore_accuracy_mean\endcsname{92.10}
\expandafter\def\csname result@rev_qwen_pissa_28_restore_accuracy_sd\endcsname{1.00}
\expandafter\def\csname result@rev_qwen_pissa_28_restore_nll_vs_reconstruct_mean\endcsname{-0.00009}
\expandafter\def\csname result@rev_qwen_pissa_28_restore_nll_vs_reconstruct_sd\endcsname{0.00060}
\expandafter\def\csname result@rev_qwen_pissa_28_rotation_band_restore_bottom_accuracy_change_mean\endcsname{0.20}
\expandafter\def\csname result@rev_qwen_pissa_28_rotation_band_restore_bottom_accuracy_change_sd\endcsname{0.66}
\expandafter\def\csname result@rev_qwen_pissa_28_rotation_band_restore_mid_accuracy_change_mean\endcsname{0.43}
\expandafter\def\csname result@rev_qwen_pissa_28_rotation_band_restore_mid_accuracy_change_sd\endcsname{0.47}
\expandafter\def\csname result@rev_qwen_pissa_28_rotation_band_restore_top_accuracy_change_mean\endcsname{-87.27}
\expandafter\def\csname result@rev_qwen_pissa_28_rotation_band_restore_top_accuracy_change_sd\endcsname{1.07}
\expandafter\def\csname result@rev_qwen_pissa_28_rotation_band_restore_top_nll_change_mean\endcsname{-0.01063}
\expandafter\def\csname result@rev_qwen_pissa_28_rotation_band_restore_top_nll_change_sd\endcsname{0.00117}
\expandafter\def\csname result@rev_qwen_pissa_28_trained_accuracy_mean\endcsname{91.57}
\expandafter\def\csname result@rev_qwen_pissa_28_trained_accuracy_sd\endcsname{1.10}
\expandafter\def\csname result@rev_qwen_pissa_28_trained_d_W_mean\endcsname{0.817}
\expandafter\def\csname result@rev_qwen_pissa_28_trained_d_W_sd\endcsname{0.001}
\expandafter\def\csname result@rev_qwen_pissa_28_trained_d_S_mean\endcsname{0.02752}
\expandafter\def\csname result@rev_qwen_pissa_28_trained_d_S_sd\endcsname{0.00050}
\expandafter\def\csname result@rev_qwen_pissa_28_trained_ratio_mean\endcsname{3.37}
\expandafter\def\csname result@rev_qwen_pissa_28_trained_ratio_sd\endcsname{0.06}
\expandafter\def\csname result@math_grid_llama_dora_7_0p0003_params\endcsname{19.73}
\expandafter\def\csname result@math_grid_llama_dora_7_0p0003_dev_mean\endcsname{71.50}
\expandafter\def\csname result@math_grid_llama_dora_7_0p0003_dev_sd\endcsname{0.78}
\expandafter\def\csname result@math_grid_llama_dora_7_0p0003_forgetting_mean\endcsname{0.0314}
\expandafter\def\csname result@math_grid_llama_dora_7_0p0003_forgetting_sd\endcsname{0.0013}
\expandafter\def\csname result@math_grid_llama_dora_7_0p0003_test_mean\endcsname{69.95}
\expandafter\def\csname result@math_grid_llama_dora_7_0p0003_test_sd\endcsname{0.12}
\expandafter\def\csname result@math_grid_llama_dora_15_0p0003_params\endcsname{40.70}
\expandafter\def\csname result@math_grid_llama_dora_15_0p0003_dev_mean\endcsname{72.97}
\expandafter\def\csname result@math_grid_llama_dora_15_0p0003_dev_sd\endcsname{1.01}
\expandafter\def\csname result@math_grid_llama_dora_15_0p0003_forgetting_mean\endcsname{0.0394}
\expandafter\def\csname result@math_grid_llama_dora_15_0p0003_forgetting_sd\endcsname{0.0014}
\expandafter\def\csname result@math_grid_llama_dora_15_0p0003_test_mean\endcsname{70.91}
\expandafter\def\csname result@math_grid_llama_dora_15_0p0003_test_sd\endcsname{0.44}
\expandafter\def\csname result@math_grid_llama_dora_30_0p0003_params\endcsname{80.02}
\expandafter\def\csname result@math_grid_llama_dora_30_0p0003_dev_mean\endcsname{73.27}
\expandafter\def\csname result@math_grid_llama_dora_30_0p0003_dev_sd\endcsname{2.18}
\expandafter\def\csname result@math_grid_llama_dora_30_0p0003_forgetting_mean\endcsname{0.0512}
\expandafter\def\csname result@math_grid_llama_dora_30_0p0003_forgetting_sd\endcsname{0.0033}
\expandafter\def\csname result@math_grid_llama_dora_30_0p0003_test_mean\endcsname{71.97}
\expandafter\def\csname result@math_grid_llama_dora_30_0p0003_test_sd\endcsname{1.16}
\expandafter\def\csname result@math_grid_llama_hra_16_0p001_params\endcsname{19.92}
\expandafter\def\csname result@math_grid_llama_hra_16_0p001_dev_mean\endcsname{69.03}
\expandafter\def\csname result@math_grid_llama_hra_16_0p001_dev_sd\endcsname{0.61}
\expandafter\def\csname result@math_grid_llama_hra_16_0p001_forgetting_mean\endcsname{0.0340}
\expandafter\def\csname result@math_grid_llama_hra_16_0p001_forgetting_sd\endcsname{0.0017}
\expandafter\def\csname result@math_grid_llama_hra_16_0p001_test_mean\endcsname{67.63}
\expandafter\def\csname result@math_grid_llama_hra_16_0p001_test_sd\endcsname{0.42}
\expandafter\def\csname result@math_grid_llama_hra_32_0p001_params\endcsname{39.85}
\expandafter\def\csname result@math_grid_llama_hra_32_0p001_dev_mean\endcsname{70.17}
\expandafter\def\csname result@math_grid_llama_hra_32_0p001_dev_sd\endcsname{1.10}
\expandafter\def\csname result@math_grid_llama_hra_32_0p001_forgetting_mean\endcsname{0.0521}
\expandafter\def\csname result@math_grid_llama_hra_32_0p001_forgetting_sd\endcsname{0.0022}
\expandafter\def\csname result@math_grid_llama_hra_32_0p001_test_mean\endcsname{70.31}
\expandafter\def\csname result@math_grid_llama_hra_32_0p001_test_sd\endcsname{1.37}
\expandafter\def\csname result@math_grid_llama_hra_64_0p001_params\endcsname{79.69}
\expandafter\def\csname result@math_grid_llama_hra_64_0p001_dev_mean\endcsname{71.47}
\expandafter\def\csname result@math_grid_llama_hra_64_0p001_dev_sd\endcsname{1.01}
\expandafter\def\csname result@math_grid_llama_hra_64_0p001_forgetting_mean\endcsname{0.0988}
\expandafter\def\csname result@math_grid_llama_hra_64_0p001_forgetting_sd\endcsname{0.0031}
\expandafter\def\csname result@math_grid_llama_hra_64_0p001_test_mean\endcsname{71.04}
\expandafter\def\csname result@math_grid_llama_hra_64_0p001_test_sd\endcsname{0.20}
\expandafter\def\csname result@math_grid_llama_lora_7_0p0003_params\endcsname{18.35}
\expandafter\def\csname result@math_grid_llama_lora_7_0p0003_dev_mean\endcsname{71.17}
\expandafter\def\csname result@math_grid_llama_lora_7_0p0003_dev_sd\endcsname{0.29}
\expandafter\def\csname result@math_grid_llama_lora_7_0p0003_forgetting_mean\endcsname{0.0285}
\expandafter\def\csname result@math_grid_llama_lora_7_0p0003_forgetting_sd\endcsname{0.0018}
\expandafter\def\csname result@math_grid_llama_lora_7_0p0003_test_mean\endcsname{70.84}
\expandafter\def\csname result@math_grid_llama_lora_7_0p0003_test_sd\endcsname{1.09}
\expandafter\def\csname result@math_grid_llama_lora_15_0p0003_params\endcsname{39.32}
\expandafter\def\csname result@math_grid_llama_lora_15_0p0003_dev_mean\endcsname{72.30}
\expandafter\def\csname result@math_grid_llama_lora_15_0p0003_dev_sd\endcsname{1.04}
\expandafter\def\csname result@math_grid_llama_lora_15_0p0003_forgetting_mean\endcsname{0.0372}
\expandafter\def\csname result@math_grid_llama_lora_15_0p0003_forgetting_sd\endcsname{0.0023}
\expandafter\def\csname result@math_grid_llama_lora_15_0p0003_test_mean\endcsname{71.65}
\expandafter\def\csname result@math_grid_llama_lora_15_0p0003_test_sd\endcsname{0.35}
\expandafter\def\csname result@math_grid_llama_lora_30_0p0003_params\endcsname{78.64}
\expandafter\def\csname result@math_grid_llama_lora_30_0p0003_dev_mean\endcsname{71.93}
\expandafter\def\csname result@math_grid_llama_lora_30_0p0003_dev_sd\endcsname{1.12}
\expandafter\def\csname result@math_grid_llama_lora_30_0p0003_forgetting_mean\endcsname{0.0486}
\expandafter\def\csname result@math_grid_llama_lora_30_0p0003_forgetting_sd\endcsname{0.0024}
\expandafter\def\csname result@math_grid_llama_lora_30_0p0003_test_mean\endcsname{71.90}
\expandafter\def\csname result@math_grid_llama_lora_30_0p0003_test_sd\endcsname{0.12}
\expandafter\def\csname result@math_grid_llama_milora_7_0p0003_params\endcsname{18.35}
\expandafter\def\csname result@math_grid_llama_milora_7_0p0003_dev_mean\endcsname{71.53}
\expandafter\def\csname result@math_grid_llama_milora_7_0p0003_dev_sd\endcsname{0.25}
\expandafter\def\csname result@math_grid_llama_milora_7_0p0003_forgetting_mean\endcsname{0.0486}
\expandafter\def\csname result@math_grid_llama_milora_7_0p0003_forgetting_sd\endcsname{0.0017}
\expandafter\def\csname result@math_grid_llama_milora_7_0p0003_test_mean\endcsname{70.43}
\expandafter\def\csname result@math_grid_llama_milora_7_0p0003_test_sd\endcsname{0.42}
\expandafter\def\csname result@math_grid_llama_milora_15_0p0003_params\endcsname{39.32}
\expandafter\def\csname result@math_grid_llama_milora_15_0p0003_dev_mean\endcsname{72.03}
\expandafter\def\csname result@math_grid_llama_milora_15_0p0003_dev_sd\endcsname{0.86}
\expandafter\def\csname result@math_grid_llama_milora_15_0p0003_forgetting_mean\endcsname{0.0685}
\expandafter\def\csname result@math_grid_llama_milora_15_0p0003_forgetting_sd\endcsname{0.0021}
\expandafter\def\csname result@math_grid_llama_milora_15_0p0003_test_mean\endcsname{71.06}
\expandafter\def\csname result@math_grid_llama_milora_15_0p0003_test_sd\endcsname{0.61}
\expandafter\def\csname result@math_grid_llama_milora_30_0p0003_params\endcsname{78.64}
\expandafter\def\csname result@math_grid_llama_milora_30_0p0003_dev_mean\endcsname{72.60}
\expandafter\def\csname result@math_grid_llama_milora_30_0p0003_dev_sd\endcsname{0.52}
\expandafter\def\csname result@math_grid_llama_milora_30_0p0003_forgetting_mean\endcsname{0.0942}
\expandafter\def\csname result@math_grid_llama_milora_30_0p0003_forgetting_sd\endcsname{0.0024}
\expandafter\def\csname result@math_grid_llama_milora_30_0p0003_test_mean\endcsname{70.68}
\expandafter\def\csname result@math_grid_llama_milora_30_0p0003_test_sd\endcsname{0.23}
\expandafter\def\csname result@math_grid_llama_oft_32_0p0003_params\endcsname{19.30}
\expandafter\def\csname result@math_grid_llama_oft_32_0p0003_dev_mean\endcsname{72.10}
\expandafter\def\csname result@math_grid_llama_oft_32_0p0003_dev_sd\endcsname{1.31}
\expandafter\def\csname result@math_grid_llama_oft_32_0p0003_forgetting_mean\endcsname{0.0371}
\expandafter\def\csname result@math_grid_llama_oft_32_0p0003_forgetting_sd\endcsname{0.0015}
\expandafter\def\csname result@math_grid_llama_oft_32_0p0003_test_mean\endcsname{71.42}
\expandafter\def\csname result@math_grid_llama_oft_32_0p0003_test_sd\endcsname{1.45}
\expandafter\def\csname result@math_grid_llama_oft_64_0p0003_params\endcsname{39.22}
\expandafter\def\csname result@math_grid_llama_oft_64_0p0003_dev_mean\endcsname{73.10}
\expandafter\def\csname result@math_grid_llama_oft_64_0p0003_dev_sd\endcsname{0.26}
\expandafter\def\csname result@math_grid_llama_oft_64_0p0003_forgetting_mean\endcsname{0.0740}
\expandafter\def\csname result@math_grid_llama_oft_64_0p0003_forgetting_sd\endcsname{0.0017}
\expandafter\def\csname result@math_grid_llama_oft_64_0p0003_test_mean\endcsname{71.34}
\expandafter\def\csname result@math_grid_llama_oft_64_0p0003_test_sd\endcsname{0.13}
\expandafter\def\csname result@math_grid_llama_oft_128_0p0001_params\endcsname{79.07}
\expandafter\def\csname result@math_grid_llama_oft_128_0p0001_dev_mean\endcsname{72.57}
\expandafter\def\csname result@math_grid_llama_oft_128_0p0001_dev_sd\endcsname{0.71}
\expandafter\def\csname result@math_grid_llama_oft_128_0p0001_forgetting_mean\endcsname{0.0366}
\expandafter\def\csname result@math_grid_llama_oft_128_0p0001_forgetting_sd\endcsname{0.0011}
\expandafter\def\csname result@math_grid_llama_oft_128_0p0001_test_mean\endcsname{71.75}
\expandafter\def\csname result@math_grid_llama_oft_128_0p0001_test_sd\endcsname{1.18}
\expandafter\def\csname result@math_grid_llama_pissa_7_0p0001_params\endcsname{18.35}
\expandafter\def\csname result@math_grid_llama_pissa_7_0p0001_dev_mean\endcsname{70.57}
\expandafter\def\csname result@math_grid_llama_pissa_7_0p0001_dev_sd\endcsname{0.21}
\expandafter\def\csname result@math_grid_llama_pissa_7_0p0001_forgetting_mean\endcsname{0.1084}
\expandafter\def\csname result@math_grid_llama_pissa_7_0p0001_forgetting_sd\endcsname{0.0025}
\expandafter\def\csname result@math_grid_llama_pissa_7_0p0001_test_mean\endcsname{70.33}
\expandafter\def\csname result@math_grid_llama_pissa_7_0p0001_test_sd\endcsname{1.49}
\expandafter\def\csname result@math_grid_llama_pissa_15_0p0001_params\endcsname{39.32}
\expandafter\def\csname result@math_grid_llama_pissa_15_0p0001_dev_mean\endcsname{70.87}
\expandafter\def\csname result@math_grid_llama_pissa_15_0p0001_dev_sd\endcsname{0.61}
\expandafter\def\csname result@math_grid_llama_pissa_15_0p0001_forgetting_mean\endcsname{0.1521}
\expandafter\def\csname result@math_grid_llama_pissa_15_0p0001_forgetting_sd\endcsname{0.0030}
\expandafter\def\csname result@math_grid_llama_pissa_15_0p0001_test_mean\endcsname{69.98}
\expandafter\def\csname result@math_grid_llama_pissa_15_0p0001_test_sd\endcsname{0.27}
\expandafter\def\csname result@math_grid_llama_pissa_30_3em05_params\endcsname{78.64}
\expandafter\def\csname result@math_grid_llama_pissa_30_3em05_dev_mean\endcsname{71.13}
\expandafter\def\csname result@math_grid_llama_pissa_30_3em05_dev_sd\endcsname{0.68}
\expandafter\def\csname result@math_grid_llama_pissa_30_3em05_forgetting_mean\endcsname{0.0650}
\expandafter\def\csname result@math_grid_llama_pissa_30_3em05_forgetting_sd\endcsname{0.0014}
\expandafter\def\csname result@math_grid_llama_pissa_30_3em05_test_mean\endcsname{70.26}
\expandafter\def\csname result@math_grid_llama_pissa_30_3em05_test_sd\endcsname{0.98}
\expandafter\def\csname result@math_grid_qwen_dora_7_0p0003_params\endcsname{19.05}
\expandafter\def\csname result@math_grid_qwen_dora_7_0p0003_dev_mean\endcsname{92.23}
\expandafter\def\csname result@math_grid_qwen_dora_7_0p0003_dev_sd\endcsname{0.40}
\expandafter\def\csname result@math_grid_qwen_dora_7_0p0003_forgetting_mean\endcsname{0.0055}
\expandafter\def\csname result@math_grid_qwen_dora_7_0p0003_forgetting_sd\endcsname{0.0007}
\expandafter\def\csname result@math_grid_qwen_dora_7_0p0003_test_mean\endcsname{83.27}
\expandafter\def\csname result@math_grid_qwen_dora_7_0p0003_test_sd\endcsname{0.04}
\expandafter\def\csname result@math_grid_qwen_dora_14_0p0003_params\endcsname{36.71}
\expandafter\def\csname result@math_grid_qwen_dora_14_0p0003_dev_mean\endcsname{92.77}
\expandafter\def\csname result@math_grid_qwen_dora_14_0p0003_dev_sd\endcsname{0.58}
\expandafter\def\csname result@math_grid_qwen_dora_14_0p0003_forgetting_mean\endcsname{0.0054}
\expandafter\def\csname result@math_grid_qwen_dora_14_0p0003_forgetting_sd\endcsname{0.0005}
\expandafter\def\csname result@math_grid_qwen_dora_14_0p0003_test_mean\endcsname{82.64}
\expandafter\def\csname result@math_grid_qwen_dora_14_0p0003_test_sd\endcsname{0.59}
\expandafter\def\csname result@math_grid_qwen_dora_28_0p0001_params\endcsname{72.04}
\expandafter\def\csname result@math_grid_qwen_dora_28_0p0001_dev_mean\endcsname{92.63}
\expandafter\def\csname result@math_grid_qwen_dora_28_0p0001_dev_sd\endcsname{0.06}
\expandafter\def\csname result@math_grid_qwen_dora_28_0p0001_forgetting_mean\endcsname{0.0048}
\expandafter\def\csname result@math_grid_qwen_dora_28_0p0001_forgetting_sd\endcsname{0.0006}
\expandafter\def\csname result@math_grid_qwen_dora_28_0p0001_test_mean\endcsname{83.12}
\expandafter\def\csname result@math_grid_qwen_dora_28_0p0001_test_sd\endcsname{0.53}
\expandafter\def\csname result@math_grid_qwen_hra_16_0p0001_params\endcsname{18.12}
\expandafter\def\csname result@math_grid_qwen_hra_16_0p0001_dev_mean\endcsname{91.97}
\expandafter\def\csname result@math_grid_qwen_hra_16_0p0001_dev_sd\endcsname{0.31}
\expandafter\def\csname result@math_grid_qwen_hra_16_0p0001_forgetting_mean\endcsname{0.0054}
\expandafter\def\csname result@math_grid_qwen_hra_16_0p0001_forgetting_sd\endcsname{0.0007}
\expandafter\def\csname result@math_grid_qwen_hra_16_0p0001_test_mean\endcsname{83.60}
\expandafter\def\csname result@math_grid_qwen_hra_16_0p0001_test_sd\endcsname{0.12}
\expandafter\def\csname result@math_grid_qwen_hra_32_0p0001_params\endcsname{36.24}
\expandafter\def\csname result@math_grid_qwen_hra_32_0p0001_dev_mean\endcsname{92.07}
\expandafter\def\csname result@math_grid_qwen_hra_32_0p0001_dev_sd\endcsname{0.32}
\expandafter\def\csname result@math_grid_qwen_hra_32_0p0001_forgetting_mean\endcsname{0.0073}
\expandafter\def\csname result@math_grid_qwen_hra_32_0p0001_forgetting_sd\endcsname{0.0001}
\expandafter\def\csname result@math_grid_qwen_hra_32_0p0001_test_mean\endcsname{83.65}
\expandafter\def\csname result@math_grid_qwen_hra_32_0p0001_test_sd\endcsname{0.49}
\expandafter\def\csname result@math_grid_qwen_hra_64_1em05_params\endcsname{72.48}
\expandafter\def\csname result@math_grid_qwen_hra_64_1em05_dev_mean\endcsname{92.00}
\expandafter\def\csname result@math_grid_qwen_hra_64_1em05_dev_sd\endcsname{0.26}
\expandafter\def\csname result@math_grid_qwen_hra_64_1em05_forgetting_mean\endcsname{0.0038}
\expandafter\def\csname result@math_grid_qwen_hra_64_1em05_forgetting_sd\endcsname{0.0002}
\expandafter\def\csname result@math_grid_qwen_hra_64_1em05_test_mean\endcsname{83.22}
\expandafter\def\csname result@math_grid_qwen_hra_64_1em05_test_sd\endcsname{0.16}
\expandafter\def\csname result@math_grid_qwen_lora_7_0p001_params\endcsname{17.66}
\expandafter\def\csname result@math_grid_qwen_lora_7_0p001_dev_mean\endcsname{90.73}
\expandafter\def\csname result@math_grid_qwen_lora_7_0p001_dev_sd\endcsname{0.35}
\expandafter\def\csname result@math_grid_qwen_lora_7_0p001_forgetting_mean\endcsname{0.0237}
\expandafter\def\csname result@math_grid_qwen_lora_7_0p001_forgetting_sd\endcsname{0.0012}
\expandafter\def\csname result@math_grid_qwen_lora_7_0p001_test_mean\endcsname{83.07}
\expandafter\def\csname result@math_grid_qwen_lora_7_0p001_test_sd\endcsname{0.54}
\expandafter\def\csname result@math_grid_qwen_lora_14_1em05_params\endcsname{35.32}
\expandafter\def\csname result@math_grid_qwen_lora_14_1em05_dev_mean\endcsname{91.40}
\expandafter\def\csname result@math_grid_qwen_lora_14_1em05_dev_sd\endcsname{0.44}
\expandafter\def\csname result@math_grid_qwen_lora_14_1em05_forgetting_mean\endcsname{0.0020}
\expandafter\def\csname result@math_grid_qwen_lora_14_1em05_forgetting_sd\endcsname{0.0001}
\expandafter\def\csname result@math_grid_qwen_lora_14_1em05_test_mean\endcsname{83.22}
\expandafter\def\csname result@math_grid_qwen_lora_14_1em05_test_sd\endcsname{0.16}
\expandafter\def\csname result@math_grid_qwen_lora_28_3em05_params\endcsname{70.65}
\expandafter\def\csname result@math_grid_qwen_lora_28_3em05_dev_mean\endcsname{92.30}
\expandafter\def\csname result@math_grid_qwen_lora_28_3em05_dev_sd\endcsname{0.46}
\expandafter\def\csname result@math_grid_qwen_lora_28_3em05_forgetting_mean\endcsname{0.0038}
\expandafter\def\csname result@math_grid_qwen_lora_28_3em05_forgetting_sd\endcsname{0.0002}
\expandafter\def\csname result@math_grid_qwen_lora_28_3em05_test_mean\endcsname{83.65}
\expandafter\def\csname result@math_grid_qwen_lora_28_3em05_test_sd\endcsname{0.42}
\expandafter\def\csname result@math_grid_qwen_milora_7_1em05_params\endcsname{17.66}
\expandafter\def\csname result@math_grid_qwen_milora_7_1em05_dev_mean\endcsname{92.60}
\expandafter\def\csname result@math_grid_qwen_milora_7_1em05_dev_sd\endcsname{0.17}
\expandafter\def\csname result@math_grid_qwen_milora_7_1em05_forgetting_mean\endcsname{0.0045}
\expandafter\def\csname result@math_grid_qwen_milora_7_1em05_forgetting_sd\endcsname{0.0002}
\expandafter\def\csname result@math_grid_qwen_milora_7_1em05_test_mean\endcsname{83.35}
\expandafter\def\csname result@math_grid_qwen_milora_7_1em05_test_sd\endcsname{0.24}
\expandafter\def\csname result@math_grid_qwen_milora_14_1em05_params\endcsname{35.32}
\expandafter\def\csname result@math_grid_qwen_milora_14_1em05_dev_mean\endcsname{92.60}
\expandafter\def\csname result@math_grid_qwen_milora_14_1em05_dev_sd\endcsname{0.10}
\expandafter\def\csname result@math_grid_qwen_milora_14_1em05_forgetting_mean\endcsname{0.0050}
\expandafter\def\csname result@math_grid_qwen_milora_14_1em05_forgetting_sd\endcsname{0.0003}
\expandafter\def\csname result@math_grid_qwen_milora_14_1em05_test_mean\endcsname{83.67}
\expandafter\def\csname result@math_grid_qwen_milora_14_1em05_test_sd\endcsname{0.46}
\expandafter\def\csname result@math_grid_qwen_milora_28_3em05_params\endcsname{70.65}
\expandafter\def\csname result@math_grid_qwen_milora_28_3em05_dev_mean\endcsname{92.60}
\expandafter\def\csname result@math_grid_qwen_milora_28_3em05_dev_sd\endcsname{0.30}
\expandafter\def\csname result@math_grid_qwen_milora_28_3em05_forgetting_mean\endcsname{0.0083}
\expandafter\def\csname result@math_grid_qwen_milora_28_3em05_forgetting_sd\endcsname{0.0010}
\expandafter\def\csname result@math_grid_qwen_milora_28_3em05_test_mean\endcsname{83.73}
\expandafter\def\csname result@math_grid_qwen_milora_28_3em05_test_sd\endcsname{0.12}
\expandafter\def\csname result@math_grid_qwen_oft_32_0p0003_params\endcsname{17.55}
\expandafter\def\csname result@math_grid_qwen_oft_32_0p0003_dev_mean\endcsname{91.37}
\expandafter\def\csname result@math_grid_qwen_oft_32_0p0003_dev_sd\endcsname{0.25}
\expandafter\def\csname result@math_grid_qwen_oft_32_0p0003_forgetting_mean\endcsname{0.0105}
\expandafter\def\csname result@math_grid_qwen_oft_32_0p0003_forgetting_sd\endcsname{0.0009}
\expandafter\def\csname result@math_grid_qwen_oft_32_0p0003_test_mean\endcsname{83.35}
\expandafter\def\csname result@math_grid_qwen_oft_32_0p0003_test_sd\endcsname{0.36}
\expandafter\def\csname result@math_grid_qwen_oft_64_0p0003_params\endcsname{35.68}
\expandafter\def\csname result@math_grid_qwen_oft_64_0p0003_dev_mean\endcsname{89.97}
\expandafter\def\csname result@math_grid_qwen_oft_64_0p0003_dev_sd\endcsname{0.38}
\expandafter\def\csname result@math_grid_qwen_oft_64_0p0003_forgetting_mean\endcsname{0.0182}
\expandafter\def\csname result@math_grid_qwen_oft_64_0p0003_forgetting_sd\endcsname{0.0008}
\expandafter\def\csname result@math_grid_qwen_oft_64_0p0003_test_mean\endcsname{82.92}
\expandafter\def\csname result@math_grid_qwen_oft_64_0p0003_test_sd\endcsname{0.46}
\expandafter\def\csname result@math_grid_qwen_oft_128_0p0001_params\endcsname{71.92}
\expandafter\def\csname result@math_grid_qwen_oft_128_0p0001_dev_mean\endcsname{90.90}
\expandafter\def\csname result@math_grid_qwen_oft_128_0p0001_dev_sd\endcsname{0.69}
\expandafter\def\csname result@math_grid_qwen_oft_128_0p0001_forgetting_mean\endcsname{0.0088}
\expandafter\def\csname result@math_grid_qwen_oft_128_0p0001_forgetting_sd\endcsname{0.0007}
\expandafter\def\csname result@math_grid_qwen_oft_128_0p0001_test_mean\endcsname{83.50}
\expandafter\def\csname result@math_grid_qwen_oft_128_0p0001_test_sd\endcsname{0.27}
\expandafter\def\csname result@math_grid_qwen_pissa_7_3em05_params\endcsname{17.66}
\expandafter\def\csname result@math_grid_qwen_pissa_7_3em05_dev_mean\endcsname{92.53}
\expandafter\def\csname result@math_grid_qwen_pissa_7_3em05_dev_sd\endcsname{0.32}
\expandafter\def\csname result@math_grid_qwen_pissa_7_3em05_forgetting_mean\endcsname{0.0108}
\expandafter\def\csname result@math_grid_qwen_pissa_7_3em05_forgetting_sd\endcsname{0.0009}
\expandafter\def\csname result@math_grid_qwen_pissa_7_3em05_test_mean\endcsname{83.85}
\expandafter\def\csname result@math_grid_qwen_pissa_7_3em05_test_sd\endcsname{0.47}
\expandafter\def\csname result@math_grid_qwen_pissa_14_3em05_params\endcsname{35.32}
\expandafter\def\csname result@math_grid_qwen_pissa_14_3em05_dev_mean\endcsname{92.40}
\expandafter\def\csname result@math_grid_qwen_pissa_14_3em05_dev_sd\endcsname{0.20}
\expandafter\def\csname result@math_grid_qwen_pissa_14_3em05_forgetting_mean\endcsname{0.0110}
\expandafter\def\csname result@math_grid_qwen_pissa_14_3em05_forgetting_sd\endcsname{0.0012}
\expandafter\def\csname result@math_grid_qwen_pissa_14_3em05_test_mean\endcsname{83.67}
\expandafter\def\csname result@math_grid_qwen_pissa_14_3em05_test_sd\endcsname{0.23}
\expandafter\def\csname result@math_grid_qwen_pissa_28_3em05_params\endcsname{70.65}
\expandafter\def\csname result@math_grid_qwen_pissa_28_3em05_dev_mean\endcsname{91.73}
\expandafter\def\csname result@math_grid_qwen_pissa_28_3em05_dev_sd\endcsname{0.83}
\expandafter\def\csname result@math_grid_qwen_pissa_28_3em05_forgetting_mean\endcsname{0.0108}
\expandafter\def\csname result@math_grid_qwen_pissa_28_3em05_forgetting_sd\endcsname{0.0010}
\expandafter\def\csname result@math_grid_qwen_pissa_28_3em05_test_mean\endcsname{83.17}
\expandafter\def\csname result@math_grid_qwen_pissa_28_3em05_test_sd\endcsname{0.20}
\expandafter\def\csname result@math_extension_llama_hra_16_0p003_dev_mean\endcsname{70.47}
\expandafter\def\csname result@math_extension_llama_hra_16_0p003_dev_sd\endcsname{0.42}
\expandafter\def\csname result@math_extension_llama_hra_16_0p003_forgetting_mean\endcsname{0.0553}
\expandafter\def\csname result@math_extension_llama_hra_16_0p003_forgetting_sd\endcsname{0.0028}
\expandafter\def\csname result@math_extension_llama_hra_16_0p003_test_mean\endcsname{70.15}
\expandafter\def\csname result@math_extension_llama_hra_16_0p003_test_sd\endcsname{0.27}
\expandafter\def\csname result@math_extension_llama_hra_32_0p003_dev_mean\endcsname{70.87}
\expandafter\def\csname result@math_extension_llama_hra_32_0p003_dev_sd\endcsname{0.90}
\expandafter\def\csname result@math_extension_llama_hra_32_0p003_forgetting_mean\endcsname{0.1132}
\expandafter\def\csname result@math_extension_llama_hra_32_0p003_forgetting_sd\endcsname{0.0060}
\expandafter\def\csname result@math_extension_llama_hra_32_0p003_test_mean\endcsname{71.37}
\expandafter\def\csname result@math_extension_llama_hra_32_0p003_test_sd\endcsname{1.20}
\expandafter\def\csname result@obj28_lora_4_1em04_test_mean\endcsname{72.95}
\expandafter\def\csname result@obj28_lora_4_1em04_test_sd\endcsname{0.42}
\expandafter\def\csname result@obj28_lora_4_1em04_flow_mean\endcsname{0.8877}
\expandafter\def\csname result@obj28_lora_4_1em04_flow_sd\endcsname{0.0035}
\expandafter\def\csname result@obj28_lora_4_1em04_drift_mean\endcsname{12.11}
\expandafter\def\csname result@obj28_lora_4_1em04_drift_sd\endcsname{0.86}
\expandafter\def\csname result@obj28_lora_4_1em04_clip_mean\endcsname{25.63}
\expandafter\def\csname result@obj28_lora_4_1em04_clip_sd\endcsname{0.52}
\expandafter\def\csname result@obj28_lora_8_5em05_test_mean\endcsname{73.94}
\expandafter\def\csname result@obj28_lora_8_5em05_test_sd\endcsname{0.85}
\expandafter\def\csname result@obj28_lora_8_5em05_flow_mean\endcsname{0.8739}
\expandafter\def\csname result@obj28_lora_8_5em05_flow_sd\endcsname{0.0015}
\expandafter\def\csname result@obj28_lora_8_5em05_drift_mean\endcsname{11.38}
\expandafter\def\csname result@obj28_lora_8_5em05_drift_sd\endcsname{0.60}
\expandafter\def\csname result@obj28_lora_8_5em05_clip_mean\endcsname{25.74}
\expandafter\def\csname result@obj28_lora_8_5em05_clip_sd\endcsname{0.24}
\expandafter\def\csname result@obj28_lora_16_3em05_test_mean\endcsname{71.88}
\expandafter\def\csname result@obj28_lora_16_3em05_test_sd\endcsname{1.42}
\expandafter\def\csname result@obj28_lora_16_3em05_flow_mean\endcsname{0.8687}
\expandafter\def\csname result@obj28_lora_16_3em05_flow_sd\endcsname{0.0017}
\expandafter\def\csname result@obj28_lora_16_3em05_drift_mean\endcsname{10.67}
\expandafter\def\csname result@obj28_lora_16_3em05_drift_sd\endcsname{0.62}
\expandafter\def\csname result@obj28_lora_16_3em05_clip_mean\endcsname{25.98}
\expandafter\def\csname result@obj28_lora_16_3em05_clip_sd\endcsname{0.43}
\expandafter\def\csname result@obj28_lora_32_1em04_test_mean\endcsname{72.84}
\expandafter\def\csname result@obj28_lora_32_1em04_test_sd\endcsname{1.18}
\expandafter\def\csname result@obj28_lora_32_1em04_flow_mean\endcsname{1.0164}
\expandafter\def\csname result@obj28_lora_32_1em04_flow_sd\endcsname{0.0058}
\expandafter\def\csname result@obj28_lora_32_1em04_drift_mean\endcsname{16.10}
\expandafter\def\csname result@obj28_lora_32_1em04_drift_sd\endcsname{1.16}
\expandafter\def\csname result@obj28_lora_32_1em04_clip_mean\endcsname{24.62}
\expandafter\def\csname result@obj28_lora_32_1em04_clip_sd\endcsname{0.93}
\expandafter\def\csname result@obj28_oft_32_1em04_test_mean\endcsname{73.44}
\expandafter\def\csname result@obj28_oft_32_1em04_test_sd\endcsname{2.27}
\expandafter\def\csname result@obj28_oft_32_1em04_flow_mean\endcsname{0.8588}
\expandafter\def\csname result@obj28_oft_32_1em04_flow_sd\endcsname{0.0009}
\expandafter\def\csname result@obj28_oft_32_1em04_drift_mean\endcsname{17.29}
\expandafter\def\csname result@obj28_oft_32_1em04_drift_sd\endcsname{0.81}
\expandafter\def\csname result@obj28_oft_32_1em04_clip_mean\endcsname{23.99}
\expandafter\def\csname result@obj28_oft_32_1em04_clip_sd\endcsname{0.05}
\expandafter\def\csname result@obj28_oft_64_5em05_test_mean\endcsname{73.60}
\expandafter\def\csname result@obj28_oft_64_5em05_test_sd\endcsname{1.95}
\expandafter\def\csname result@obj28_oft_64_5em05_flow_mean\endcsname{0.8596}
\expandafter\def\csname result@obj28_oft_64_5em05_flow_sd\endcsname{0.0005}
\expandafter\def\csname result@obj28_oft_64_5em05_drift_mean\endcsname{19.21}
\expandafter\def\csname result@obj28_oft_64_5em05_drift_sd\endcsname{0.52}
\expandafter\def\csname result@obj28_oft_64_5em05_clip_mean\endcsname{23.41}
\expandafter\def\csname result@obj28_oft_64_5em05_clip_sd\endcsname{0.72}
\expandafter\def\csname result@obj28_oft_128_1em04_test_mean\endcsname{73.39}
\expandafter\def\csname result@obj28_oft_128_1em04_test_sd\endcsname{1.76}
\expandafter\def\csname result@obj28_oft_128_1em04_flow_mean\endcsname{0.9692}
\expandafter\def\csname result@obj28_oft_128_1em04_flow_sd\endcsname{0.0030}
\expandafter\def\csname result@obj28_oft_128_1em04_drift_mean\endcsname{43.59}
\expandafter\def\csname result@obj28_oft_128_1em04_drift_sd\endcsname{0.88}
\expandafter\def\csname result@obj28_oft_128_1em04_clip_mean\endcsname{12.24}
\expandafter\def\csname result@obj28_oft_128_1em04_clip_sd\endcsname{0.49}
\expandafter\def\csname result@obj28_oft_256_5em05_test_mean\endcsname{74.13}
\expandafter\def\csname result@obj28_oft_256_5em05_test_sd\endcsname{1.36}
\expandafter\def\csname result@obj28_oft_256_5em05_flow_mean\endcsname{0.9886}
\expandafter\def\csname result@obj28_oft_256_5em05_flow_sd\endcsname{0.0038}
\expandafter\def\csname result@obj28_oft_256_5em05_drift_mean\endcsname{44.91}
\expandafter\def\csname result@obj28_oft_256_5em05_drift_sd\endcsname{0.10}
\expandafter\def\csname result@obj28_oft_256_5em05_clip_mean\endcsname{11.74}
\expandafter\def\csname result@obj28_oft_256_5em05_clip_sd\endcsname{0.41}
\expandafter\def\csname result@spectrum_qwen_r7_ratio\endcsname{20.53}
\expandafter\def\csname result@spectrum_qwen_r7_dS_percent\endcsname{1.67}
\expandafter\def\csname result@rev_capacity_llama_dora_7_valid\endcsname{3}
\expandafter\def\csname result@rev_capacity_llama_dora_7_top\endcsname{1}
\expandafter\def\csname result@rev_capacity_llama_dora_7_complement\endcsname{0}
\expandafter\def\csname result@rev_capacity_llama_dora_15_valid\endcsname{3}
\expandafter\def\csname result@rev_capacity_llama_dora_15_top\endcsname{0}
\expandafter\def\csname result@rev_capacity_llama_dora_15_complement\endcsname{0}
\expandafter\def\csname result@rev_capacity_llama_dora_30_valid\endcsname{3}
\expandafter\def\csname result@rev_capacity_llama_dora_30_top\endcsname{0}
\expandafter\def\csname result@rev_capacity_llama_dora_30_complement\endcsname{0}
\expandafter\def\csname result@rev_capacity_llama_hra_16_valid\endcsname{3}
\expandafter\def\csname result@rev_capacity_llama_hra_16_top\endcsname{1}
\expandafter\def\csname result@rev_capacity_llama_hra_16_complement\endcsname{0}
\expandafter\def\csname result@rev_capacity_llama_hra_32_valid\endcsname{3}
\expandafter\def\csname result@rev_capacity_llama_hra_32_top\endcsname{2}
\expandafter\def\csname result@rev_capacity_llama_hra_32_complement\endcsname{0}
\expandafter\def\csname result@rev_capacity_llama_hra_64_valid\endcsname{3}
\expandafter\def\csname result@rev_capacity_llama_hra_64_top\endcsname{2}
\expandafter\def\csname result@rev_capacity_llama_hra_64_complement\endcsname{0}
\expandafter\def\csname result@rev_capacity_llama_lora_7_valid\endcsname{3}
\expandafter\def\csname result@rev_capacity_llama_lora_7_top\endcsname{2}
\expandafter\def\csname result@rev_capacity_llama_lora_7_complement\endcsname{0}
\expandafter\def\csname result@rev_capacity_llama_lora_15_valid\endcsname{3}
\expandafter\def\csname result@rev_capacity_llama_lora_15_top\endcsname{1}
\expandafter\def\csname result@rev_capacity_llama_lora_15_complement\endcsname{0}
\expandafter\def\csname result@rev_capacity_llama_lora_30_valid\endcsname{3}
\expandafter\def\csname result@rev_capacity_llama_lora_30_top\endcsname{3}
\expandafter\def\csname result@rev_capacity_llama_lora_30_complement\endcsname{0}
\expandafter\def\csname result@rev_capacity_llama_milora_7_valid\endcsname{3}
\expandafter\def\csname result@rev_capacity_llama_milora_7_top\endcsname{0}
\expandafter\def\csname result@rev_capacity_llama_milora_7_complement\endcsname{0}
\expandafter\def\csname result@rev_capacity_llama_milora_15_valid\endcsname{3}
\expandafter\def\csname result@rev_capacity_llama_milora_15_top\endcsname{0}
\expandafter\def\csname result@rev_capacity_llama_milora_15_complement\endcsname{0}
\expandafter\def\csname result@rev_capacity_llama_milora_30_valid\endcsname{3}
\expandafter\def\csname result@rev_capacity_llama_milora_30_top\endcsname{0}
\expandafter\def\csname result@rev_capacity_llama_milora_30_complement\endcsname{0}
\expandafter\def\csname result@rev_capacity_llama_oft_32_valid\endcsname{3}
\expandafter\def\csname result@rev_capacity_llama_oft_32_top\endcsname{1}
\expandafter\def\csname result@rev_capacity_llama_oft_32_complement\endcsname{0}
\expandafter\def\csname result@rev_capacity_llama_oft_64_valid\endcsname{3}
\expandafter\def\csname result@rev_capacity_llama_oft_64_top\endcsname{0}
\expandafter\def\csname result@rev_capacity_llama_oft_64_complement\endcsname{0}
\expandafter\def\csname result@rev_capacity_llama_oft_128_valid\endcsname{3}
\expandafter\def\csname result@rev_capacity_llama_oft_128_top\endcsname{2}
\expandafter\def\csname result@rev_capacity_llama_oft_128_complement\endcsname{0}
\expandafter\def\csname result@rev_capacity_llama_pissa_7_valid\endcsname{3}
\expandafter\def\csname result@rev_capacity_llama_pissa_7_top\endcsname{3}
\expandafter\def\csname result@rev_capacity_llama_pissa_7_complement\endcsname{0}
\expandafter\def\csname result@rev_capacity_llama_pissa_15_valid\endcsname{3}
\expandafter\def\csname result@rev_capacity_llama_pissa_15_top\endcsname{3}
\expandafter\def\csname result@rev_capacity_llama_pissa_15_complement\endcsname{0}
\expandafter\def\csname result@rev_capacity_llama_pissa_30_valid\endcsname{3}
\expandafter\def\csname result@rev_capacity_llama_pissa_30_top\endcsname{3}
\expandafter\def\csname result@rev_capacity_llama_pissa_30_complement\endcsname{0}
\expandafter\def\csname result@rev_capacity_qwen_dora_7_valid\endcsname{3}
\expandafter\def\csname result@rev_capacity_qwen_dora_7_top\endcsname{3}
\expandafter\def\csname result@rev_capacity_qwen_dora_7_complement\endcsname{3}
\expandafter\def\csname result@rev_capacity_qwen_dora_14_valid\endcsname{3}
\expandafter\def\csname result@rev_capacity_qwen_dora_14_top\endcsname{3}
\expandafter\def\csname result@rev_capacity_qwen_dora_14_complement\endcsname{1}
\expandafter\def\csname result@rev_capacity_qwen_dora_28_valid\endcsname{3}
\expandafter\def\csname result@rev_capacity_qwen_dora_28_top\endcsname{3}
\expandafter\def\csname result@rev_capacity_qwen_dora_28_complement\endcsname{3}
\expandafter\def\csname result@rev_capacity_qwen_hra_16_valid\endcsname{3}
\expandafter\def\csname result@rev_capacity_qwen_hra_16_top\endcsname{0}
\expandafter\def\csname result@rev_capacity_qwen_hra_16_complement\endcsname{0}
\expandafter\def\csname result@rev_capacity_qwen_hra_32_valid\endcsname{3}
\expandafter\def\csname result@rev_capacity_qwen_hra_32_top\endcsname{2}
\expandafter\def\csname result@rev_capacity_qwen_hra_32_complement\endcsname{0}
\expandafter\def\csname result@rev_capacity_qwen_hra_64_valid\endcsname{3}
\expandafter\def\csname result@rev_capacity_qwen_hra_64_top\endcsname{0}
\expandafter\def\csname result@rev_capacity_qwen_hra_64_complement\endcsname{0}
\expandafter\def\csname result@rev_capacity_qwen_lora_7_valid\endcsname{3}
\expandafter\def\csname result@rev_capacity_qwen_lora_7_top\endcsname{3}
\expandafter\def\csname result@rev_capacity_qwen_lora_7_complement\endcsname{0}
\expandafter\def\csname result@rev_capacity_qwen_lora_14_valid\endcsname{3}
\expandafter\def\csname result@rev_capacity_qwen_lora_14_top\endcsname{0}
\expandafter\def\csname result@rev_capacity_qwen_lora_14_complement\endcsname{0}
\expandafter\def\csname result@rev_capacity_qwen_lora_28_valid\endcsname{3}
\expandafter\def\csname result@rev_capacity_qwen_lora_28_top\endcsname{3}
\expandafter\def\csname result@rev_capacity_qwen_lora_28_complement\endcsname{3}
\expandafter\def\csname result@rev_capacity_qwen_milora_7_valid\endcsname{3}
\expandafter\def\csname result@rev_capacity_qwen_milora_7_top\endcsname{0}
\expandafter\def\csname result@rev_capacity_qwen_milora_7_complement\endcsname{0}
\expandafter\def\csname result@rev_capacity_qwen_milora_14_valid\endcsname{3}
\expandafter\def\csname result@rev_capacity_qwen_milora_14_top\endcsname{0}
\expandafter\def\csname result@rev_capacity_qwen_milora_14_complement\endcsname{3}
\expandafter\def\csname result@rev_capacity_qwen_milora_28_valid\endcsname{3}
\expandafter\def\csname result@rev_capacity_qwen_milora_28_top\endcsname{0}
\expandafter\def\csname result@rev_capacity_qwen_milora_28_complement\endcsname{3}
\expandafter\def\csname result@rev_capacity_qwen_oft_32_valid\endcsname{3}
\expandafter\def\csname result@rev_capacity_qwen_oft_32_top\endcsname{3}
\expandafter\def\csname result@rev_capacity_qwen_oft_32_complement\endcsname{3}
\expandafter\def\csname result@rev_capacity_qwen_oft_64_valid\endcsname{3}
\expandafter\def\csname result@rev_capacity_qwen_oft_64_top\endcsname{3}
\expandafter\def\csname result@rev_capacity_qwen_oft_64_complement\endcsname{3}
\expandafter\def\csname result@rev_capacity_qwen_oft_128_valid\endcsname{3}
\expandafter\def\csname result@rev_capacity_qwen_oft_128_top\endcsname{3}
\expandafter\def\csname result@rev_capacity_qwen_oft_128_complement\endcsname{3}
\expandafter\def\csname result@rev_capacity_qwen_pissa_7_valid\endcsname{3}
\expandafter\def\csname result@rev_capacity_qwen_pissa_7_top\endcsname{3}
\expandafter\def\csname result@rev_capacity_qwen_pissa_7_complement\endcsname{0}
\expandafter\def\csname result@rev_capacity_qwen_pissa_14_valid\endcsname{3}
\expandafter\def\csname result@rev_capacity_qwen_pissa_14_top\endcsname{3}
\expandafter\def\csname result@rev_capacity_qwen_pissa_14_complement\endcsname{0}
\expandafter\def\csname result@rev_capacity_qwen_pissa_28_valid\endcsname{3}
\expandafter\def\csname result@rev_capacity_qwen_pissa_28_top\endcsname{3}
\expandafter\def\csname result@rev_capacity_qwen_pissa_28_complement\endcsname{0}
\expandafter\def\csname result@rev_qwen_milora_7_restore_accuracy_mean\endcsname{92.77}
\expandafter\def\csname result@rev_qwen_milora_7_restore_accuracy_sd\endcsname{0.32}
\expandafter\def\csname result@rev_qwen_milora_7_restore_nll_vs_reconstruct_mean\endcsname{0.00000}
\expandafter\def\csname result@rev_qwen_milora_7_restore_nll_vs_reconstruct_sd\endcsname{0.00012}
\expandafter\def\csname result@rev_qwen_milora_14_restore_accuracy_mean\endcsname{92.57}
\expandafter\def\csname result@rev_qwen_milora_14_restore_accuracy_sd\endcsname{0.29}
\expandafter\def\csname result@rev_qwen_milora_14_restore_nll_vs_reconstruct_mean\endcsname{0.00006}
\expandafter\def\csname result@rev_qwen_milora_14_restore_nll_vs_reconstruct_sd\endcsname{0.00021}

\expandafter\def\csname result@spectrum_qwen_r7_ratio\endcsname{20.53}
\expandafter\def\csname result@spectrum_qwen_r7_dS_percent\endcsname{1.67}

\expandafter\def\csname result@he_llama_pissa_forgetting\endcsname{0.0646}
\expandafter\def\csname result@he_llama_lora_forgetting\endcsname{0.0487}
\expandafter\def\csname result@he_max_checkpoint_percent\endcsname{0.00980}

\begin{document}
\renewcommand{\floatpagefraction}{.85}
\maketitle
\begin{abstract}
Parameter-efficient fine-tuning (PEFT) offers many parameterizations, yet their methodological and functional differences remain unclear. We compare six methods in language and diffusion models to examine how their parameterizations relate to task performance, forgetting, and changes in pretrained weight geometry. Motivated by the spectrum-preserving design of orthogonal fine-tuning (OFT), we first ask whether spectral preservation is itself important for adaptation and retention. We find that the selected LoRA-family methods also approximately preserve pretrained geometry, and that restoring their slightly drifted singular-value spectra largely preserves task performance, questioning the necessity of explicit geometric preservation. Beyond this, we observe that some methods exhibit distinct adaptation--retention trade-offs that vary across settings: LoRA most consistently limits forgetting at competitive performance, DoRA achieves higher mean task scores than LoRA in most comparisons, while PiSSA often incurs greater retention costs. Further intervention experiments suggest that while performance gains from different PEFT methods can be attributed to modifications in different groups of spectral components, we consistently find that restoring dominant rather than intermediate or trailing components produces the largest mean reduction in general-text NLL or base-image drift. Together, these results motivate evaluating geometric constraints through their functional consequences rather than preservation alone.
Code is available at \url{https://github.com/Kuaaannn/PEFT_methods}.
\end{abstract}
\section{Introduction}

Parameter-efficient fine-tuning (PEFT) \citep{houlsby2019adapters,li2021prefix,hu2022lora}
has become a standard approach for adapting large pretrained models.
Over the past several years, a large number of PEFT algorithms have been proposed,
often by modifying a small number of recurring parameterizations.

One major paradigm is the LoRA family, which adapts pretrained weights through
\emph{additive} low-rank updates \citep{hu2022lora}.
Subsequent variants modify the optimization, scaling, initialization, or
magnitude--direction decomposition of these updates
\citep{hayou2024loraplus,kalajdzievski2023rslora,
wang2024loraga,meng2024pissa,liu2024dora}.
A second paradigm is the OFT family, which adapts weights
\emph{multiplicatively} through structured orthogonal transformations and their
variants \citep{qiu2023oft,qiu2025oftv2,qiu2025poet,yuan2024hra}.
These methods are motivated by seemingly different principles, and many are
presented as principled improvements over vanilla LoRA.
Yet it remains unclear which of these differences survive careful tuning and,
more importantly, which differences matter functionally. This raises a basic question:

\begin{center}
\textbf{\emph{When do typical PEFT methods really differ, and which differences matter?}}
\end{center}

\begin{figure}[t]
\centering
\includegraphics[width=\linewidth]{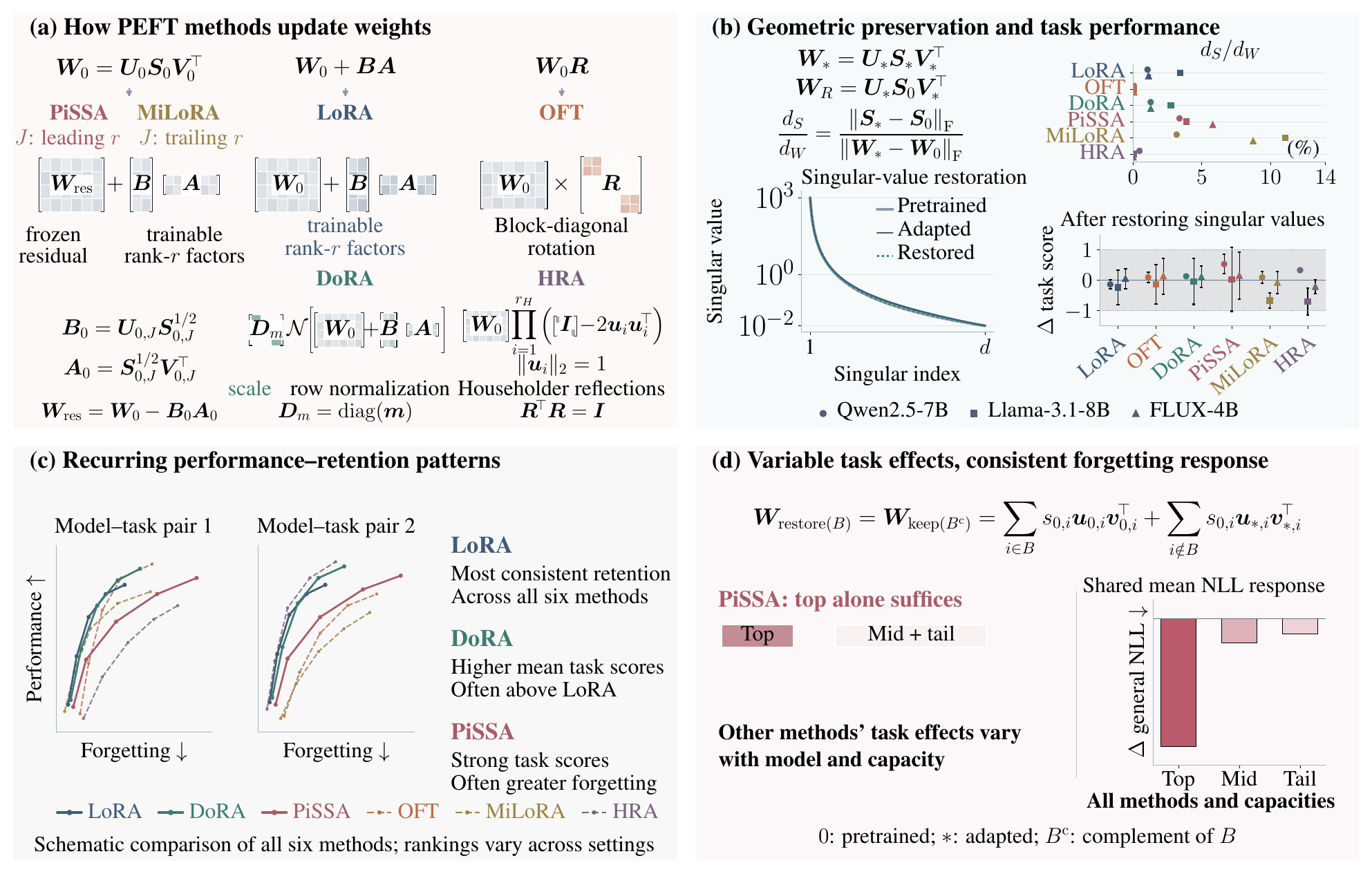}
\caption{
Overview of the questions studied in this work.
(a) Six representative PEFT parameterizations.
(b) Changes in pretrained singular spectra and the singular-value restoration
intervention.
(c) Performance--retention comparisons across PEFT methods.
(d) Spectral interventions that selectively retain or restore dominant,
intermediate, and trailing components.
Singular-value decay curves and bar heights are schematic.
}
\label{fig:overview}
\end{figure}

Recent work has begun to question whether apparent methodological improvements
reflect genuinely different operating behavior.
For the LoRA family, \citet{schulman2025lora} identify rescaling invariances
linking initialization scale, update scaling, and factor learning rates, and
show that vanilla LoRA can be highly competitive when properly configured.
\citet{lee2026learningrate} likewise find similar peak performance between
LoRA and its variants after method-specific learning-rate tuning, yet 46 of
64 surveyed studies report a fixed learning rate.
Hugging Face PEFT's mathematical-reasoning and image-generation benchmarks
\citep[v0.21.0]{mangrulkar2022peft} also report a single training seed and,
for most methods, only one learning-rate setting.
Comparison settings also differ across studies. On GSM8K and MATH, PiSSA's
main comparison adapts all linear layers across three 7B backbones for one
epoch \citep{meng2024pissa}, whereas HRA \citep{yuan2024hra} adapts only
LLaMA2-7B's query and value projections for two epochs.
These observations motivate common experimental settings, method-specific
tuning and multiple training seeds to separate configuration effects from
methodological gains.

Performance alone, however, may not reveal the full distinction between methods.
Two adapters can achieve similar downstream scores while differing substantially
in how much pretrained behavior they disrupt.
This issue is particularly relevant to orthogonal adaptation.
OFT explicitly preserves hyperspherical geometry
\citep{qiu2023oft}, while HRA connects orthogonal transformations and
low-rank adaptation through Householder reflections \citep{yuan2024hra}.
Such geometric constraints are motivated in part by improved generalization
and preservation of pretrained knowledge
\citep{qiu2023oft,yuan2024hra,huang2026peftarena}.
But two questions remain unresolved:
do unconstrained LoRA-family methods actually depart substantially from these
geometries in practice, and does tighter geometric preservation translate into
better adaptation or retention?

\subsection{Contributions and findings}

We address these questions by combining controlled empirical comparisons with
direct interventions on pretrained weight geometry.
We also review experimental controls in published PEFT comparisons to guide our study design (\Cref{sec:peft-comparisons}).
We study six representative methods, including LoRA, DoRA, PiSSA, MiLoRA, OFT, and
HRA, across language-model adaptation and diffusion personalization.
Rather than comparing parameterizations only through final task scores, we
examine three levels of behavior:
their operating performance after method-specific tuning,
their retention of pretrained capabilities, and the functional consequences of
the geometric changes they induce.
As summarized in \Cref{fig:overview}, this analysis reveals a recurring
mismatch:
\emph{parameterization labels can overstate geometric differences, whereas
task scores can understate functional differences.}
Our study leads to three main findings.

\noindent {\bf Explicit geometric preservation is not what primarily separates PEFT methods.}
We first characterize adaptation relative to the two-sided orthogonal orbit of
the pretrained weights.
For small isotropically oriented additive updates, the fraction of weight
displacement appearing as singular-value displacement scales as
$1/\sqrt{\max(m,n)}$, showing that limited spectral drift can arise even
without an explicit orthogonality constraint.
Consistent with this prediction, the LoRA-family methods we study exhibit only
modest spectral drift and very small changes in hyperspherical energy.
More importantly, restoring pretrained singular values while retaining the
adapted singular bases largely preserves task performance.
Thus, explicit geometric preservation is neither unique to successful PEFT nor,
in our tested settings, necessary for retaining its task gains.

\noindent {\bf Similar task performance can conceal more persistent differences in retention.}
After method-specific learning-rate tuning, the six methods often achieve
competitive downstream performance, but their retention behavior differs more
systematically.
Across the tested settings, LoRA most consistently limits forgetting while
maintaining competitive performance.
DoRA achieves higher mean task scores than LoRA in most comparisons, typically
with some additional forgetting, whereas PiSSA often incurs larger retention
costs.
Crucially, tighter preservation of pretrained geometry does not reliably
predict better retention.
The more persistent distinctions between methods therefore appear in their
functional consequences, rather than simply in the geometric invariants imposed
by their parameterizations.

\noindent {\bf Adaptation is spectrally heterogeneous, whereas forgetting is spectrally structured.}
We next intervene on dominant, intermediate, and trailing spectral components,
ordered by pretrained singular value.
For task adaptation, there is no universal spectral location of useful change:
PiSSA consistently concentrates task-relevant changes in the dominant
components, whereas for the other methods the sufficient components vary with
the model, capacity, and task.
For forgetting, however, a markedly more consistent pattern emerges:
restoring the dominant components produces the largest mean reduction in
general-text NLL or base-image drift across all six methods and tested settings.
This ordering persists even in cases where the dominant-component edit is
smaller in norm than alternative edits.
These results suggest that downstream adaptation can be realized through
multiple spectral routes, while forgetting is more consistently coupled to
changes in dominant pretrained components.

Taken together, our results suggest that PEFT methods are better distinguished
by their operating regimes and the functional consequences of learned weight
geometry than by parameterization labels or explicit geometric preservation alone.

\subsection{Related work}
\paragraph{Adaptation, retention and parameterization.}
Within the LoRA family, several methods use singular directions of pretrained weights or gradients to initialize adaptation \citep{meng2024pissa,wang2024loraga,zhang2025loraone,ponkshe2024lorasb}. 
Optimization and scaling studies explain how parameterization affects adaptation \citep{hayou2024loraplus,kalajdzievski2023rslora,wang2024loraga}; \citet{lee2026learningrate} link preferred learning rates to loss curvature. \citet{biderman2024loraforgetting} compare LoRA's adaptation and forgetting with full fine-tuning. PEFT-Arena compares PEFT methods using a shared SFT learning rate and one run per configuration, leaving method-specific tuning and seed uncertainty unresolved \citep{huang2026peftarena}. ETHER questions the practical necessity of hyperspherical-energy preservation \citep{bini2024ether}. We examine preservation alongside tuned retention comparisons and spectral interventions.

\paragraph{Spectral structure and functional interventions.}
\citet{shuttleworth2024illusion} identify intruder singular directions after LoRA adaptation. \citet{he2025posttrainingstructure} study spectrum and vector replacement after full fine-tuning. \citet{jin2026rlheals} connect forgetting and recovery during full-parameter SFT/RL to singular-vector changes. CASA studies spectral preservation in diffusion-LoRA transfer \citep{wang2026casa}, while SVFT learns combinations of pretrained singular-vector pairs \citep{lingam2024svft}. Our contribution is to connect tuning-controlled comparisons with complementary keep/reset interventions, distinguishing recurring retention responses from setting-dependent contributions to task performance.

\section{Experimental comparisons in PEFT}\label{sec:peft-comparisons}

\citet{lee2026learningrate} survey whether vanilla LoRA was evaluated across multiple hyperparameter settings.
We extend this assessment to uncertainty across training runs, validation-based learning-rate selection, disclosure of numerical search spaces and comparable experimental conditions, considering both the proposed method and its principal baseline.

We also review method papers associated with Hugging Face PEFT v0.21.0 \citep{mangrulkar2022peft}, using sources available by September 2026.
The inventory comprises 71 supported entries, 70 identified papers and 64 papers with eligible comparisons between trained methods.
Appendix~\ref{app:peft-review} details the criteria, review procedure and limitations.

Most papers do not clearly establish key experimental controls (\Cref{fig:peft-controls}).
Although shared checkpoints and data are documented in at least one main comparison in 63 and 62 of the 64 papers, respectively, only 25 establish all applicable comparison conditions together, including data exposure, adapted modules and trainable budgets.
Evidence for repeated training and learning-rate selection is even scarcer: only nine document repeated training of both methods with uncertainty attributable to those runs, and seven establish separate validation-based selection from at least three learning rates per method.
Most remaining cases are unclear in the reviewed sources.

\begin{figure}[t]
\centering
\includegraphics[width=\linewidth]{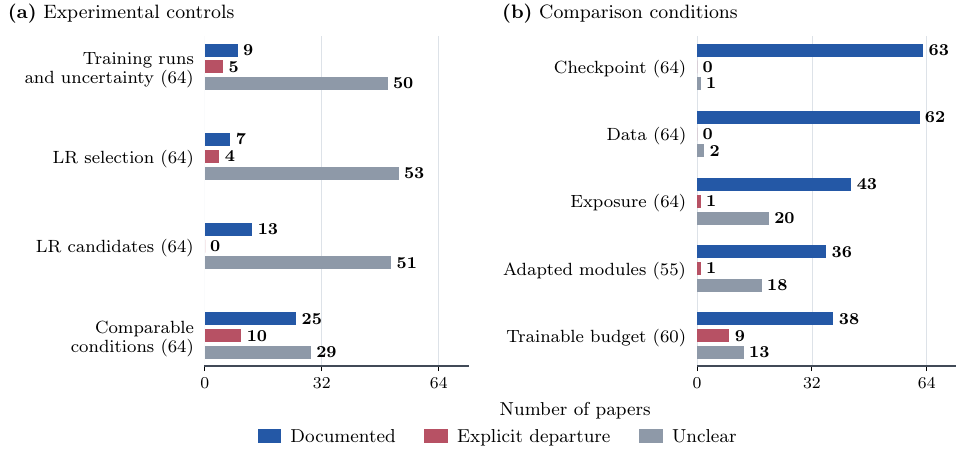}
\captionsetup{labelsep=period}
\caption{Experimental controls documented in PEFT method papers.
(a) Four controls in the 64 eligible papers.
(b) Components of comparable conditions, excluding inapplicable cases.
Each criterion counts a paper once. Denominators appear in parentheses. Bar lengths and labels show paper counts.
A control is documented when one eligible main comparison meets all applicable requirements.
An explicit departure requires evidence that every eligible comparison fails the definition.
Unclear covers missing or conflicting evidence and unresolved comparison coverage.
Component positives may come from different comparisons.
Definitions and coverage summaries appear in Appendix~\ref{app:peft-review}.}
\label{fig:peft-controls}
\end{figure}

These gaps matter because experimental conditions shape what a comparison can establish.
For example, gains with fewer trainable parameters can support an efficiency claim even when budgets differ.
Comparisons should therefore specify which conditions are comparable, select settings separately for each method and report uncertainty attributable to training runs.

For modest adaptation datasets, independent supervised PEFT runs can be much cheaper than pretraining or rollout-intensive post-training such as reinforcement learning with verifiable rewards (RLVR).
Method-specific learning-rate selection over an exponentially spaced grid and replication across training seeds should therefore be standard practice, with costs that are often manageable.
In RL, higher training costs often limit replication, but DataFlex-RL illustrates why stronger controls matter: its primary comparison of 13 configurations over 12 matched seeds under a common training recipe finds no statistically reliable gains from the tested data policies over their respective baselines \citep{liang2026dataflexrlevaluationplatformrlvr}.
These findings strengthen the case for controlled comparisons in PEFT, where replication is more affordable.
Accordingly, our study combines common experimental settings, method-specific tuning and multiple training seeds to examine task performance and retention.

\FloatBarrier
\section{Methods and study design}\label{sec:protocol}

We consider six representative PEFT algorithms: LoRA \citep{hu2022lora}, DoRA \citep{liu2024dora}, PiSSA \citep{meng2024pissa}, MiLoRA \citep{wang2025milora}, OFT \citep{qiu2023oft}, and HRA \citep{yuan2024hra}.

\paragraph{Parameterizations.}
The six methods span low-rank addition, magnitude--direction decomposition, spectral initialization and orthogonal transformations.
For pretrained $\bm W_0\in\mathbb R^{m\times n}$, \Cref{tab:parameterizations} summarizes the six methods. The low-rank factors are $\bm B\in\mathbb R^{m\times r}$ and $\bm A\in\mathbb R^{r\times n}$, with scale $s=\alpha/r$. PiSSA and MiLoRA freeze $\bm W_{\mathrm{res}}=\bm W_0-s\bm B_0\bm A_0$, preserving the starting model. For DoRA, $\mathcal N$ normalizes each row and $\bm m$ contains learned row magnitudes.
HRA uses $\bm u_i$ with $\|\bm u_i\|_2=1$.
OFT uses Cayley blocks $\bm R_j=(\bm I+\bm Q_j)(\bm I-\bm Q_j)^{-1}$ with skew-symmetric $\bm Q_j$.
Our experiments use the OFTv2 implementation, which approximates the Cayley inverse with a truncated Neumann series \citep{qiu2025oftv2}. The resulting transform is only approximately orthogonal. Appendix~\ref{app:native-orthogonality} quantifies this deviation, and all functional evaluations use the native implementation.

\begin{table}[!htbp]
\caption{Selected PEFT methods and their parameterizations.}
\label{tab:parameterizations}
\begin{center}
\begin{tabular}{@{}p{.12\linewidth}p{.39\linewidth}p{.43\linewidth}@{}}
\toprule
Method & Adapted matrix $\bm W_*$ & Initialization or structure \\
\midrule
\textcolor{LoRAColor}{LoRA} & $\bm W_0+s\bm B\bm A$ & Zero initial update \\
\textcolor{DoRAColor}{DoRA} & $\operatorname{diag}(\bm m)\mathcal N(\bm W_0+s\bm B\bm A)$ & Learned magnitudes and directions \\
\textcolor{PiSSAColor}{PiSSA} & $\bm W_{\mathrm{res}}+s\bm B\bm A$ & Leading singular components \\
\textcolor{MiLoRAColor}{MiLoRA} & $\bm W_{\mathrm{res}}+s\bm B\bm A$ & Trailing singular components \\
\textcolor{OFTColor}{OFT} & $\bm W_0\bm R$ & $\bm R=\operatorname{blkdiag}(\bm R_j)$ \\
\textcolor{HRAColor}{HRA} & $\bm W_0\bm R$ & $\bm R=\prod_{i=1}^{r_H}(\bm I-2\bm u_i\bm u_i^{\!\top})$ \\
\bottomrule
\end{tabular}
\end{center}
\end{table}

\paragraph{Experiment setup.}
We evaluate mathematics, coding and diffusion personalization, reporting means and sample SDs over three training seeds unless otherwise specified. Qwen2.5-7B \citep{qwen2024model} and Llama-3.1-8B \citep{dubey2024llama3} each train on 20,000 examples per task from MetaMath \citep{yu2024metamath} for mathematics and CodeFeedback \citep{zheng2024opencodeinterpreter} for coding. Methods share data exposure and adapted modules, with approximately matched parameter budgets. We select LRs by mean GSM8K \citep{cobbe2021gsm8k} development accuracy for mathematics or held-out response NLL for coding, then report held-out GSM8K accuracy and greedy HumanEval+/MBPP+ pass@1 \citep{liu2023evalplus}. FineWiki \citep{penedo2025finewiki} $\Delta\mathrm{NLL}=\mathrm{NLL}_*-\mathrm{NLL}_0$ measures forgetting, with lower values indicating better retention. For FLUX.2-klein-base-4B personalization \citep{bfl2026klein}, we compare all six methods on the benchmark cat at four approximately matched budgets and extend the LoRA/OFT comparison to 28 objects \citep{ruiz2023dreambooth}. Validation DINO \citep{oquab2023dinov2} selects LRs, while held-out test DINO measures fidelity. Base-image drift, CLIP alignment \citep{radford2021clip} and COCO flow loss \citep{lin2014coco} assess complementary retention outcomes. Spectral interventions use the same development/validation selection criteria across all six methods on mathematics and diffusion and jointly edit all adapted matrices. Each Pareto point represents one capacity/LR configuration.

\FloatBarrier
\section{Geometric preservation and functional consequences}\label{sec:mechanism}\label{sec:restoration}

We examine how PEFT methods change pretrained weight geometry and whether these changes are necessary for adaptation. We first explain singular-value and singular-basis changes theoretically, then measure spectral and hyperspherical energy preservation in checkpoints. Finally, we restore pretrained singular values to test whether adapted singular bases alone retain task gains.

\paragraph{Theoretical perspective.}
For $\bm W_0,\bm W_*\in\mathbb R^{m\times n}$, let $\bm W_j=\bm U_j\operatorname{diag}(\bm s_j)\bm V_j^{\!\top}$ denote thin SVDs with decreasing singular values, and let $O(k)$ denote the orthogonal group.
Relative weight and spectral displacement are defined by
$d_W={\|\bm W_*-\bm W_0\|_{\mathrm{F}}}/{\|\bm W_0\|_{\mathrm{F}}}$
and
$d_S={\|\bm s_*-\bm s_0\|_2}/{\|\bm s_0\|_2}$.

\begin{proposition}[Distance to the orthogonal orbit]\label{prop:orthogonal-distance}
The restored matrix $\bm W_R=\bm U_*\operatorname{diag}(\bm s_0)\bm V_*^{\!\top}$ satisfies
\begin{equation}\label{eq:orthogonal-distance}
\min_{\bm Q_L\in O(m),\,\bm Q_R\in O(n)}
\|\bm W_*-\bm Q_L\bm W_0\bm Q_R^{\!\top}\|_{\mathrm F}
=\|\bm W_*-\bm W_R\|_{\mathrm F}=\|\bm s_*-\bm s_0\|_2,
\end{equation}
and attains this minimum \citep{mirsky1975trace}. For nonzero base and update norms, $d_S/d_W$ is the distance to this orbit relative to the update norm.
\end{proposition}

A small ratio means that changing the singular bases while keeping pretrained singular values closely approximates the adapted weights, with error measured relative to the update norm.

\begin{proposition}[First-order spectral and basis motion]\label{prop:first-order-motion}
Let $q=\min(m,n)$ and suppose $\bm W_0$ has $q$ distinct positive singular values. For $\bm W_*(\epsilon)=\bm W_0+\epsilon\bm E$, set $\dot{\bm s}=\left.\mathrm d\bm s(\epsilon)/\mathrm d\epsilon\right|_{\epsilon=0}$. Then
\begin{equation}\label{eq:tangent-normal}
\dot s_i=\bm u_{0,i}^{\!\top}\bm E\bm v_{0,i},\qquad
\bm E_N=\bm U_0\operatorname{diag}(\dot{\bm s})\bm V_0^{\!\top},\qquad
\bm E_T=\bm E-\bm E_N
\end{equation}
are the spectral derivative \citep{stewart1990}, the normal component and the tangent component to the two-sided orthogonal orbit, respectively. In particular, $\bm E_T$ changes only singular bases to first order, and $\|\bm E\|_{\mathrm F}^2=\|\dot{\bm s}\|_2^2+\|\bm E_T\|_{\mathrm F}^2$. If $\|\bm E\|_{\mathrm F}=1$ and $\mathbb E[\operatorname{vec}(\bm E)\operatorname{vec}(\bm E)^{\!\top}]=\bm I_{mn}/(mn)$, then
\begin{equation}\label{eq:relative-spectral-scale}
\mathbb E\|\dot{\bm s}\|_2^2=\frac{1}{\max(m,n)},\qquad
\lim_{\epsilon\to0}\sqrt{\mathbb E\left[\left(\frac{d_S}{d_W}\right)^2\right]}=\frac{1}{\sqrt{\max(m,n)}}.
\end{equation}
\end{proposition}
Small isotropic updates therefore have RMS $d_S/d_W\approx1/\sqrt{\max(m,n)}$: spectral displacement can be a small share of weight displacement without an orthogonality constraint. This conditional prediction also holds for uniformly oriented low-rank updates (proofs and pooling in Appendix~\ref{app:geometry}).

\paragraph{Empirical evaluation.}
At the largest capacities, spectral displacement is a modest fraction of weight displacement even for unconstrained methods (\Cref{fig:overview}b): about 11\% for Llama MiLoRA versus 3--4\% for LoRA, DoRA, and PiSSA, and 5.8\% and 8.7\% for FLUX PiSSA and MiLoRA, respectively.
We also measure hyperspherical energy (HE), whose preservation motivates OFT's approach to adaptation \citep{qiu2023oft}. For nonzero rows with distinct normalized directions, $H(\bm W)=\sum_{i\ne j}\|\widehat{\bm w}_i-\widehat{\bm w}_j\|_2^{-1}$, where $\widehat{\bm w}_i=\bm w_i/\|\bm w_i\|_2$. Right-orthogonal transformations preserve hyperspherical energy exactly, whereas preserving singular values alone does not. Nevertheless, all six methods approximately preserve HE: every checkpoint's mean absolute relative change across adapted matrices is below $0.01\%$ at the largest capacities (Appendix~\ref{app:hyperspherical}).
LoRA-family methods therefore approximately preserve both quantities without enforcing orthogonality. Their small spectral displacement further supports a two-sided orthogonal approximation of adapted weights (Proposition~\ref{prop:orthogonal-distance}). To assess the functional importance of singular-basis changes relative to singular-value scaling, we restore pretrained singular values while retaining adapted bases \citep{he2025posttrainingstructure,jin2026rlheals}.
Across LLM and FLUX checkpoints, all points at the largest capacities remain within one percentage point of reconstruction (\Cref{fig:overview}b). See Appendix~\ref{app:complete-interventions} for additional results at other capacities.
\begin{figure}[t]\centering
\includegraphics[width=\linewidth]{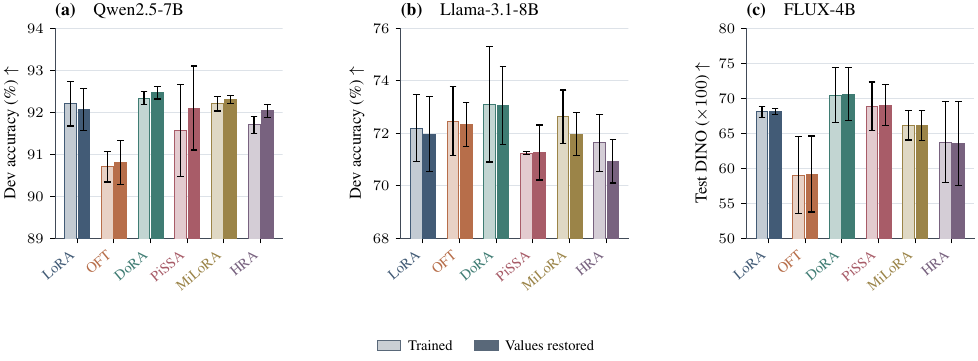}
\caption{Task performance largely survives singular-value restoration. At the largest capacities and development/validation-selected LRs, pale and saturated bars show trained and restored scores with sample SD error bars. LLM panels report development accuracy, and FLUX reports same-run test DINO for all six methods.}
\label{fig:restoration}
\end{figure}

\paragraph{Summary.}
In our settings, LoRA-family methods approximately preserve the spectra and hyperspherical energy that OFT is designed to preserve exactly. Singular-value restoration largely retains task performance. Tighter preservation must be judged by its task and retention outcomes.

\section{Competitive performance, recurring retention differences}\label{sec:comparisons}

In this section, we study task performance and forgetting across PEFT methods and ask whether closer geometric preservation leads to better adaptation--retention trade-offs. We also examine method-specific patterns across settings. 

\subsection{Language models: competitive performance, different retention}
We first examine whether comparable task performance in mathematics and coding is accompanied by comparable retention of general language modeling. We consider both development-selected test results and the mathematics performance--retention frontiers across learning rates and capacities.

\begin{table}[t]
\caption{Development-selected LLM test scores (\%) and FineWiki forgetting ($10^3\Delta\mathrm{NLL}$). Budgets are approximately matched: S/M/L for GSM8K; one for CodeFeedback. HE+ denotes HumanEval+. Base is one unadapted evaluation. Shading marks each column\textquotesingle s best adapted mean, including ties.}
\label{tab:math-test}\label{tab:coding-test}
\begin{center}
\textbf{(a) Mathematics: GSM8K test accuracy (\%, $\uparrow$)}\par\smallskip
\setlength{\tabcolsep}{5.9pt}\begin{tabular*}{\linewidth}{@{\extracolsep{\fill}}lcccccc}\toprule
& \multicolumn{3}{c}{Qwen2.5-7B} & \multicolumn{3}{c}{Llama-3.1-8B} \\ \cmidrule(lr){2-4}\cmidrule(l){5-7} Method & S & M & L & S & M & L \\ \midrule
\textcolor{LoRAColor}{LoRA} & $\result{math_grid_qwen_lora_7_0p001_test_mean}_{\pm\result{math_grid_qwen_lora_7_0p001_test_sd}}$ & $\result{math_grid_qwen_lora_14_1em05_test_mean}_{\pm\result{math_grid_qwen_lora_14_1em05_test_sd}}$ & $\result{math_grid_qwen_lora_28_3em05_test_mean}_{\pm\result{math_grid_qwen_lora_28_3em05_test_sd}}$ & $\result{math_grid_llama_lora_7_0p0003_test_mean}_{\pm\result{math_grid_llama_lora_7_0p0003_test_sd}}$ & \bestcell{LoRAColor}{$\bestmean{\result{math_grid_llama_lora_15_0p0003_test_mean}}_{\pm\result{math_grid_llama_lora_15_0p0003_test_sd}}$} & $\result{math_grid_llama_lora_30_0p0003_test_mean}_{\pm\result{math_grid_llama_lora_30_0p0003_test_sd}}$ \\
\textcolor{OFTColor}{OFT} & $\result{math_grid_qwen_oft_32_0p0003_test_mean}_{\pm\result{math_grid_qwen_oft_32_0p0003_test_sd}}$ & $\result{math_grid_qwen_oft_64_0p0003_test_mean}_{\pm\result{math_grid_qwen_oft_64_0p0003_test_sd}}$ & $\result{math_grid_qwen_oft_128_0p0001_test_mean}_{\pm\result{math_grid_qwen_oft_128_0p0001_test_sd}}$ & \bestcell{OFTColor}{$\bestmean{\result{math_grid_llama_oft_32_0p0003_test_mean}}_{\pm\result{math_grid_llama_oft_32_0p0003_test_sd}}$} & $\result{math_grid_llama_oft_64_0p0003_test_mean}_{\pm\result{math_grid_llama_oft_64_0p0003_test_sd}}$ & $\result{math_grid_llama_oft_128_0p0001_test_mean}_{\pm\result{math_grid_llama_oft_128_0p0001_test_sd}}$ \\
\textcolor{DoRAColor}{DoRA} & $\result{math_grid_qwen_dora_7_0p0003_test_mean}_{\pm\result{math_grid_qwen_dora_7_0p0003_test_sd}}$ & $\result{math_grid_qwen_dora_14_0p0003_test_mean}_{\pm\result{math_grid_qwen_dora_14_0p0003_test_sd}}$ & $\result{math_grid_qwen_dora_28_0p0001_test_mean}_{\pm\result{math_grid_qwen_dora_28_0p0001_test_sd}}$ & $\result{math_grid_llama_dora_7_0p0003_test_mean}_{\pm\result{math_grid_llama_dora_7_0p0003_test_sd}}$ & $\result{math_grid_llama_dora_15_0p0003_test_mean}_{\pm\result{math_grid_llama_dora_15_0p0003_test_sd}}$ & \bestcell{DoRAColor}{$\bestmean{\result{math_grid_llama_dora_30_0p0003_test_mean}}_{\pm\result{math_grid_llama_dora_30_0p0003_test_sd}}$} \\
\textcolor{PiSSAColor}{PiSSA} & \bestcell{PiSSAColor}{$\bestmean{\result{math_grid_qwen_pissa_7_3em05_test_mean}}_{\pm\result{math_grid_qwen_pissa_7_3em05_test_sd}}$} & \bestcell{PiSSAColor}{$\bestmean{\result{math_grid_qwen_pissa_14_3em05_test_mean}}_{\pm\result{math_grid_qwen_pissa_14_3em05_test_sd}}$} & $\result{math_grid_qwen_pissa_28_3em05_test_mean}_{\pm\result{math_grid_qwen_pissa_28_3em05_test_sd}}$ & $\result{math_grid_llama_pissa_7_0p0001_test_mean}_{\pm\result{math_grid_llama_pissa_7_0p0001_test_sd}}$ & $\result{math_grid_llama_pissa_15_0p0001_test_mean}_{\pm\result{math_grid_llama_pissa_15_0p0001_test_sd}}$ & $\result{math_grid_llama_pissa_30_3em05_test_mean}_{\pm\result{math_grid_llama_pissa_30_3em05_test_sd}}$ \\
\textcolor{MiLoRAColor}{MiLoRA} & $\result{math_grid_qwen_milora_7_1em05_test_mean}_{\pm\result{math_grid_qwen_milora_7_1em05_test_sd}}$ & \bestcell{MiLoRAColor}{$\bestmean{\result{math_grid_qwen_milora_14_1em05_test_mean}}_{\pm\result{math_grid_qwen_milora_14_1em05_test_sd}}$} & \bestcell{MiLoRAColor}{$\bestmean{\result{math_grid_qwen_milora_28_3em05_test_mean}}_{\pm\result{math_grid_qwen_milora_28_3em05_test_sd}}$} & $\result{math_grid_llama_milora_7_0p0003_test_mean}_{\pm\result{math_grid_llama_milora_7_0p0003_test_sd}}$ & $\result{math_grid_llama_milora_15_0p0003_test_mean}_{\pm\result{math_grid_llama_milora_15_0p0003_test_sd}}$ & $\result{math_grid_llama_milora_30_0p0003_test_mean}_{\pm\result{math_grid_llama_milora_30_0p0003_test_sd}}$ \\
\textcolor{HRAColor}{HRA} & $\result{math_grid_qwen_hra_16_0p0001_test_mean}_{\pm\result{math_grid_qwen_hra_16_0p0001_test_sd}}$ & $\result{math_grid_qwen_hra_32_0p0001_test_mean}_{\pm\result{math_grid_qwen_hra_32_0p0001_test_sd}}$ & $\result{math_grid_qwen_hra_64_1em05_test_mean}_{\pm\result{math_grid_qwen_hra_64_1em05_test_sd}}$ & $\result{math_grid_llama_hra_16_0p001_test_mean}_{\pm\result{math_grid_llama_hra_16_0p001_test_sd}}$ & $\result{math_grid_llama_hra_32_0p001_test_mean}_{\pm\result{math_grid_llama_hra_32_0p001_test_sd}}$ & $\result{math_grid_llama_hra_64_0p001_test_mean}_{\pm\result{math_grid_llama_hra_64_0p001_test_sd}}$ \\
\bottomrule\end{tabular*}
\par\medskip\textbf{(b) Coding: test pass@1 and general-text forgetting}\par\smallskip
\setlength{\tabcolsep}{3pt}\begin{tabular*}{\linewidth}{@{\extracolsep{\fill}}lcccccc}\toprule
& \multicolumn{3}{c}{Qwen2.5-7B} & \multicolumn{3}{c}{Llama-3.1-8B} \\
\cmidrule(lr){2-4}\cmidrule(l){5-7}
Method & HE+ $\uparrow$ & MBPP+ $\uparrow$ & $10^3\Delta$NLL $\downarrow$ & HE+ $\uparrow$ & MBPP+ $\uparrow$ & $10^3\Delta$NLL $\downarrow$ \\ \midrule
Base & $\result{ext_code_qwen_base_hep}$ & $\result{ext_code_qwen_base_mbp}$ & 0 & $\result{ext_code_llama_base_hep}$ & $\result{ext_code_llama_base_mbp}$ & 0 \\ \midrule
\textcolor{LoRAColor}{LoRA} & $\result{ext_code_qwen_lora_0p0001_hep_mean}_{\pm\result{ext_code_qwen_lora_0p0001_hep_sd}}$ & $\result{ext_code_qwen_lora_0p0001_mbp_mean}_{\pm\result{ext_code_qwen_lora_0p0001_mbp_sd}}$ & \bestcell{LoRAColor}{$\bestmean{\result{ext_code_qwen_lora_0p0001_nll_mean}}_{\pm\result{ext_code_qwen_lora_0p0001_nll_sd}}$} & $\result{ext_code_llama_lora_0p0003_hep_mean}_{\pm\result{ext_code_llama_lora_0p0003_hep_sd}}$ & $\result{ext_code_llama_lora_0p0003_mbp_mean}_{\pm\result{ext_code_llama_lora_0p0003_mbp_sd}}$ & \bestcell{LoRAColor}{$\bestmean{\result{ext_code_llama_lora_0p0003_nll_mean}}_{\pm\result{ext_code_llama_lora_0p0003_nll_sd}}$} \\
\textcolor{OFTColor}{OFT} & $\result{ext_code_qwen_oft_0p0001_hep_mean}_{\pm\result{ext_code_qwen_oft_0p0001_hep_sd}}$ & \bestcell{OFTColor}{$\bestmean{\result{ext_code_qwen_oft_0p0001_mbp_mean}}_{\pm\result{ext_code_qwen_oft_0p0001_mbp_sd}}$} & $\result{ext_code_qwen_oft_0p0001_nll_mean}_{\pm\result{ext_code_qwen_oft_0p0001_nll_sd}}$ & $\result{ext_code_llama_oft_0p0003_hep_mean}_{\pm\result{ext_code_llama_oft_0p0003_hep_sd}}$ & $\result{ext_code_llama_oft_0p0003_mbp_mean}_{\pm\result{ext_code_llama_oft_0p0003_mbp_sd}}$ & $\result{ext_code_llama_oft_0p0003_nll_mean}_{\pm\result{ext_code_llama_oft_0p0003_nll_sd}}$ \\
\textcolor{DoRAColor}{DoRA} & \bestcell{DoRAColor}{$\bestmean{\result{ext_code_qwen_dora_0p0001_hep_mean}}_{\pm\result{ext_code_qwen_dora_0p0001_hep_sd}}$} & $\result{ext_code_qwen_dora_0p0001_mbp_mean}_{\pm\result{ext_code_qwen_dora_0p0001_mbp_sd}}$ & $\result{ext_code_qwen_dora_0p0001_nll_mean}_{\pm\result{ext_code_qwen_dora_0p0001_nll_sd}}$ & $\result{ext_code_llama_dora_0p0003_hep_mean}_{\pm\result{ext_code_llama_dora_0p0003_hep_sd}}$ & \bestcell{DoRAColor}{$\bestmean{\result{ext_code_llama_dora_0p0003_mbp_mean}}_{\pm\result{ext_code_llama_dora_0p0003_mbp_sd}}$} & $\result{ext_code_llama_dora_0p0003_nll_mean}_{\pm\result{ext_code_llama_dora_0p0003_nll_sd}}$ \\
\textcolor{PiSSAColor}{PiSSA} & $\result{ext_code_qwen_pissa_1em05_hep_mean}_{\pm\result{ext_code_qwen_pissa_1em05_hep_sd}}$ & $\result{ext_code_qwen_pissa_1em05_mbp_mean}_{\pm\result{ext_code_qwen_pissa_1em05_mbp_sd}}$ & $\result{ext_code_qwen_pissa_1em05_nll_mean}_{\pm\result{ext_code_qwen_pissa_1em05_nll_sd}}$ & \bestcell{PiSSAColor}{$\bestmean{\result{ext_code_llama_pissa_0p0001_hep_mean}}_{\pm\result{ext_code_llama_pissa_0p0001_hep_sd}}$} & $\result{ext_code_llama_pissa_0p0001_mbp_mean}_{\pm\result{ext_code_llama_pissa_0p0001_mbp_sd}}$ & $\result{ext_code_llama_pissa_0p0001_nll_mean}_{\pm\result{ext_code_llama_pissa_0p0001_nll_sd}}$ \\
\textcolor{MiLoRAColor}{MiLoRA} & $\result{ext_code_qwen_milora_0p0001_hep_mean}_{\pm\result{ext_code_qwen_milora_0p0001_hep_sd}}$ & $\result{ext_code_qwen_milora_0p0001_mbp_mean}_{\pm\result{ext_code_qwen_milora_0p0001_mbp_sd}}$ & $\result{ext_code_qwen_milora_0p0001_nll_mean}_{\pm\result{ext_code_qwen_milora_0p0001_nll_sd}}$ & $\result{ext_code_llama_milora_0p0003_hep_mean}_{\pm\result{ext_code_llama_milora_0p0003_hep_sd}}$ & $\result{ext_code_llama_milora_0p0003_mbp_mean}_{\pm\result{ext_code_llama_milora_0p0003_mbp_sd}}$ & $\result{ext_code_llama_milora_0p0003_nll_mean}_{\pm\result{ext_code_llama_milora_0p0003_nll_sd}}$ \\
\textcolor{HRAColor}{HRA} & $\result{ext_code_qwen_hra_0p001_hep_mean}_{\pm\result{ext_code_qwen_hra_0p001_hep_sd}}$ & $\result{ext_code_qwen_hra_0p001_mbp_mean}_{\pm\result{ext_code_qwen_hra_0p001_mbp_sd}}$ & $\result{ext_code_qwen_hra_0p001_nll_mean}_{\pm\result{ext_code_qwen_hra_0p001_nll_sd}}$ & $\result{ext_code_llama_hra_0p001_hep_mean}_{\pm\result{ext_code_llama_hra_0p001_hep_sd}}$ & $\result{ext_code_llama_hra_0p001_mbp_mean}_{\pm\result{ext_code_llama_hra_0p001_mbp_sd}}$ & $\result{ext_code_llama_hra_0p001_nll_mean}_{\pm\result{ext_code_llama_hra_0p001_nll_sd}}$ \\
\bottomrule\end{tabular*}

\end{center}
\end{table}

\textbf{LoRA most consistently limits forgetting among all six methods.} Its mean NLL increase is lowest in \result{cross_lora_math_lowest} of 6 mathematics settings and both coding backbones (\Cref{tab:math-test}), while task performance remains competitive across settings. See Appendix~\ref{app:retention-patterns} for robustness checks. OFT forgets least for Llama at the largest capacity, with LoRA second. For Qwen at the smallest capacity, LoRA's high-LR selection forgets most and MiLoRA least.

\textbf{DoRA usually exceeds LoRA's mean task scores.} It leads in 10/14 comparisons: all four coding comparisons, all four cat budgets, and two mathematics settings (small Qwen and large Llama; \Cref{tab:math-test}, Appendix~\ref{app:retention-patterns}). DoRA's forgetting ranks among the lowest three methods in every selected LLM setting.

\begin{figure}[!htbp]\centering
\includegraphics[width=\linewidth]{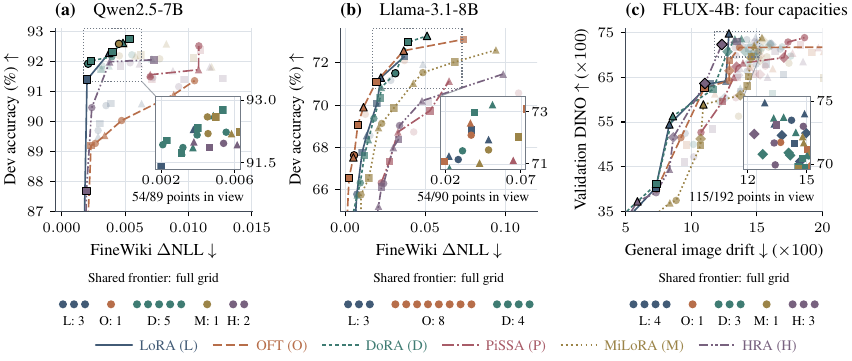}
\caption{Development performance--retention envelopes across LRs and capacities. Shapes denote increasing capacity: circle, square, triangle, then diamond (FLUX only). Faint points are dominated within their method. Line styles identify method frontiers. Black borders mark the shared frontier across methods. Each icon counts one configuration in the full grid (89 Qwen, 90 Llama, 192 FLUX). Insets magnify boxed regions with uniform point styling. (a,b) Dev accuracy versus FineWiki $\Delta$NLL. (c) Cat validation DINO versus general-image drift.}
\label{fig:operating-envelopes}
\end{figure}

\textbf{PiSSA frequently incurs greater retention costs.} It has the largest mean NLL increase on both coding backbones and half the mathematics settings, yet leads Llama HumanEval+. At matched LRs and capacities, PiSSA forgets most in every complete LLM comparison and constituent seed (further comparisons in Appendix~\ref{app:retention-patterns}).

PiSSA also illustrates why closer geometric preservation need not improve retention. At the largest Llama capacity, it has smaller HE and pairwise-cosine changes than LoRA (Appendix~\ref{app:hyperspherical}) but a greater NLL increase in every seed (means $\result{he_llama_pissa_forgetting}$ versus $\result{he_llama_lora_forgetting}$).

On the development frontiers, LoRA and DoRA also appear repeatedly, with DoRA reaching the highest accuracy on both backbones (\Cref{fig:operating-envelopes}a,b). OFT contributes many Llama frontier points, whereas PiSSA contributes none.

\subsection{Diffusion: recurring drift patterns, distinct retention metrics}
We next examine whether the retention patterns observed in language extend to diffusion personalization. We compare all six methods on cat personalization across parameter budgets, then extend the LoRA--OFT comparison across objects using complementary measures of retention.

\begin{table}[!htbp]
\caption{FLUX-4B test/retention scores. DINO and general drift/CLIP are $\times100$; COCO loss is unscaled. Panel (b) averages objects per seed. Shaded bold means mark the best method per budget/metric in (a), or capacity pair in (b). Budgets list rank/block sizes; HRA uses 16/32/62/124 vectors.}
\label{tab:diffusion}\label{tab:benchmark28}
\begin{center}
\textbf{(a) Cat: six methods at four approximately matched budgets}\par\smallskip
\setlength{\tabcolsep}{3pt}
\begin{tabular*}{\linewidth}{@{\extracolsep{\fill}}llcccccc}\toprule
Budget & Metric & \textcolor{LoRAColor}{LoRA} & \textcolor{OFTColor}{OFT} & \textcolor{DoRAColor}{DoRA} & \textcolor{PiSSAColor}{PiSSA} & \textcolor{MiLoRAColor}{MiLoRA} & \textcolor{HRAColor}{HRA} \\ \midrule
r4/b32 & DINO $\uparrow$ & $\result{ext_flux_flux_lora_cap4_0p0003_dino_mean}_{\pm\result{ext_flux_flux_lora_cap4_0p0003_dino_sd}}$ & $\result{ext_flux_flux_oft_cap32_0p0003_dino_mean}_{\pm\result{ext_flux_flux_oft_cap32_0p0003_dino_sd}}$ & \bestcell{DoRAColor}{$\bestmean{\result{ext_flux_flux_dora_cap4_0p0005_dino_mean}}_{\pm\result{ext_flux_flux_dora_cap4_0p0005_dino_sd}}$} & $\result{ext_flux_flux_pissa_cap4_0p0001_dino_mean}_{\pm\result{ext_flux_flux_pissa_cap4_0p0001_dino_sd}}$ & $\result{ext_flux_flux_milora_cap4_0p0001_dino_mean}_{\pm\result{ext_flux_flux_milora_cap4_0p0001_dino_sd}}$ & $\result{ext_flux_flux_hra_cap16_0p0005_dino_mean}_{\pm\result{ext_flux_flux_hra_cap16_0p0005_dino_sd}}$ \\
 & Drift $\downarrow$ & \bestcell{LoRAColor}{$\bestmean{\result{ext_flux_flux_lora_cap4_0p0003_drift_mean}}_{\pm\result{ext_flux_flux_lora_cap4_0p0003_drift_sd}}$} & $\result{ext_flux_flux_oft_cap32_0p0003_drift_mean}_{\pm\result{ext_flux_flux_oft_cap32_0p0003_drift_sd}}$ & $\result{ext_flux_flux_dora_cap4_0p0005_drift_mean}_{\pm\result{ext_flux_flux_dora_cap4_0p0005_drift_sd}}$ & $\result{ext_flux_flux_pissa_cap4_0p0001_drift_mean}_{\pm\result{ext_flux_flux_pissa_cap4_0p0001_drift_sd}}$ & $\result{ext_flux_flux_milora_cap4_0p0001_drift_mean}_{\pm\result{ext_flux_flux_milora_cap4_0p0001_drift_sd}}$ & $\result{ext_flux_flux_hra_cap16_0p0005_drift_mean}_{\pm\result{ext_flux_flux_hra_cap16_0p0005_drift_sd}}$ \\
 & CLIP $\uparrow$ & $\result{ext_flux_flux_lora_cap4_0p0003_clip_mean}_{\pm\result{ext_flux_flux_lora_cap4_0p0003_clip_sd}}$ & $\result{ext_flux_flux_oft_cap32_0p0003_clip_mean}_{\pm\result{ext_flux_flux_oft_cap32_0p0003_clip_sd}}$ & $\result{ext_flux_flux_dora_cap4_0p0005_clip_mean}_{\pm\result{ext_flux_flux_dora_cap4_0p0005_clip_sd}}$ & \bestcell{PiSSAColor}{$\bestmean{\result{ext_flux_flux_pissa_cap4_0p0001_clip_mean}}_{\pm\result{ext_flux_flux_pissa_cap4_0p0001_clip_sd}}$} & $\result{ext_flux_flux_milora_cap4_0p0001_clip_mean}_{\pm\result{ext_flux_flux_milora_cap4_0p0001_clip_sd}}$ & $\result{ext_flux_flux_hra_cap16_0p0005_clip_mean}_{\pm\result{ext_flux_flux_hra_cap16_0p0005_clip_sd}}$ \\
\addlinespace[2pt]
r8/b64 & DINO $\uparrow$ & $\result{ext_flux_flux_lora_0p0001_dino_mean}_{\pm\result{ext_flux_flux_lora_0p0001_dino_sd}}$ & $\result{ext_flux_flux_oft_5em05_dino_mean}_{\pm\result{ext_flux_flux_oft_5em05_dino_sd}}$ & $\result{ext_flux_flux_dora_0p0003_dino_mean}_{\pm\result{ext_flux_flux_dora_0p0003_dino_sd}}$ & $\result{ext_flux_flux_pissa_5em05_dino_mean}_{\pm\result{ext_flux_flux_pissa_5em05_dino_sd}}$ & $\result{ext_flux_flux_milora_5em05_dino_mean}_{\pm\result{ext_flux_flux_milora_5em05_dino_sd}}$ & \bestcell{HRAColor}{$\bestmean{\result{ext_flux_flux_hra_0p0003_dino_mean}}_{\pm\result{ext_flux_flux_hra_0p0003_dino_sd}}$} \\
 & Drift $\downarrow$ & $\result{ext_flux_flux_lora_0p0001_drift_mean}_{\pm\result{ext_flux_flux_lora_0p0001_drift_sd}}$ & $\result{ext_flux_flux_oft_5em05_drift_mean}_{\pm\result{ext_flux_flux_oft_5em05_drift_sd}}$ & $\result{ext_flux_flux_dora_0p0003_drift_mean}_{\pm\result{ext_flux_flux_dora_0p0003_drift_sd}}$ & $\result{ext_flux_flux_pissa_5em05_drift_mean}_{\pm\result{ext_flux_flux_pissa_5em05_drift_sd}}$ & $\result{ext_flux_flux_milora_5em05_drift_mean}_{\pm\result{ext_flux_flux_milora_5em05_drift_sd}}$ & \bestcell{HRAColor}{$\bestmean{\result{ext_flux_flux_hra_0p0003_drift_mean}}_{\pm\result{ext_flux_flux_hra_0p0003_drift_sd}}$} \\
 & CLIP $\uparrow$ & $\result{ext_flux_flux_lora_0p0001_clip_mean}_{\pm\result{ext_flux_flux_lora_0p0001_clip_sd}}$ & $\result{ext_flux_flux_oft_5em05_clip_mean}_{\pm\result{ext_flux_flux_oft_5em05_clip_sd}}$ & $\result{ext_flux_flux_dora_0p0003_clip_mean}_{\pm\result{ext_flux_flux_dora_0p0003_clip_sd}}$ & $\result{ext_flux_flux_pissa_5em05_clip_mean}_{\pm\result{ext_flux_flux_pissa_5em05_clip_sd}}$ & \bestcell{MiLoRAColor}{$\bestmean{\result{ext_flux_flux_milora_5em05_clip_mean}}_{\pm\result{ext_flux_flux_milora_5em05_clip_sd}}$} & $\result{ext_flux_flux_hra_0p0003_clip_mean}_{\pm\result{ext_flux_flux_hra_0p0003_clip_sd}}$ \\
\addlinespace[2pt]
r16/b128 & DINO $\uparrow$ & $\result{ext_flux_flux_lora_cap16_0p0001_dino_mean}_{\pm\result{ext_flux_flux_lora_cap16_0p0001_dino_sd}}$ & $\result{ext_flux_flux_oft_cap128_0p0001_dino_mean}_{\pm\result{ext_flux_flux_oft_cap128_0p0001_dino_sd}}$ & $\result{ext_flux_flux_dora_cap16_5em05_dino_mean}_{\pm\result{ext_flux_flux_dora_cap16_5em05_dino_sd}}$ & \bestcell{PiSSAColor}{$\bestmean{\result{ext_flux_flux_pissa_cap16_5em05_dino_mean}}_{\pm\result{ext_flux_flux_pissa_cap16_5em05_dino_sd}}$} & $\result{ext_flux_flux_milora_cap16_0p0001_dino_mean}_{\pm\result{ext_flux_flux_milora_cap16_0p0001_dino_sd}}$ & $\result{ext_flux_flux_hra_cap62_0p0005_dino_mean}_{\pm\result{ext_flux_flux_hra_cap62_0p0005_dino_sd}}$ \\
 & Drift $\downarrow$ & \bestcell{LoRAColor}{$\bestmean{\result{ext_flux_flux_lora_cap16_0p0001_drift_mean}}_{\pm\result{ext_flux_flux_lora_cap16_0p0001_drift_sd}}$} & $\result{ext_flux_flux_oft_cap128_0p0001_drift_mean}_{\pm\result{ext_flux_flux_oft_cap128_0p0001_drift_sd}}$ & $\result{ext_flux_flux_dora_cap16_5em05_drift_mean}_{\pm\result{ext_flux_flux_dora_cap16_5em05_drift_sd}}$ & $\result{ext_flux_flux_pissa_cap16_5em05_drift_mean}_{\pm\result{ext_flux_flux_pissa_cap16_5em05_drift_sd}}$ & $\result{ext_flux_flux_milora_cap16_0p0001_drift_mean}_{\pm\result{ext_flux_flux_milora_cap16_0p0001_drift_sd}}$ & $\result{ext_flux_flux_hra_cap62_0p0005_drift_mean}_{\pm\result{ext_flux_flux_hra_cap62_0p0005_drift_sd}}$ \\
 & CLIP $\uparrow$ & $\result{ext_flux_flux_lora_cap16_0p0001_clip_mean}_{\pm\result{ext_flux_flux_lora_cap16_0p0001_clip_sd}}$ & $\result{ext_flux_flux_oft_cap128_0p0001_clip_mean}_{\pm\result{ext_flux_flux_oft_cap128_0p0001_clip_sd}}$ & \bestcell{DoRAColor}{$\bestmean{\result{ext_flux_flux_dora_cap16_5em05_clip_mean}}_{\pm\result{ext_flux_flux_dora_cap16_5em05_clip_sd}}$} & $\result{ext_flux_flux_pissa_cap16_5em05_clip_mean}_{\pm\result{ext_flux_flux_pissa_cap16_5em05_clip_sd}}$ & $\result{ext_flux_flux_milora_cap16_0p0001_clip_mean}_{\pm\result{ext_flux_flux_milora_cap16_0p0001_clip_sd}}$ & $\result{ext_flux_flux_hra_cap62_0p0005_clip_mean}_{\pm\result{ext_flux_flux_hra_cap62_0p0005_clip_sd}}$ \\
\addlinespace[2pt]
r32/b256 & DINO $\uparrow$ & $\result{ext_flux_flux_lora_cap32_5em05_dino_mean}_{\pm\result{ext_flux_flux_lora_cap32_5em05_dino_sd}}$ & $\result{ext_flux_flux_oft_cap256_0p0001_dino_mean}_{\pm\result{ext_flux_flux_oft_cap256_0p0001_dino_sd}}$ & \bestcell{DoRAColor}{$\bestmean{\result{ext_flux_flux_dora_cap32_0p0001_dino_mean}}_{\pm\result{ext_flux_flux_dora_cap32_0p0001_dino_sd}}$} & $\result{ext_flux_flux_pissa_cap32_5em05_dino_mean}_{\pm\result{ext_flux_flux_pissa_cap32_5em05_dino_sd}}$ & $\result{ext_flux_flux_milora_cap32_0p0001_dino_mean}_{\pm\result{ext_flux_flux_milora_cap32_0p0001_dino_sd}}$ & $\result{ext_flux_flux_hra_cap124_0p0003_dino_mean}_{\pm\result{ext_flux_flux_hra_cap124_0p0003_dino_sd}}$ \\
 & Drift $\downarrow$ & \bestcell{LoRAColor}{$\bestmean{\result{ext_flux_flux_lora_cap32_5em05_drift_mean}}_{\pm\result{ext_flux_flux_lora_cap32_5em05_drift_sd}}$} & $\result{ext_flux_flux_oft_cap256_0p0001_drift_mean}_{\pm\result{ext_flux_flux_oft_cap256_0p0001_drift_sd}}$ & $\result{ext_flux_flux_dora_cap32_0p0001_drift_mean}_{\pm\result{ext_flux_flux_dora_cap32_0p0001_drift_sd}}$ & $\result{ext_flux_flux_pissa_cap32_5em05_drift_mean}_{\pm\result{ext_flux_flux_pissa_cap32_5em05_drift_sd}}$ & $\result{ext_flux_flux_milora_cap32_0p0001_drift_mean}_{\pm\result{ext_flux_flux_milora_cap32_0p0001_drift_sd}}$ & $\result{ext_flux_flux_hra_cap124_0p0003_drift_mean}_{\pm\result{ext_flux_flux_hra_cap124_0p0003_drift_sd}}$ \\
 & CLIP $\uparrow$ & \bestcell{LoRAColor}{$\bestmean{\result{ext_flux_flux_lora_cap32_5em05_clip_mean}}_{\pm\result{ext_flux_flux_lora_cap32_5em05_clip_sd}}$} & $\result{ext_flux_flux_oft_cap256_0p0001_clip_mean}_{\pm\result{ext_flux_flux_oft_cap256_0p0001_clip_sd}}$ & $\result{ext_flux_flux_dora_cap32_0p0001_clip_mean}_{\pm\result{ext_flux_flux_dora_cap32_0p0001_clip_sd}}$ & $\result{ext_flux_flux_pissa_cap32_5em05_clip_mean}_{\pm\result{ext_flux_flux_pissa_cap32_5em05_clip_sd}}$ & $\result{ext_flux_flux_milora_cap32_0p0001_clip_mean}_{\pm\result{ext_flux_flux_milora_cap32_0p0001_clip_sd}}$ & $\result{ext_flux_flux_hra_cap124_0p0003_clip_mean}_{\pm\result{ext_flux_flux_hra_cap124_0p0003_clip_sd}}$ \\
\bottomrule\end{tabular*}
\par\medskip\textbf{(b) 28 objects: four capacities, three LRs}\par\smallskip
\setlength{\tabcolsep}{3pt}
\begin{tabular*}{\linewidth}{@{\extracolsep{\fill}}lccccc}\toprule
Adapter & LR & Test DINO $\uparrow$ & COCO flow $\downarrow$ & General drift $\downarrow$ & General CLIP $\uparrow$\\\midrule
\loraname{} $r=4$ & $1\!\times\!10^{-4}$ & $\result{obj28_lora_4_1em04_test_mean}_{\pm\result{obj28_lora_4_1em04_test_sd}}$ & $\result{obj28_lora_4_1em04_flow_mean}_{\pm\result{obj28_lora_4_1em04_flow_sd}}$ & \bestcell{LoRAColor}{$\bestmean{\result{obj28_lora_4_1em04_drift_mean}}_{\pm\result{obj28_lora_4_1em04_drift_sd}}$} & \bestcell{LoRAColor}{$\bestmean{\result{obj28_lora_4_1em04_clip_mean}}_{\pm\result{obj28_lora_4_1em04_clip_sd}}$}\\
\loraname{} $r=8$ & $5\!\times\!10^{-5}$ & \bestcell{LoRAColor}{$\bestmean{\result{obj28_lora_8_5em05_test_mean}}_{\pm\result{obj28_lora_8_5em05_test_sd}}$} & $\result{obj28_lora_8_5em05_flow_mean}_{\pm\result{obj28_lora_8_5em05_flow_sd}}$ & \bestcell{LoRAColor}{$\bestmean{\result{obj28_lora_8_5em05_drift_mean}}_{\pm\result{obj28_lora_8_5em05_drift_sd}}$} & \bestcell{LoRAColor}{$\bestmean{\result{obj28_lora_8_5em05_clip_mean}}_{\pm\result{obj28_lora_8_5em05_clip_sd}}$}\\
\loraname{} $r=16$ & $3\!\times\!10^{-5}$ & $\result{obj28_lora_16_3em05_test_mean}_{\pm\result{obj28_lora_16_3em05_test_sd}}$ & \bestcell{LoRAColor}{$\bestmean{\result{obj28_lora_16_3em05_flow_mean}}_{\pm\result{obj28_lora_16_3em05_flow_sd}}$} & \bestcell{LoRAColor}{$\bestmean{\result{obj28_lora_16_3em05_drift_mean}}_{\pm\result{obj28_lora_16_3em05_drift_sd}}$} & \bestcell{LoRAColor}{$\bestmean{\result{obj28_lora_16_3em05_clip_mean}}_{\pm\result{obj28_lora_16_3em05_clip_sd}}$}\\
\loraname{} $r=32$ & $1\!\times\!10^{-4}$ & $\result{obj28_lora_32_1em04_test_mean}_{\pm\result{obj28_lora_32_1em04_test_sd}}$ & $\result{obj28_lora_32_1em04_flow_mean}_{\pm\result{obj28_lora_32_1em04_flow_sd}}$ & \bestcell{LoRAColor}{$\bestmean{\result{obj28_lora_32_1em04_drift_mean}}_{\pm\result{obj28_lora_32_1em04_drift_sd}}$} & \bestcell{LoRAColor}{$\bestmean{\result{obj28_lora_32_1em04_clip_mean}}_{\pm\result{obj28_lora_32_1em04_clip_sd}}$}\\
\midrule
\oftname{} $b=32$ & $1\!\times\!10^{-4}$ & \bestcell{OFTColor}{$\bestmean{\result{obj28_oft_32_1em04_test_mean}}_{\pm\result{obj28_oft_32_1em04_test_sd}}$} & \bestcell{OFTColor}{$\bestmean{\result{obj28_oft_32_1em04_flow_mean}}_{\pm\result{obj28_oft_32_1em04_flow_sd}}$} & $\result{obj28_oft_32_1em04_drift_mean}_{\pm\result{obj28_oft_32_1em04_drift_sd}}$ & $\result{obj28_oft_32_1em04_clip_mean}_{\pm\result{obj28_oft_32_1em04_clip_sd}}$\\
\oftname{} $b=64$ & $5\!\times\!10^{-5}$ & $\result{obj28_oft_64_5em05_test_mean}_{\pm\result{obj28_oft_64_5em05_test_sd}}$ & \bestcell{OFTColor}{$\bestmean{\result{obj28_oft_64_5em05_flow_mean}}_{\pm\result{obj28_oft_64_5em05_flow_sd}}$} & $\result{obj28_oft_64_5em05_drift_mean}_{\pm\result{obj28_oft_64_5em05_drift_sd}}$ & $\result{obj28_oft_64_5em05_clip_mean}_{\pm\result{obj28_oft_64_5em05_clip_sd}}$\\
\oftname{} $b=128$ & $1\!\times\!10^{-4}$ & \bestcell{OFTColor}{$\bestmean{\result{obj28_oft_128_1em04_test_mean}}_{\pm\result{obj28_oft_128_1em04_test_sd}}$} & $\result{obj28_oft_128_1em04_flow_mean}_{\pm\result{obj28_oft_128_1em04_flow_sd}}$ & $\result{obj28_oft_128_1em04_drift_mean}_{\pm\result{obj28_oft_128_1em04_drift_sd}}$ & $\result{obj28_oft_128_1em04_clip_mean}_{\pm\result{obj28_oft_128_1em04_clip_sd}}$\\
\oftname{} $b=256$ & $5\!\times\!10^{-5}$ & \bestcell{OFTColor}{$\bestmean{\result{obj28_oft_256_5em05_test_mean}}_{\pm\result{obj28_oft_256_5em05_test_sd}}$} & \bestcell{OFTColor}{$\bestmean{\result{obj28_oft_256_5em05_flow_mean}}_{\pm\result{obj28_oft_256_5em05_flow_sd}}$} & $\result{obj28_oft_256_5em05_drift_mean}_{\pm\result{obj28_oft_256_5em05_drift_sd}}$ & $\result{obj28_oft_256_5em05_clip_mean}_{\pm\result{obj28_oft_256_5em05_clip_sd}}$\\
\bottomrule\end{tabular*}

\end{center}
\end{table}

\textbf{The retention pattern recurs across cat budgets.} LoRA has the lowest mean general drift at \result{cross_cat_lora_lowest} of \result{cross_cat_selected_settings} validation-selected budgets and the second lowest at the other (\Cref{tab:diffusion}a). DoRA has higher mean test DINO than LoRA at every budget, with additional drift, and leads fidelity at the smallest and largest budgets. PiSSA exceeds LoRA's drift in all \result{cross_cat_pissa_above_lora_seeds} paired seeds; its medium-budget fidelity lead again accompanies a retention cost.

HRA has the highest mean test fidelity and lowest drift at the small budget and extends the pooled frontier, but its advantage does not persist at larger capacities. LoRA and DoRA also contribute; PiSSA contributes no frontier point (\Cref{fig:operating-envelopes}c). MiLoRA has lower mean fidelity and higher drift than LoRA at every selected budget.

Across 28 objects, LoRA and OFT have close test fidelity. LoRA has lower image drift and higher CLIP alignment at all four selected capacity pairs, while OFT has lower COCO flow loss at three (\Cref{tab:diffusion}b). These metrics distinguish preservation of generated behavior from preservation of denoising predictions at fixed noisy states.

\paragraph{Summary.}
Although task scores are often close, LoRA most consistently limits forgetting among all six methods at competitive performance. DoRA usually exceeds LoRA's mean task scores, while PiSSA often forgets more. Other methods' performance and forgetting depend on the setting. The orthogonal methods OFT and HRA offer no consistent performance--retention advantage over the four LoRA-family methods across the tested settings.

\section{Spectrally Heterogeneous Adaptation, Structured Forgetting}\label{sec:sensitivity}
We examine how dominant, intermediate, and trailing singular components contribute to adaptation and forgetting. Using pretrained singular values throughout, we selectively restore or retain adapted singular vectors to measure each group's contribution.

\subsection{Restoring groups of spectral components}
For a component group $B$, we restore its pretrained singular vectors and retain adapted vectors elsewhere:
\begin{equation}\label{eq:band-restore}
\bm W_{\mathrm{restore}(B)}=\bm W_{\mathrm{keep}(B^{\mathrm c})}
=\sum_{i\in B}s_{0,i}\bm u_{0,i}\bm v_{0,i}^{\!\top}
 +\sum_{i\notin B}s_{0,i}\bm u_{*,i}\bm v_{*,i}^{\!\top}.
\end{equation}

\begin{figure}[!htbp]\centering
\includegraphics[width=\linewidth]{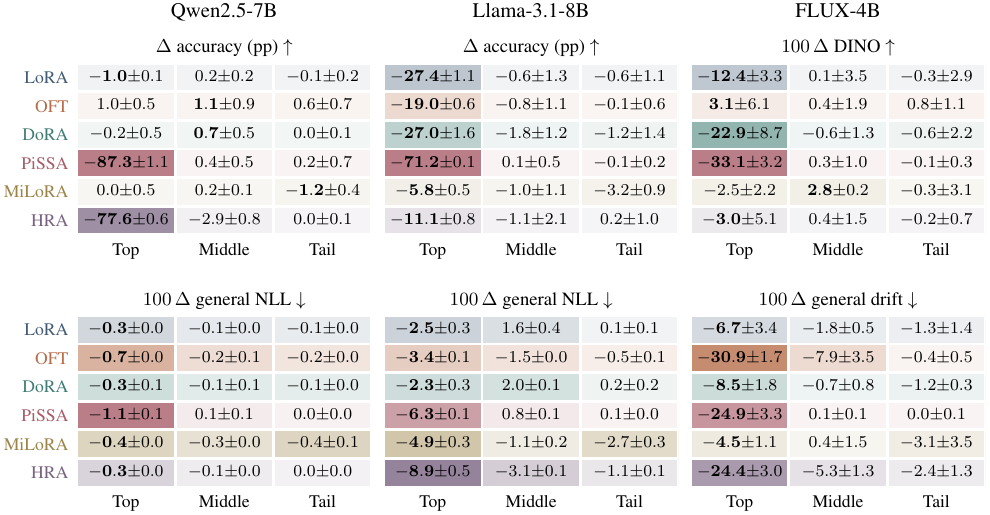}
\caption{Dominant-component restoration consistently reduces forgetting, with varying effects on task performance. Cells show mean $\pm$ SD changes from trained models at the largest capacities and selected LRs. Higher task scores and lower NLL/drift are better. FLUX DINO/drift are scaled by 100. Bold marks the largest absolute mean change across component groups per method.}
\label{fig:band-response}
\end{figure}

\paragraph{Restoring dominant components consistently reduces forgetting.}
At the largest capacities, restoring dominant components gives the largest mean reduction in general-text NLL or general-image drift for every method (\Cref{fig:band-response}). This ordering holds across language capacities (Appendix~\ref{app:complete-interventions}). MiLoRA shares this pattern despite greater spectral drift concentration in trailing components.

\paragraph{The same restoration has different effects on task performance.}
At the largest language capacities, restoring dominant components largely preserves Qwen accuracy for LoRA, OFT, DoRA and MiLoRA, whereas PiSSA and HRA suffer large losses. The same edit reduces LoRA and DoRA accuracy by about 27 percentage points on Llama. PiSSA nearly loses all recorded accuracy on both models. In FLUX, PiSSA, LoRA and DoRA all incur substantial losses. OFT improves on average, with substantial seed variation. HRA and MiLoRA incur smaller mean losses. Thus, restoring dominant components consistently reduces forgetting, but its effect on task performance shows no consistent pattern across methods and models.

\subsection{Which components retain task performance on their own?}
Retaining only adapted dominant components preserves PiSSA's language accuracy within 1 percentage point of the full value-restored model $\bm W_R$ across every tested setting, whereas retaining only the remaining components does not. Other methods show no consistent sufficiency pattern across settings (Appendix~\ref{app:complete-interventions}, \Cref{fig:component-sufficiency,tab:capacity-sufficiency}).

\paragraph{Summary.}
Task dependence varies across settings: PiSSA’s adapted dominant components suffice across the tested settings, whereas other methods may require intermediate and trailing components. Restoring the dominant components consistently yields the largest mean reduction in NLL and drift, suggesting that forgetting is systematically associated with changes in dominant components.

\section{Discussion}
Across the tested settings, our comparison of selected PEFT methods finds no consistent performance--retention advantage from explicit geometric preservation. With method-specific learning-rate tuning, LoRA most consistently limits forgetting while maintaining competitive task performance. Restoring pretrained singular values largely preserves task performance, while restoring dominant components consistently yields the largest mean reduction in forgetting, although its effects on task performance vary. These findings highlight the importance of evaluating adaptation and retention jointly. In addition, shared learning rates and single-seed results can bias comparisons, underscoring the need for method-specific tuning and evaluation across multiple seeds. We hope this study supports more reliable evaluations, informs method selection, and guides PEFT designs that better balance adaptation with the retention of pretrained capabilities.
\label{paper:mainend}
\section*{Acknowledgments}
We thank \href{https://zulip.com/}{Zulip} for the project organization tool and \href{https://warwick.ac.uk/research/rtp/sc/sulis/}{Sulis} for computational resources.

\bibliography{references}

@article{huang2026peftarena,
      title={{PEFT-Arena}: Understanding Parameter-Efficient Finetuning from a Stability-Plasticity Perspective},
      author={Huang, Yangyi and Peng, Ruotian and Qiu, Zeju and Kang, Jiale and Wen, Yandong and Sch\"olkopf, Bernhard and Liu, Weiyang},
      journal={arXiv preprint arXiv:2605.28819},
      year={2026}}

@inproceedings{
hu2022lora,
title={Lo{RA}: Low-Rank Adaptation of Large Language Models},
author={Edward J Hu and Yelong Shen and Phillip Wallis and Zeyuan Allen-Zhu and Yuanzhi Li and Shean Wang and Lu Wang and Weizhu Chen},
booktitle={International Conference on Learning Representations},
year={2022},
url={https://openreview.net/forum?id=nZeVKeeFYf9}
}

@inproceedings{qiu2023oft,
 author = {Qiu, Zeju and Liu, Weiyang and Feng, Haiwen and Xue, Yuxuan and Feng, Yao and Liu, Zhen and Zhang, Dan and Weller, Adrian and Sch\"{o}lkopf, Bernhard},
 booktitle = {Advances in Neural Information Processing Systems},
 doi = {10.52202/075280-3472},
 editor = {A. Oh and T. Naumann and A. Globerson and K. Saenko and M. Hardt and S. Levine},
 pages = {79320--79362},
 publisher = {Curran Associates, Inc.},
 title = {Controlling Text-to-Image Diffusion by Orthogonal Finetuning},
 url = {https://proceedings.neurips.cc/paper_files/paper/2023/file/faacb7a4827b4d51e201666b93ab5fa7-Paper-Conference.pdf},
 volume = {36},
 year = {2023}
}

@InProceedings{liu2024dora,
  title = 	 {{D}o{RA}: Weight-Decomposed Low-Rank Adaptation},
  author =       {Liu, Shih-Yang and Wang, Chien-Yi and Yin, Hongxu and Molchanov, Pavlo and Wang, Yu-Chiang Frank and Cheng, Kwang-Ting and Chen, Min-Hung},
  booktitle = 	 {Proceedings of the 41st International Conference on Machine Learning},
  pages = 	 {32100--32121},
  year = 	 {2024},
  editor = 	 {Salakhutdinov, Ruslan and Kolter, Zico and Heller, Katherine and Weller, Adrian and Oliver, Nuria and Scarlett, Jonathan and Berkenkamp, Felix},
  volume = 	 {235},
  series = 	 {Proceedings of Machine Learning Research},
  month = 	 {21--27 Jul},
  publisher =    {PMLR},
  url = 	 {https://proceedings.mlr.press/v235/liu24bn.html}
}

@inproceedings{qiu2025oftv2,
    title = "Orthogonal Finetuning Made Scalable",
    author = {Qiu, Zeju  and
      Liu, Weiyang  and
      Weller, Adrian  and
      Sch{\"o}lkopf, Bernhard},
    editor = "Christodoulopoulos, Christos  and
      Chakraborty, Tanmoy  and
      Rose, Carolyn  and
      Peng, Violet",
    booktitle = "Proceedings of the 2025 Conference on Empirical Methods in Natural Language Processing",
    month = nov,
    year = "2025",
    address = "Suzhou, China",
    publisher = "Association for Computational Linguistics",
    url = "https://aclanthology.org/2025.emnlp-main.1627/",
    doi = "10.18653/v1/2025.emnlp-main.1627",
    pages = "31946--31963",
    ISBN = "979-8-89176-332-6"
}

@InProceedings{hayou2024loraplus,
  title = 	 {{L}o{RA}+: Efficient Low Rank Adaptation of Large Models},
  author =       {Hayou, Soufiane and Ghosh, Nikhil and Yu, Bin},
  booktitle = 	 {Proceedings of the 41st International Conference on Machine Learning},
  pages = 	 {17783--17806},
  year = 	 {2024},
  editor = 	 {Salakhutdinov, Ruslan and Kolter, Zico and Heller, Katherine and Weller, Adrian and Oliver, Nuria and Scarlett, Jonathan and Berkenkamp, Felix},
  volume = 	 {235},
  series = 	 {Proceedings of Machine Learning Research},
  month = 	 {21--27 Jul},
  publisher =    {PMLR},
  url = 	 {https://proceedings.mlr.press/v235/hayou24a.html}
}

@misc{kalajdzievski2023rslora,
  doi = {10.48550/ARXIV.2312.03732},
  url = {https://arxiv.org/abs/2312.03732},
  author = {Kalajdzievski, Damjan},
  title = {A Rank Stabilization Scaling Factor for Fine-Tuning with {LoRA}},
  publisher = {arXiv},
  year = {2023},
  copyright = {arXiv.org perpetual, non-exclusive license}
}

@inproceedings{wang2024loraga,
 author = {Wang, Shaowen and Yu, Linxi and Li, Jian},
 booktitle = {Advances in Neural Information Processing Systems},
 doi = {10.52202/079017-1741},
 editor = {A. Globerson and L. Mackey and D. Belgrave and A. Fan and U. Paquet and J. Tomczak and C. Zhang},
 pages = {54905--54931},
 publisher = {Curran Associates, Inc.},
 title = {{LoRA-GA}: Low-Rank Adaptation with Gradient Approximation},
 url = {https://proceedings.neurips.cc/paper_files/paper/2024/file/62c4718cc334f6a0a62fb81c4a2095a1-Paper-Conference.pdf},
 volume = {37},
 year = {2024}
}

@article{
biderman2024loraforgetting,
title={Lo{RA} Learns Less and Forgets Less},
author={Dan Biderman and Jacob Portes and Jose Javier Gonzalez Ortiz and Mansheej Paul and Philip Greengard and Connor Jennings and Daniel King and Sam Havens and Vitaliy Chiley and Jonathan Frankle and Cody Blakeney and John Patrick Cunningham},
journal={Transactions on Machine Learning Research},
issn={2835-8856},
year={2024},
url={https://openreview.net/forum?id=aloEru2qCG},
note={Featured Certification}
}

@InProceedings{ruiz2023dreambooth,
    author    = {Ruiz, Nataniel and Li, Yuanzhen and Jampani, Varun and Pritch, Yael and Rubinstein, Michael and Aberman, Kfir},
    title     = {{DreamBooth}: Fine Tuning Text-to-Image Diffusion Models for Subject-Driven Generation},
    booktitle = {Proceedings of the IEEE/CVF Conference on Computer Vision and Pattern Recognition (CVPR)},
    month     = {June},
    year      = {2023},
    pages     = {22500-22510}
}

@article{
oquab2023dinov2,
title={{DINO}v2: Learning Robust Visual Features without Supervision},
author={Maxime Oquab and Timoth{\'e}e Darcet and Th{\'e}o Moutakanni and Huy V. Vo and Marc Szafraniec and Vasil Khalidov and Pierre Fernandez and Daniel HAZIZA and Francisco Massa and Alaaeldin El-Nouby and Mido Assran and Nicolas Ballas and Wojciech Galuba and Russell Howes and Po-Yao Huang and Shang-Wen Li and Ishan Misra and Michael Rabbat and Vasu Sharma and Gabriel Synnaeve and Hu Xu and Herve Jegou and Julien Mairal and Patrick Labatut and Armand Joulin and Piotr Bojanowski},
journal={Transactions on Machine Learning Research},
issn={2835-8856},
year={2024},
url={https://openreview.net/forum?id=a68SUt6zFt},
note={Featured Certification}
}

@inproceedings{yu2024metamath,
 author = {Yu, Longhui and JIANG, Weisen and Shi, Han and YU, Jincheng and Liu, Zhengying and Zhang, Yu and Kwok, James and Li, Zhenguo and Weller, Adrian and Liu, Weiyang},
 booktitle = {International Conference on Learning Representations},
 editor = {B. Kim and Y. Yue and S. Chaudhuri and K. Fragkiadaki and M. Khan and Y. Sun},
 pages = {45040--45061},
 title = {{MetaMath}: Bootstrap Your Own Mathematical Questions for Large Language Models},
 url = {https://proceedings.iclr.cc/paper_files/paper/2024/file/c400474e8a36d0812fdee52739288b12-Paper-Conference.pdf},
 volume = {2024},
 year = {2024}
}

@misc{bfl2026klein,
 title={{FLUX.2-klein-base-4B} Model Card},author={{Black Forest Labs}},year={2026},
 howpublished={Official model release},url={https://huggingface.co/black-forest-labs/FLUX.2-klein-base-4B}
}

@InProceedings{radford2021clip,
  title = 	 {Learning Transferable Visual Models From Natural Language Supervision},
  author =       {Radford, Alec and Kim, Jong Wook and Hallacy, Chris and Ramesh, Aditya and Goh, Gabriel and Agarwal, Sandhini and Sastry, Girish and Askell, Amanda and Mishkin, Pamela and Clark, Jack and Krueger, Gretchen and Sutskever, Ilya},
  booktitle = 	 {Proceedings of the 38th International Conference on Machine Learning},
  pages = 	 {8748--8763},
  year = 	 {2021},
  editor = 	 {Meila, Marina and Zhang, Tong},
  volume = 	 {139},
  series = 	 {Proceedings of Machine Learning Research},
  month = 	 {18--24 Jul},
  publisher =    {PMLR},
  url = 	 {https://proceedings.mlr.press/v139/radford21a.html}
}

@misc{qwen2024model,
    title = {Qwen2.5: A Party of Foundation Models},
    url = {https://qwenlm.github.io/blog/qwen2.5/},
    author = {{Qwen Team}},
    month = {September},
    year = {2024}
}

@InProceedings{lin2014coco,
author="Lin, Tsung-Yi
and Maire, Michael
and Belongie, Serge
and Hays, James
and Perona, Pietro
and Ramanan, Deva
and Doll{\'a}r, Piotr
and Zitnick, C. Lawrence",
editor="Fleet, David
and Pajdla, Tomas
and Schiele, Bernt
and Tuytelaars, Tinne",
title="Microsoft {COCO}: Common Objects in Context",
booktitle="Computer Vision -- ECCV 2014",
year="2014",
publisher="Springer International Publishing",
address="Cham",
pages="740--755",
isbn="978-3-319-10602-1"
}

@misc{penedo2025finewiki,
  author    = {Guilherme Penedo},
  title     = {{FineWiki}},
  year      = {2025},
  publisher = {Hugging Face Datasets},
  url       = {https://huggingface.co/datasets/HuggingFaceFW/finewiki},
  urldate   = {2025-10-20},
  note      = {Source: Wikimedia Enterprise Snapshot API (https://api.enterprise.wikimedia.com/v2/snapshots). Text licensed under CC BY-SA 4.0 with attribution to Wikipedia contributors.}
}

@inproceedings{lingam2024svft,
 author = {Lingam, Vijay and Tejaswi, Atula and Vavre, Aditya and Shetty, Aneesh and Gudur, Gautham Krishna and Ghosh, Joydeep and Dimakis, Alex and Choi, Eunsol and Bojchevski, Aleksandar and Sanghavi, Sujay},
 booktitle = {Advances in Neural Information Processing Systems},
 doi = {10.52202/079017-1311},
 editor = {A. Globerson and L. Mackey and D. Belgrave and A. Fan and U. Paquet and J. Tomczak and C. Zhang},
 pages = {41425--41446},
 publisher = {Curran Associates, Inc.},
 title = {{SVFT}: Parameter-Efficient Fine-Tuning with Singular Vectors},
 url = {https://proceedings.neurips.cc/paper_files/paper/2024/file/48c368f105e8145b945227b73255635a-Paper-Conference.pdf},
 volume = {37},
 year = {2024}
}

@inproceedings{shuttleworth2024illusion,
 author = {Shuttleworth, Reece and Andreas, Jacob and Torralba, Antonio and Sharma, Pratyusha},
 booktitle = {Advances in Neural Information Processing Systems},
 doi = {10.52202/085713-5808},
 editor = {D. Belgrave and C. Zhang and H. Lin and R. Pascanu and P. Koniusz and M. Ghassemi and N. Chen},
 pages = {174627--174662},
 publisher = {Curran Associates, Inc.},
 title = {{LoRA} vs Full Fine-tuning: An Illusion of Equivalence},
 url = {https://proceedings.neurips.cc/paper_files/paper/2025/file/ff541950d1e885af90f523571564a401-Paper-Conference.pdf},
 volume = {38, Main Conference},
 year = {2025}
}

@misc{he2025posttrainingstructure,
  doi = {10.48550/ARXIV.2509.17866},
  url = {https://arxiv.org/abs/2509.17866},
  author = {He, Xinyu and Cao, Xianghui},
  title = {Understanding Post-Training Structural Changes in Large Language Models},
  publisher = {arXiv},
  year = {2025},
  copyright = {arXiv.org perpetual, non-exclusive license}
}

@inproceedings{wang2026casa,
  title = {Exploring Data-Free {LoRA} Transferability for Video Diffusion Models},
  author = {Wang, Yuchen and Zhong, Wenliang and Bai, Lichen and Zhou, Zikai and Shao, Shitong and Cheng, Bojun and Chen, Shuo and Yang, Shuo and Xie, Zeke},
  booktitle = {International Conference on Machine Learning},
  year = {2026},
  url = {https://icml.cc/virtual/2026/poster/64894}
}

@inproceedings{qiu2025poet,
 author = {Qiu, Zeju and Buchholz, Simon and Xiao, Tim and Dax, Maximilian and Sch\"{o}lkopf, Bernhard and Liu, Weiyang},
 booktitle = {Advances in Neural Information Processing Systems},
 doi = {10.52202/085713-4704},
 editor = {D. Belgrave and C. Zhang and H. Lin and R. Pascanu and P. Koniusz and M. Ghassemi and N. Chen},
 pages = {140775--140821},
 publisher = {Curran Associates, Inc.},
 title = {Reparameterized {LLM} Training via Orthogonal Equivalence Transformation},
 url = {https://proceedings.neurips.cc/paper_files/paper/2025/file/cde43c5d698076ae17f3838a80f5a1bc-Paper-Conference.pdf},
 volume = {38, Main Conference},
 year = {2025}
}

@inproceedings{meng2024pissa,
 author = {Meng, Fanxu and Wang, Zhaohui and Zhang, Muhan},
 booktitle = {Advances in Neural Information Processing Systems},
 doi = {10.52202/079017-3846},
 editor = {A. Globerson and L. Mackey and D. Belgrave and A. Fan and U. Paquet and J. Tomczak and C. Zhang},
 pages = {121038--121072},
 publisher = {Curran Associates, Inc.},
 title = {{PiSSA}: Principal Singular Values and Singular Vectors Adaptation of Large Language Models},
 url = {https://proceedings.neurips.cc/paper_files/paper/2024/file/db36f4d603cc9e3a2a5e10b93e6428f2-Paper-Conference.pdf},
 volume = {37},
 year = {2024}
}

@techreport{stewart1990,
  author = {G. W. Stewart},
  title = {Perturbation Theory for the Singular Value Decomposition},
  institution = {University of Maryland, Institute for Advanced Computer Studies},
  number = {UMIACS-TR-90-124, CS-TR-2539},
  year = {1990},
  url = {https://api.drum.lib.umd.edu/server/api/core/bitstreams/a38cc6db-9bb7-4f3b-ae43-55b5827bb9d4/content}
}

@inproceedings{yuan2024hra,
 author = {Yuan, Shen and Liu, Haotian and Xu, Hongteng},
 booktitle = {Advances in Neural Information Processing Systems},
 doi = {10.52202/079017-3606},
 editor = {A. Globerson and L. Mackey and D. Belgrave and A. Fan and U. Paquet and J. Tomczak and C. Zhang},
 pages = {113484--113518},
 publisher = {Curran Associates, Inc.},
 title = {Bridging The Gap between Low-rank and Orthogonal Adaptation via {Householder} Reflection Adaptation},
 url = {https://proceedings.neurips.cc/paper_files/paper/2024/file/cdd0640218a27e9e2c0e52e324e25db0-Paper-Conference.pdf},
 volume = {37},
 year = {2024}
}

@inproceedings{wang2025milora,
    title = "{M}i{L}o{RA}: Harnessing Minor Singular Components for Parameter-Efficient {LLM} Finetuning",
    author = "Wang, Hanqing  and
      Li, Yixia  and
      Wang, Shuo  and
      Chen, Guanhua  and
      Chen, Yun",
    editor = "Chiruzzo, Luis  and
      Ritter, Alan  and
      Wang, Lu",
    booktitle = "Proceedings of the 2025 Conference of the Nations of the Americas Chapter of the Association for Computational Linguistics: Human Language Technologies (Volume 1: Long Papers)",
    month = apr,
    year = "2025",
    address = "Albuquerque, New Mexico",
    publisher = "Association for Computational Linguistics",
    url = "https://aclanthology.org/2025.naacl-long.248/",
    doi = "10.18653/v1/2025.naacl-long.248",
    pages = "4823--4836",
    ISBN = "979-8-89176-189-6"
}

@article{schulman2025lora,
  author  = {Schulman, John and {Thinking Machines Lab}},
  title   = {{LoRA} Without Regret},
  journal = {Thinking Machines Lab: Connectionism},
  year    = {2025},
  month   = sep,
  doi     = {10.64434/tml.20250929},
  url     = {https://thinkingmachines.ai/blog/lora/}
}

@inproceedings{liu2023evalplus,
 author = {Liu, Jiawei and Xia, Chunqiu Steven and Wang, Yuyao and ZHANG, LINGMING},
 booktitle = {Advances in Neural Information Processing Systems},
 doi = {10.52202/075280-0943},
 editor = {A. Oh and T. Naumann and A. Globerson and K. Saenko and M. Hardt and S. Levine},
 pages = {21558--21572},
 publisher = {Curran Associates, Inc.},
 title = {Is Your Code Generated by {ChatGPT} Really Correct? {Rigorous} Evaluation of Large Language Models for Code Generation},
 url = {https://proceedings.neurips.cc/paper_files/paper/2023/file/43e9d647ccd3e4b7b5baab53f0368686-Paper-Conference.pdf},
 volume = {36},
 year = {2023}
}

@inproceedings{zheng2024opencodeinterpreter,
    title = "{O}pen{C}ode{I}nterpreter: Integrating Code Generation with Execution and Refinement",
    author = "Zheng, Tianyu  and
      Zhang, Ge  and
      Shen, Tianhao  and
      Liu, Xueling  and
      Lin, Bill Yuchen  and
      Fu, Jie  and
      Chen, Wenhu  and
      Yue, Xiang",
    editor = "Ku, Lun-Wei  and
      Martins, Andre  and
      Srikumar, Vivek",
    booktitle = "Findings of the Association for Computational Linguistics: ACL 2024",
    month = aug,
    year = "2024",
    address = "Bangkok, Thailand",
    publisher = "Association for Computational Linguistics",
    url = "https://aclanthology.org/2024.findings-acl.762/",
    doi = "10.18653/v1/2024.findings-acl.762",
    pages = "12834--12859"
}

@misc{lee2026learningrate,
  title={Learning Rate Matters: Vanilla {LoRA} May Suffice for {LLM} Fine-tuning},
  author={Yu-Ang Lee and Ching-Yun Ko and Pin-Yu Chen and Mi-Yen Yeh},
  year={2026},
  eprint={2602.04998},
  archivePrefix={arXiv},
  primaryClass={cs.LG},
  url={https://arxiv.org/abs/2602.04998}
}

@misc{jin2026rlheals,
      title={{RL} Fine-Tuning Heals {OOD} Forgetting in {SFT}},
      author={Hangzhan Jin and Sitao Luan and Tianwei Ni and Sicheng Lyu and Guillaume Rabusseau and Reihaneh Rabbany and Doina Precup and Mohammad Hamdaqa},
      year={2026},
      eprint={2509.12235},
      archivePrefix={arXiv},
      primaryClass={cs.LG},
      url={https://arxiv.org/abs/2509.12235},
}

@InProceedings{bini2024ether,
  title = {{ETHER}: Efficient Finetuning of Large-Scale Models with Hyperplane Reflections},
  author = {Bini, Massimo and Roth, Karsten and Akata, Zeynep and Khoreva, Anna},
  booktitle = {Proceedings of the 41st International Conference on Machine Learning},
  pages = {4007--4026},
  year = {2024},
  editor = {Salakhutdinov, Ruslan and Kolter, Zico and Heller, Katherine and Weller, Adrian and Oliver, Nuria and Scarlett, Jonathan and Berkenkamp, Felix},
  volume = {235},
  series = {Proceedings of Machine Learning Research},
  month = {21--27 Jul},
  publisher = {PMLR},
  url = {https://proceedings.mlr.press/v235/bini24a.html}
}

@InProceedings{zhang2025loraone,
  title = 	 {{L}o{RA}-{One}: One-Step Full Gradient Could Suffice for Fine-Tuning Large Language Models, Provably and Efficiently},
  author =       {Zhang, Yuanhe and Liu, Fanghui and Chen, Yudong},
  booktitle = 	 {Proceedings of the 42nd International Conference on Machine Learning},
  pages = 	 {75513--75574},
  year = 	 {2025},
  editor = 	 {Singh, Aarti and Fazel, Maryam and Hsu, Daniel and Lacoste-Julien, Simon and Berkenkamp, Felix and Maharaj, Tegan and Wagstaff, Kiri and Zhu, Jerry},
  volume = 	 {267},
  series = 	 {Proceedings of Machine Learning Research},
  month = 	 {13--19 Jul},
  publisher =    {PMLR},
  url = 	 {https://proceedings.mlr.press/v267/zhang25ax.html}
}

@article{ponkshe2024lorasb,
  title={Initialization using update approximation is a silver bullet for extremely efficient low-rank fine-tuning},
  author={Ponkshe, Kaustubh and Singhal, Raghav and Gorbunov, Eduard and Tumanov, Alexey and Horvath, Samuel and Vepakomma, Praneeth},
  journal={arXiv preprint arXiv:2411.19557},
  year={2024}
}

@article{schreiber1989wy, title={A Storage-Efficient {$WY$} Representation for Products of {Householder} Transformations}, volume={10}, ISSN={2168-3417}, url={http://dx.doi.org/10.1137/0910005}, DOI={10.1137/0910005}, number={1}, journal={SIAM Journal on Scientific and Statistical Computing}, publisher={Society for Industrial & Applied Mathematics (SIAM)}, author={Schreiber, Robert and Van Loan, Charles}, year={1989}, month=Jan, pages={53–57} }

@article{mirsky1975trace, title={A trace inequality of {John von Neumann}}, volume={79}, ISSN={1436-5081}, url={http://dx.doi.org/10.1007/BF01647331}, DOI={10.1007/bf01647331}, number={4}, journal={Monatshefte für Mathematik}, publisher={Springer Science and Business Media LLC}, author={Mirsky, L.}, year={1975}, month=Dec, pages={303–306} }

@article{higham1986polar, title={Computing the Polar Decomposition—with Applications}, volume={7}, ISSN={2168-3417}, url={http://dx.doi.org/10.1137/0907079}, DOI={10.1137/0907079}, number={4}, journal={SIAM Journal on Scientific and Statistical Computing}, publisher={Society for Industrial & Applied Mathematics (SIAM)}, author={Higham, Nicholas J.}, year={1986}, month=Oct, pages={1160–1174} }

@inproceedings{liu2018mhe,
 author = {Liu, Weiyang and Lin, Rongmei and Liu, Zhen and Liu, Lixin and Yu, Zhiding and Dai, Bo and Song, Le},
 booktitle = {Advances in Neural Information Processing Systems},
 editor = {S. Bengio and H. Wallach and H. Larochelle and K. Grauman and N. Cesa-Bianchi and R. Garnett},
 pages = {},
 publisher = {Curran Associates, Inc.},
 title = {Learning towards Minimum Hyperspherical Energy},
 url = {https://proceedings.neurips.cc/paper_files/paper/2018/file/177540c7bcb8db31697b601642eac8d4-Paper.pdf},
 volume = {31},
 year = {2018}
}

@InProceedings{likhosherstov2021cwy,
  title = 	 { {CWY} Parametrization: a Solution for Parallelized Optimization of Orthogonal and {Stiefel} Matrices },
  author =       {Likhosherstov, Valerii and Davis, Jared and Choromanski, Krzysztof and Weller, Adrian},
  booktitle = 	 {Proceedings of The 24th International Conference on Artificial Intelligence and Statistics},
  pages = 	 {55--63},
  year = 	 {2021},
  editor = 	 {Banerjee, Arindam and Fukumizu, Kenji},
  volume = 	 {130},
  series = 	 {Proceedings of Machine Learning Research},
  month = 	 {13--15 Apr},
  publisher =    {PMLR},
  url = 	 {https://proceedings.mlr.press/v130/likhosherstov21a.html}
}

@misc{dubey2024llama3,
      title={The {Llama 3} Herd of Models}, 
      author={Aaron Grattafiori and Abhimanyu Dubey and Abhinav Jauhri and Abhinav Pandey and Abhishek Kadian and Ahmad Al-Dahle and Aiesha Letman and Akhil Mathur and Alan Schelten and Alex Vaughan and others},
      year={2024},
      eprint={2407.21783},
      archivePrefix={arXiv},
      primaryClass={cs.AI},
      url={https://arxiv.org/abs/2407.21783}, 
}

@misc{cobbe2021gsm8k,
  title = {Training Verifiers to Solve Math Word Problems},
  author = {Cobbe, Karl and Kosaraju, Vineet and Bavarian, Mohammad and Chen, Mark and Jun, Heewoo and Kaiser, Lukasz and Plappert, Matthias and Tworek, Jerry and Hilton, Jacob and Nakano, Reiichiro and Hesse, Christopher and Schulman, John},
  year = {2021},
  eprint = {2110.14168},
  archivePrefix = {arXiv},
  primaryClass = {cs.LG},
  url = {https://arxiv.org/abs/2110.14168}
}

@misc{chen2021humaneval,
  title = {Evaluating Large Language Models Trained on Code},
  author = {Chen, Mark and Tworek, Jerry and Jun, Heewoo and Yuan, Qiming and Pinto, Henrique Ponde de Oliveira and Kaplan, Jared and Edwards, Harri and Burda, Yuri and Joseph, Nicholas and Brockman, Greg and others},
  year = {2021},
  eprint = {2107.03374},
  archivePrefix = {arXiv},
  primaryClass = {cs.LG},
  url = {https://arxiv.org/abs/2107.03374}
}

@misc{austin2021mbpp,
  title = {Program Synthesis with Large Language Models},
  author = {Austin, Jacob and Odena, Augustus and Nye, Maxwell and Bosma, Maarten and Michalewski, Henryk and Dohan, David and Jiang, Ellen and Cai, Carrie and Terry, Michael and Le, Quoc and Sutton, Charles},
  year = {2021},
  eprint = {2108.07732},
  archivePrefix = {arXiv},
  primaryClass = {cs.PL},
  url = {https://arxiv.org/abs/2108.07732}
}

@inproceedings{loshchilov2019adamw,
  title = {Decoupled Weight Decay Regularization},
  author = {Loshchilov, Ilya and Hutter, Frank},
  booktitle = {International Conference on Learning Representations},
  year = {2019},
  url = {https://openreview.net/forum?id=Bkg6RiCqY7}
}

@inproceedings{liu2023rectifiedflow,
  title = {Flow Straight and Fast: Learning to Generate and Transfer Data with Rectified Flow},
  author = {Liu, Xingchao and Gong, Chengyue and Liu, Qiang},
  booktitle = {International Conference on Learning Representations},
  year = {2023},
  url = {https://openreview.net/forum?id=XVjTT1nw5z}
}

@InProceedings{houlsby2019adapters,
  title = 	 {Parameter-Efficient Transfer Learning for {NLP}},
  author =       {Houlsby, Neil and Giurgiu, Andrei and Jastrzebski, Stanislaw and Morrone, Bruna and De Laroussilhe, Quentin and Gesmundo, Andrea and Attariyan, Mona and Gelly, Sylvain},
  booktitle = 	 {Proceedings of the 36th International Conference on Machine Learning},
  pages = 	 {2790--2799},
  year = 	 {2019},
  editor = 	 {Chaudhuri, Kamalika and Salakhutdinov, Ruslan},
  volume = 	 {97},
  series = 	 {Proceedings of Machine Learning Research},
  month = 	 {09--15 Jun},
  publisher =    {PMLR},
  url = 	 {https://proceedings.mlr.press/v97/houlsby19a.html}
}

@inproceedings{li2021prefix,
    title = "Prefix-Tuning: Optimizing Continuous Prompts for Generation",
    author = "Li, Xiang Lisa  and
      Liang, Percy",
    editor = "Zong, Chengqing  and
      Xia, Fei  and
      Li, Wenjie  and
      Navigli, Roberto",
    booktitle = "Proceedings of the 59th Annual Meeting of the Association for Computational Linguistics and the 11th International Joint Conference on Natural Language Processing (Volume 1: Long Papers)",
    month = aug,
    year = "2021",
    address = "Online",
    publisher = "Association for Computational Linguistics",
    url = "https://aclanthology.org/2021.acl-long.353/",
    doi = "10.18653/v1/2021.acl-long.353",
    pages = "4582--4597"
}

@Misc{mangrulkar2022peft,
  title =        {{PEFT}: State-of-the-art Parameter-Efficient Fine-Tuning methods},
  author =       {Sourab Mangrulkar and Sylvain Gugger and Lysandre Debut and Younes Belkada and Sayak Paul and Benjamin Bossan and Marian Tietz},
  howpublished = {\url{https://github.com/huggingface/peft}},
  year =         {2022}
}

@misc{liang2026dataflexrlevaluationplatformrlvr,
  title={{DataFlex-RL}: An Evaluation Platform for {RLVR} Data Policies},
  author={Hao Liang and Mingrui Chen and Hengyi Feng and Meiyi Qiang and Wentao Zhang},
  year={2026},
  eprint={2609.06107},
  archivePrefix={arXiv},
  primaryClass={cs.LG},
  url={https://arxiv.org/abs/2609.06107}
}
\bibliographystyle{reference}
\clearpage
\begingroup
\setcounter{tocdepth}{-1}
\hypersetup{linkcolor=ContentsBlue,linktoc=section}
\phantomsection\label{paper:contentsstart}
\tableofcontents
\label{paper:contentsend}
\endgroup
\hypersetup{linkcolor=LinkColor,linktoc=section}

\clearpage
\appendix
\addtocontents{toc}{\protect\setcounter{tocdepth}{2}}
\section{Experimental settings and metrics}\label{app:settings}
Training and functional evaluations used a cluster of NVIDIA A100 GPUs. Each task used a single GPU, with at most 30 GPUs in use simultaneously. The post-hoc hyperspherical-energy measurements used NVIDIA L40 GPUs.

\begin{table}[!htbp]
\caption{Study overview. Each comparison uses a shared pretrained backbone. Divergent runs are reported with the corresponding results.}
\label{tab:studies}
\begin{center}
\setlength{\tabcolsep}{3pt}
\begin{tabular*}{\linewidth}{@{\extracolsep{\fill}}lll@{}}\toprule
Backbone & Adaptation & Primary retention view \\\midrule
Qwen2.5-7B & MetaMath / six methods & General-text NLL \\
Llama-3.1-8B & MetaMath / six methods & General-text NLL \\
Qwen2.5-7B & CodeFeedback / six methods & General-text NLL \\
Llama-3.1-8B & CodeFeedback / six methods & General-text NLL \\
FLUX-4B & Cat / six methods & Image drift / CLIP \\
FLUX-4B & 28 objects & Flow loss / drift / CLIP \\\bottomrule
\end{tabular*}
\end{center}
\end{table}

\subsection{Language adaptation}
Qwen2.5-7B and Llama-3.1-8B use 20,000 MetaMath examples \citep{yu2024metamath}, microbatch size two, gradient accumulation over 16 steps, sequence length 768, 625 optimizer updates, cosine decay, 3\% warmup, gradient clipping at one and zero weight decay. LoRA uses $\alpha=2r$ and zero dropout. OFT uses unshared coordinate blocks and OFTv2's Cayley--Neumann parameterization with five terms \citep{qiu2025oftv2}. The seven projection types are query, key, value, output, up, gate and down, giving 196 adapted matrices in Qwen and 224 in Llama. LRs are $\{1,3,10,30,100\}\times10^{-5}$ and seeds are 13, 37 and 73. 

All six methods use AdamW \citep{loshchilov2019adamw}. LoRA-family ranks are 7/14/28 for Qwen and 7/15/30 for Llama. OFT block sizes are 32/64/128, and HRA uses 16/32/64 Householder vectors. For each model/method/capacity, mean accuracy on the fixed 1,000-item GSM8K development panel selects one LR from fully completed configurations, with ties favoring the lower LR. The resulting 108 checkpoints are evaluated on the official 1,319-item test set. The diverged Qwen PiSSA configuration is excluded from LR selection and complete-configuration frontiers. Appendix~\ref{app:language} reports all selected outcomes and the later HRA LR extension separately.

Coding uses the same two backbones, six methods and five-rate grid at one capacity per method. Its held-out response-NLL selection, completion-based testing and full settings are specified in Appendix~\ref{app:coding}.

Standard and intervention retention use FineWiki-200 \citep{penedo2025finewiki}, with contexts truncated to 768 tokens. The intervention panel evaluates GSM8K development accuracy on 1,000 items. General NLL and full-vocabulary KL use all 200 FineWiki documents, with token-pooled aggregation. General NLL scores valid next-token positions. All intervention comparisons use the same evaluation inputs and checkpoint-local references.

\subsection{Diffusion personalization}
FLUX.2-klein-base-4B uses AdamW, batch size two, weight decay $10^{-4}$, gradient clipping at one, constant LR and 750 updates. The benchmark-cat set has 14 training, two validation and four test images. Four approximately matched parameter pairs are r4/b32, r8/b64, r16/b128 and r32/b256. LoRA uses $\alpha=r$. OFT uses unshared coordinate blocks with OFTv2's Cayley--Neumann parameterization \citep{qiu2025oftv2}. The LR grid is $\{0.5,1,3,5,10,30,50,100\}\times10^{-5}$, with seeds 0, 1 and 2. Generation uses resolution 512, 20 inference steps and guidance 3.5. For each method/capacity, we select a shared LR by the mean validation DINO at step 700, the last validation evaluation, then report test and retention scores at final step 750. The two validation and four test images are disjoint.

The complete six-method cat comparison uses LoRA/DoRA/PiSSA/MiLoRA ranks 4/8/16/32, OFT blocks 32/64/128/256, and HRA 16/32/62/124 Householder vectors. All 576 runs across six methods, four capacities and eight LRs are complete. 

Target DINOv2 similarity averages the paired generated/reference test-image cosines over four examples and three generation repeats. General and class drift compare paired base/adapted generations using $(1-\mathrm{cosine})/2$, over five prompts in each bank. General CLIP is normalized text/image cosine averaged over five general prompts. Training and generation randomness contribute to the reported seed variation.

The 28-object study uses three training images, one validation image and one or two held-out test images per object. It crosses the same four capacity pairs with LRs $\{3,5,10\}\times10^{-5}$ and seeds 0, 1 and 2, for 2,016 runs. For each method/capacity, we select one LR by mean step-700 validation DINO across all 28 objects and seeds, breaking ties toward the smaller LR. That LR is shared across objects and seeds. Test DINO, averaged over three generation repeats, and retention are measured at step 750. The retention bank uses 200 COCO image/caption pairs and four scheduler-index fractions, $\{0.125,0.375,0.625,0.875\}$. Image drift and CLIP use the same four general prompts for every object, depicting an angry Mona Lisa, the Eiffel Tower at night, an astronaut riding a horse, and an avocado-shaped chair.

\subsection{Metrics and uncertainty}
\paragraph{Language likelihood.}
For document $j$ and its valid next-token positions $T_j$, general-text NLL is
\begin{equation*}
 \mathrm{NLL}(\bm\Theta)=-\frac{1}{\sum_j |T_j|}\sum_j\sum_{t\in T_j}
 \log p_{\bm\Theta}\!\left(x_{j,t}\middle|x_{j,<t}\right).
\end{equation*}
All models use identical teacher-forced contexts and masks, excluding padding. The score weights tokens equally and measures fit to the general-text bank.

\paragraph{Diffusion outputs and predictive change.}
Write $\cos(\bm a,\bm b)=\frac{\bm a^{\!\top}\bm b}{\left\|\bm a\right\|_2\left\|\bm b\right\|_2}$, with normalized DINOv2 features $\bm{\mathrm{f}}(I)$ and CLIP image/text features $\bm{\mathrm{g}}(I),\bm{\mathrm{h}}(c)$. For generated image $I_j^{\bm\Theta}$, paired base image $I_j^0$, target reference image $I_j^{\rm ref}$ and prompt $c_j$, the reported averages are
\begin{align*}
 \mathrm{DINO}&=\frac{1}{J}\sum_j\cos\!\left(\bm{\mathrm{f}}(I_j^{\bm\Theta}),\bm{\mathrm{f}}(I_j^{\rm ref})\right),
 &D_{\rm image}&=\frac{1}{J}\sum_j\frac{1-\cos\!\left(\bm{\mathrm{f}}(I_j^{\bm\Theta}),\bm{\mathrm{f}}(I_j^0)\right)}{2},\\
 \mathrm{CLIP}&=\frac{1}{J}\sum_j\cos\!\left(\bm{\mathrm{g}}(I_j^{\bm\Theta}),\bm{\mathrm{h}}(c_j)\right).
\end{align*}
Each metric uses its own stated bank of $J$ pairs. General and class drift use different prompt banks. Paired images share generation seeds. Repeats are averaged before training-seed summaries. The 28-object tables report absolute target DINO, COCO flow loss and CLIP alignment, together with base-relative image drift.

For the 28-object COCO flow loss, $\bm y_{j,k}$ contains the fixed noisy latent, scheduler time and conditioning for image/caption pair $j$ in time bucket $k$. With $d_j$ velocity coordinates, $K$ buckets, flow target $\bm u_{j,k}$ and loss weight $w_{j,k}$,
\begin{equation*}
 \mathcal L_{\rm flow}(\bm\Theta)=\frac{1}{NK}\sum_{j,k}\frac{w_{j,k}}{d_j}
 \left\|\bm v_{\bm\Theta}(\bm y_{j,k})-\bm u_{j,k}\right\|_2^2.
\end{equation*}
The rectified-flow target is noise minus clean latent \citep{liu2023rectifiedflow}. This fixed-state prediction loss is distinct from generated-image drift.

\paragraph{Uncertainty.}
For $S$ training seeds with scores $z_s$, we report $\bar z=S^{-1}\sum_s z_s$ and sample $\mathrm{SD}=\left(\sum_s(z_s-\bar z)^2/(S-1)\right)^{1/2}$. Perturbation and generation repetitions are averaged within a checkpoint first. For each 28-object method/capacity/LR setting, $z_s=28^{-1}\sum_{o=1}^{28}z_{s,o}$. The SD describes variation across training seeds.

\FloatBarrier
\section{Complete selected comparisons and retention checks}\label{app:comparisons}
\subsection{Six-method mathematics comparison and held-out evaluation}\label{app:math-six}\label{app:language}
We extend the same Qwen2.5-7B/Llama-3.1-8B mathematics setup with four methods that vary adaptation structure. DoRA separates weight-vector magnitudes and directions, learning magnitudes alongside a low-rank directional update \citep{liu2024dora}. PiSSA and MiLoRA initialize the trainable factors from leading and trailing singular components, respectively \citep{meng2024pissa,wang2025milora}. In both cases, subtracting the initial low-rank component from the frozen base preserves the starting model. HRA learns a product of Householder reflections \citep{yuan2024hra}. In an input-side convention,
\begin{equation*}
 \bm W=\bm W_0\prod_{j=1}^{r_H}\bm H_j,
 \qquad \bm H_j=\bm I-2\frac{\bm v_j\bm v_j^{\!\top}}{\bm v_j^{\!\top}\bm v_j}.
\end{equation*}
The implementation uses the compact WY representation \citep{schreiber1989wy,likhosherstov2021cwy}, paired identity initialization and no Gram--Schmidt orthogonalization. Here $r_H$ counts Householder vectors. DoRA/PiSSA/MiLoRA use training ranks 7/14/28 for Qwen and 7/15/30 for Llama, with $\alpha=2r$. HRA uses 16/32/64 vectors. The capacity tiers approximately match trainable parameter counts. PiSSA/MiLoRA serialization expresses their final differences against the original base at rank $2r$. The tables report the actual training rank and trainable parameter count.

All methods use the same seven attention/MLP projection types, AdamW, batch 32, sequence length 768, 625 updates, cosine decay, warmup fraction 0.03, zero weight decay and gradient clipping at one. The original grid uses LRs $\{1,3,10,30,100\}\times10^{-5}$ and seeds 13/37/73. It contains 540 training outcomes, comprising 180 LoRA/OFT and 360 added-method runs. Qwen PiSSA r7 at LR $10^{-3}$, seed 73, diverged at step 55, leaving 539 final development evaluations. The group with two successful seeds is excluded from LR selection and complete-configuration frontiers.

\paragraph{Selection and test separation.}
One LR per model/method/capacity maximizes the exact total correct count across seeds on the fixed 1,000-problem, group-held-out GSM8K development panel. Ties choose the smaller LR. The resulting 36 configurations and 108 checkpoints are fixed before test evaluation, with one LR shared across seeds. Each checkpoint is evaluated on all 1,319 official GSM8K test problems. Prior test-set use for evaluation-harness development was separate from this selection. Generation uses the same native BF16 merged-adapter path, prompt and answer parser, greedy decoding, a 300-new-token cap, an 800-total-token cap and a 1,024-token engine context. The six-method intervention figures use development scores.

Retention reuses each run's recorded standard step-625 FineWiki-200 measurement, at length 768 and batch two. We report $\Delta\mathrm{NLL}=\mathrm{NLL}_{625}-\mathrm{NLL}_0$, using the corresponding pretrained reference. Smaller values indicate better general-text retention. We report sample SD across training seeds.

\begin{table}[!htbp]
\caption{Qwen2.5-7B results for all development-selected configurations from the original five-LR grid. Parameters are in millions. Ranks are training ranks (HRA counts Householder vectors). Accuracy is in percent, and $\Delta$NLL is the recorded FineWiki-200 NLL minus its matching pretrained value. LRs were selected using development accuracy only.}
\label{tab:math-selected-qwen}
\begin{center}
\setlength{\tabcolsep}{3pt}\renewcommand{\arraystretch}{1.12}
\begin{tabular*}{\linewidth}{@{\extracolsep{\fill}}lrcccc@{}}\toprule
Adapter & Params (M) & LR & Dev $\uparrow$ & Test $\uparrow$ & $\Delta$NLL $\downarrow$ \\ \midrule
\textcolor{LoRAColor}{LoRA} r7 & \result{math_grid_qwen_lora_7_0p001_params} & $1\times10^{-3}$ & $\result{math_grid_qwen_lora_7_0p001_dev_mean}_{\pm\result{math_grid_qwen_lora_7_0p001_dev_sd}}$ & $\result{math_grid_qwen_lora_7_0p001_test_mean}_{\pm\result{math_grid_qwen_lora_7_0p001_test_sd}}$ & $\result{math_grid_qwen_lora_7_0p001_forgetting_mean}_{\pm\result{math_grid_qwen_lora_7_0p001_forgetting_sd}}$ \\
\textcolor{LoRAColor}{LoRA} r14 & \result{math_grid_qwen_lora_14_1em05_params} & $1\times10^{-5}$ & $\result{math_grid_qwen_lora_14_1em05_dev_mean}_{\pm\result{math_grid_qwen_lora_14_1em05_dev_sd}}$ & $\result{math_grid_qwen_lora_14_1em05_test_mean}_{\pm\result{math_grid_qwen_lora_14_1em05_test_sd}}$ & $\result{math_grid_qwen_lora_14_1em05_forgetting_mean}_{\pm\result{math_grid_qwen_lora_14_1em05_forgetting_sd}}$ \\
\textcolor{LoRAColor}{LoRA} r28 & \result{math_grid_qwen_lora_28_3em05_params} & $3\times10^{-5}$ & $\result{math_grid_qwen_lora_28_3em05_dev_mean}_{\pm\result{math_grid_qwen_lora_28_3em05_dev_sd}}$ & $\result{math_grid_qwen_lora_28_3em05_test_mean}_{\pm\result{math_grid_qwen_lora_28_3em05_test_sd}}$ & $\result{math_grid_qwen_lora_28_3em05_forgetting_mean}_{\pm\result{math_grid_qwen_lora_28_3em05_forgetting_sd}}$ \\
\addlinespace[2pt]
\addlinespace[2pt]
\textcolor{OFTColor}{OFT} b32 & \result{math_grid_qwen_oft_32_0p0003_params} & $3\times10^{-4}$ & $\result{math_grid_qwen_oft_32_0p0003_dev_mean}_{\pm\result{math_grid_qwen_oft_32_0p0003_dev_sd}}$ & $\result{math_grid_qwen_oft_32_0p0003_test_mean}_{\pm\result{math_grid_qwen_oft_32_0p0003_test_sd}}$ & $\result{math_grid_qwen_oft_32_0p0003_forgetting_mean}_{\pm\result{math_grid_qwen_oft_32_0p0003_forgetting_sd}}$ \\
\textcolor{OFTColor}{OFT} b64 & \result{math_grid_qwen_oft_64_0p0003_params} & $3\times10^{-4}$ & $\result{math_grid_qwen_oft_64_0p0003_dev_mean}_{\pm\result{math_grid_qwen_oft_64_0p0003_dev_sd}}$ & $\result{math_grid_qwen_oft_64_0p0003_test_mean}_{\pm\result{math_grid_qwen_oft_64_0p0003_test_sd}}$ & $\result{math_grid_qwen_oft_64_0p0003_forgetting_mean}_{\pm\result{math_grid_qwen_oft_64_0p0003_forgetting_sd}}$ \\
\textcolor{OFTColor}{OFT} b128 & \result{math_grid_qwen_oft_128_0p0001_params} & $1\times10^{-4}$ & $\result{math_grid_qwen_oft_128_0p0001_dev_mean}_{\pm\result{math_grid_qwen_oft_128_0p0001_dev_sd}}$ & $\result{math_grid_qwen_oft_128_0p0001_test_mean}_{\pm\result{math_grid_qwen_oft_128_0p0001_test_sd}}$ & $\result{math_grid_qwen_oft_128_0p0001_forgetting_mean}_{\pm\result{math_grid_qwen_oft_128_0p0001_forgetting_sd}}$ \\
\addlinespace[2pt]
\addlinespace[2pt]
\textcolor{DoRAColor}{DoRA} r7 & \result{math_grid_qwen_dora_7_0p0003_params} & $3\times10^{-4}$ & $\result{math_grid_qwen_dora_7_0p0003_dev_mean}_{\pm\result{math_grid_qwen_dora_7_0p0003_dev_sd}}$ & $\result{math_grid_qwen_dora_7_0p0003_test_mean}_{\pm\result{math_grid_qwen_dora_7_0p0003_test_sd}}$ & $\result{math_grid_qwen_dora_7_0p0003_forgetting_mean}_{\pm\result{math_grid_qwen_dora_7_0p0003_forgetting_sd}}$ \\
\textcolor{DoRAColor}{DoRA} r14 & \result{math_grid_qwen_dora_14_0p0003_params} & $3\times10^{-4}$ & $\result{math_grid_qwen_dora_14_0p0003_dev_mean}_{\pm\result{math_grid_qwen_dora_14_0p0003_dev_sd}}$ & $\result{math_grid_qwen_dora_14_0p0003_test_mean}_{\pm\result{math_grid_qwen_dora_14_0p0003_test_sd}}$ & $\result{math_grid_qwen_dora_14_0p0003_forgetting_mean}_{\pm\result{math_grid_qwen_dora_14_0p0003_forgetting_sd}}$ \\
\textcolor{DoRAColor}{DoRA} r28 & \result{math_grid_qwen_dora_28_0p0001_params} & $1\times10^{-4}$ & $\result{math_grid_qwen_dora_28_0p0001_dev_mean}_{\pm\result{math_grid_qwen_dora_28_0p0001_dev_sd}}$ & $\result{math_grid_qwen_dora_28_0p0001_test_mean}_{\pm\result{math_grid_qwen_dora_28_0p0001_test_sd}}$ & $\result{math_grid_qwen_dora_28_0p0001_forgetting_mean}_{\pm\result{math_grid_qwen_dora_28_0p0001_forgetting_sd}}$ \\
\addlinespace[2pt]
\addlinespace[2pt]
\textcolor{PiSSAColor}{PiSSA} r7 & \result{math_grid_qwen_pissa_7_3em05_params} & $3\times10^{-5}$ & $\result{math_grid_qwen_pissa_7_3em05_dev_mean}_{\pm\result{math_grid_qwen_pissa_7_3em05_dev_sd}}$ & $\result{math_grid_qwen_pissa_7_3em05_test_mean}_{\pm\result{math_grid_qwen_pissa_7_3em05_test_sd}}$ & $\result{math_grid_qwen_pissa_7_3em05_forgetting_mean}_{\pm\result{math_grid_qwen_pissa_7_3em05_forgetting_sd}}$ \\
\textcolor{PiSSAColor}{PiSSA} r14 & \result{math_grid_qwen_pissa_14_3em05_params} & $3\times10^{-5}$ & $\result{math_grid_qwen_pissa_14_3em05_dev_mean}_{\pm\result{math_grid_qwen_pissa_14_3em05_dev_sd}}$ & $\result{math_grid_qwen_pissa_14_3em05_test_mean}_{\pm\result{math_grid_qwen_pissa_14_3em05_test_sd}}$ & $\result{math_grid_qwen_pissa_14_3em05_forgetting_mean}_{\pm\result{math_grid_qwen_pissa_14_3em05_forgetting_sd}}$ \\
\textcolor{PiSSAColor}{PiSSA} r28 & \result{math_grid_qwen_pissa_28_3em05_params} & $3\times10^{-5}$ & $\result{math_grid_qwen_pissa_28_3em05_dev_mean}_{\pm\result{math_grid_qwen_pissa_28_3em05_dev_sd}}$ & $\result{math_grid_qwen_pissa_28_3em05_test_mean}_{\pm\result{math_grid_qwen_pissa_28_3em05_test_sd}}$ & $\result{math_grid_qwen_pissa_28_3em05_forgetting_mean}_{\pm\result{math_grid_qwen_pissa_28_3em05_forgetting_sd}}$ \\
\addlinespace[2pt]
\addlinespace[2pt]
\textcolor{MiLoRAColor}{MiLoRA} r7 & \result{math_grid_qwen_milora_7_1em05_params} & $1\times10^{-5}$ & $\result{math_grid_qwen_milora_7_1em05_dev_mean}_{\pm\result{math_grid_qwen_milora_7_1em05_dev_sd}}$ & $\result{math_grid_qwen_milora_7_1em05_test_mean}_{\pm\result{math_grid_qwen_milora_7_1em05_test_sd}}$ & $\result{math_grid_qwen_milora_7_1em05_forgetting_mean}_{\pm\result{math_grid_qwen_milora_7_1em05_forgetting_sd}}$ \\
\textcolor{MiLoRAColor}{MiLoRA} r14 & \result{math_grid_qwen_milora_14_1em05_params} & $1\times10^{-5}$ & $\result{math_grid_qwen_milora_14_1em05_dev_mean}_{\pm\result{math_grid_qwen_milora_14_1em05_dev_sd}}$ & $\result{math_grid_qwen_milora_14_1em05_test_mean}_{\pm\result{math_grid_qwen_milora_14_1em05_test_sd}}$ & $\result{math_grid_qwen_milora_14_1em05_forgetting_mean}_{\pm\result{math_grid_qwen_milora_14_1em05_forgetting_sd}}$ \\
\textcolor{MiLoRAColor}{MiLoRA} r28 & \result{math_grid_qwen_milora_28_3em05_params} & $3\times10^{-5}$ & $\result{math_grid_qwen_milora_28_3em05_dev_mean}_{\pm\result{math_grid_qwen_milora_28_3em05_dev_sd}}$ & $\result{math_grid_qwen_milora_28_3em05_test_mean}_{\pm\result{math_grid_qwen_milora_28_3em05_test_sd}}$ & $\result{math_grid_qwen_milora_28_3em05_forgetting_mean}_{\pm\result{math_grid_qwen_milora_28_3em05_forgetting_sd}}$ \\
\addlinespace[2pt]
\addlinespace[2pt]
\textcolor{HRAColor}{HRA} r16 & \result{math_grid_qwen_hra_16_0p0001_params} & $1\times10^{-4}$ & $\result{math_grid_qwen_hra_16_0p0001_dev_mean}_{\pm\result{math_grid_qwen_hra_16_0p0001_dev_sd}}$ & $\result{math_grid_qwen_hra_16_0p0001_test_mean}_{\pm\result{math_grid_qwen_hra_16_0p0001_test_sd}}$ & $\result{math_grid_qwen_hra_16_0p0001_forgetting_mean}_{\pm\result{math_grid_qwen_hra_16_0p0001_forgetting_sd}}$ \\
\textcolor{HRAColor}{HRA} r32 & \result{math_grid_qwen_hra_32_0p0001_params} & $1\times10^{-4}$ & $\result{math_grid_qwen_hra_32_0p0001_dev_mean}_{\pm\result{math_grid_qwen_hra_32_0p0001_dev_sd}}$ & $\result{math_grid_qwen_hra_32_0p0001_test_mean}_{\pm\result{math_grid_qwen_hra_32_0p0001_test_sd}}$ & $\result{math_grid_qwen_hra_32_0p0001_forgetting_mean}_{\pm\result{math_grid_qwen_hra_32_0p0001_forgetting_sd}}$ \\
\textcolor{HRAColor}{HRA} r64 & \result{math_grid_qwen_hra_64_1em05_params} & $1\times10^{-5}$ & $\result{math_grid_qwen_hra_64_1em05_dev_mean}_{\pm\result{math_grid_qwen_hra_64_1em05_dev_sd}}$ & $\result{math_grid_qwen_hra_64_1em05_test_mean}_{\pm\result{math_grid_qwen_hra_64_1em05_test_sd}}$ & $\result{math_grid_qwen_hra_64_1em05_forgetting_mean}_{\pm\result{math_grid_qwen_hra_64_1em05_forgetting_sd}}$ \\
\bottomrule\end{tabular*}

\end{center}
\end{table}

\begin{table}[!htbp]
\caption{Llama-3.1-8B results for all development-selected configurations from the original five-LR grid. Parameters are in millions. Ranks are training ranks (HRA counts Householder vectors). Accuracy is in percent, and $\Delta$NLL is the recorded FineWiki-200 NLL minus its matching pretrained value. LRs were selected using development accuracy only.}
\label{tab:math-selected-llama}
\begin{center}
\setlength{\tabcolsep}{3pt}\renewcommand{\arraystretch}{1.12}
\begin{tabular*}{\linewidth}{@{\extracolsep{\fill}}lrcccc@{}}\toprule
Adapter & Params (M) & LR & Dev $\uparrow$ & Test $\uparrow$ & $\Delta$NLL $\downarrow$ \\ \midrule
\textcolor{LoRAColor}{LoRA} r7 & \result{math_grid_llama_lora_7_0p0003_params} & $3\times10^{-4}$ & $\result{math_grid_llama_lora_7_0p0003_dev_mean}_{\pm\result{math_grid_llama_lora_7_0p0003_dev_sd}}$ & $\result{math_grid_llama_lora_7_0p0003_test_mean}_{\pm\result{math_grid_llama_lora_7_0p0003_test_sd}}$ & $\result{math_grid_llama_lora_7_0p0003_forgetting_mean}_{\pm\result{math_grid_llama_lora_7_0p0003_forgetting_sd}}$ \\
\textcolor{LoRAColor}{LoRA} r15 & \result{math_grid_llama_lora_15_0p0003_params} & $3\times10^{-4}$ & $\result{math_grid_llama_lora_15_0p0003_dev_mean}_{\pm\result{math_grid_llama_lora_15_0p0003_dev_sd}}$ & $\result{math_grid_llama_lora_15_0p0003_test_mean}_{\pm\result{math_grid_llama_lora_15_0p0003_test_sd}}$ & $\result{math_grid_llama_lora_15_0p0003_forgetting_mean}_{\pm\result{math_grid_llama_lora_15_0p0003_forgetting_sd}}$ \\
\textcolor{LoRAColor}{LoRA} r30 & \result{math_grid_llama_lora_30_0p0003_params} & $3\times10^{-4}$ & $\result{math_grid_llama_lora_30_0p0003_dev_mean}_{\pm\result{math_grid_llama_lora_30_0p0003_dev_sd}}$ & $\result{math_grid_llama_lora_30_0p0003_test_mean}_{\pm\result{math_grid_llama_lora_30_0p0003_test_sd}}$ & $\result{math_grid_llama_lora_30_0p0003_forgetting_mean}_{\pm\result{math_grid_llama_lora_30_0p0003_forgetting_sd}}$ \\
\addlinespace[2pt]
\addlinespace[2pt]
\textcolor{OFTColor}{OFT} b32 & \result{math_grid_llama_oft_32_0p0003_params} & $3\times10^{-4}$ & $\result{math_grid_llama_oft_32_0p0003_dev_mean}_{\pm\result{math_grid_llama_oft_32_0p0003_dev_sd}}$ & $\result{math_grid_llama_oft_32_0p0003_test_mean}_{\pm\result{math_grid_llama_oft_32_0p0003_test_sd}}$ & $\result{math_grid_llama_oft_32_0p0003_forgetting_mean}_{\pm\result{math_grid_llama_oft_32_0p0003_forgetting_sd}}$ \\
\textcolor{OFTColor}{OFT} b64 & \result{math_grid_llama_oft_64_0p0003_params} & $3\times10^{-4}$ & $\result{math_grid_llama_oft_64_0p0003_dev_mean}_{\pm\result{math_grid_llama_oft_64_0p0003_dev_sd}}$ & $\result{math_grid_llama_oft_64_0p0003_test_mean}_{\pm\result{math_grid_llama_oft_64_0p0003_test_sd}}$ & $\result{math_grid_llama_oft_64_0p0003_forgetting_mean}_{\pm\result{math_grid_llama_oft_64_0p0003_forgetting_sd}}$ \\
\textcolor{OFTColor}{OFT} b128 & \result{math_grid_llama_oft_128_0p0001_params} & $1\times10^{-4}$ & $\result{math_grid_llama_oft_128_0p0001_dev_mean}_{\pm\result{math_grid_llama_oft_128_0p0001_dev_sd}}$ & $\result{math_grid_llama_oft_128_0p0001_test_mean}_{\pm\result{math_grid_llama_oft_128_0p0001_test_sd}}$ & $\result{math_grid_llama_oft_128_0p0001_forgetting_mean}_{\pm\result{math_grid_llama_oft_128_0p0001_forgetting_sd}}$ \\
\addlinespace[2pt]
\addlinespace[2pt]
\textcolor{DoRAColor}{DoRA} r7 & \result{math_grid_llama_dora_7_0p0003_params} & $3\times10^{-4}$ & $\result{math_grid_llama_dora_7_0p0003_dev_mean}_{\pm\result{math_grid_llama_dora_7_0p0003_dev_sd}}$ & $\result{math_grid_llama_dora_7_0p0003_test_mean}_{\pm\result{math_grid_llama_dora_7_0p0003_test_sd}}$ & $\result{math_grid_llama_dora_7_0p0003_forgetting_mean}_{\pm\result{math_grid_llama_dora_7_0p0003_forgetting_sd}}$ \\
\textcolor{DoRAColor}{DoRA} r15 & \result{math_grid_llama_dora_15_0p0003_params} & $3\times10^{-4}$ & $\result{math_grid_llama_dora_15_0p0003_dev_mean}_{\pm\result{math_grid_llama_dora_15_0p0003_dev_sd}}$ & $\result{math_grid_llama_dora_15_0p0003_test_mean}_{\pm\result{math_grid_llama_dora_15_0p0003_test_sd}}$ & $\result{math_grid_llama_dora_15_0p0003_forgetting_mean}_{\pm\result{math_grid_llama_dora_15_0p0003_forgetting_sd}}$ \\
\textcolor{DoRAColor}{DoRA} r30 & \result{math_grid_llama_dora_30_0p0003_params} & $3\times10^{-4}$ & $\result{math_grid_llama_dora_30_0p0003_dev_mean}_{\pm\result{math_grid_llama_dora_30_0p0003_dev_sd}}$ & $\result{math_grid_llama_dora_30_0p0003_test_mean}_{\pm\result{math_grid_llama_dora_30_0p0003_test_sd}}$ & $\result{math_grid_llama_dora_30_0p0003_forgetting_mean}_{\pm\result{math_grid_llama_dora_30_0p0003_forgetting_sd}}$ \\
\addlinespace[2pt]
\addlinespace[2pt]
\textcolor{PiSSAColor}{PiSSA} r7 & \result{math_grid_llama_pissa_7_0p0001_params} & $1\times10^{-4}$ & $\result{math_grid_llama_pissa_7_0p0001_dev_mean}_{\pm\result{math_grid_llama_pissa_7_0p0001_dev_sd}}$ & $\result{math_grid_llama_pissa_7_0p0001_test_mean}_{\pm\result{math_grid_llama_pissa_7_0p0001_test_sd}}$ & $\result{math_grid_llama_pissa_7_0p0001_forgetting_mean}_{\pm\result{math_grid_llama_pissa_7_0p0001_forgetting_sd}}$ \\
\textcolor{PiSSAColor}{PiSSA} r15 & \result{math_grid_llama_pissa_15_0p0001_params} & $1\times10^{-4}$ & $\result{math_grid_llama_pissa_15_0p0001_dev_mean}_{\pm\result{math_grid_llama_pissa_15_0p0001_dev_sd}}$ & $\result{math_grid_llama_pissa_15_0p0001_test_mean}_{\pm\result{math_grid_llama_pissa_15_0p0001_test_sd}}$ & $\result{math_grid_llama_pissa_15_0p0001_forgetting_mean}_{\pm\result{math_grid_llama_pissa_15_0p0001_forgetting_sd}}$ \\
\textcolor{PiSSAColor}{PiSSA} r30 & \result{math_grid_llama_pissa_30_3em05_params} & $3\times10^{-5}$ & $\result{math_grid_llama_pissa_30_3em05_dev_mean}_{\pm\result{math_grid_llama_pissa_30_3em05_dev_sd}}$ & $\result{math_grid_llama_pissa_30_3em05_test_mean}_{\pm\result{math_grid_llama_pissa_30_3em05_test_sd}}$ & $\result{math_grid_llama_pissa_30_3em05_forgetting_mean}_{\pm\result{math_grid_llama_pissa_30_3em05_forgetting_sd}}$ \\
\addlinespace[2pt]
\addlinespace[2pt]
\textcolor{MiLoRAColor}{MiLoRA} r7 & \result{math_grid_llama_milora_7_0p0003_params} & $3\times10^{-4}$ & $\result{math_grid_llama_milora_7_0p0003_dev_mean}_{\pm\result{math_grid_llama_milora_7_0p0003_dev_sd}}$ & $\result{math_grid_llama_milora_7_0p0003_test_mean}_{\pm\result{math_grid_llama_milora_7_0p0003_test_sd}}$ & $\result{math_grid_llama_milora_7_0p0003_forgetting_mean}_{\pm\result{math_grid_llama_milora_7_0p0003_forgetting_sd}}$ \\
\textcolor{MiLoRAColor}{MiLoRA} r15 & \result{math_grid_llama_milora_15_0p0003_params} & $3\times10^{-4}$ & $\result{math_grid_llama_milora_15_0p0003_dev_mean}_{\pm\result{math_grid_llama_milora_15_0p0003_dev_sd}}$ & $\result{math_grid_llama_milora_15_0p0003_test_mean}_{\pm\result{math_grid_llama_milora_15_0p0003_test_sd}}$ & $\result{math_grid_llama_milora_15_0p0003_forgetting_mean}_{\pm\result{math_grid_llama_milora_15_0p0003_forgetting_sd}}$ \\
\textcolor{MiLoRAColor}{MiLoRA} r30 & \result{math_grid_llama_milora_30_0p0003_params} & $3\times10^{-4}$ & $\result{math_grid_llama_milora_30_0p0003_dev_mean}_{\pm\result{math_grid_llama_milora_30_0p0003_dev_sd}}$ & $\result{math_grid_llama_milora_30_0p0003_test_mean}_{\pm\result{math_grid_llama_milora_30_0p0003_test_sd}}$ & $\result{math_grid_llama_milora_30_0p0003_forgetting_mean}_{\pm\result{math_grid_llama_milora_30_0p0003_forgetting_sd}}$ \\
\addlinespace[2pt]
\addlinespace[2pt]
\textcolor{HRAColor}{HRA} r16 & \result{math_grid_llama_hra_16_0p001_params} & $1\times10^{-3}$ & $\result{math_grid_llama_hra_16_0p001_dev_mean}_{\pm\result{math_grid_llama_hra_16_0p001_dev_sd}}$ & $\result{math_grid_llama_hra_16_0p001_test_mean}_{\pm\result{math_grid_llama_hra_16_0p001_test_sd}}$ & $\result{math_grid_llama_hra_16_0p001_forgetting_mean}_{\pm\result{math_grid_llama_hra_16_0p001_forgetting_sd}}$ \\
\textcolor{HRAColor}{HRA} r32 & \result{math_grid_llama_hra_32_0p001_params} & $1\times10^{-3}$ & $\result{math_grid_llama_hra_32_0p001_dev_mean}_{\pm\result{math_grid_llama_hra_32_0p001_dev_sd}}$ & $\result{math_grid_llama_hra_32_0p001_test_mean}_{\pm\result{math_grid_llama_hra_32_0p001_test_sd}}$ & $\result{math_grid_llama_hra_32_0p001_forgetting_mean}_{\pm\result{math_grid_llama_hra_32_0p001_forgetting_sd}}$ \\
\textcolor{HRAColor}{HRA} r64 & \result{math_grid_llama_hra_64_0p001_params} & $1\times10^{-3}$ & $\result{math_grid_llama_hra_64_0p001_dev_mean}_{\pm\result{math_grid_llama_hra_64_0p001_dev_sd}}$ & $\result{math_grid_llama_hra_64_0p001_test_mean}_{\pm\result{math_grid_llama_hra_64_0p001_test_sd}}$ & $\result{math_grid_llama_hra_64_0p001_forgetting_mean}_{\pm\result{math_grid_llama_hra_64_0p001_forgetting_sd}}$ \\
\bottomrule\end{tabular*}

\end{center}
\end{table}

\begin{figure}[!htbp]\centering
\includegraphics[width=\linewidth]{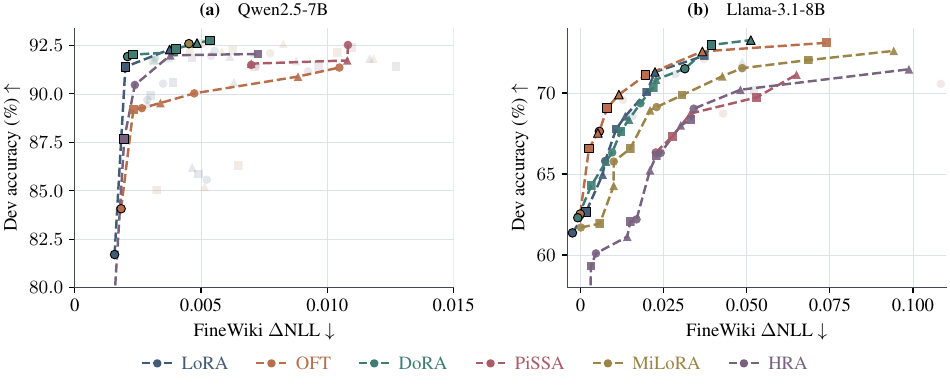}
\caption{Six-method mathematics trade-offs, focused on the frontier regions. Each point is a complete capacity/LR configuration, with shapes denoting increasing capacity. Dashed lines and solid points mark within-method frontiers, and black borders mark the pooled frontier. Dominated points are faint. Frontiers use all 89 Qwen and 90 Llama configurations before axis clipping. The later HRA extension and diverged PiSSA configuration are excluded.}
\label{fig:pareto-full}
\end{figure}
\paragraph{Performance and retention across settings.}
Qwen's small test-accuracy differences coexist with different retention costs. Its high-LR, small-rank LoRA selection is an outlier with greater forgetting than the other methods. LoRA's larger-rank selections have less forgetting. HRA's Qwen scores are competitive, and its original-grid Llama scores improve with capacity alongside increasing retention cost. These results support the main finding that competitive task scores can accompany different forgetting costs, with outcomes depending on the selected LR and capacity. Appendix~\ref{app:retention-patterns} summarizes the trends across tasks.
\subsection{Separate HRA learning-rate extension}\label{app:hra-extension}
All three original Llama HRA winners occur at the largest LR, $10^{-3}$. A follow-up therefore adds $3\times10^{-3}$ and $10^{-2}$ at each capacity and seed. The extension was proposed after the original test report and is reported separately. Of 18 training runs, 15 complete both development and official-test evaluations. All three r64 runs at $10^{-2}$ diverge. Selection over the seven-LR development grid chooses $3\times10^{-3}$ at r16/r32 and retains $10^{-3}$ at r64. The main comparison retains the original five-LR selections.
\begin{table}[!htbp]
\caption{Llama-3.1-8B HRA development-selected outcomes after the separate seven-LR follow-up. The r64 selection is unchanged from the original grid. $n$ is the successful seed count. LR selection uses development scores.}
\label{tab:math-hra-extension}
\begin{center}
\setlength{\tabcolsep}{3pt}
\begin{tabular*}{\linewidth}{@{\extracolsep{\fill}}lccccc@{}}\toprule
Adapter & LR & $n$ & Dev (\%) $\uparrow$ & Test (\%) $\uparrow$ & $\Delta$NLL $\downarrow$ \\\midrule
\textcolor{HRAColor}{HRA} r16 & $3\times10^{-3}$ & 3 & $\result{math_extension_llama_hra_16_0p003_dev_mean}_{\pm\result{math_extension_llama_hra_16_0p003_dev_sd}}$ & $\result{math_extension_llama_hra_16_0p003_test_mean}_{\pm\result{math_extension_llama_hra_16_0p003_test_sd}}$ & $\result{math_extension_llama_hra_16_0p003_forgetting_mean}_{\pm\result{math_extension_llama_hra_16_0p003_forgetting_sd}}$ \\
\textcolor{HRAColor}{HRA} r32 & $3\times10^{-3}$ & 3 & $\result{math_extension_llama_hra_32_0p003_dev_mean}_{\pm\result{math_extension_llama_hra_32_0p003_dev_sd}}$ & $\result{math_extension_llama_hra_32_0p003_test_mean}_{\pm\result{math_extension_llama_hra_32_0p003_test_sd}}$ & $\result{math_extension_llama_hra_32_0p003_forgetting_mean}_{\pm\result{math_extension_llama_hra_32_0p003_forgetting_sd}}$ \\
\textcolor{HRAColor}{HRA} r64 & $1\times10^{-3}$ & 3 & $\result{math_grid_llama_hra_64_0p001_dev_mean}_{\pm\result{math_grid_llama_hra_64_0p001_dev_sd}}$ & $\result{math_grid_llama_hra_64_0p001_test_mean}_{\pm\result{math_grid_llama_hra_64_0p001_test_sd}}$ & $\result{math_grid_llama_hra_64_0p001_forgetting_mean}_{\pm\result{math_grid_llama_hra_64_0p001_forgetting_sd}}$ \\
\bottomrule\end{tabular*}
\end{center}
\end{table}
\FloatBarrier
The development-selected $3\times10^{-3}$ extension raises mean test accuracy at both r16 and r32, while increasing their retention costs. At r64, $3\times10^{-3}$ reduces accuracy and increases NLL, so development selection retains the original LR. HRA therefore illustrates the setting-dependent trade-offs in Section~\ref{sec:comparisons}, with the useful LR range changing across capacities.
\FloatBarrier

\FloatBarrier
\subsection{CodeFeedback selection and evaluation}\label{app:coding}
We use the Python subset of CodeFeedback \citep{zheng2024opencodeinterpreter}, distributed for PiSSA \citep{meng2024pissa}. After exact prompt deduplication, a fixed split contains 20,000 training and 2,000 development examples. Exact prompt matching against the test sets removes no additional examples. All methods use one epoch, effective batch size 32, maximum training length 1,024, a cosine schedule with 3\% warmup, and the same attention/MLP projections. LoRA, DoRA, PiSSA and MiLoRA train at rank 7. OFT uses block size 32 and HRA uses 16 Householder vectors. PiSSA/MiLoRA's converted rank-14 difference adapters retain the same training capacity.

For each method/backbone, the mean development response NLL across seeds 13, 37 and 73 selects one of $\{10^{-5},3\times10^{-5},10^{-4},3\times10^{-4},10^{-3}\}$, with ties favoring the lower LR. NLL pools valid response-token losses, with a 4,096-token evaluation limit. Thus the 180 development runs yield 36 selected test checkpoints, each using the best LR within the tested grid. 

Test evaluation uses EvalPlus base completion prompts, greedy decoding, and one completion per problem, capped at 1,024 new tokens. We evaluate 164 HumanEval \citep{chen2021humaneval} and 378 MBPP \citep{austin2021mbpp} problems. Original and strengthened tests grade the same completion. \Cref{tab:coding-full} reports both versions. The base model has one fixed evaluation, so its row has no training-seed SD.

General retention uses the fixed FineWiki-200 corpus, a 768-token limit, and valid-token-pooled NLL with float32 logits. Each method's forgetting is its NLL minus the corresponding unadapted model's NLL. 

\begin{table}[!htbp]
\caption{CodeFeedback results for all original and strengthened-test pass@1 scores (\%). A single greedy completion is graded under both test suites. LR minimizes mean held-out response NLL. Only selected checkpoints have test evaluations. The Base row has no training-seed SD.}
\label{tab:coding-full}
\begin{center}
\setlength{\tabcolsep}{3pt}
\begin{tabular*}{\linewidth}{@{\extracolsep{\fill}}llccccc@{}}\toprule
Model & Method & Selected LR & HumanEval & HE+ & MBPP & MBPP+ \\ \midrule
Qwen7B & Base & --- & $\result{ext_code_qwen_base_he}$ & $\result{ext_code_qwen_base_hep}$ & $\result{ext_code_qwen_base_mb}$ & $\result{ext_code_qwen_base_mbp}$ \\
 & \textcolor{LoRAColor}{LoRA} & $1\times10^{-4}$ & $\result{ext_code_qwen_lora_0p0001_he_mean}_{\pm\result{ext_code_qwen_lora_0p0001_he_sd}}$ & $\result{ext_code_qwen_lora_0p0001_hep_mean}_{\pm\result{ext_code_qwen_lora_0p0001_hep_sd}}$ & $\result{ext_code_qwen_lora_0p0001_mb_mean}_{\pm\result{ext_code_qwen_lora_0p0001_mb_sd}}$ & $\result{ext_code_qwen_lora_0p0001_mbp_mean}_{\pm\result{ext_code_qwen_lora_0p0001_mbp_sd}}$ \\
 & \textcolor{OFTColor}{OFT} & $1\times10^{-4}$ & $\result{ext_code_qwen_oft_0p0001_he_mean}_{\pm\result{ext_code_qwen_oft_0p0001_he_sd}}$ & $\result{ext_code_qwen_oft_0p0001_hep_mean}_{\pm\result{ext_code_qwen_oft_0p0001_hep_sd}}$ & $\result{ext_code_qwen_oft_0p0001_mb_mean}_{\pm\result{ext_code_qwen_oft_0p0001_mb_sd}}$ & $\result{ext_code_qwen_oft_0p0001_mbp_mean}_{\pm\result{ext_code_qwen_oft_0p0001_mbp_sd}}$ \\
 & \textcolor{DoRAColor}{DoRA} & $1\times10^{-4}$ & $\result{ext_code_qwen_dora_0p0001_he_mean}_{\pm\result{ext_code_qwen_dora_0p0001_he_sd}}$ & $\result{ext_code_qwen_dora_0p0001_hep_mean}_{\pm\result{ext_code_qwen_dora_0p0001_hep_sd}}$ & $\result{ext_code_qwen_dora_0p0001_mb_mean}_{\pm\result{ext_code_qwen_dora_0p0001_mb_sd}}$ & $\result{ext_code_qwen_dora_0p0001_mbp_mean}_{\pm\result{ext_code_qwen_dora_0p0001_mbp_sd}}$ \\
 & \textcolor{PiSSAColor}{PiSSA} & $1\times10^{-5}$ & $\result{ext_code_qwen_pissa_1em05_he_mean}_{\pm\result{ext_code_qwen_pissa_1em05_he_sd}}$ & $\result{ext_code_qwen_pissa_1em05_hep_mean}_{\pm\result{ext_code_qwen_pissa_1em05_hep_sd}}$ & $\result{ext_code_qwen_pissa_1em05_mb_mean}_{\pm\result{ext_code_qwen_pissa_1em05_mb_sd}}$ & $\result{ext_code_qwen_pissa_1em05_mbp_mean}_{\pm\result{ext_code_qwen_pissa_1em05_mbp_sd}}$ \\
 & \textcolor{MiLoRAColor}{MiLoRA} & $1\times10^{-4}$ & $\result{ext_code_qwen_milora_0p0001_he_mean}_{\pm\result{ext_code_qwen_milora_0p0001_he_sd}}$ & $\result{ext_code_qwen_milora_0p0001_hep_mean}_{\pm\result{ext_code_qwen_milora_0p0001_hep_sd}}$ & $\result{ext_code_qwen_milora_0p0001_mb_mean}_{\pm\result{ext_code_qwen_milora_0p0001_mb_sd}}$ & $\result{ext_code_qwen_milora_0p0001_mbp_mean}_{\pm\result{ext_code_qwen_milora_0p0001_mbp_sd}}$ \\
 & \textcolor{HRAColor}{HRA} & $1\times10^{-3}$ & $\result{ext_code_qwen_hra_0p001_he_mean}_{\pm\result{ext_code_qwen_hra_0p001_he_sd}}$ & $\result{ext_code_qwen_hra_0p001_hep_mean}_{\pm\result{ext_code_qwen_hra_0p001_hep_sd}}$ & $\result{ext_code_qwen_hra_0p001_mb_mean}_{\pm\result{ext_code_qwen_hra_0p001_mb_sd}}$ & $\result{ext_code_qwen_hra_0p001_mbp_mean}_{\pm\result{ext_code_qwen_hra_0p001_mbp_sd}}$ \\
\addlinespace[3pt]
Llama8B & Base & --- & $\result{ext_code_llama_base_he}$ & $\result{ext_code_llama_base_hep}$ & $\result{ext_code_llama_base_mb}$ & $\result{ext_code_llama_base_mbp}$ \\
 & \textcolor{LoRAColor}{LoRA} & $3\times10^{-4}$ & $\result{ext_code_llama_lora_0p0003_he_mean}_{\pm\result{ext_code_llama_lora_0p0003_he_sd}}$ & $\result{ext_code_llama_lora_0p0003_hep_mean}_{\pm\result{ext_code_llama_lora_0p0003_hep_sd}}$ & $\result{ext_code_llama_lora_0p0003_mb_mean}_{\pm\result{ext_code_llama_lora_0p0003_mb_sd}}$ & $\result{ext_code_llama_lora_0p0003_mbp_mean}_{\pm\result{ext_code_llama_lora_0p0003_mbp_sd}}$ \\
 & \textcolor{OFTColor}{OFT} & $3\times10^{-4}$ & $\result{ext_code_llama_oft_0p0003_he_mean}_{\pm\result{ext_code_llama_oft_0p0003_he_sd}}$ & $\result{ext_code_llama_oft_0p0003_hep_mean}_{\pm\result{ext_code_llama_oft_0p0003_hep_sd}}$ & $\result{ext_code_llama_oft_0p0003_mb_mean}_{\pm\result{ext_code_llama_oft_0p0003_mb_sd}}$ & $\result{ext_code_llama_oft_0p0003_mbp_mean}_{\pm\result{ext_code_llama_oft_0p0003_mbp_sd}}$ \\
 & \textcolor{DoRAColor}{DoRA} & $3\times10^{-4}$ & $\result{ext_code_llama_dora_0p0003_he_mean}_{\pm\result{ext_code_llama_dora_0p0003_he_sd}}$ & $\result{ext_code_llama_dora_0p0003_hep_mean}_{\pm\result{ext_code_llama_dora_0p0003_hep_sd}}$ & $\result{ext_code_llama_dora_0p0003_mb_mean}_{\pm\result{ext_code_llama_dora_0p0003_mb_sd}}$ & $\result{ext_code_llama_dora_0p0003_mbp_mean}_{\pm\result{ext_code_llama_dora_0p0003_mbp_sd}}$ \\
 & \textcolor{PiSSAColor}{PiSSA} & $1\times10^{-4}$ & $\result{ext_code_llama_pissa_0p0001_he_mean}_{\pm\result{ext_code_llama_pissa_0p0001_he_sd}}$ & $\result{ext_code_llama_pissa_0p0001_hep_mean}_{\pm\result{ext_code_llama_pissa_0p0001_hep_sd}}$ & $\result{ext_code_llama_pissa_0p0001_mb_mean}_{\pm\result{ext_code_llama_pissa_0p0001_mb_sd}}$ & $\result{ext_code_llama_pissa_0p0001_mbp_mean}_{\pm\result{ext_code_llama_pissa_0p0001_mbp_sd}}$ \\
 & \textcolor{MiLoRAColor}{MiLoRA} & $3\times10^{-4}$ & $\result{ext_code_llama_milora_0p0003_he_mean}_{\pm\result{ext_code_llama_milora_0p0003_he_sd}}$ & $\result{ext_code_llama_milora_0p0003_hep_mean}_{\pm\result{ext_code_llama_milora_0p0003_hep_sd}}$ & $\result{ext_code_llama_milora_0p0003_mb_mean}_{\pm\result{ext_code_llama_milora_0p0003_mb_sd}}$ & $\result{ext_code_llama_milora_0p0003_mbp_mean}_{\pm\result{ext_code_llama_milora_0p0003_mbp_sd}}$ \\
 & \textcolor{HRAColor}{HRA} & $1\times10^{-3}$ & $\result{ext_code_llama_hra_0p001_he_mean}_{\pm\result{ext_code_llama_hra_0p001_he_sd}}$ & $\result{ext_code_llama_hra_0p001_hep_mean}_{\pm\result{ext_code_llama_hra_0p001_hep_sd}}$ & $\result{ext_code_llama_hra_0p001_mb_mean}_{\pm\result{ext_code_llama_hra_0p001_mb_sd}}$ & $\result{ext_code_llama_hra_0p001_mbp_mean}_{\pm\result{ext_code_llama_hra_0p001_mbp_sd}}$ \\
\addlinespace[3pt]
\bottomrule\end{tabular*}

\end{center}
\end{table}

\FloatBarrier
\paragraph{Task gains and retention costs.}
All six methods improve Llama's HumanEval+ and MBPP+ means over the base model, with rankings that differ across benchmarks and retention. PiSSA leads on HumanEval+, alongside higher retention costs. Together, these results extend the distinction between task performance and forgetting to coding, with the size of the task gain depending on the backbone.

\FloatBarrier
\subsection{Six-method cat comparison across learning rates and capacities}\label{app:flux-four}\label{app:diffusion}\label{app:extendedimage}
The complete cat study crosses six methods, four approximately matched parameter budgets, eight LRs and seeds 0, 1 and 2, giving 576 successful runs and 192 configuration means. All methods share the cat split (14 training, two validation and four test images), adapted modules, training exposure and evaluation steps. LoRA, DoRA, PiSSA and MiLoRA use ranks 4/8/16/32, OFT uses block sizes 32/64/128/256, and HRA uses 16/32/62/124 Householder vectors. PiSSA/MiLoRA ranks refer to training, with doubled ranks used to serialize their differences from the base model.

The LR grid is $\{5\times10^{-6},10^{-5},3\times10^{-5},5\times10^{-5},10^{-4},3\times10^{-4},5\times10^{-4},10^{-3}\}$. Within each method/capacity, mean validation DINO at step 700 selects one LR shared across seeds, with ties favoring smaller LR. Test and retention use step 750. \Cref{tab:flux-selected-all} gives all 24 selections and parameter counts, with sample seed SD. 

\begin{table}[!htbp]
\caption{Complete validation-selected FLUX cat results. Budgets are approximately matched, with HRA using 16/32/62/124 vectors and DoRA adding magnitude parameters. Parameters are in millions. Validation DINO at step 700 selects LR separately for each method/capacity. Step 750 supplies test DINO, general drift and general CLIP. Scores are $\times100$.}
\label{tab:flux-selected-all}
\begin{center}
\setlength{\tabcolsep}{3pt}
\begin{tabular*}{\linewidth}{@{\extracolsep{\fill}}lccccccc@{}}\toprule
Method & Capacity & Params & LR & Validation & Test DINO & Gen. drift & Gen. CLIP \\ \midrule
\textcolor{LoRAColor}{LoRA} & r4 & \result{ext_flux_flux_lora_cap4_0p0003_params} & $3\times10^{-4}$ & $\result{ext_flux_flux_lora_cap4_0p0003_validation_mean}_{\pm\result{ext_flux_flux_lora_cap4_0p0003_validation_sd}}$ & $\result{ext_flux_flux_lora_cap4_0p0003_dino_mean}_{\pm\result{ext_flux_flux_lora_cap4_0p0003_dino_sd}}$ & $\result{ext_flux_flux_lora_cap4_0p0003_drift_mean}_{\pm\result{ext_flux_flux_lora_cap4_0p0003_drift_sd}}$ & $\result{ext_flux_flux_lora_cap4_0p0003_clip_mean}_{\pm\result{ext_flux_flux_lora_cap4_0p0003_clip_sd}}$ \\
\textcolor{OFTColor}{OFT} & b32 & \result{ext_flux_flux_oft_cap32_0p0003_params} & $3\times10^{-4}$ & $\result{ext_flux_flux_oft_cap32_0p0003_validation_mean}_{\pm\result{ext_flux_flux_oft_cap32_0p0003_validation_sd}}$ & $\result{ext_flux_flux_oft_cap32_0p0003_dino_mean}_{\pm\result{ext_flux_flux_oft_cap32_0p0003_dino_sd}}$ & $\result{ext_flux_flux_oft_cap32_0p0003_drift_mean}_{\pm\result{ext_flux_flux_oft_cap32_0p0003_drift_sd}}$ & $\result{ext_flux_flux_oft_cap32_0p0003_clip_mean}_{\pm\result{ext_flux_flux_oft_cap32_0p0003_clip_sd}}$ \\
\textcolor{DoRAColor}{DoRA} & r4 & \result{ext_flux_flux_dora_cap4_0p0005_params} & $5\times10^{-4}$ & $\result{ext_flux_flux_dora_cap4_0p0005_validation_mean}_{\pm\result{ext_flux_flux_dora_cap4_0p0005_validation_sd}}$ & $\result{ext_flux_flux_dora_cap4_0p0005_dino_mean}_{\pm\result{ext_flux_flux_dora_cap4_0p0005_dino_sd}}$ & $\result{ext_flux_flux_dora_cap4_0p0005_drift_mean}_{\pm\result{ext_flux_flux_dora_cap4_0p0005_drift_sd}}$ & $\result{ext_flux_flux_dora_cap4_0p0005_clip_mean}_{\pm\result{ext_flux_flux_dora_cap4_0p0005_clip_sd}}$ \\
\textcolor{PiSSAColor}{PiSSA} & r4 & \result{ext_flux_flux_pissa_cap4_0p0001_params} & $1\times10^{-4}$ & $\result{ext_flux_flux_pissa_cap4_0p0001_validation_mean}_{\pm\result{ext_flux_flux_pissa_cap4_0p0001_validation_sd}}$ & $\result{ext_flux_flux_pissa_cap4_0p0001_dino_mean}_{\pm\result{ext_flux_flux_pissa_cap4_0p0001_dino_sd}}$ & $\result{ext_flux_flux_pissa_cap4_0p0001_drift_mean}_{\pm\result{ext_flux_flux_pissa_cap4_0p0001_drift_sd}}$ & $\result{ext_flux_flux_pissa_cap4_0p0001_clip_mean}_{\pm\result{ext_flux_flux_pissa_cap4_0p0001_clip_sd}}$ \\
\textcolor{MiLoRAColor}{MiLoRA} & r4 & \result{ext_flux_flux_milora_cap4_0p0001_params} & $1\times10^{-4}$ & $\result{ext_flux_flux_milora_cap4_0p0001_validation_mean}_{\pm\result{ext_flux_flux_milora_cap4_0p0001_validation_sd}}$ & $\result{ext_flux_flux_milora_cap4_0p0001_dino_mean}_{\pm\result{ext_flux_flux_milora_cap4_0p0001_dino_sd}}$ & $\result{ext_flux_flux_milora_cap4_0p0001_drift_mean}_{\pm\result{ext_flux_flux_milora_cap4_0p0001_drift_sd}}$ & $\result{ext_flux_flux_milora_cap4_0p0001_clip_mean}_{\pm\result{ext_flux_flux_milora_cap4_0p0001_clip_sd}}$ \\
\textcolor{HRAColor}{HRA} & h16 & \result{ext_flux_flux_hra_cap16_0p0005_params} & $5\times10^{-4}$ & $\result{ext_flux_flux_hra_cap16_0p0005_validation_mean}_{\pm\result{ext_flux_flux_hra_cap16_0p0005_validation_sd}}$ & $\result{ext_flux_flux_hra_cap16_0p0005_dino_mean}_{\pm\result{ext_flux_flux_hra_cap16_0p0005_dino_sd}}$ & $\result{ext_flux_flux_hra_cap16_0p0005_drift_mean}_{\pm\result{ext_flux_flux_hra_cap16_0p0005_drift_sd}}$ & $\result{ext_flux_flux_hra_cap16_0p0005_clip_mean}_{\pm\result{ext_flux_flux_hra_cap16_0p0005_clip_sd}}$ \\
\midrule
\textcolor{LoRAColor}{LoRA} & r8 & \result{ext_flux_flux_lora_0p0001_params} & $1\times10^{-4}$ & $\result{ext_flux_flux_lora_0p0001_validation_mean}_{\pm\result{ext_flux_flux_lora_0p0001_validation_sd}}$ & $\result{ext_flux_flux_lora_0p0001_dino_mean}_{\pm\result{ext_flux_flux_lora_0p0001_dino_sd}}$ & $\result{ext_flux_flux_lora_0p0001_drift_mean}_{\pm\result{ext_flux_flux_lora_0p0001_drift_sd}}$ & $\result{ext_flux_flux_lora_0p0001_clip_mean}_{\pm\result{ext_flux_flux_lora_0p0001_clip_sd}}$ \\
\textcolor{OFTColor}{OFT} & b64 & \result{ext_flux_flux_oft_5em05_params} & $5\times10^{-5}$ & $\result{ext_flux_flux_oft_5em05_validation_mean}_{\pm\result{ext_flux_flux_oft_5em05_validation_sd}}$ & $\result{ext_flux_flux_oft_5em05_dino_mean}_{\pm\result{ext_flux_flux_oft_5em05_dino_sd}}$ & $\result{ext_flux_flux_oft_5em05_drift_mean}_{\pm\result{ext_flux_flux_oft_5em05_drift_sd}}$ & $\result{ext_flux_flux_oft_5em05_clip_mean}_{\pm\result{ext_flux_flux_oft_5em05_clip_sd}}$ \\
\textcolor{DoRAColor}{DoRA} & r8 & \result{ext_flux_flux_dora_0p0003_params} & $3\times10^{-4}$ & $\result{ext_flux_flux_dora_0p0003_validation_mean}_{\pm\result{ext_flux_flux_dora_0p0003_validation_sd}}$ & $\result{ext_flux_flux_dora_0p0003_dino_mean}_{\pm\result{ext_flux_flux_dora_0p0003_dino_sd}}$ & $\result{ext_flux_flux_dora_0p0003_drift_mean}_{\pm\result{ext_flux_flux_dora_0p0003_drift_sd}}$ & $\result{ext_flux_flux_dora_0p0003_clip_mean}_{\pm\result{ext_flux_flux_dora_0p0003_clip_sd}}$ \\
\textcolor{PiSSAColor}{PiSSA} & r8 & \result{ext_flux_flux_pissa_5em05_params} & $5\times10^{-5}$ & $\result{ext_flux_flux_pissa_5em05_validation_mean}_{\pm\result{ext_flux_flux_pissa_5em05_validation_sd}}$ & $\result{ext_flux_flux_pissa_5em05_dino_mean}_{\pm\result{ext_flux_flux_pissa_5em05_dino_sd}}$ & $\result{ext_flux_flux_pissa_5em05_drift_mean}_{\pm\result{ext_flux_flux_pissa_5em05_drift_sd}}$ & $\result{ext_flux_flux_pissa_5em05_clip_mean}_{\pm\result{ext_flux_flux_pissa_5em05_clip_sd}}$ \\
\textcolor{MiLoRAColor}{MiLoRA} & r8 & \result{ext_flux_flux_milora_5em05_params} & $5\times10^{-5}$ & $\result{ext_flux_flux_milora_5em05_validation_mean}_{\pm\result{ext_flux_flux_milora_5em05_validation_sd}}$ & $\result{ext_flux_flux_milora_5em05_dino_mean}_{\pm\result{ext_flux_flux_milora_5em05_dino_sd}}$ & $\result{ext_flux_flux_milora_5em05_drift_mean}_{\pm\result{ext_flux_flux_milora_5em05_drift_sd}}$ & $\result{ext_flux_flux_milora_5em05_clip_mean}_{\pm\result{ext_flux_flux_milora_5em05_clip_sd}}$ \\
\textcolor{HRAColor}{HRA} & h32 & \result{ext_flux_flux_hra_0p0003_params} & $3\times10^{-4}$ & $\result{ext_flux_flux_hra_0p0003_validation_mean}_{\pm\result{ext_flux_flux_hra_0p0003_validation_sd}}$ & $\result{ext_flux_flux_hra_0p0003_dino_mean}_{\pm\result{ext_flux_flux_hra_0p0003_dino_sd}}$ & $\result{ext_flux_flux_hra_0p0003_drift_mean}_{\pm\result{ext_flux_flux_hra_0p0003_drift_sd}}$ & $\result{ext_flux_flux_hra_0p0003_clip_mean}_{\pm\result{ext_flux_flux_hra_0p0003_clip_sd}}$ \\
\midrule
\textcolor{LoRAColor}{LoRA} & r16 & \result{ext_flux_flux_lora_cap16_0p0001_params} & $1\times10^{-4}$ & $\result{ext_flux_flux_lora_cap16_0p0001_validation_mean}_{\pm\result{ext_flux_flux_lora_cap16_0p0001_validation_sd}}$ & $\result{ext_flux_flux_lora_cap16_0p0001_dino_mean}_{\pm\result{ext_flux_flux_lora_cap16_0p0001_dino_sd}}$ & $\result{ext_flux_flux_lora_cap16_0p0001_drift_mean}_{\pm\result{ext_flux_flux_lora_cap16_0p0001_drift_sd}}$ & $\result{ext_flux_flux_lora_cap16_0p0001_clip_mean}_{\pm\result{ext_flux_flux_lora_cap16_0p0001_clip_sd}}$ \\
\textcolor{OFTColor}{OFT} & b128 & \result{ext_flux_flux_oft_cap128_0p0001_params} & $1\times10^{-4}$ & $\result{ext_flux_flux_oft_cap128_0p0001_validation_mean}_{\pm\result{ext_flux_flux_oft_cap128_0p0001_validation_sd}}$ & $\result{ext_flux_flux_oft_cap128_0p0001_dino_mean}_{\pm\result{ext_flux_flux_oft_cap128_0p0001_dino_sd}}$ & $\result{ext_flux_flux_oft_cap128_0p0001_drift_mean}_{\pm\result{ext_flux_flux_oft_cap128_0p0001_drift_sd}}$ & $\result{ext_flux_flux_oft_cap128_0p0001_clip_mean}_{\pm\result{ext_flux_flux_oft_cap128_0p0001_clip_sd}}$ \\
\textcolor{DoRAColor}{DoRA} & r16 & \result{ext_flux_flux_dora_cap16_5em05_params} & $5\times10^{-5}$ & $\result{ext_flux_flux_dora_cap16_5em05_validation_mean}_{\pm\result{ext_flux_flux_dora_cap16_5em05_validation_sd}}$ & $\result{ext_flux_flux_dora_cap16_5em05_dino_mean}_{\pm\result{ext_flux_flux_dora_cap16_5em05_dino_sd}}$ & $\result{ext_flux_flux_dora_cap16_5em05_drift_mean}_{\pm\result{ext_flux_flux_dora_cap16_5em05_drift_sd}}$ & $\result{ext_flux_flux_dora_cap16_5em05_clip_mean}_{\pm\result{ext_flux_flux_dora_cap16_5em05_clip_sd}}$ \\
\textcolor{PiSSAColor}{PiSSA} & r16 & \result{ext_flux_flux_pissa_cap16_5em05_params} & $5\times10^{-5}$ & $\result{ext_flux_flux_pissa_cap16_5em05_validation_mean}_{\pm\result{ext_flux_flux_pissa_cap16_5em05_validation_sd}}$ & $\result{ext_flux_flux_pissa_cap16_5em05_dino_mean}_{\pm\result{ext_flux_flux_pissa_cap16_5em05_dino_sd}}$ & $\result{ext_flux_flux_pissa_cap16_5em05_drift_mean}_{\pm\result{ext_flux_flux_pissa_cap16_5em05_drift_sd}}$ & $\result{ext_flux_flux_pissa_cap16_5em05_clip_mean}_{\pm\result{ext_flux_flux_pissa_cap16_5em05_clip_sd}}$ \\
\textcolor{MiLoRAColor}{MiLoRA} & r16 & \result{ext_flux_flux_milora_cap16_0p0001_params} & $1\times10^{-4}$ & $\result{ext_flux_flux_milora_cap16_0p0001_validation_mean}_{\pm\result{ext_flux_flux_milora_cap16_0p0001_validation_sd}}$ & $\result{ext_flux_flux_milora_cap16_0p0001_dino_mean}_{\pm\result{ext_flux_flux_milora_cap16_0p0001_dino_sd}}$ & $\result{ext_flux_flux_milora_cap16_0p0001_drift_mean}_{\pm\result{ext_flux_flux_milora_cap16_0p0001_drift_sd}}$ & $\result{ext_flux_flux_milora_cap16_0p0001_clip_mean}_{\pm\result{ext_flux_flux_milora_cap16_0p0001_clip_sd}}$ \\
\textcolor{HRAColor}{HRA} & h62 & \result{ext_flux_flux_hra_cap62_0p0005_params} & $5\times10^{-4}$ & $\result{ext_flux_flux_hra_cap62_0p0005_validation_mean}_{\pm\result{ext_flux_flux_hra_cap62_0p0005_validation_sd}}$ & $\result{ext_flux_flux_hra_cap62_0p0005_dino_mean}_{\pm\result{ext_flux_flux_hra_cap62_0p0005_dino_sd}}$ & $\result{ext_flux_flux_hra_cap62_0p0005_drift_mean}_{\pm\result{ext_flux_flux_hra_cap62_0p0005_drift_sd}}$ & $\result{ext_flux_flux_hra_cap62_0p0005_clip_mean}_{\pm\result{ext_flux_flux_hra_cap62_0p0005_clip_sd}}$ \\
\midrule
\textcolor{LoRAColor}{LoRA} & r32 & \result{ext_flux_flux_lora_cap32_5em05_params} & $5\times10^{-5}$ & $\result{ext_flux_flux_lora_cap32_5em05_validation_mean}_{\pm\result{ext_flux_flux_lora_cap32_5em05_validation_sd}}$ & $\result{ext_flux_flux_lora_cap32_5em05_dino_mean}_{\pm\result{ext_flux_flux_lora_cap32_5em05_dino_sd}}$ & $\result{ext_flux_flux_lora_cap32_5em05_drift_mean}_{\pm\result{ext_flux_flux_lora_cap32_5em05_drift_sd}}$ & $\result{ext_flux_flux_lora_cap32_5em05_clip_mean}_{\pm\result{ext_flux_flux_lora_cap32_5em05_clip_sd}}$ \\
\textcolor{OFTColor}{OFT} & b256 & \result{ext_flux_flux_oft_cap256_0p0001_params} & $1\times10^{-4}$ & $\result{ext_flux_flux_oft_cap256_0p0001_validation_mean}_{\pm\result{ext_flux_flux_oft_cap256_0p0001_validation_sd}}$ & $\result{ext_flux_flux_oft_cap256_0p0001_dino_mean}_{\pm\result{ext_flux_flux_oft_cap256_0p0001_dino_sd}}$ & $\result{ext_flux_flux_oft_cap256_0p0001_drift_mean}_{\pm\result{ext_flux_flux_oft_cap256_0p0001_drift_sd}}$ & $\result{ext_flux_flux_oft_cap256_0p0001_clip_mean}_{\pm\result{ext_flux_flux_oft_cap256_0p0001_clip_sd}}$ \\
\textcolor{DoRAColor}{DoRA} & r32 & \result{ext_flux_flux_dora_cap32_0p0001_params} & $1\times10^{-4}$ & $\result{ext_flux_flux_dora_cap32_0p0001_validation_mean}_{\pm\result{ext_flux_flux_dora_cap32_0p0001_validation_sd}}$ & $\result{ext_flux_flux_dora_cap32_0p0001_dino_mean}_{\pm\result{ext_flux_flux_dora_cap32_0p0001_dino_sd}}$ & $\result{ext_flux_flux_dora_cap32_0p0001_drift_mean}_{\pm\result{ext_flux_flux_dora_cap32_0p0001_drift_sd}}$ & $\result{ext_flux_flux_dora_cap32_0p0001_clip_mean}_{\pm\result{ext_flux_flux_dora_cap32_0p0001_clip_sd}}$ \\
\textcolor{PiSSAColor}{PiSSA} & r32 & \result{ext_flux_flux_pissa_cap32_5em05_params} & $5\times10^{-5}$ & $\result{ext_flux_flux_pissa_cap32_5em05_validation_mean}_{\pm\result{ext_flux_flux_pissa_cap32_5em05_validation_sd}}$ & $\result{ext_flux_flux_pissa_cap32_5em05_dino_mean}_{\pm\result{ext_flux_flux_pissa_cap32_5em05_dino_sd}}$ & $\result{ext_flux_flux_pissa_cap32_5em05_drift_mean}_{\pm\result{ext_flux_flux_pissa_cap32_5em05_drift_sd}}$ & $\result{ext_flux_flux_pissa_cap32_5em05_clip_mean}_{\pm\result{ext_flux_flux_pissa_cap32_5em05_clip_sd}}$ \\
\textcolor{MiLoRAColor}{MiLoRA} & r32 & \result{ext_flux_flux_milora_cap32_0p0001_params} & $1\times10^{-4}$ & $\result{ext_flux_flux_milora_cap32_0p0001_validation_mean}_{\pm\result{ext_flux_flux_milora_cap32_0p0001_validation_sd}}$ & $\result{ext_flux_flux_milora_cap32_0p0001_dino_mean}_{\pm\result{ext_flux_flux_milora_cap32_0p0001_dino_sd}}$ & $\result{ext_flux_flux_milora_cap32_0p0001_drift_mean}_{\pm\result{ext_flux_flux_milora_cap32_0p0001_drift_sd}}$ & $\result{ext_flux_flux_milora_cap32_0p0001_clip_mean}_{\pm\result{ext_flux_flux_milora_cap32_0p0001_clip_sd}}$ \\
\textcolor{HRAColor}{HRA} & h124 & \result{ext_flux_flux_hra_cap124_0p0003_params} & $3\times10^{-4}$ & $\result{ext_flux_flux_hra_cap124_0p0003_validation_mean}_{\pm\result{ext_flux_flux_hra_cap124_0p0003_validation_sd}}$ & $\result{ext_flux_flux_hra_cap124_0p0003_dino_mean}_{\pm\result{ext_flux_flux_hra_cap124_0p0003_dino_sd}}$ & $\result{ext_flux_flux_hra_cap124_0p0003_drift_mean}_{\pm\result{ext_flux_flux_hra_cap124_0p0003_drift_sd}}$ & $\result{ext_flux_flux_hra_cap124_0p0003_clip_mean}_{\pm\result{ext_flux_flux_hra_cap124_0p0003_clip_sd}}$ \\
\bottomrule\end{tabular*}

\end{center}
\end{table}

\FloatBarrier

\paragraph{Low selected drift and capacity-dependent fidelity.}
LoRA has the lowest selected mean general drift at \result{cross_cat_lora_lowest} of \result{cross_cat_selected_settings} budgets and ranks among the lowest two at all \result{cross_cat_lora_lowest_two}. PiSSA has greater drift than LoRA in all \result{cross_cat_pissa_above_lora_seeds} paired seeds. PiSSA has the highest selected drift at the small budget, OFT at the smallest and largest budgets, and HRA at the medium budget. DoRA exceeds LoRA's mean test DINO at every selected budget and leads at the smallest and largest. MiLoRA has lower mean test DINO than LoRA at all four budgets, with a mean CLIP advantage only at the small budget. These results support LoRA's recurring retention advantage and DoRA's higher mean task scores, alongside setting-dependent outcomes for the other methods.

\begin{figure}[!htbp]\centering
\includegraphics[width=\linewidth]{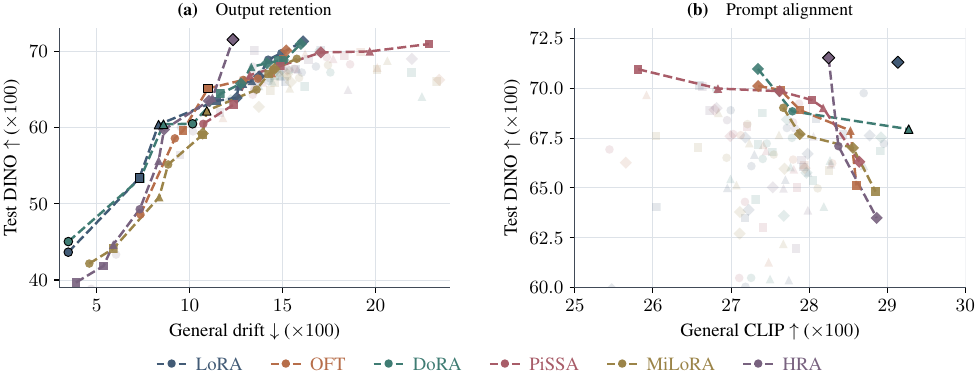}
\caption{Six-method FLUX cat test trade-offs, focused on the frontier regions. Shapes denote increasing capacity. Dashed lines and solid points mark within-method frontiers, and black borders mark the pooled frontier. Dominated points are faint. Frontiers use all 192 configurations before axis clipping, with separate y ranges for each panel. Test DINO describes final-step outcomes, while validation DINO selects LR.}
\label{fig:flux-four-pareto}
\end{figure}
\FloatBarrier

\paragraph{Retention metrics.}
General drift and CLIP use five unrelated prompts, while class drift uses five cat prompts. All six methods have these measurements. Drift measures preservation of generated outputs, and CLIP measures prompt alignment. Appendix~\ref{app:flux22} examines the role of spectral components through six-method interventions at the largest budget.

\FloatBarrier
\subsection{28-object comparison}
\Cref{tab:diffusion}b reports every validation-selected capacity pair. Each score first averages the 28 objects within a seed, then reports the mean and sample SD across seeds. The shared-LR selection and evaluation banks are specified in Appendix~\ref{app:settings}.

\paragraph{Retention across objects and metrics.}
At the validation-selected settings, OFT has lower COCO flow loss at three of four capacities. LoRA has lower general-image drift and higher prompt alignment at every capacity. The same distinction appears at a shared LR of $5\times10^{-5}$, where OFT b64 has lower COCO loss than LoRA r8, but higher image drift and lower CLIP. Flow loss evaluates velocity predictions on fixed noisy latents, whereas drift and CLIP evaluate the final generated images. The object comparison supports LoRA's retention advantage in generated outputs and shows why the metrics capture different aspects of preservation.

\subsection{Retention across all six methods}\label{app:retention-patterns}
\expandafter\def\csname result@rob_lora_first_original\endcsname{9}
\expandafter\def\csname result@rob_lora_first_other_seeds\endcsname{8}
\expandafter\def\csname result@rob_original_rank\endcsname{1.58}
\expandafter\def\csname result@rob_other_seeds_rank\endcsname{2.08}

We compare LoRA's forgetting with \emph{each of the other five methods} across twelve settings, comprising six mathematics model/capacity pairs, two coding backbones and four cat budgets. Each method's LR is chosen by task performance on the development set. We compare forgetting within each setting, so language NLL and image drift remain on their own scales.

\paragraph{LoRA most consistently limits forgetting.}
LoRA has the lowest mean forgetting in \result{rob_lora_first_original}/12 settings. OFT, MiLoRA and HRA each have the lowest forgetting in one setting, and DoRA and PiSSA in none (\Cref{tab:retention-all-methods}). LoRA also has the lowest average rank, \result{rob_original_rank}, followed by DoRA at 2.75. The three exceptions are small-capacity Qwen, where LoRA's outlying high-LR selection forgets most, large-capacity Llama, where OFT has the lowest forgetting, and cat rank 8, where HRA has the lowest forgetting.

\begin{table}[!htbp]
\caption{Forgetting across all six methods. The lowest-forgetting columns count settings in which a method has the lowest mean forgetting. A rank of 1 means least forgetting, and lower average rank is better. The right columns choose LR using the other two seeds, evaluate the remaining seed, then average those evaluations. LoRA has the lowest forgetting in the most settings and the lowest average rank under both selections.}
\label{tab:retention-all-methods}
\begin{center}
\begin{tabular*}{\linewidth}{@{\extracolsep{\fill}}lrrrr@{}}\toprule
& \multicolumn{2}{c}{Original selection} & \multicolumn{2}{c}{LR chosen on other seeds}\\Method & \shortstack{Lowest\\forgetting} & Average rank & \shortstack{Lowest\\forgetting} & Average rank\\\midrule
\textcolor{LoRAColor}{LoRA} & \textbf{9/12} & \textbf{1.58} & \textbf{8/12} & \textbf{2.08} \\
\textcolor{OFTColor}{OFT} & 1/12 & 4.58 & 1/12 & 4.33 \\
\textcolor{DoRAColor}{DoRA} & 0/12 & 2.75 & 1/12 & 2.50 \\
\textcolor{PiSSAColor}{PiSSA} & 0/12 & 5.25 & 0/12 & 5.42 \\
\textcolor{MiLoRAColor}{MiLoRA} & 1/12 & 3.42 & 1/12 & 3.25 \\
\textcolor{HRAColor}{HRA} & 1/12 & 3.42 & 1/12 & 3.42 \\
\bottomrule\end{tabular*}
\end{center}
\end{table}

\paragraph{Consistency across selection and weighting choices.}
First, we choose LR using two seeds and measure forgetting on the remaining seed, repeating for every seed. The original grids, eligibility rules and smaller-LR tie rule remain unchanged. LoRA still has the lowest mean forgetting in \result{rob_lora_first_other_seeds}/12 settings. Each other method does so in at most one. Its average rank remains lowest (\result{rob_other_seeds_rank}, versus 2.50 for DoRA).
Second, we give each model/task group equal importance, preventing groups with more capacities from dominating the summary.
Third, we remove each backbone or individual setting in turn.
Finally, we compare results using one language seed and one image seed at a time, covering all nine combinations rather than matching unrelated seed numbers.

Across all these checks, LoRA has the lowest forgetting in more settings than any other method (\Cref{tab:retention-robustness}). This supports the claim that LoRA most consistently limits forgetting across the tested methods. After LR selection on other seeds, DoRA has a better average rank in three seed combinations and when FLUX is removed.

\begin{table}[!htbp]
\caption{Robustness of the retention comparison. Frequency columns count versions in which LoRA has the lowest forgetting in more settings than any other method. Average-rank columns count versions in which LoRA has the lowest average rank. Equal weighting averages within the five model/task groups before averaging groups. The nine seed combinations cross all language seeds with all image seeds.}
\label{tab:retention-robustness}
\begin{center}
\begin{tabular*}{\linewidth}{@{\extracolsep{\fill}}lrrrr@{}}\toprule
& \multicolumn{2}{c}{Original selection} & \multicolumn{2}{c}{LR chosen on other seeds}\\Check & Frequency & Average rank & Frequency & Average rank\\\midrule
Equal model/task weights & 1/1 & 1/1 & 1/1 & 1/1 \\
Remove one backbone & 3/3 & 3/3 & 3/3 & 2/3 \\
Remove one setting & 12/12 & 12/12 & 12/12 & 12/12 \\
Use individual seeds & 9/9 & 9/9 & 9/9 & 6/9 \\
\bottomrule\end{tabular*}
\end{center}
\end{table}

\paragraph{Task scores and PiSSA's retention costs.}
DoRA has higher mean task scores than LoRA in 10/14 comparisons, covering two of six mathematics settings, all four coding comparisons and all four cat budgets. This supports the reported trend in mean task scores. Its coding and cat gains accompany additional forgetting.

With backbone, capacity and LR fixed across all six methods, PiSSA forgets most in all \result{cross_math_matched_pissa_highest}/\result{cross_math_matched_blocks} complete mathematics comparisons and \result{cross_coding_matched_pissa_highest}/\result{cross_coding_matched_blocks} coding comparisons, including all \result{cross_math_matched_pissa_seed_highest} and \result{cross_coding_matched_pissa_seed_highest} seed comparisons. In cat, it ranks highest in \result{cross_cat_matched_pissa_highest}/\result{cross_cat_matched_blocks} comparisons and among the highest two in all \result{cross_cat_matched_pissa_highest_two}. HRA ranks lowest in \result{cross_cat_hra_matched_lowest}/\result{cross_cat_matched_blocks}. These matched-LR comparisons reinforce the finding that PiSSA often incurs greater retention costs and show that HRA can retain well in specific settings.

\paragraph{LoRA versus OFT across objects.}
After averaging seeds within each object, LoRA has lower drift than OFT for 23, 27, 28 and 28 of the 28 objects across increasing budgets. This extends LoRA's image-drift advantage beyond cat personalization.

\FloatBarrier
\section{Geometric measurements and language spectral interventions}\label{app:complete-interventions}
\FloatBarrier
\subsection{Checkpoints, controls and reference models}
The intervention study covers 108 checkpoints selected by mean development accuracy within each model/method/capacity, with 22 interventions or controls per checkpoint. Llama HRA uses the extended-grid selections in Appendix~\ref{app:hra-extension}. All edits use the same GSM8K development and general-text banks. The tables below extend the largest-capacity main figures to all capacities and selected LRs.

Task and NLL changes use the trained model as reference. Value-restoration controls compare against reconstruction, while component sufficiency and factorial analyses use full value restoration. Forgetting recovery is the fraction of the trained model's NLL increase recovered by an edit, computed within each checkpoint before seed averaging.

Edits jointly cover all adapted matrices. SVD computations use FP32, followed by BF16 storage and evaluation. The dominant, intermediate and trailing component groups divide the pretrained spectrum into approximate thirds, with boundaries adjusted to nearby spectral gaps. Plots label these groups top, middle and tail.

For the collection $\mathcal M$ of adapted matrices, the normalized displacements are
\begin{equation*}
 d_W^2=\frac{\sum_{M\in\mathcal M}\left\|\bm W_{*,M}-\bm W_{0,M}\right\|_{\mathrm F}^2}{\sum_{M\in\mathcal M}\left\|\bm W_{0,M}\right\|_{\mathrm F}^2},\qquad
 d_S^2=\frac{\sum_{M\in\mathcal M}\left\|\bm s_{*,M}-\bm s_{0,M}\right\|_2^2}{\sum_{M\in\mathcal M}\left\|\bm s_{0,M}\right\|_2^2}.
\end{equation*}

Each checkpoint pools all 196 Qwen or 224 Llama adapted matrices (80 for the FLUX results in Appendix~\ref{app:flux22}), weighting their contributions through the squared norms.

\begin{table}[!htbp]
\caption{Qwen2.5-7B spectral displacement at every selected capacity and LR.}
\label{tab:peft-spectrum-qwen}
\begin{center}
\setlength{\tabcolsep}{3pt}
\begin{tabular*}{\linewidth}{@{\extracolsep{\fill}}lrrrrr@{}}\toprule
Method & Capacity & LR & $100d_W$ & $100d_S$ & $100d_S/d_W$\\\midrule
\textcolor{LoRAColor}{LoRA} & 7 & $1\times10^{-3}$ & $\result{rev_qwen_lora_7_trained_d_W_mean}_{\pm \result{rev_qwen_lora_7_trained_d_W_sd}}$ & $\result{rev_qwen_lora_7_trained_d_S_mean}_{\pm \result{rev_qwen_lora_7_trained_d_S_sd}}$ & $\result{rev_qwen_lora_7_trained_ratio_mean}_{\pm \result{rev_qwen_lora_7_trained_ratio_sd}}$\\
\textcolor{LoRAColor}{LoRA} & 14 & $1\times10^{-5}$ & $\result{rev_qwen_lora_14_trained_d_W_mean}_{\pm \result{rev_qwen_lora_14_trained_d_W_sd}}$ & $\result{rev_qwen_lora_14_trained_d_S_mean}_{\pm \result{rev_qwen_lora_14_trained_d_S_sd}}$ & $\result{rev_qwen_lora_14_trained_ratio_mean}_{\pm \result{rev_qwen_lora_14_trained_ratio_sd}}$\\
\textcolor{LoRAColor}{LoRA} & 28 & $3\times10^{-5}$ & $\result{rev_qwen_lora_28_trained_d_W_mean}_{\pm \result{rev_qwen_lora_28_trained_d_W_sd}}$ & $\result{rev_qwen_lora_28_trained_d_S_mean}_{\pm \result{rev_qwen_lora_28_trained_d_S_sd}}$ & $\result{rev_qwen_lora_28_trained_ratio_mean}_{\pm \result{rev_qwen_lora_28_trained_ratio_sd}}$\\
\addlinespace[2pt]
\textcolor{OFTColor}{OFT} & 32 & $3\times10^{-4}$ & $\result{rev_qwen_oft_32_trained_d_W_mean}_{\pm \result{rev_qwen_oft_32_trained_d_W_sd}}$ & $\result{rev_qwen_oft_32_trained_d_S_mean}_{\pm \result{rev_qwen_oft_32_trained_d_S_sd}}$ & $\result{rev_qwen_oft_32_trained_ratio_mean}_{\pm \result{rev_qwen_oft_32_trained_ratio_sd}}$\\
\textcolor{OFTColor}{OFT} & 64 & $3\times10^{-4}$ & $\result{rev_qwen_oft_64_trained_d_W_mean}_{\pm \result{rev_qwen_oft_64_trained_d_W_sd}}$ & $\result{rev_qwen_oft_64_trained_d_S_mean}_{\pm \result{rev_qwen_oft_64_trained_d_S_sd}}$ & $\result{rev_qwen_oft_64_trained_ratio_mean}_{\pm \result{rev_qwen_oft_64_trained_ratio_sd}}$\\
\bottomrule\end{tabular*}
\end{center}
\end{table}
\FloatBarrier
\begin{table}[!htbp]
\caption{Qwen2.5-7B singular-value restoration and reconstruction controls, with accuracy in percent.}
\label{tab:peft-restoration-qwen}
\begin{center}
\setlength{\tabcolsep}{3pt}
\begin{tabular*}{\linewidth}{@{\extracolsep{\fill}}lrrrrrr@{}}\toprule
Method & Capacity & LR & Trained acc. & Rebuilt acc. & Restored acc. & $\Delta$NLL vs. rebuilt\\\midrule
\textcolor{LoRAColor}{LoRA} & 7 & $1\times10^{-3}$ & $\result{rev_qwen_lora_7_trained_accuracy_mean}_{\pm \result{rev_qwen_lora_7_trained_accuracy_sd}}$ & $\result{rev_qwen_lora_7_reconstruct_accuracy_mean}_{\pm \result{rev_qwen_lora_7_reconstruct_accuracy_sd}}$ & $\result{rev_qwen_lora_7_restore_accuracy_mean}_{\pm \result{rev_qwen_lora_7_restore_accuracy_sd}}$ & $\result{rev_qwen_lora_7_restore_nll_vs_reconstruct_mean}_{\pm \result{rev_qwen_lora_7_restore_nll_vs_reconstruct_sd}}$\\
\textcolor{LoRAColor}{LoRA} & 14 & $1\times10^{-5}$ & $\result{rev_qwen_lora_14_trained_accuracy_mean}_{\pm \result{rev_qwen_lora_14_trained_accuracy_sd}}$ & $\result{rev_qwen_lora_14_reconstruct_accuracy_mean}_{\pm \result{rev_qwen_lora_14_reconstruct_accuracy_sd}}$ & $\result{rev_qwen_lora_14_restore_accuracy_mean}_{\pm \result{rev_qwen_lora_14_restore_accuracy_sd}}$ & $\result{rev_qwen_lora_14_restore_nll_vs_reconstruct_mean}_{\pm \result{rev_qwen_lora_14_restore_nll_vs_reconstruct_sd}}$\\
\textcolor{LoRAColor}{LoRA} & 28 & $3\times10^{-5}$ & $\result{rev_qwen_lora_28_trained_accuracy_mean}_{\pm \result{rev_qwen_lora_28_trained_accuracy_sd}}$ & $\result{rev_qwen_lora_28_reconstruct_accuracy_mean}_{\pm \result{rev_qwen_lora_28_reconstruct_accuracy_sd}}$ & $\result{rev_qwen_lora_28_restore_accuracy_mean}_{\pm \result{rev_qwen_lora_28_restore_accuracy_sd}}$ & $\result{rev_qwen_lora_28_restore_nll_vs_reconstruct_mean}_{\pm \result{rev_qwen_lora_28_restore_nll_vs_reconstruct_sd}}$\\
\bottomrule\end{tabular*}
\end{center}
\end{table}
\FloatBarrier
\begin{table}[!htbp]
\caption{Qwen2.5-7B component-restoration effects relative to the trained model, with accuracy changes in percentage points.}
\label{tab:peft-bands-qwen}
\begin{center}
\setlength{\tabcolsep}{3pt}
\begin{tabular*}{\linewidth}{@{\extracolsep{\fill}}lrrrrrr@{}}\toprule
Method & Capacity & LR & Top $\Delta$acc. & Middle $\Delta$acc. & Tail $\Delta$acc. & Top $\Delta$NLL\\\midrule
\textcolor{LoRAColor}{LoRA} & 7 & $1\times10^{-3}$ & $\result{rev_qwen_lora_7_rotation_band_restore_top_accuracy_change_mean}_{\pm \result{rev_qwen_lora_7_rotation_band_restore_top_accuracy_change_sd}}$ & $\result{rev_qwen_lora_7_rotation_band_restore_mid_accuracy_change_mean}_{\pm \result{rev_qwen_lora_7_rotation_band_restore_mid_accuracy_change_sd}}$ & $\result{rev_qwen_lora_7_rotation_band_restore_bottom_accuracy_change_mean}_{\pm \result{rev_qwen_lora_7_rotation_band_restore_bottom_accuracy_change_sd}}$ & $\result{rev_qwen_lora_7_rotation_band_restore_top_nll_change_mean}_{\pm \result{rev_qwen_lora_7_rotation_band_restore_top_nll_change_sd}}$\\
\textcolor{LoRAColor}{LoRA} & 14 & $1\times10^{-5}$ & $\result{rev_qwen_lora_14_rotation_band_restore_top_accuracy_change_mean}_{\pm \result{rev_qwen_lora_14_rotation_band_restore_top_accuracy_change_sd}}$ & $\result{rev_qwen_lora_14_rotation_band_restore_mid_accuracy_change_mean}_{\pm \result{rev_qwen_lora_14_rotation_band_restore_mid_accuracy_change_sd}}$ & $\result{rev_qwen_lora_14_rotation_band_restore_bottom_accuracy_change_mean}_{\pm \result{rev_qwen_lora_14_rotation_band_restore_bottom_accuracy_change_sd}}$ & $\result{rev_qwen_lora_14_rotation_band_restore_top_nll_change_mean}_{\pm \result{rev_qwen_lora_14_rotation_band_restore_top_nll_change_sd}}$\\
\bottomrule\end{tabular*}
\end{center}
\end{table}
\FloatBarrier
\begin{table}[!htbp]
\caption{Llama-3.1-8B spectral displacement at every selected capacity and LR.}
\label{tab:peft-spectrum-llama}
\begin{center}
\setlength{\tabcolsep}{3pt}
\begin{tabular*}{\linewidth}{@{\extracolsep{\fill}}lrrrrr@{}}\toprule
Method & Capacity & LR & $100d_W$ & $100d_S$ & $100d_S/d_W$\\\midrule
\textcolor{LoRAColor}{LoRA} & 7 & $3\times10^{-4}$ & $\result{rev_llama_lora_7_trained_d_W_mean}_{\pm \result{rev_llama_lora_7_trained_d_W_sd}}$ & $\result{rev_llama_lora_7_trained_d_S_mean}_{\pm \result{rev_llama_lora_7_trained_d_S_sd}}$ & $\result{rev_llama_lora_7_trained_ratio_mean}_{\pm \result{rev_llama_lora_7_trained_ratio_sd}}$\\
\textcolor{LoRAColor}{LoRA} & 15 & $3\times10^{-4}$ & $\result{rev_llama_lora_15_trained_d_W_mean}_{\pm \result{rev_llama_lora_15_trained_d_W_sd}}$ & $\result{rev_llama_lora_15_trained_d_S_mean}_{\pm \result{rev_llama_lora_15_trained_d_S_sd}}$ & $\result{rev_llama_lora_15_trained_ratio_mean}_{\pm \result{rev_llama_lora_15_trained_ratio_sd}}$\\
\textcolor{LoRAColor}{LoRA} & 30 & $3\times10^{-4}$ & $\result{rev_llama_lora_30_trained_d_W_mean}_{\pm \result{rev_llama_lora_30_trained_d_W_sd}}$ & $\result{rev_llama_lora_30_trained_d_S_mean}_{\pm \result{rev_llama_lora_30_trained_d_S_sd}}$ & $\result{rev_llama_lora_30_trained_ratio_mean}_{\pm \result{rev_llama_lora_30_trained_ratio_sd}}$\\
\addlinespace[2pt]
\textcolor{OFTColor}{OFT} & 32 & $3\times10^{-4}$ & $\result{rev_llama_oft_32_trained_d_W_mean}_{\pm \result{rev_llama_oft_32_trained_d_W_sd}}$ & $\result{rev_llama_oft_32_trained_d_S_mean}_{\pm \result{rev_llama_oft_32_trained_d_S_sd}}$ & $\result{rev_llama_oft_32_trained_ratio_mean}_{\pm \result{rev_llama_oft_32_trained_ratio_sd}}$\\
\textcolor{OFTColor}{OFT} & 64 & $3\times10^{-4}$ & $\result{rev_llama_oft_64_trained_d_W_mean}_{\pm \result{rev_llama_oft_64_trained_d_W_sd}}$ & $\result{rev_llama_oft_64_trained_d_S_mean}_{\pm \result{rev_llama_oft_64_trained_d_S_sd}}$ & $\result{rev_llama_oft_64_trained_ratio_mean}_{\pm \result{rev_llama_oft_64_trained_ratio_sd}}$\\
\bottomrule\end{tabular*}
\end{center}
\end{table}
\FloatBarrier
\begin{table}[!htbp]
\caption{Llama-3.1-8B singular-value restoration and reconstruction controls, with accuracy in percent.}
\label{tab:peft-restoration-llama}
\begin{center}
\setlength{\tabcolsep}{3pt}
\begin{tabular*}{\linewidth}{@{\extracolsep{\fill}}lrrrrrr@{}}\toprule
Method & Capacity & LR & Trained acc. & Rebuilt acc. & Restored acc. & $\Delta$NLL vs. rebuilt\\\midrule
\textcolor{LoRAColor}{LoRA} & 7 & $3\times10^{-4}$ & $\result{rev_llama_lora_7_trained_accuracy_mean}_{\pm \result{rev_llama_lora_7_trained_accuracy_sd}}$ & $\result{rev_llama_lora_7_reconstruct_accuracy_mean}_{\pm \result{rev_llama_lora_7_reconstruct_accuracy_sd}}$ & $\result{rev_llama_lora_7_restore_accuracy_mean}_{\pm \result{rev_llama_lora_7_restore_accuracy_sd}}$ & $\result{rev_llama_lora_7_restore_nll_vs_reconstruct_mean}_{\pm \result{rev_llama_lora_7_restore_nll_vs_reconstruct_sd}}$\\
\textcolor{LoRAColor}{LoRA} & 15 & $3\times10^{-4}$ & $\result{rev_llama_lora_15_trained_accuracy_mean}_{\pm \result{rev_llama_lora_15_trained_accuracy_sd}}$ & $\result{rev_llama_lora_15_reconstruct_accuracy_mean}_{\pm \result{rev_llama_lora_15_reconstruct_accuracy_sd}}$ & $\result{rev_llama_lora_15_restore_accuracy_mean}_{\pm \result{rev_llama_lora_15_restore_accuracy_sd}}$ & $\result{rev_llama_lora_15_restore_nll_vs_reconstruct_mean}_{\pm \result{rev_llama_lora_15_restore_nll_vs_reconstruct_sd}}$\\
\textcolor{LoRAColor}{LoRA} & 30 & $3\times10^{-4}$ & $\result{rev_llama_lora_30_trained_accuracy_mean}_{\pm \result{rev_llama_lora_30_trained_accuracy_sd}}$ & $\result{rev_llama_lora_30_reconstruct_accuracy_mean}_{\pm \result{rev_llama_lora_30_reconstruct_accuracy_sd}}$ & $\result{rev_llama_lora_30_restore_accuracy_mean}_{\pm \result{rev_llama_lora_30_restore_accuracy_sd}}$ & $\result{rev_llama_lora_30_restore_nll_vs_reconstruct_mean}_{\pm \result{rev_llama_lora_30_restore_nll_vs_reconstruct_sd}}$\\
\bottomrule\end{tabular*}
\end{center}
\end{table}
\FloatBarrier
\begin{table}[!htbp]
\caption{Llama-3.1-8B component-restoration effects relative to the trained model, with accuracy changes in percentage points.}
\label{tab:peft-bands-llama}
\begin{center}
\setlength{\tabcolsep}{3pt}
\begin{tabular*}{\linewidth}{@{\extracolsep{\fill}}lrrrrrr@{}}\toprule
Method & Capacity & LR & Top $\Delta$acc. & Middle $\Delta$acc. & Tail $\Delta$acc. & Top $\Delta$NLL\\\midrule
\textcolor{LoRAColor}{LoRA} & 7 & $3\times10^{-4}$ & $\result{rev_llama_lora_7_rotation_band_restore_top_accuracy_change_mean}_{\pm \result{rev_llama_lora_7_rotation_band_restore_top_accuracy_change_sd}}$ & $\result{rev_llama_lora_7_rotation_band_restore_mid_accuracy_change_mean}_{\pm \result{rev_llama_lora_7_rotation_band_restore_mid_accuracy_change_sd}}$ & $\result{rev_llama_lora_7_rotation_band_restore_bottom_accuracy_change_mean}_{\pm \result{rev_llama_lora_7_rotation_band_restore_bottom_accuracy_change_sd}}$ & $\result{rev_llama_lora_7_rotation_band_restore_top_nll_change_mean}_{\pm \result{rev_llama_lora_7_rotation_band_restore_top_nll_change_sd}}$\\
\textcolor{LoRAColor}{LoRA} & 15 & $3\times10^{-4}$ & $\result{rev_llama_lora_15_rotation_band_restore_top_accuracy_change_mean}_{\pm \result{rev_llama_lora_15_rotation_band_restore_top_accuracy_change_sd}}$ & $\result{rev_llama_lora_15_rotation_band_restore_mid_accuracy_change_mean}_{\pm \result{rev_llama_lora_15_rotation_band_restore_mid_accuracy_change_sd}}$ & $\result{rev_llama_lora_15_rotation_band_restore_bottom_accuracy_change_mean}_{\pm \result{rev_llama_lora_15_rotation_band_restore_bottom_accuracy_change_sd}}$ & $\result{rev_llama_lora_15_rotation_band_restore_top_nll_change_mean}_{\pm \result{rev_llama_lora_15_rotation_band_restore_top_nll_change_sd}}$\\
\bottomrule\end{tabular*}
\end{center}
\end{table}
\FloatBarrier

\FloatBarrier
\Needspace{8\baselineskip}
\subsection{Singular-value displacement and restoration}
At largest capacity, PiSSA concentrates over 99\% of squared singular-value displacement in dominant components on all three models. MiLoRA concentrates about 99\% on Qwen/FLUX and 95\% on Llama in trailing components. These profiles follow their initialization choices, while restoration measures how much the spectral changes contribute to task performance.

Across all 108 checkpoints, singular-value restoration changes general NLL by $-0.00535$ to $+0.01365$ relative to reconstruction, alongside a maximum absolute task-accuracy change of 1.3 percentage points. These small task effects extend the restoration finding across language capacities. 

\paragraph{A high-LR LoRA outlier.}
Qwen LoRA r7 selects the highest tested LR, $10^{-3}$. Its spectral displacement is $\result{spectrum_qwen_r7_dS_percent}\%$ of the base norm and $\result{spectrum_qwen_r7_ratio}\%$ of the weight update, exceeding the ratios in the low-drift settings. Restoration still largely preserves task performance, with a mean accuracy change of $+\result{rev_qwen_lora_7_restore_accuracy_vs_reconstruct_mean}$ pp relative to reconstruction (\Cref{tab:peft-restoration-qwen}). The restoration finding therefore also holds in this larger-ratio setting.
\begin{figure}[!htbp]\centering
\includegraphics[width=\linewidth,height=.72\textheight,keepaspectratio]{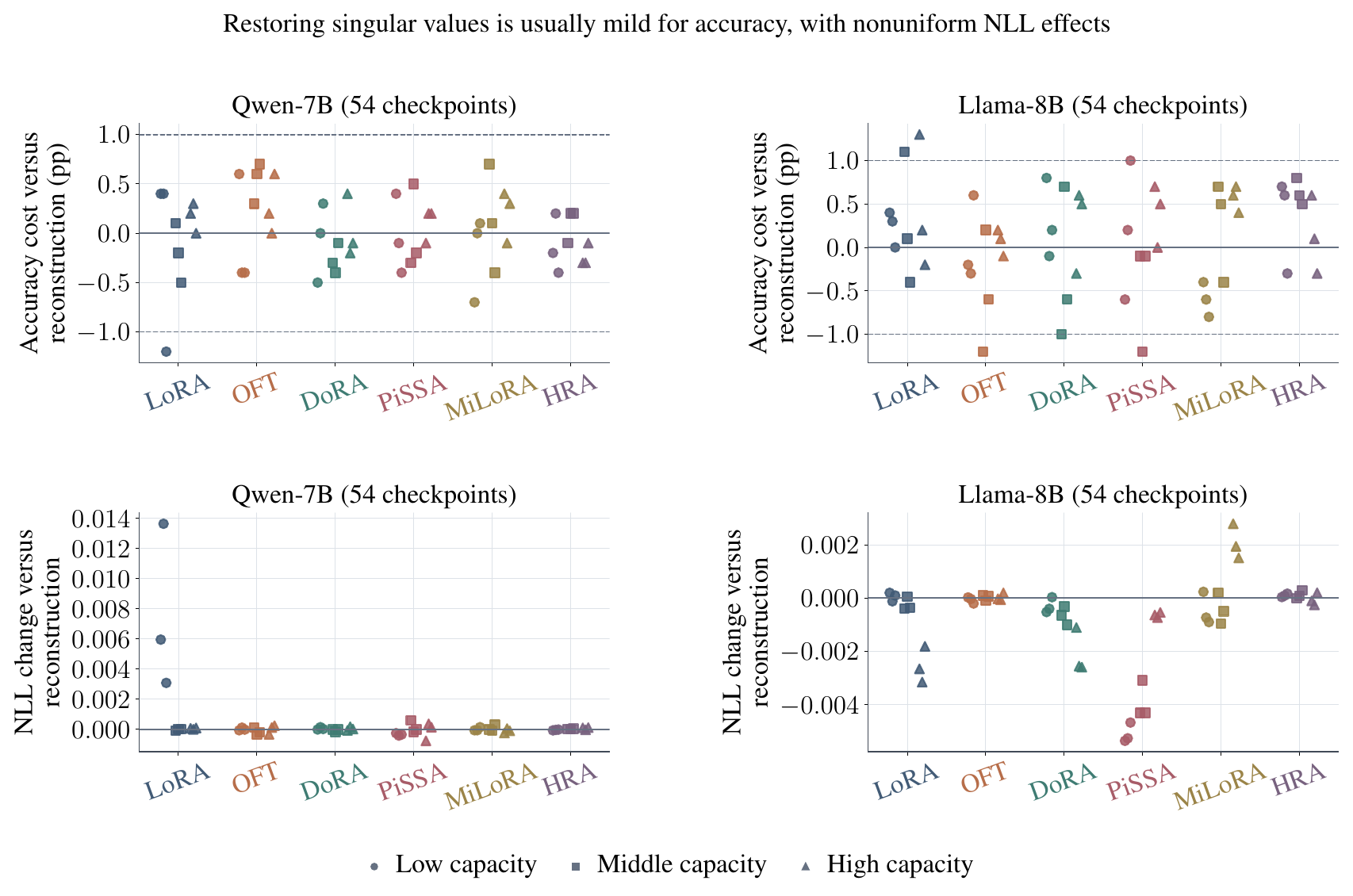}
\caption{Singular-value restoration across capacities and seeds. Trained, reconstructed and restored models are shown separately, extending the main accuracy comparison with reconstruction controls and retention measurements.}
\label{fig:complete-values}
\end{figure}
\begin{figure}[!htbp]\centering
\includegraphics[width=\linewidth,height=.72\textheight,keepaspectratio]{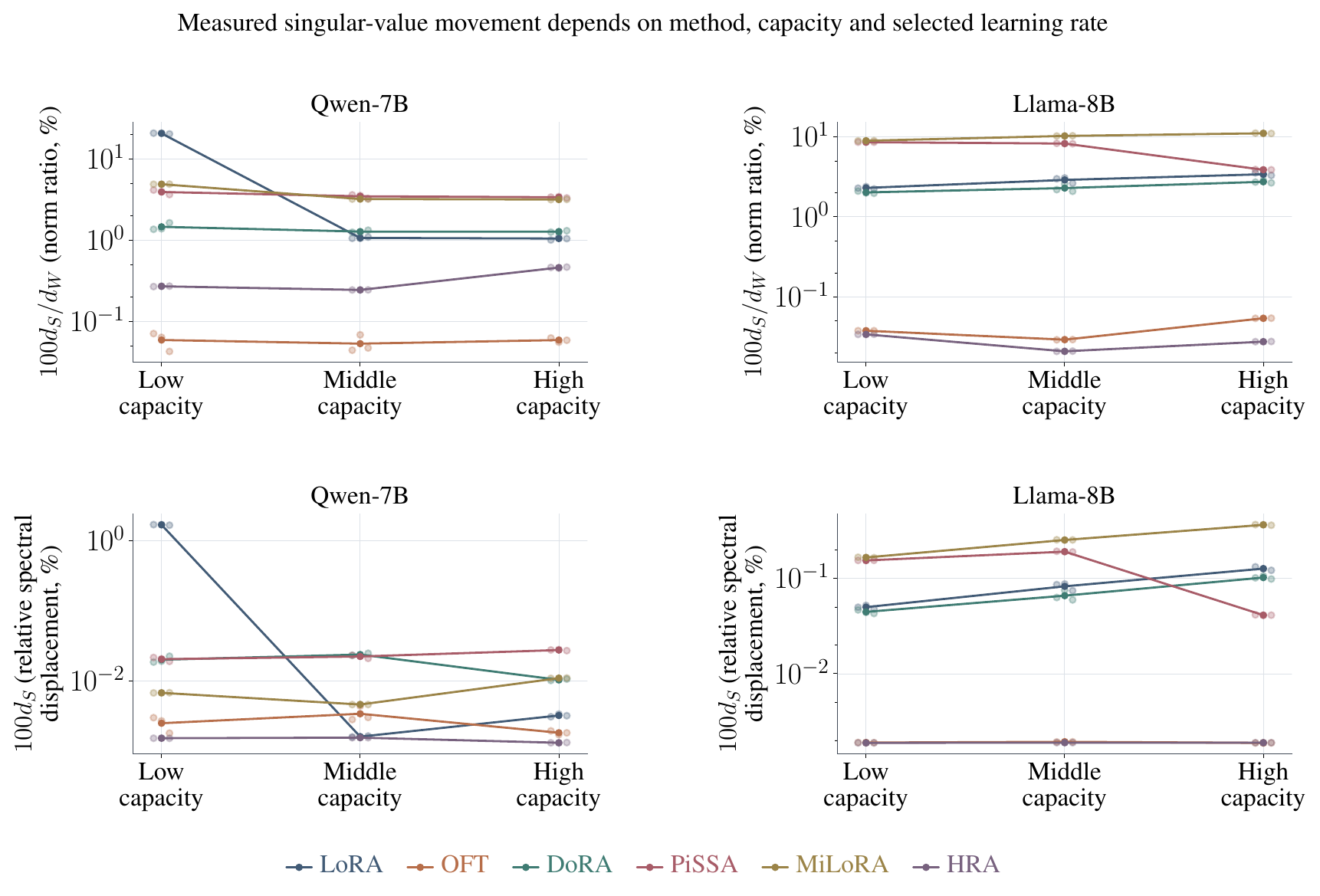}
\caption{Measured weight and singular-value displacement across all six methods and both models at each selected capacity and LR. Ratios express spectral displacement relative to weight displacement.}
\end{figure}
\begin{figure}[!htbp]\centering
\includegraphics[width=\linewidth,height=.72\textheight,keepaspectratio]{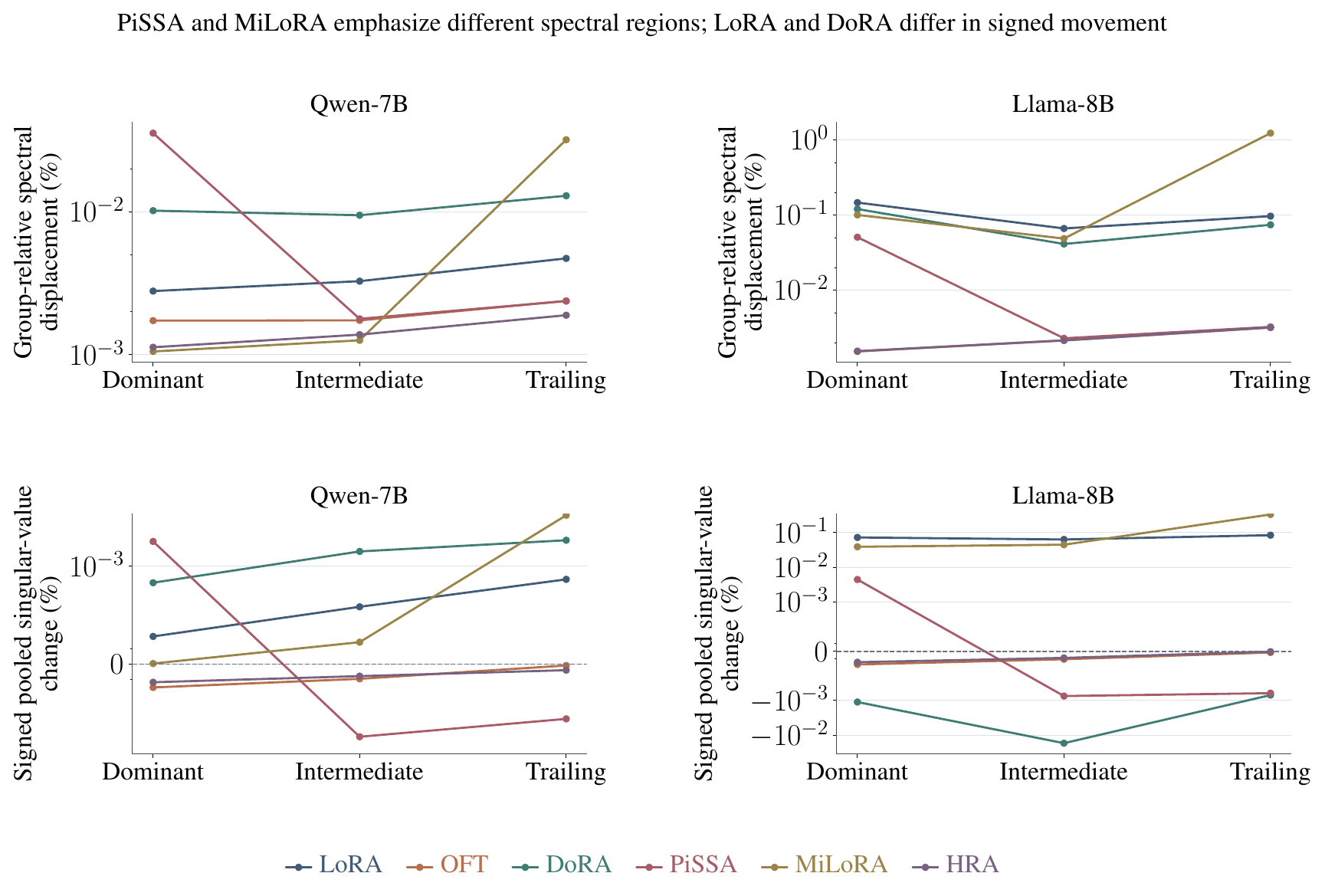}
\caption{Singular-value changes across spectral components at largest capacity. PiSSA's displacement concentrates near the leading end, whereas MiLoRA's concentrates near the tail, consistent with their initialization choices. The pretrained reference and method colors are retained.}
\end{figure}

\Needspace{20\baselineskip}
\subsection{Hyperspherical energy and pairwise geometry}\label{app:hyperspherical}
We measure all adapted matrices in the largest-capacity intervention checkpoints, with 196 per Qwen model, 224 per Llama model and 80 per FLUX cat model. The 54 checkpoints and 9,000 matrix comparisons use the same selected LRs as the spectral and functional analyses. This tests whether the approximate geometric preservation observed for LoRA-family methods also extends to HE.

\paragraph{Definitions and aggregation.}
For stored weights $\bm W\in\mathbb R^{n\times d}$, rows represent output neurons. Let $\widehat{\bm w}_i=\bm w_i/\|\bm w_i\|_2$. Following \citet{liu2018mhe} and \citet{qiu2023oft}, we compute
\begin{equation*}
H(\bm W)=\sum_{i\ne j}\frac{1}{\|\widehat{\bm w}_i-\widehat{\bm w}_j\|_2},\qquad
\delta_H=\frac{|H(\bm W_*)-H(\bm W_0)|}{H(\bm W_0)}.
\end{equation*}
The sum covers every ordered distinct row pair. To measure pairwise changes without signed cancellation, we also compute
\begin{equation*}
\delta_C=\left[\frac{1}{n(n-1)}\sum_{i\ne j}
\left(\widehat{\bm w}_{*,i}^{\!\top}\widehat{\bm w}_{*,j}-\widehat{\bm w}_{0,i}^{\!\top}\widehat{\bm w}_{0,j}\right)^2\right]^{1/2}.
\end{equation*}
We average each metric equally over matrices within a checkpoint, then summarize training seeds. The absolute HE change is taken before matrix averaging. A common right-orthogonal transformation preserves row inner products and HE.

\begin{table}[!htbp]
\caption{Geometric preservation at largest selected capacities. Each checkpoint averages matrices equally, followed by the seed mean and sample SD. HE changes are in units of $10^{-4}\%$, and cosine changes are in units of $10^{-4}$.}
\label{tab:hyperspherical}
\begin{center}
\begin{tabular*}{\linewidth}{@{\extracolsep{\fill}}lccc@{}}\toprule
Method & Qwen2.5-7B & Llama-3.1-8B & FLUX-4B cat \\\midrule
\multicolumn{4}{l}{(a) Absolute relative HE change ($\times10^{-4}\%$)} \\
\textcolor{LoRAColor}{LoRA} & $14.8_{\pm 0.96}$ & $1.72_{\pm 0.17}$ & $0.281_{\pm 0.025}$ \\
\textcolor{OFTColor}{OFT} & $1.5_{\pm 0.2}$ & $0.0215_{\pm 0.0014}$ & $0.533_{\pm 0.065}$ \\
\textcolor{DoRAColor}{DoRA} & $80.3_{\pm 15}$ & $1.7_{\pm 0.069}$ & $0.714_{\pm 0.072}$ \\
\textcolor{PiSSAColor}{PiSSA} & $91.9_{\pm 4.6}$ & $1.53_{\pm 0.099}$ & $3.02_{\pm 0.24}$ \\
\textcolor{MiLoRAColor}{MiLoRA} & $1.35_{\pm 0.37}$ & $4.36_{\pm 0.17}$ & $0.987_{\pm 0.032}$ \\
\textcolor{HRAColor}{HRA} & $0.625_{\pm 0.12}$ & $0.0264_{\pm 0.0013}$ & $0.00804_{\pm 0.0008}$ \\
\midrule
\multicolumn{4}{l}{(b) Pairwise-cosine RMS change ($\times10^{-4}$)} \\
\textcolor{LoRAColor}{LoRA} & $0.906_{\pm 0.015}$ & $10.7_{\pm 0.15}$ & $2.51_{\pm 0.11}$ \\
\textcolor{OFTColor}{OFT} & $0.372_{\pm 0.00052}$ & $0.35_{\pm 2.3\times10^{-5}}$ & $0.48_{\pm 0.022}$ \\
\textcolor{DoRAColor}{DoRA} & $2.26_{\pm 0.034}$ & $10.6_{\pm 0.069}$ & $4.65_{\pm 0.08}$ \\
\textcolor{PiSSAColor}{PiSSA} & $5.33_{\pm 0.045}$ & $5.98_{\pm 0.011}$ & $17_{\pm 0.1}$ \\
\textcolor{MiLoRAColor}{MiLoRA} & $1.12_{\pm 0.0044}$ & $10.6_{\pm 0.1}$ & $4.11_{\pm 0.19}$ \\
\textcolor{HRAColor}{HRA} & $0.307_{\pm 0.001}$ & $0.35_{\pm 1.6\times10^{-5}}$ & $0.407_{\pm 5.9\times10^{-6}}$ \\
\bottomrule
\end{tabular*}
\end{center}
\end{table}

\begin{samepage}
\paragraph{Results.}
All checkpoint means of $\delta_H$ are below $0.01\%$, with maximum $\result{he_max_checkpoint_percent}\%$ (\Cref{tab:hyperspherical}). This supports approximate HE preservation across the selected LoRA-family methods as well as the orthogonal methods. The pairwise measure shows how closely individual neuron relationships are preserved alongside aggregate HE.
\par
\end{samepage}

\subsection{Component effects, interactions and task performance}
Restoring dominant components lowers general-text NLL at every capacity and gives the largest reduction among the three groups in 106/108 checkpoints. The exceptions to the largest reduction are Qwen DoRA r14/seed 73 and MiLoRA r28/seed 13. Dominant restoration remains strongest in every configuration mean, supporting the consistent forgetting response across language capacities.

\begin{figure}[!htbp]\centering
\includegraphics[width=\linewidth,height=.76\textheight,keepaspectratio]{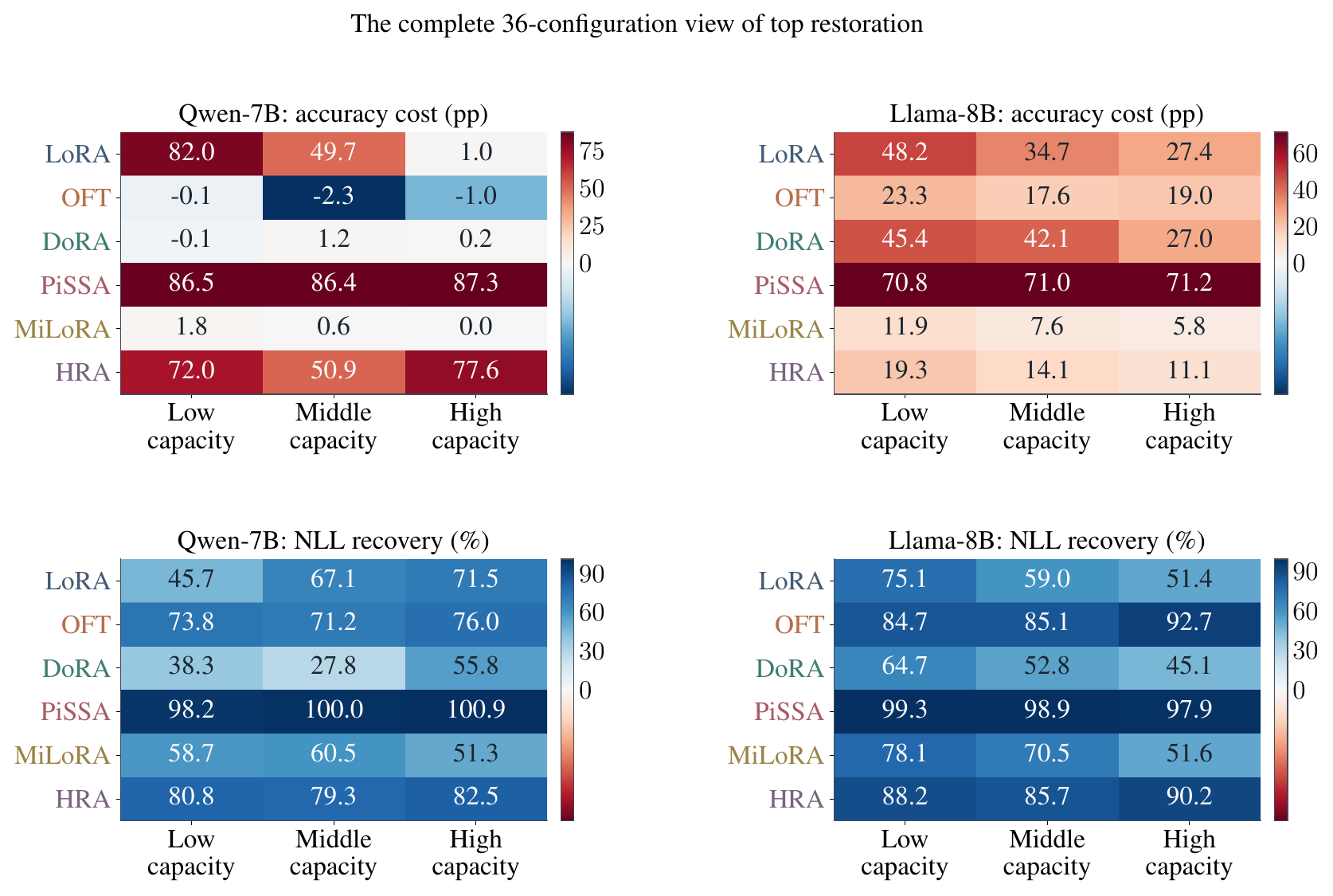}
\caption{Dominant-component restoration at every selected configuration and seed. The retention response persists across the grid, while task-accuracy responses vary. Recovery is normalized by each checkpoint's trained-model NLL increase.}
\end{figure}

For a set $A$ of adapted component groups, define
\begin{equation*}
\bm W(A)=\sum_{B\in A}\sum_{i\in B}s_{0,i}\bm u_{*,i}\bm v_{*,i}^{\!\top}
+\sum_{B\notin A}\sum_{i\in B}s_{0,i}\bm u_{0,i}\bm v_{0,i}^{\!\top}.
\end{equation*}
Its endpoints are the pretrained-component reconstruction and the full value-restored model. For general NLL $f(A)$, the effect of adapting group $B$ is $f(A\cup\{B\})-f(A)$. We evaluate it under all four combinations of pretrained or adapted components in the other groups. Adapting dominant components increases NLL in every combination across all 108 checkpoints. Intermediate and trailing effects change sign across combinations in 38 and 42 checkpoints, using a $10^{-4}$ tolerance. Thus, the dominant-component forgetting response persists as the other components change.

All 36 configurations include evaluations of every component combination for all three seeds.
\begin{figure}[!htbp]\centering
\includegraphics[width=\linewidth,height=.75\textheight,keepaspectratio]{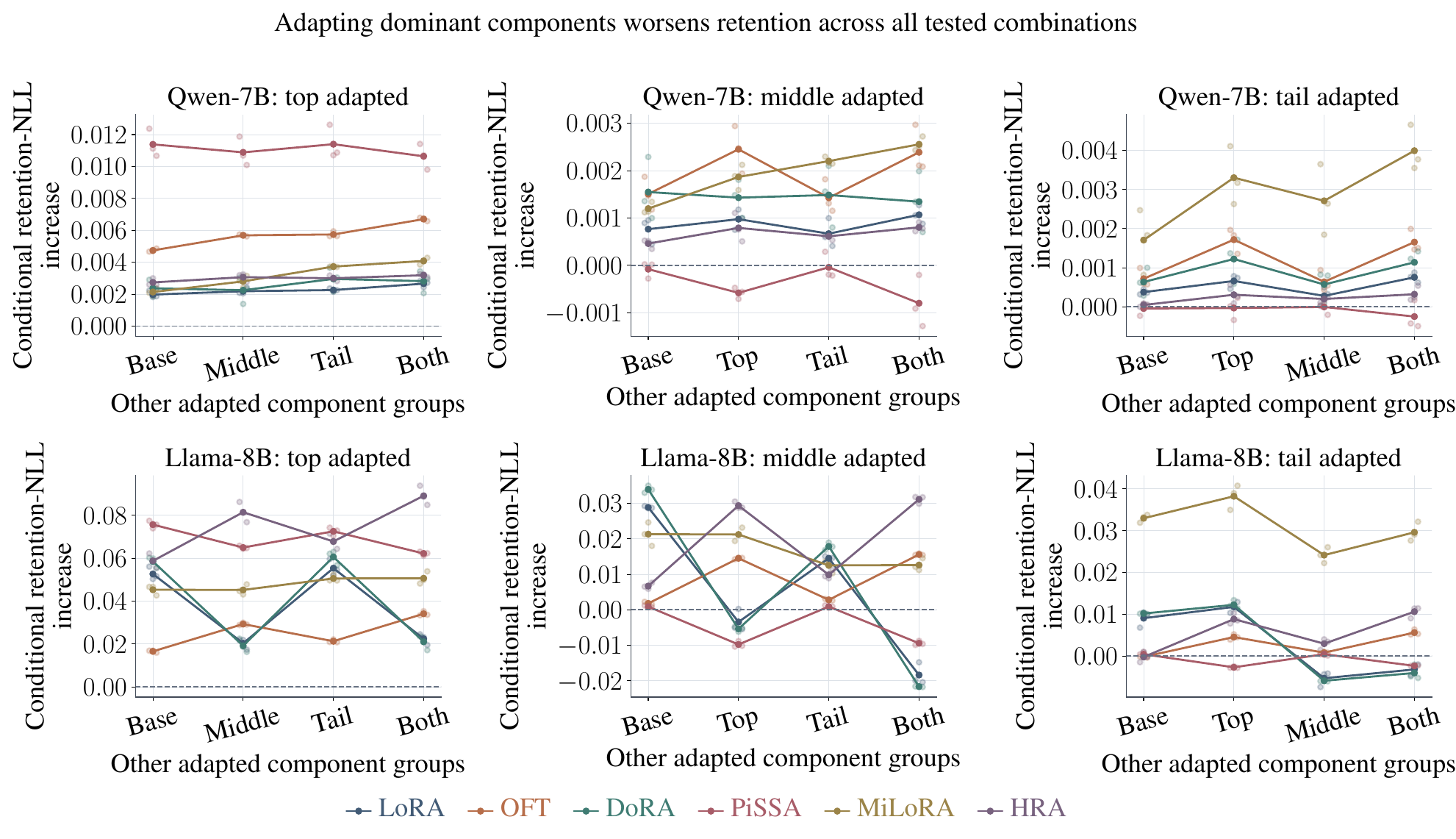}
\caption{Effects of adapting each component group under all combinations of the other groups. Dominant-component adaptation consistently increases general NLL, while intermediate and trailing effects can change sign. Both NLL evaluations agree on the dominant-component direction.}
\end{figure}
\begin{figure}[t]\centering
\includegraphics[width=\linewidth]{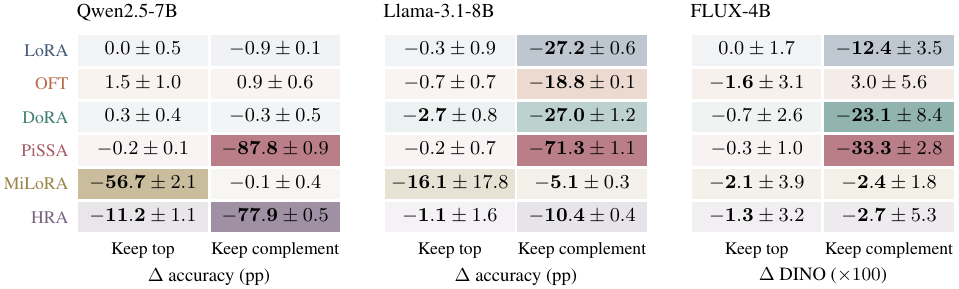}
\caption{Task performance when retaining only some adapted components. Cells show mean $\pm$ SD changes from $\bm W_R$ (all singular values restored) at the largest capacities. Columns retain the adapted dominant components alone or the intermediate and trailing components together. Bold means mark losses exceeding 1 pp for LLMs or 1 DINO$\times100$ unit for FLUX.}
\label{fig:component-sufficiency}
\end{figure}

\paragraph{Task performance across selected capacities.}
\Cref{tab:capacity-sufficiency} extends the largest-capacity comparison in \Cref{fig:component-sufficiency}. Qwen LoRA retains task performance with dominant components alone at small capacity, with neither component set alone at medium capacity, and with either set alone at large capacity. Llama LoRA's outcomes also vary across capacities and seeds. PiSSA retains the dominant-only pattern at every language capacity, supporting its consistent concentration of task gains in dominant components. The other methods illustrate the setting-dependent component contributions described in the main text, alongside the consistent forgetting response.
\begin{table}[!htbp]
\caption{Task sufficiency across all selected language capacities. Entries count seeds with accuracy loss $\leq1$ pp relative to full value restoration, out of three seeds per configuration. Top and complement retain those adapted components on a pretrained background. Tables~\ref{tab:peft-spectrum-qwen} and~\ref{tab:peft-spectrum-llama} give capacities and selected LRs.}
\label{tab:capacity-sufficiency}
\begin{center}
\setlength{\tabcolsep}{4pt}
\begin{tabular*}{\linewidth}{@{\extracolsep{\fill}}llcccccc@{}}\toprule
Model & Method & \multicolumn{2}{c}{Small} & \multicolumn{2}{c}{Medium} & \multicolumn{2}{c}{Large}\\
\cmidrule(lr){3-4}\cmidrule(lr){5-6}\cmidrule(l){7-8}
& & Top & Complement & Top & Complement & Top & Complement\\\midrule
Qwen2.5-7B & \textcolor{LoRAColor}{LoRA} & \result{rev_capacity_qwen_lora_7_top}/\result{rev_capacity_qwen_lora_7_valid} & \result{rev_capacity_qwen_lora_7_complement}/\result{rev_capacity_qwen_lora_7_valid} & \result{rev_capacity_qwen_lora_14_top}/\result{rev_capacity_qwen_lora_14_valid} & \result{rev_capacity_qwen_lora_14_complement}/\result{rev_capacity_qwen_lora_14_valid} & \result{rev_capacity_qwen_lora_28_top}/\result{rev_capacity_qwen_lora_28_valid} & \result{rev_capacity_qwen_lora_28_complement}/\result{rev_capacity_qwen_lora_28_valid}\\
 & \textcolor{OFTColor}{OFT} & \result{rev_capacity_qwen_oft_32_top}/\result{rev_capacity_qwen_oft_32_valid} & \result{rev_capacity_qwen_oft_32_complement}/\result{rev_capacity_qwen_oft_32_valid} & \result{rev_capacity_qwen_oft_64_top}/\result{rev_capacity_qwen_oft_64_valid} & \result{rev_capacity_qwen_oft_64_complement}/\result{rev_capacity_qwen_oft_64_valid} & \result{rev_capacity_qwen_oft_128_top}/\result{rev_capacity_qwen_oft_128_valid} & \result{rev_capacity_qwen_oft_128_complement}/\result{rev_capacity_qwen_oft_128_valid}\\
 & \textcolor{DoRAColor}{DoRA} & \result{rev_capacity_qwen_dora_7_top}/\result{rev_capacity_qwen_dora_7_valid} & \result{rev_capacity_qwen_dora_7_complement}/\result{rev_capacity_qwen_dora_7_valid} & \result{rev_capacity_qwen_dora_14_top}/\result{rev_capacity_qwen_dora_14_valid} & \result{rev_capacity_qwen_dora_14_complement}/\result{rev_capacity_qwen_dora_14_valid} & \result{rev_capacity_qwen_dora_28_top}/\result{rev_capacity_qwen_dora_28_valid} & \result{rev_capacity_qwen_dora_28_complement}/\result{rev_capacity_qwen_dora_28_valid}\\
 & \textcolor{PiSSAColor}{PiSSA} & \result{rev_capacity_qwen_pissa_7_top}/\result{rev_capacity_qwen_pissa_7_valid} & \result{rev_capacity_qwen_pissa_7_complement}/\result{rev_capacity_qwen_pissa_7_valid} & \result{rev_capacity_qwen_pissa_14_top}/\result{rev_capacity_qwen_pissa_14_valid} & \result{rev_capacity_qwen_pissa_14_complement}/\result{rev_capacity_qwen_pissa_14_valid} & \result{rev_capacity_qwen_pissa_28_top}/\result{rev_capacity_qwen_pissa_28_valid} & \result{rev_capacity_qwen_pissa_28_complement}/\result{rev_capacity_qwen_pissa_28_valid}\\
 & \textcolor{MiLoRAColor}{MiLoRA} & \result{rev_capacity_qwen_milora_7_top}/\result{rev_capacity_qwen_milora_7_valid} & \result{rev_capacity_qwen_milora_7_complement}/\result{rev_capacity_qwen_milora_7_valid} & \result{rev_capacity_qwen_milora_14_top}/\result{rev_capacity_qwen_milora_14_valid} & \result{rev_capacity_qwen_milora_14_complement}/\result{rev_capacity_qwen_milora_14_valid} & \result{rev_capacity_qwen_milora_28_top}/\result{rev_capacity_qwen_milora_28_valid} & \result{rev_capacity_qwen_milora_28_complement}/\result{rev_capacity_qwen_milora_28_valid}\\
 & \textcolor{HRAColor}{HRA} & \result{rev_capacity_qwen_hra_16_top}/\result{rev_capacity_qwen_hra_16_valid} & \result{rev_capacity_qwen_hra_16_complement}/\result{rev_capacity_qwen_hra_16_valid} & \result{rev_capacity_qwen_hra_32_top}/\result{rev_capacity_qwen_hra_32_valid} & \result{rev_capacity_qwen_hra_32_complement}/\result{rev_capacity_qwen_hra_32_valid} & \result{rev_capacity_qwen_hra_64_top}/\result{rev_capacity_qwen_hra_64_valid} & \result{rev_capacity_qwen_hra_64_complement}/\result{rev_capacity_qwen_hra_64_valid}\\
\midrule
Llama-3.1-8B & \textcolor{LoRAColor}{LoRA} & \result{rev_capacity_llama_lora_7_top}/\result{rev_capacity_llama_lora_7_valid} & \result{rev_capacity_llama_lora_7_complement}/\result{rev_capacity_llama_lora_7_valid} & \result{rev_capacity_llama_lora_15_top}/\result{rev_capacity_llama_lora_15_valid} & \result{rev_capacity_llama_lora_15_complement}/\result{rev_capacity_llama_lora_15_valid} & \result{rev_capacity_llama_lora_30_top}/\result{rev_capacity_llama_lora_30_valid} & \result{rev_capacity_llama_lora_30_complement}/\result{rev_capacity_llama_lora_30_valid}\\
 & \textcolor{OFTColor}{OFT} & \result{rev_capacity_llama_oft_32_top}/\result{rev_capacity_llama_oft_32_valid} & \result{rev_capacity_llama_oft_32_complement}/\result{rev_capacity_llama_oft_32_valid} & \result{rev_capacity_llama_oft_64_top}/\result{rev_capacity_llama_oft_64_valid} & \result{rev_capacity_llama_oft_64_complement}/\result{rev_capacity_llama_oft_64_valid} & \result{rev_capacity_llama_oft_128_top}/\result{rev_capacity_llama_oft_128_valid} & \result{rev_capacity_llama_oft_128_complement}/\result{rev_capacity_llama_oft_128_valid}\\
 & \textcolor{DoRAColor}{DoRA} & \result{rev_capacity_llama_dora_7_top}/\result{rev_capacity_llama_dora_7_valid} & \result{rev_capacity_llama_dora_7_complement}/\result{rev_capacity_llama_dora_7_valid} & \result{rev_capacity_llama_dora_15_top}/\result{rev_capacity_llama_dora_15_valid} & \result{rev_capacity_llama_dora_15_complement}/\result{rev_capacity_llama_dora_15_valid} & \result{rev_capacity_llama_dora_30_top}/\result{rev_capacity_llama_dora_30_valid} & \result{rev_capacity_llama_dora_30_complement}/\result{rev_capacity_llama_dora_30_valid}\\
 & \textcolor{PiSSAColor}{PiSSA} & \result{rev_capacity_llama_pissa_7_top}/\result{rev_capacity_llama_pissa_7_valid} & \result{rev_capacity_llama_pissa_7_complement}/\result{rev_capacity_llama_pissa_7_valid} & \result{rev_capacity_llama_pissa_15_top}/\result{rev_capacity_llama_pissa_15_valid} & \result{rev_capacity_llama_pissa_15_complement}/\result{rev_capacity_llama_pissa_15_valid} & \result{rev_capacity_llama_pissa_30_top}/\result{rev_capacity_llama_pissa_30_valid} & \result{rev_capacity_llama_pissa_30_complement}/\result{rev_capacity_llama_pissa_30_valid}\\
 & \textcolor{MiLoRAColor}{MiLoRA} & \result{rev_capacity_llama_milora_7_top}/\result{rev_capacity_llama_milora_7_valid} & \result{rev_capacity_llama_milora_7_complement}/\result{rev_capacity_llama_milora_7_valid} & \result{rev_capacity_llama_milora_15_top}/\result{rev_capacity_llama_milora_15_valid} & \result{rev_capacity_llama_milora_15_complement}/\result{rev_capacity_llama_milora_15_valid} & \result{rev_capacity_llama_milora_30_top}/\result{rev_capacity_llama_milora_30_valid} & \result{rev_capacity_llama_milora_30_complement}/\result{rev_capacity_llama_milora_30_valid}\\
 & \textcolor{HRAColor}{HRA} & \result{rev_capacity_llama_hra_16_top}/\result{rev_capacity_llama_hra_16_valid} & \result{rev_capacity_llama_hra_16_complement}/\result{rev_capacity_llama_hra_16_valid} & \result{rev_capacity_llama_hra_32_top}/\result{rev_capacity_llama_hra_32_valid} & \result{rev_capacity_llama_hra_32_complement}/\result{rev_capacity_llama_hra_32_valid} & \result{rev_capacity_llama_hra_64_top}/\result{rev_capacity_llama_hra_64_valid} & \result{rev_capacity_llama_hra_64_complement}/\result{rev_capacity_llama_hra_64_valid}\\
\bottomrule\end{tabular*}
\end{center}
\end{table}

\FloatBarrier
\subsection{Consistency of the dominant-component response}\label{app:component-consistency}
The comparison covers 42 selected configurations and 126 checkpoints across all six methods, all tested language capacities and the largest FLUX budget for each method. Restoring dominant components gives the largest mean reduction in general-text NLL or general-image drift in all 42 configurations and a larger reduction than restoring either intermediate or trailing components in 123/126 checkpoints. These counts support the recurring dominant-component forgetting response across models and tasks.

\paragraph{Consistency with smaller edits.}
We measure edit norm as the square root of the summed squared Frobenius changes from trained weights across all adapted matrices, using the stored weights. We select cases where restoring intermediate or trailing components produces a larger edit norm than restoring dominant components. Configuration comparisons use seed means of norms and forgetting reductions. Dominant restoration still gives the largest mean forgetting reduction among all three restorations in 17/17 eligible configurations, comprising 8 Qwen, 6 Llama and 3 FLUX configurations. At the checkpoint level, the count is 48/51, with 22/24 for Qwen, 18/18 for Llama and 8/9 for FLUX. Each configuration or checkpoint is counted once, including when both alternatives have larger norms. Requiring both alternatives to have larger norms gives 4/4 configurations and 10/12 checkpoints. The eligible cases involve LoRA, DoRA and MiLoRA. These counts show that the dominant-component advantage can persist with smaller edits, so edit magnitude alone does not determine the ordering of forgetting reductions.

\FloatBarrier
\subsection{Native OFT orthogonality and polar projection}\label{app:native-orthogonality}
OFTv2 approximates the Cayley inverse with a truncated Neumann series \citep{qiu2025oftv2}. We measure $\|\bm R\bm R^{\!\top}-\bm I\|_{\mathrm{F}}$, its maximum absolute entry, and $\|\bm R-\widehat{\bm R}\|_{\mathrm{F}}/\|\bm R\|_{\mathrm{F}}$, where $\widehat{\bm R}$ is the polar factor \citep{higham1986polar}. The relative distance expresses the deviation from orthogonality on a scale normalized by the transform norm. These diagnostics cover 21 checkpoints each for Qwen and Llama, and 33 for FLUX. The maximum native-to-polar distances are about 0.154\%, 0.094\% and 11.0\%, respectively (\Cref{tab:peft-orthogonality}). Task and retention evaluations use native transforms. Polar projection measures their deviation from an exact orthogonal transform, clarifying the numerical precision of the geometric-preservation comparison.

\begin{table}[!htbp]\caption{Native OFT-transform diagnostics. Maxima are taken over all matrices and checkpoints in each row. Frobenius and entrywise maxima describe different norms.}\label{tab:peft-orthogonality}
\begin{center}
\begin{tabular}{lrrrr}\toprule Model & Checkpoints & Max. $\|\bm R\bm R^{\!\top}-\bm I\|_{\mathrm{F}}$ & Max. entry & Max. polar residual\\\midrule
FLUX-4B & 33 & $8.04$ & $0.136$ & $0.11$\\
Llama-8B & 21 & $0.118$ & $0.00528$ & $0.000936$\\
Qwen-7B & 21 & $0.42$ & $0.0233$ & $0.00154$\\
\bottomrule\end{tabular}
\end{center}
\end{table}

\FloatBarrier
\section{FLUX spectral interventions and numerical controls}\label{app:interventions}\label{app:spectrum}\label{app:diffspec}
\subsection{Complete six-method FLUX interventions}\label{app:flux22}
The FLUX cat intervention study contains 18 checkpoints and 396 evaluations, covering all six methods at their largest tested budgets. LR maximizes mean step-700 validation DINO, with ties favoring lower LR. Interventions use the final checkpoints and jointly replace all 80 adapted matrices.

\begin{table}[!htbp]
\caption{FLUX intervention operating points selected by mean validation DINO. Parameters are in millions. Displacements pool all adapted matrices before normalization. All edits act jointly on 80 matrices.}
\label{tab:flux22-protocol}
\begin{center}
\setlength{\tabcolsep}{3pt}
\begin{tabular*}{\linewidth}{@{\extracolsep{\fill}}lrrrrrr@{}}\toprule
Method & Capacity & LR & Params. (M) & $100d_W$ & $100d_S$ & $100d_S/d_W$\\\midrule
\textcolor{LoRAColor}{LoRA} & 32 & $5\times10^{-5}$ & $\result{fx22_lora_params}$ & $\result{fx22_lora_trained_dW_mean}_{\pm \result{fx22_lora_trained_dW_sd}}$ & $\result{fx22_lora_trained_dS_mean}_{\pm \result{fx22_lora_trained_dS_sd}}$ & $\result{fx22_lora_trained_ratio_mean}_{\pm \result{fx22_lora_trained_ratio_sd}}$\\
\textcolor{OFTColor}{OFT} & 256 & $1\times10^{-4}$ & $\result{fx22_oft_params}$ & $\result{fx22_oft_trained_dW_mean}_{\pm \result{fx22_oft_trained_dW_sd}}$ & $\result{fx22_oft_trained_dS_mean}_{\pm \result{fx22_oft_trained_dS_sd}}$ & $\result{fx22_oft_trained_ratio_mean}_{\pm \result{fx22_oft_trained_ratio_sd}}$\\
\textcolor{DoRAColor}{DoRA} & 32 & $1\times10^{-4}$ & $\result{fx22_dora_params}$ & $\result{fx22_dora_trained_dW_mean}_{\pm \result{fx22_dora_trained_dW_sd}}$ & $\result{fx22_dora_trained_dS_mean}_{\pm \result{fx22_dora_trained_dS_sd}}$ & $\result{fx22_dora_trained_ratio_mean}_{\pm \result{fx22_dora_trained_ratio_sd}}$\\
\textcolor{PiSSAColor}{PiSSA} & 32 & $5\times10^{-5}$ & $\result{fx22_pissa_params}$ & $\result{fx22_pissa_trained_dW_mean}_{\pm \result{fx22_pissa_trained_dW_sd}}$ & $\result{fx22_pissa_trained_dS_mean}_{\pm \result{fx22_pissa_trained_dS_sd}}$ & $\result{fx22_pissa_trained_ratio_mean}_{\pm \result{fx22_pissa_trained_ratio_sd}}$\\
\textcolor{MiLoRAColor}{MiLoRA} & 32 & $1\times10^{-4}$ & $\result{fx22_milora_params}$ & $\result{fx22_milora_trained_dW_mean}_{\pm \result{fx22_milora_trained_dW_sd}}$ & $\result{fx22_milora_trained_dS_mean}_{\pm \result{fx22_milora_trained_dS_sd}}$ & $\result{fx22_milora_trained_ratio_mean}_{\pm \result{fx22_milora_trained_ratio_sd}}$\\
\textcolor{HRAColor}{HRA} & 124 & $3\times10^{-4}$ & $\result{fx22_hra_params}$ & $\result{fx22_hra_trained_dW_mean}_{\pm \result{fx22_hra_trained_dW_sd}}$ & $\result{fx22_hra_trained_dS_mean}_{\pm \result{fx22_hra_trained_dS_sd}}$ & $\result{fx22_hra_trained_ratio_mean}_{\pm \result{fx22_hra_trained_ratio_sd}}$\\
\bottomrule\end{tabular*}
\end{center}
\end{table}
\FloatBarrier
\begin{table}[!htbp]
\caption{FLUX value-restoration controls. DINO scores and paired differences are multiplied by 100. The last column compares restoration with reconstruction through the same SVD/storage path.}
\label{tab:flux22-controls}
\begin{center}
\setlength{\tabcolsep}{3pt}
\begin{tabular*}{\linewidth}{@{\extracolsep{\fill}}lrrrrr@{}}\toprule
Method & Trained & Reconstructed & Restored & Values only & $\Delta$ vs. recon.\\\midrule
\textcolor{LoRAColor}{LoRA} & $\result{fx22_lora_trained_dino_mean}_{\pm \result{fx22_lora_trained_dino_sd}}$ & $\result{fx22_lora_reconstruct_dino_mean}_{\pm \result{fx22_lora_reconstruct_dino_sd}}$ & $\result{fx22_lora_restore_dino_mean}_{\pm \result{fx22_lora_restore_dino_sd}}$ & $\result{fx22_lora_spectrum_only_dino_mean}_{\pm \result{fx22_lora_spectrum_only_dino_sd}}$ & $\result{fx22_lora_restore_dino_vs_reconstruct_mean}_{\pm \result{fx22_lora_restore_dino_vs_reconstruct_sd}}$\\
\textcolor{OFTColor}{OFT} & $\result{fx22_oft_trained_dino_mean}_{\pm \result{fx22_oft_trained_dino_sd}}$ & $\result{fx22_oft_reconstruct_dino_mean}_{\pm \result{fx22_oft_reconstruct_dino_sd}}$ & $\result{fx22_oft_restore_dino_mean}_{\pm \result{fx22_oft_restore_dino_sd}}$ & $\result{fx22_oft_spectrum_only_dino_mean}_{\pm \result{fx22_oft_spectrum_only_dino_sd}}$ & $\result{fx22_oft_restore_dino_vs_reconstruct_mean}_{\pm \result{fx22_oft_restore_dino_vs_reconstruct_sd}}$\\
\textcolor{DoRAColor}{DoRA} & $\result{fx22_dora_trained_dino_mean}_{\pm \result{fx22_dora_trained_dino_sd}}$ & $\result{fx22_dora_reconstruct_dino_mean}_{\pm \result{fx22_dora_reconstruct_dino_sd}}$ & $\result{fx22_dora_restore_dino_mean}_{\pm \result{fx22_dora_restore_dino_sd}}$ & $\result{fx22_dora_spectrum_only_dino_mean}_{\pm \result{fx22_dora_spectrum_only_dino_sd}}$ & $\result{fx22_dora_restore_dino_vs_reconstruct_mean}_{\pm \result{fx22_dora_restore_dino_vs_reconstruct_sd}}$\\
\textcolor{PiSSAColor}{PiSSA} & $\result{fx22_pissa_trained_dino_mean}_{\pm \result{fx22_pissa_trained_dino_sd}}$ & $\result{fx22_pissa_reconstruct_dino_mean}_{\pm \result{fx22_pissa_reconstruct_dino_sd}}$ & $\result{fx22_pissa_restore_dino_mean}_{\pm \result{fx22_pissa_restore_dino_sd}}$ & $\result{fx22_pissa_spectrum_only_dino_mean}_{\pm \result{fx22_pissa_spectrum_only_dino_sd}}$ & $\result{fx22_pissa_restore_dino_vs_reconstruct_mean}_{\pm \result{fx22_pissa_restore_dino_vs_reconstruct_sd}}$\\
\textcolor{MiLoRAColor}{MiLoRA} & $\result{fx22_milora_trained_dino_mean}_{\pm \result{fx22_milora_trained_dino_sd}}$ & $\result{fx22_milora_reconstruct_dino_mean}_{\pm \result{fx22_milora_reconstruct_dino_sd}}$ & $\result{fx22_milora_restore_dino_mean}_{\pm \result{fx22_milora_restore_dino_sd}}$ & $\result{fx22_milora_spectrum_only_dino_mean}_{\pm \result{fx22_milora_spectrum_only_dino_sd}}$ & $\result{fx22_milora_restore_dino_vs_reconstruct_mean}_{\pm \result{fx22_milora_restore_dino_vs_reconstruct_sd}}$\\
\textcolor{HRAColor}{HRA} & $\result{fx22_hra_trained_dino_mean}_{\pm \result{fx22_hra_trained_dino_sd}}$ & $\result{fx22_hra_reconstruct_dino_mean}_{\pm \result{fx22_hra_reconstruct_dino_sd}}$ & $\result{fx22_hra_restore_dino_mean}_{\pm \result{fx22_hra_restore_dino_sd}}$ & $\result{fx22_hra_spectrum_only_dino_mean}_{\pm \result{fx22_hra_spectrum_only_dino_sd}}$ & $\result{fx22_hra_restore_dino_vs_reconstruct_mean}_{\pm \result{fx22_hra_restore_dino_vs_reconstruct_sd}}$\\
\bottomrule\end{tabular*}
\end{center}
\end{table}
\FloatBarrier

\paragraph{Evaluation and numerical references.}
The BF16 evaluator generates $512\times512$ images with 20 sampling steps and guidance 3.5. Test DINO averages four held-out image/prompt pairs with three generation repeats. The drift bank has five class prompts and five general prompts, with general drift and CLIP using the latter. Drift is $(1-\text{DINO cosine})/2$ relative to base-generated images. 

Every change is paired within a checkpoint before computing the mean and sample SD. Component effects use the same-run merged trained model, restoration controls use reconstruction through the same SVD/storage path, and comparisons of retained task performance use full value restoration. Same-run references account for this difference. All 396 evaluations succeed, with every restored matrix matched and within the recorded numerical tolerance.

\begin{figure}[!htbp]\centering
\includegraphics[width=.98\linewidth]{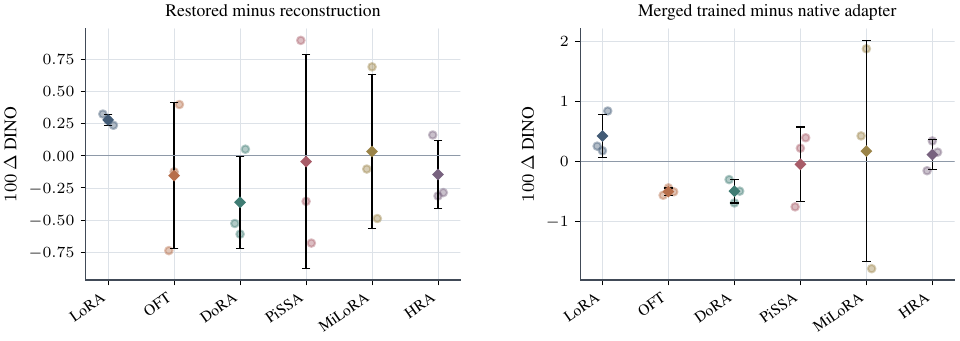}
\caption{FLUX numerical controls. The left panel compares value restoration with reconstruction. The right panel compares the same-run merged trained model with the earlier native-adapter evaluation. Diamonds and bars show means and sample SDs, with faint points for individual checkpoints. Differences are in DINO scores $\times100$.}
\label{fig:flux22-controls}
\end{figure}
\FloatBarrier

\paragraph{Values and adapted bases.}
Value restoration nearly preserves fidelity for every method, whereas learned values on pretrained singular bases yield near-base scores (\Cref{tab:flux22-controls}). This extends the main restoration finding to diffusion and points to the adapted bases as the principal source of task gains. Spectral displacement concentrates in PiSSA's dominant components and MiLoRA's trailing components (\Cref{fig:flux22-spectrum}). 

\begin{figure}[!htbp]\centering
\includegraphics[width=.98\linewidth]{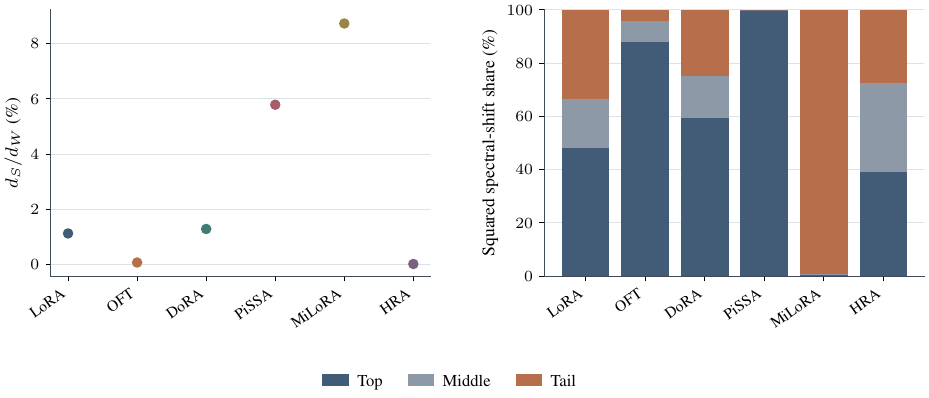}
\caption{FLUX spectral displacement. The left panel shows the pooled spectral-to-weight displacement ratio with sample SDs. The right panel shows mean shares of squared singular-value displacement across component groups, with boundaries adjusted to pretrained spectral gaps. Shares are computed from all measured spectral vectors.}
\label{fig:flux22-spectrum}
\end{figure}
\FloatBarrier

\subsection{Task performance and forgetting across spectral components}
The dominant, intermediate and trailing component groups are approximate thirds adjusted to nearby pretrained spectral gaps. Restore and keep operations replace complete rank-one contributions, using pretrained singular values as in \cref{eq:band-restore}. The measured response includes these interactions between components.

Restoring dominant components reduces general drift in every checkpoint and gives the largest mean reduction for each method. It gives the largest individual reduction in 17 of 18 checkpoints. MiLoRA seed zero has a slightly larger trailing-component response. Class-prompt drift also decreases in every checkpoint, with dominant restoration giving the largest reduction. Both drift measures therefore support the recurring dominant-component response. CLIP responses differ, with LoRA's mean general-prompt alignment decreasing. OFT gains fidelity on average but loses fidelity in one checkpoint.

Keeping only the adapted dominant components nearly preserves mean fidelity for LoRA, DoRA and PiSSA. LoRA and DoRA reduce general drift in every checkpoint, whereas PiSSA retains essentially all its drift. These responses show that similar retained task performance can accompany different retention outcomes. Fidelity losses vary across seeds, including a larger DoRA loss. We report DINO changes directly and use the 1 pp sufficiency rule only for language accuracy. OFT/HRA have greater trained drift than PiSSA in these selected settings.

\begin{figure}[!htbp]\centering
\includegraphics[width=.98\linewidth]{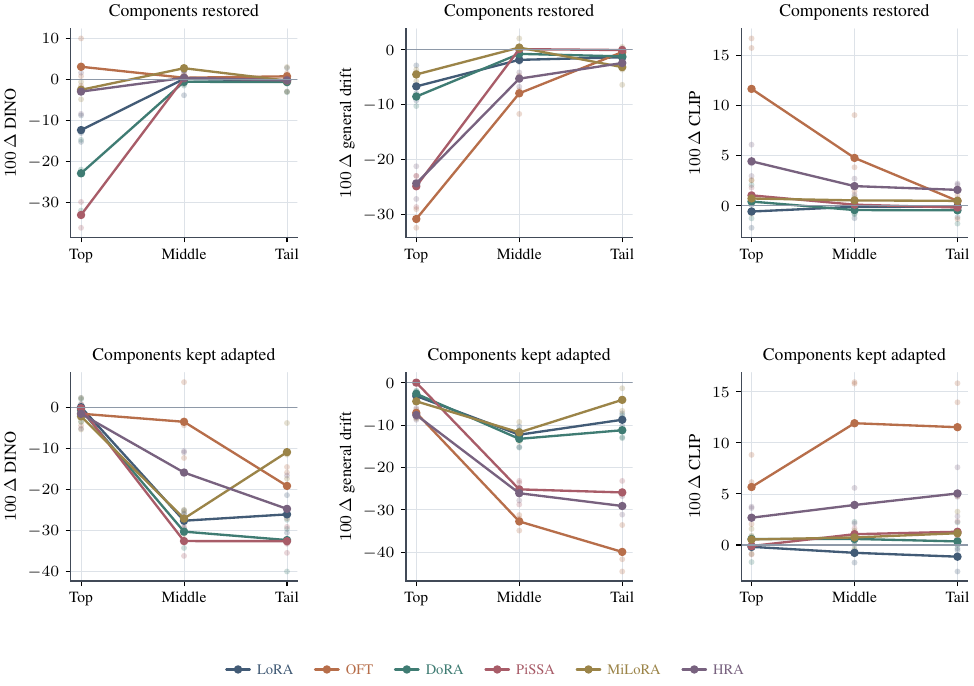}
\caption{FLUX component-restoration responses. The top row restores the named component group, and the bottom row keeps only that adapted group. All changes use the same-run trained model. Lines show method means and faint points show individual checkpoints.} 
\label{fig:flux22-bands}
\end{figure}

\FloatBarrier

\paragraph{Summary.}
The FLUX results extend the singular-value restoration and dominant-component forgetting findings to image personalization.
\FloatBarrier

\clearpage
\section{Geometry of additive updates}\label{app:geometry}
\subsection{Distance to a two-sided orthogonal transformation}
Write $q=\min(m,n)$ and let $O(k)$ denote the real $k\times k$ orthogonal group. The orbit $\{\bm Q_L\bm W_0\bm Q_R^{\!\top}:\bm Q_L\in O(m),\bm Q_R\in O(n)\}$ contains exactly the matrices with the same singular values as $\bm W_0$. It includes reflections as well as rotations. A local smooth change of singular bases is generated by skew-symmetric matrices.

\begin{proof}[Proof of Proposition~\ref{prop:orthogonal-distance}]
\textbf{Step 1. Bound the distance from below.} Let $\bm X=\bm Q_L\bm W_0\bm Q_R^{\!\top}$ be any matrix in the orbit. Orthogonal multiplication preserves singular values, so $\bm X$ has singular values $\bm s_0$ and $\|\bm X\|_{\mathrm F}=\|\bm s_0\|_2$. Write $\langle\bm A,\bm B\rangle_{\mathrm F}=\operatorname{tr}(\bm A^{\!\top}\bm B)$. The singular-value trace inequality \citep{mirsky1975trace} bounds this inner product by the sum of the paired, decreasing singular values, giving
\begin{equation*}
\langle\bm W_*,\bm X\rangle_{\mathrm F}\leq\sum_{i=1}^q s_{*,i}s_{0,i}.
\end{equation*}
Expanding the squared distance therefore gives
\begin{equation*}
\begin{aligned}
\|\bm W_*-\bm X\|_{\mathrm F}^2
&=\|\bm s_*\|_2^2+\|\bm s_0\|_2^2-2\langle\bm W_*,\bm X\rangle_{\mathrm F}\\
&\geq\sum_{i=1}^q(s_{*,i}-s_{0,i})^2=\|\bm s_*-\bm s_0\|_2^2.
\end{aligned}
\end{equation*}
\textbf{Step 2. Attain the bound by aligning the singular bases.} Take full SVDs $\bm W_j=\bm U_j\bm\Sigma_j\bm V_j^{\!\top}$ for $j\in\{0,*\}$. Here $\bm U_j$ and $\bm V_j$ are square orthogonal completions of the thin bases, and $\bm\Sigma_j\in\mathbb R^{m\times n}$ is rectangular diagonal. Choose $\bm Q_L=\bm U_*\bm U_0^{\!\top}$ and $\bm Q_R=\bm V_*\bm V_0^{\!\top}$. These matrices are orthogonal and give
\begin{equation*}
\bm Q_L\bm W_0\bm Q_R^{\!\top}=\bm U_*\bm\Sigma_0\bm V_*^{\!\top}=\bm W_R.
\end{equation*}
Thus $\bm W_R$ keeps the base singular values while using the adapted singular bases. Orthogonal invariance of the Frobenius norm yields
\begin{equation*}
\|\bm W_*-\bm W_R\|_{\mathrm F}
=\|\bm\Sigma_*-\bm\Sigma_0\|_{\mathrm F}
=\|\bm s_*-\bm s_0\|_2.
\end{equation*}
This equals the lower bound from Step 1, proving the minimum. No positivity or gap assumption was used, so the result also covers rank-deficient matrices and repeated singular values.

\textbf{Step 3. Relate the distance to $d_S/d_W$.} For nonzero base and update norms, the common normalization $\|\bm s_0\|_2=\|\bm W_0\|_{\mathrm F}$ cancels, giving
\begin{equation*}
\frac{d_S}{d_W}
=\frac{\|\bm s_*-\bm s_0\|_2}{\|\bm W_*-\bm W_0\|_{\mathrm F}}
=\frac{\|\bm W_*-\bm W_R\|_{\mathrm F}}{\|\bm W_*-\bm W_0\|_{\mathrm F}}.
\end{equation*}
For multiple matrices, summing the squared residuals before taking square roots gives the same interpretation for the pooled ratio.
\end{proof}

\subsection{First-order spectral and basis motion}
\begin{proof}[Proof of Proposition~\ref{prop:first-order-motion}]
\textbf{Step 1. Express the derivative in singular coordinates.} The positive, distinct singular values allow a smooth local choice of singular vectors \citep{stewart1990}. Complete them to smooth square orthogonal bases and write $\bm W_0+\epsilon\bm E=\bm U(\epsilon)\bm\Sigma(\epsilon)\bm V(\epsilon)^{\!\top}$. A dot denotes differentiation at $\epsilon=0$. Define
\begin{equation*}
\bm C=\bm U_0^{\!\top}\bm E\bm V_0,\qquad
\bm\Omega_U=\bm U_0^{\!\top}\dot{\bm U},\qquad
\bm\Omega_V=\bm V_0^{\!\top}\dot{\bm V}.
\end{equation*}
Differentiating $\bm U^{\!\top}\bm U=\bm I$ gives $\bm\Omega_U^{\!\top}=-\bm\Omega_U$. The same argument applies to $\bm\Omega_V$. Now differentiate the SVD and multiply on the left by $\bm U_0^{\!\top}$ and on the right by $\bm V_0$ to obtain
\begin{equation*}
\bm C=\bm\Omega_U\bm\Sigma_0+\dot{\bm\Sigma}-\bm\Sigma_0\bm\Omega_V,
\end{equation*}
where $\dot{\bm\Sigma}$ is rectangular diagonal with entries $\dot s_i$. Since skew-symmetric matrices have zero diagonal, the two basis terms contribute nothing to $C_{ii}$. Hence
\begin{equation*}
\dot s_i=C_{ii}=\bm u_{0,i}^{\!\top}\bm E\bm v_{0,i},\qquad 1\leq i\leq q.
\end{equation*}

\textbf{Step 2. Show that the other coordinates describe basis motion.} Holding the singular values fixed leaves $\bm\Omega_U\bm\Sigma_0-\bm\Sigma_0\bm\Omega_V$, the tangent directions to the orbit in singular coordinates. These directions have zero diagonal. Conversely, for $i\neq j\leq q$, the equations for any pair of off-diagonal entries are
\begin{equation*}
C_{ij}=(\Omega_U)_{ij}s_{0,j}-s_{0,i}(\Omega_V)_{ij},\qquad
C_{ji}=-(\Omega_U)_{ij}s_{0,i}+s_{0,j}(\Omega_V)_{ij}.
\end{equation*}
Solving this two-variable system gives
\begin{equation*}
(\Omega_U)_{ij}=\frac{s_{0,j}C_{ij}+s_{0,i}C_{ji}}{s_{0,j}^2-s_{0,i}^2},\qquad
(\Omega_V)_{ij}=\frac{s_{0,i}C_{ij}+s_{0,j}C_{ji}}{s_{0,j}^2-s_{0,i}^2}.
\end{equation*}
The denominators are nonzero because the singular values are distinct. If $m>n$, the extra rows satisfy $C_{aj}=(\Omega_U)_{aj}s_{0,j}$ for $a>q$. If $n>m$, the extra columns satisfy $C_{ib}=-s_{0,i}(\Omega_V)_{ib}$ for $b>q$. These equations are solvable because each $s_{0,i}>0$. Thus every coordinate outside the first $q$ diagonal entries can be produced by a change of orthogonal bases.

\textbf{Step 3. Separate the normal and tangent components.} The matrices $\bm u_{0,i}\bm v_{0,i}^{\!\top}$ are Frobenius-orthonormal because their pairwise inner products are $(\bm u_{0,i}^{\!\top}\bm u_{0,j})(\bm v_{0,i}^{\!\top}\bm v_{0,j})=\delta_{ij}$. By Step 2 they span the orthogonal complement of the tangent space. The normal projection of $\bm E$ is therefore
\begin{equation*}
\bm E_N=\sum_{i=1}^q\dot s_i\bm u_{0,i}\bm v_{0,i}^{\!\top},\qquad
\bm E_T=\bm E-\bm E_N.
\end{equation*}
The remainder $\bm E_T$ is tangent and changes only the bases to first order. Orthogonality gives
\begin{equation*}
\|\bm E\|_{\mathrm F}^2=\|\bm E_N\|_{\mathrm F}^2+\|\bm E_T\|_{\mathrm F}^2
=\|\dot{\bm s}\|_2^2+\|\bm E_T\|_{\mathrm F}^2.
\end{equation*}

\textbf{Step 4. Average the first-order spectral energy.} Let $\bm z_i=\operatorname{vec}(\bm u_{0,i}\bm v_{0,i}^{\!\top})$, so $\|\bm z_i\|_2=1$ and $\dot s_i=\bm z_i^{\!\top}\operatorname{vec}(\bm E)$. The assumed second moment yields
\begin{equation*}
\mathbb E[\dot s_i^2]
=\bm z_i^{\!\top}\mathbb E[\operatorname{vec}(\bm E)\operatorname{vec}(\bm E)^{\!\top}]\bm z_i
=\frac{\|\bm z_i\|_2^2}{mn}=\frac{1}{mn}.
\end{equation*}
Summing these $q$ equal contributions gives
\begin{equation*}
\mathbb E\|\dot{\bm s}\|_2^2=\frac{q}{mn}=\frac{1}{\max(m,n)}.
\end{equation*}
Since $\|\bm E\|_{\mathrm F}=1$, Step 3 gives $\mathbb E\|\bm E_T\|_{\mathrm F}^2=1-1/\max(m,n)$ for the complementary tangent energy.

\textbf{Step 5. Pass from the derivative to $d_S/d_W$.} The base norms cancel as in Proposition~\ref{prop:orthogonal-distance}, and $\|\epsilon\bm E\|_{\mathrm F}=|\epsilon|$. Thus, for $\epsilon\neq0$,
\begin{equation*}
\left(\frac{d_S}{d_W}\right)^2=\frac{\|\bm s(\bm W_0+\epsilon\bm E)-\bm s_0\|_2^2}{\epsilon^2}.
\end{equation*}
With $\bm W_0$ fixed, positivity and the nonzero singular-value gaps give the expansion
\begin{equation*}
\bm s(\bm W_0+\epsilon\bm E)-\bm s_0=\epsilon\dot{\bm s}+O(\epsilon^2),
\end{equation*}
whose remainder is bounded uniformly over $\|\bm E\|_{\mathrm F}=1$. Substitution into the ratio therefore gives $(d_S/d_W)^2=\|\dot{\bm s}\|_2^2+O(|\epsilon|)$ with a uniform error. Taking expectations and applying Step 4 gives
\begin{equation*}
\lim_{\epsilon\to0}\mathbb E\left[\left(\frac{d_S}{d_W}\right)^2\right]
=\mathbb E\|\dot{\bm s}\|_2^2=\frac{1}{\max(m,n)}.
\end{equation*}
Taking the square root proves the stated RMS limit.
\end{proof}

\begin{corollary}[Isotropic fixed-rank perturbations]\label{cor:isotropic-fraction}
Let $\bm W_0$ satisfy Proposition~\ref{prop:first-order-motion} and let $\bm E=\bm P\bm\Lambda\bm Q^{\!\top}/\|\bm\Lambda\|_{\mathrm F}$, where $\bm\Lambda\in\mathbb R^{r\times r}$ is fixed, nonzero diagonal and $\bm P\in\mathbb R^{m\times r},\bm Q\in\mathbb R^{n\times r}$ are independent uniform orthonormal frames, with $1\leq r\leq q$. Then \Cref{eq:relative-spectral-scale} holds.
\end{corollary}
\begin{proof}
\textbf{Step 1. Check the norm.} Since $\bm P^{\!\top}\bm P=\bm Q^{\!\top}\bm Q=\bm I_r$, multiplication by these frames preserves the Frobenius norm of $\bm\Lambda$. Hence $\|\bm E\|_{\mathrm F}=1$.

\textbf{Step 2. Compute the second moment.} Uniformity and independence of the frames make the distribution of $\bm E$ unchanged by any fixed left or right orthogonal multiplication. In particular, a row sign flip shows $\mathbb E[E_{ab}E_{cd}]=0$ when $a\neq c$. When $a=c$ but $b\neq d$, a column sign flip gives the same conclusion. Row and column permutations make all $\mathbb E[E_{ab}^2]$ equal. Their sum is $\mathbb E\|\bm E\|_{\mathrm F}^2=1$, so each equals $1/(mn)$. Consequently,
\begin{equation*}
\mathbb E[\operatorname{vec}(\bm E)\operatorname{vec}(\bm E)^{\!\top}]=\frac{\bm I_{mn}}{mn}.
\end{equation*}
\textbf{Step 3. Apply Proposition~\ref{prop:first-order-motion}.} Steps 1 and 2 verify its two assumptions on $\bm E$, so both conclusions in \Cref{eq:relative-spectral-scale} follow.
\end{proof}

\begin{corollary}[Pooling matrices]\label{cor:pooled-spectral-scale}
For a finite collection of matrices satisfying the isotropic assumptions of Proposition~\ref{prop:first-order-motion}, let $\bm W_{*,j}=\bm W_{0,j}+\epsilon a_j\bm E_j$, with fixed $a_j\geq0$ and $\sum_j a_j^2>0$. The pooled displacements satisfy
\begin{equation*}
\lim_{\epsilon\to0}\mathbb E\left[\left(\frac{d_S}{d_W}\right)^2\right]
=\frac{\sum_j a_j^2/\max(m_j,n_j)}{\sum_j a_j^2}.
\end{equation*}
\end{corollary}
\begin{proof}
\textbf{Step 1. Write the pooled ratio.} Let $\bm s_{*,j}$ and $\bm s_{0,j}$ be the singular-value vectors of the adapted and base matrices. The common base-norm normalization cancels, while $\|\bm E_j\|_{\mathrm F}=1$ fixes each update norm. Therefore
\begin{equation*}
\left(\frac{d_S}{d_W}\right)^2
=\frac{\sum_j\|\bm s_{*,j}-\bm s_{0,j}\|_2^2}{\sum_j\|\epsilon a_j\bm E_j\|_{\mathrm F}^2}
=\frac{\sum_j\|\bm s_{*,j}-\bm s_{0,j}\|_2^2}{\epsilon^2\sum_j a_j^2}.
\end{equation*}
\textbf{Step 2. Expand each numerator.} Write $\dot{\bm s}_j$ for the spectral derivative at $\bm W_{0,j}$ in direction $\bm E_j$. The uniform expansion from Proposition~\ref{prop:first-order-motion} gives, for fixed $a_j$,
\begin{equation*}
\frac{\mathbb E\|\bm s_{*,j}-\bm s_{0,j}\|_2^2}{\epsilon^2}
=a_j^2\mathbb E\|\dot{\bm s}_j\|_2^2+O(|\epsilon|)
=\frac{a_j^2}{\max(m_j,n_j)}+O(|\epsilon|).
\end{equation*}
When $a_j=0$, the corresponding term is identically zero.

\textbf{Step 3. Sum and take the limit.} There are finitely many matrices and $\sum_j a_j^2$ is fixed and positive. Substituting Step 2 into Step 1 yields
\begin{equation*}
\lim_{\epsilon\to0}\mathbb E\left[\left(\frac{d_S}{d_W}\right)^2\right]
=\frac{\sum_j a_j^2/\max(m_j,n_j)}{\sum_j a_j^2}.
\end{equation*}
Only linearity of expectation is used, so independence between matrices is unnecessary. The result is an average of the individual squared-ratio scales, weighted by their update energies $a_j^2$.
\end{proof}

Under the stated isotropic assumptions, the inverse-dimension scale explains how an additive update can have small relative spectral displacement. This motivates the empirical geometry measurements, while restoration tests the functional role of the spectral changes. 

\clearpage
\section{Review of experimental settings in PEFT studies}\label{app:peft-review}

\paragraph{Papers and comparisons.}
We review method papers associated with Hugging Face PEFT v0.21.0 \citep{mangrulkar2022peft}, using paper versions available by 26 September 2026.
Multiple supported entries from the same paper are counted once, and background citations are excluded.
The inventory contains 71 entries and 70 papers, with one entry lacking an identified method paper.
Sixty-four papers have eligible main comparisons.
The other six concern training from scratch, contributions without adaptation training or results without a trained principal comparator outside auxiliary analyses.

We read the papers, appendices, available supplements and relevant official experimental descriptions.
A comparison pairs the proposed method with its principal trained baseline for a stated pretrained model and training task.
Vanilla LoRA is preferred when present.
Several evaluations of the same adapted model remain one comparison, while separately trained tasks are distinguished.
Additional ablations do not establish controls for the main comparisons.
Capacity strata are retained within each comparison.

\paragraph{Controls and evidence.}
\Cref{tab:peft-review-criteria} defines the four controls.
Both methods must have the required evidence in the same comparison.
Comparable conditions are assessed separately for checkpoint, data, exposure, adapted modules and trainable budget.
Shared stopping rules can establish comparable exposure.
Parameter counts within 10\% of the larger count, an explicit matched design or a controlled comparison across budgets establish budget comparability.
Equal rank alone is insufficient.
Intrinsic module differences and full fine-tuning budgets are treated as inapplicable where direct matching is inappropriate.
A different budget can still support a useful efficiency comparison.

Training runs may vary initialization, data order or sampled training examples.
Their uncertainty is distinguished from variation across generated samples, subjects, evaluation subsets or evaluation resamples.
Undefined error bars and dispersion pooled across training seeds and other units do not establish uncertainty attributable specifically to training.
A numerical search space does not establish equal tuning effort or separate validation selection.
An explicitly exhaustive design with one candidate can document disclosure.
Selection using the reported validation split is recorded separately from selection using an independent split.

Two AI-assisted source reviews were compared, and disagreements were revisited against the primary sources.
Differences in eligibility, reviewed versions, baseline attribution and experimental scope were reconciled using the definitions above.
Ambiguous or contradictory evidence remains unclear.
Relevant official descriptions provide supporting evidence, while generic defaults do not establish which procedures produced reported results.
AI-assisted review can retain correlated errors.

\paragraph{Aggregation and coverage.}
A paper documents a control if at least one eligible main comparison satisfies it.
An explicit departure requires evidence that every eligible comparison fails the definition.
Otherwise, the classification is unclear.
Inapplicable cases are excluded from the relevant denominator.
Missing information is never interpreted as evidence that a procedure was not performed.
For incomplete comparison coverage, known positives are retained, while universal claims are withheld unless a statement explicitly covers all main results.

\Cref{tab:peft-review-coverage,tab:peft-review-domains} report coverage and domain summaries.
Sixteen papers have unresolved comparison coverage.
No reviewed comparison establishes all four controls together under these definitions.
Requiring three training runs retains the same nine replication positives, while requiring explicitly stated training randomness retains seven.
Excluding learning-rate selection on the reported or an unnamed split leaves four selection positives.

\paragraph{Scope of interpretation.}
Library membership defines a practical sample spanning several domains and research objectives.
It is not a representative sample of all PEFT research.
Domains overlap at paper level, and controls documented in different comparisons cannot be combined into a single qualifying experiment.
The cohort and validation-selection definition differ from \citet{lee2026learningrate}, so the percentages cannot be compared directly.
The review characterizes available evidence and does not estimate how undocumented procedures would change method rankings.

\begin{table}[!htbp]
\captionsetup{labelsep=period}
\caption{Definitions used in the review. Evidence must cover both methods and the stated main comparison. These are operational criteria for this review.}
\label{tab:peft-review-criteria}
\centering
\begin{tabular}{@{}p{.25\linewidth}p{.71\linewidth}@{}}
\toprule
Control & Required evidence \\
\midrule
Training runs and uncertainty & At least two independent training runs per method with a defined SD, SE or CI attributable to training variation. \\
LR selection & At least three learning rates evaluated per method, with separate selection using a stated validation criterion. \\
LR candidates & Numerical candidate values or a search range stated for both methods. \\
Comparable conditions & Shared checkpoint and data, comparable exposure, and adapted modules and trainable budgets matched or addressed through an appropriate controlled comparison. \\
\bottomrule
\end{tabular}
\end{table}

\begin{table}[!htbp]
\captionsetup{labelsep=period}
\caption{Paper-level coverage of experimental controls. The first three columns partition the 64 eligible papers using the at least one comparison rule. The final column requires established coverage of every applicable main comparison. Unresolved coverage prevents a universal positive.}
\label{tab:peft-review-coverage}
\centering
\begin{tabular*}{\linewidth}{@{\extracolsep{\fill}}lrrrr@{}}
\toprule
Control & Documented & Explicit departure & Unclear & \shortstack{All main\\comparisons} \\
\midrule
Training runs and uncertainty & 9 & 5 & 50 & 2 \\
LR selection & 7 & 4 & 53 & 1 \\
LR candidates & 13 & 0 & 51 & 4 \\
Comparable conditions & 25 & 10 & 29 & 1 \\
\bottomrule
\end{tabular*}
\end{table}

\begin{table}[!htbp]
\captionsetup{labelsep=period}
\caption{Documented controls within each domain. Entries give qualifying papers over applicable papers. A paper can occur in several domains, so rows cannot be summed. The decoder subset restricts comparisons to standard decoder adaptation against vanilla LoRA, with the exclusions described by Lee et al. Its definitions and sampling frame still differ from theirs.}
\label{tab:peft-review-domains}
\centering
\begin{tabular*}{\linewidth}{@{\extracolsep{\fill}}lrrrr@{}}
\toprule
Domain & \shortstack{Training\\uncertainty} & LR selection & LR candidates & Conditions \\
\midrule
Decoder language & 4/51 & 4/51 & 7/51 & 18/51 \\
Other language & 5/40 & 3/40 & 6/40 & 7/40 \\
Vision and multimodal & 3/20 & 1/20 & 3/20 & 2/20 \\
Diffusion & 0/11 & 0/11 & 1/11 & 3/11 \\
Other & 0/1 & 0/1 & 0/1 & 0/1 \\
\midrule
Decoder subset & 3/40 & 2/40 & 4/40 & 12/40 \\
\bottomrule
\end{tabular*}
\end{table}

\FloatBarrier
\end{document}